\documentclass[12pt]{ucsddissertation}
\usepackage{amsmath}
\usepackage{thmtools}
\usepackage[utf8]{inputenc}
\usepackage[T1]{fontenc}    %

\usepackage{tikz}
\usepackage{graphicx}
\usepackage{subcaption}
\newcommand{\edgeweight}   {{edge weight continuous model}\xspace}
\newcommand{\edgeprob}  {{edge probability discrete model}\xspace}
\newcommand{\cIGN}  {{cIGN}\xspace}
\newcommand{\textNewnorm}  {{Partition-norm}\xspace}
\newcommand{\textnewnorm}  {{partition-norm}\xspace}
\newcommand{\parspace}  {{\Gamma}}
\usepackage[breaklinks=true,pdfborder={0 0 0}]{hyperref}
\usepackage[capitalize]{cleveref}
\usepackage{thm-restate}
\newcommand\numberthis{\addtocounter{equation}{1}\tag{\theequation}} %
\usepackage{makecell}

\usepackage{changepage,threeparttable}
\newcommand{\mv}{{\bm{V}}}

\newcommand{\tabtopvspace}{{\vspace{-0pt}}} 
\newcommand{\tabbottomvspace}{{\vspace{-0pt}}}

\newcommand{\gat}{\text{GATv2}}
\newcommand{\method}{GPS\xspace}
\newcommand\tmp{\textsf{tmp}}
\newcommand\ps{\textsf{partialsum}}
\newcommand{\ip}[1]{\langle{#1}\rangle}
\newcommand{\gnvn}{\text{gn-vn}}
\newcommand{\vngn}{\text{vn-gn}}
\newcommand{\vn}{\text{vn}}
\newcommand{\gn}{\text{gn}}
\newcommand{\DS}{\text{DeepSets}}
\newcommand{\ds}{\text{ds}}

\newcommand{\pool}{\tau}
\newcommand{\gngn}{\text{gn-gn}}
\newcommand{\tooned}[1]{\texttt{ReshapeTo1D}{(#1)}}
\newcommand{\gatii}{\texttt{GATv2}\xspace}

\newcommand{\pepfunc}{\texttt{Peptides-func}\xspace}
\newcommand{\pepstruct}{\texttt{Peptides-struct}\xspace}

\newcommand{\first}[1]{\textbf{#1}}

\usepackage{url}
\usepackage{paralist}
\usepackage{color}
\definecolor{darkblue}{rgb}{0.0, 0.0, 0.8}
\definecolor{darkred}{rgb}{0.8, 0.0, 0.0}
\definecolor{darkgreen}{rgb}{0.0, 0.8, 0.0}
\newcommand{\todo}[1]{\blue{TODO: #1}}
\usepackage{subfiles}
\usepackage{booktabs}
\usepackage{makecell}
\usepackage{graphicx}
\usepackage[linesnumbered,ruled,vlined]{algorithm2e}
\SetKwInput{KwInput}{Input}                %
\SetKwInput{KwOutput}{Output}              %
\usepackage{times}
\usepackage{amssymb}
\usepackage{amsthm}
\usepackage{multirow}
\usepackage{xspace}

\usepackage{graphicx}
\usepackage{ifpdf}
\usepackage{url}
\usepackage{latexsym}
\usepackage{euscript}
\usepackage{xspace}
\usepackage{color}
\usepackage{makeidx}
\usepackage{wrapfig}
\usepackage{enumitem}
\long\def\remove#1{}

\newtheorem{theorem}{Theorem}[section] %
\newtheorem{proposition}[theorem]{Proposition}
\newtheorem{lemma}[theorem]{Lemma}

\newtheorem{definition}{Definition}
\definecolor{darkred}{rgb}{1, 0.1, 0.3}
\definecolor{darkgreen}{rgb}{0.5, 0.8, 0.1}
\definecolor{darkpurple}{rgb}{1.0, 0, 1.0}
\definecolor{darkblue}{rgb}{0, 0, 1.0}

\newcommand {\mm}[1] {\ifmmode{#1}\else{\mbox{\(#1\)}}\fi}

\usepackage{amsmath,amsfonts,bm}

\def\eqref#1{equation~\ref{#1}}

\def\1{\bm{1}}

\def\va{{\bm{a}}}
\def\vb{{\bm{b}}}
\def\vc{{\bm{c}}}

\def\ve{{\bm{e}}}

\def\vk{{\bm{k}}}

\def\vn{{\bm{n}}}

\def\vq{{\bm{q}}}

\def\vu{{\bm{u}}}
\def\vv{{\bm{v}}}
\def\vw{{\bm{w}}}
\def\vx{{\bm{x}}}
\def\vy{{\bm{y}}}
\def\vz{{\bm{z}}}

\def\mA{{\bm{A}}}
\def\mB{{\bm{B}}}

\def\mF{{\bm{F}}}

\def\mH{{\bm{H}}}

\def\mL{{\bm{L}}}

\def\mV{{\bm{V}}}
\def\mW{{\bm{W}}}
\def\mX{{\bm{X}}}

\DeclareMathAlphabet{\mathsfit}{\encodingdefault}{\sfdefault}{m}{sl}
\SetMathAlphabet{\mathsfit}{bold}{\encodingdefault}{\sfdefault}{bx}{n}

\newcommand\R[1]{\mathbb{R}^{#1} }

\newcommand{\reals} {{\mathbb{R}}}

\newcommand{\hL}    {{\widehat{L}}}

\newcommand{\MLP}   {{\mathrm{MLP}}}

\newcommand{\Vhat}      {{\widehat{V}}}
\newcommand{\Ehat}      {{\widehat{E}}}
\newcommand{\What}      {{\widehat{W}}}
\newcommand{\Dhat}      {{\widehat{D}}} 
\newcommand{\what}      {{\hat{w}}}
\newcommand{\vhat}      {{\hat{v}}}
\newcommand{\vmap}      {{\pi}}
\newcommand{\Lhat}      {{\widehat{L}}}
\newcommand{\Fcal}      {{\mathcal{F}}}
\newcommand{\Lcal}      {{\mathcal{L}}}
\newcommand{\fhat}      {{\hat{f}}}
\newcommand{\vertexmap}  {{vertex map}}
\newcommand{\Qform}     {{\mathsf{Q}}}
\newcommand{\RQ}        {{\mathsf{R}}}

\newcommand{\dwL}       {{\mathsf{L}}}
\newcommand{\proj}      {{\mathcal{P}}}
\newcommand{\lift}      {{\mathcal{U}}}
\newcommand{\myMu}      {{\mathcal{M}_{\theta}}}

\newcommand{\N}      {N} 
\newcommand{\genL}      {{\mathcal{O}_G}} %
\newcommand{\combL}      {{L}}%

\newcommand{\n}      {n} 
\newcommand{\sx}      {\hat{x}} 
\newcommand{\sgenL}      {{\mathcal{O}_{\Ghat}}}
\newcommand{\scombL}      {{\Lhat}} %

\newcommand{\sdwL}      {\widehat{\dwL}}%
\newcommand{\LtildeVhat}      {\scombL}  %

\newcommand{\bx}      {x} 

\newcommand{\Ghat}      {{\widehat{G}}}
\newcommand{\nn}  {{\texttt{GOREN} }}
\newcommand{\nntight}  {{\texttt{GOREN}}}
\newcommand{\nL}  {{\mathcal{L}}}
\newcommand{\snL}  {{\widehat{\mathcal{L}}}}
\newcommand{\IN}  {{\Pi}} %
\newcommand{\ppt}  {({P^+})^T}

\newcommand{\lbd} {{\boldsymbol{\lambda}}}
\newcommand{\w} {\mathbf{w}}
\newcommand{\wt} {\mathbf{w}^t}

\newcommand{\wg} {\mathbf{w}_{\Ghat}}
\newcommand{\wgt} {{\mathbf{w}_{\Ghat}^t}}
\newcommand{\mini}{{\operatorname{minimize}}}

\newcommand{\A}{{U \operatorname{Diag}(\boldsymbol{\lambda}) U^{T}}}
\newcommand{\mask}[1]{\overline{#1}}
\newcommand{\lap}[1]{\mathfrak{L}{#1} }
\newcommand{\lapstar}[1]{\mathfrak{L^*}{#1} }
\newcommand{\er}{Erd\H{o}s-R\' {e}nyi }

\newcommand{\mc}{\mathcal}
\newcommand{\mb}{\mathbb}

\newcommand{\dnorm}{diagonal norm}
\newcommand{\mnorm}{matrix norm}
\newcommand{\IGN}{$2$-IGN}
\newcommand{\smallIGN}{IGN-small}
\newcommand{\kIGN}{$k$-IGN}

\newcommand{\sGNN}{spectral GNN}
\newcommand\dg[1]{\text{Diag}(#1) }
\newcommand\dgdual[1]{\text{Diag}^{*}(#1)}
\newcommand\LE{\textnormal{LE}}

\newcommand\MSE{\operatorname{RMSE}}
\newcommand{\one}{\mathbf{1}}
\newcommand{\dtensor}[2]{\mb{R}^{n^{#1}\times {#2}}}

\newcommand{\axis}[1]{\text{ax}(#1)}
\newcommand{\ord}[1]{\textnormal{ord}(#1)}
\newcommand{\bs}[1]{\boldsymbol{#1}}
\newcommand{\name}[1]{\textnormal{name}_{#1}}
\newcommand{\set}[1]{\{#1\}}

\newcommand\norm[1]{\|#1\| }
\newcommand\Linf[1]{\|#1\|_{L_\infty} }

\newcommand\newnorm[1]{\|#1\|_{\text{pn}} }
\newcommand\newnorminf[1]{\|#1\|_{\text{pn}-\infty} }
\newcommand\specnorm[1]{\|#1\|_{\text{spec}} }

\newcommand{\Lnorm}{$L_{2}$ norm}
\newcommand{\lnorm}{$\ell_{2}$ norm}

\newcommand{\Linfnorm}{$L_{\infty}$ norm }

\newcommand{\wbound}{\sqrt{\frac{A_1}{n}}}
\newcommand{\filterbound}{A_2}

\newcommand{\degmean}{D(A)_{\textnormal{mean}}}

\newcommand{\inducedW}{\widetilde{W_n} }

\newcommand{\sampleW}{W_{n \times n}}
\newcommand{\inducedX}{\widetilde{X_n} }
\newcommand{\inducedEW}{\widetilde{ W_{n, E} }}
\newcommand{\inducedEX}{\widetilde{ X_{n, E} }}
\newcommand{\inducedEf}{\widetilde{ f_{n, E} }}
\newcommand{\sampleE}{S_{n} }
\newcommand{\iid}{i.i.d. }
\newcommand{\lip}{Lipschitz}

\newcommand{\bell}[1]{\textnormal{bell}(#1)}

\newcommand{\xred}{X_{\gamma, \text{reduction}}}
\newcommand{\yrep}{Y_{\gamma, \text{replication}}}

\newcommand{\red}[1]{\textcolor{red}{#1}}
\newcommand{\blue}[1]{\textcolor{blue}{#1}}

\newcommand{\tn}[1]{\textnormal{#1}}

\newtheorem{assumption}{\hspace{0pt}\bf AS\hspace{-0.075cm}}

\newtheorem{remark}{\hspace{0pt}\bf Remark}

\usepackage{mathptmx}
\usepackage[NoDate]{currvita}
\usepackage{array}
\usepackage{tabularx}
\usepackage{booktabs}
\usepackage{ragged2e}
\usepackage{microtype}
\usepackage{graphicx}
\AtBeginDocument{%
    \settowidth\cvlabelwidth{\cvlabelfont 0000--0000}%
}

\increasemargins{.1in}
\usepackage{trace}
\usepackage[english]{babel}
\usepackage{blindtext}

\title{Local-to-global Perspectives on Graph Neural Networks}
\author{Chen Cai}
\degree{Doctor of Philosophy}{}%

\chair{Professor Jingbo Shang}
\cochair{Professor Yusu Wang}
\committee{Professor Gal Mishne}
\committee{Professor Rose Yu}

\degreeyear{2023}

\begin{document}
\frontmatter
\maketitle
\makecopyright
\makesignature

\begin{epigraph}

\begin{quote}
Some people may sit back and say, I want to solve this problem and they sit down and say, ``How do I solve this problem?'' I don’t. I just move around in the mathematical waters, thinking about things, being curious, interested, talking to people, stirring up ideas; things emerge and I follow them up. Or I see something which connects up with something else I know about, and I try to put them together and things develop. I have practically never started off with any idea of what I’m going to be doing or where it’s going to go. I’m interested in mathematics; I talk, I learn, I discuss and then interesting questions simply emerge. I have never started off with a particular goal, except the goal of understanding mathematics. \\
\hspace*{\fill} -- Michael Atiyah
\end{quote}

\begin{quote}
Talk is cheap. Show me the code. \\
\hspace*{\fill} -- Linus Torvalds
\end{quote}
\end{epigraph}
\tableofcontents
\listoffigures
\listoftables

\begin{acknowledgements}
I want to express my gratitude to my Ph.D. advisor, Prof. Yusu Wang.
Her guidance, support, and encouragement have been invaluable throughout my doctoral studies.
I took a meandering path to computer science and made a switch from computational topology to
graph representation learning in the middle of my Ph.D. Such a transition will not be possible without her support.
I am grateful for the patience, kindness, vision, and academic freedom she offered.

Besides Yusu, I also want four professors who shaped my research career. 
I would like to thank Prof. Victoria Sadovskaya for the opportunity of attending the MASS (Mathematics Advanced Study Semesters) program at Penn State University when I was an undergrad. It is one of the best times I will always remember. I also want to thank Prof. Yakov Pesin, who taught me the Erlangen program in the projective geometry course. It converts me into a geometer and leads me to geometric deep learning many years later.  I want to thank Prof. Joseph Mitchell from Stony Brook University and Prof. Jie Gao from Rutgers University for introducing me to the field of algorithm design and computational geometry. Joe's kindness and devotion to computational geometry are contagious and left a mark on me.
Jie's advice of grounding self in important applications balanced my tendency of going too abstract.

I benefited much from two summer schools, ``Mathematical Methods for High-Dimensional Data Analysis'' summer school organized by Ulrich Bauer, where I am intrigued by the topological data analysis course given by Steve Oudot, and London Geometry and Machine Learning (LOGML) summer school, where I had the opportunity to work with Haggai Maron on equivariant subgraph aggregation networks. I also want to thank Simons Institute for the Theory of Computing, and Institute for Pure and Applied Mathematics (IPAM) for making most talks available to the public. They have always been great learning resources throughout the years.

I am fortunate to work with many excellent researchers from different disciplines across the world,
from whom I learned and enjoyed so much.
Thank Bahador Bahmani, Goonmeet Bajaj, Gopinath Balamurugan, Beatrice Bevilacqua,  Michael M. Bronstein, Fabrizio Frasca, Saket Gurukar, Rafael Gómez-Bombarelli, Truong Son Hy, Moniba Keymanesh, Woojin Kim, Saravana Kumar, Derek Lim, Ran Ma, Lucas Magee, Pranav Maneriker, Haggai Maron,  Facundo Memoli, Benjamin Kurt Miller, Anasua Mitra, Srinivasan Parthasarathy, Vedang Patel,  Balaraman Ravindran, Tess Smidt, Aakash Srinivasan, Jian Tang, Priyesh Vijayan, Nikolaos Vlassis, Wujie Wang,  Yusu Wang,  Dingkang Wang, Teng-Fong Wong,  Zeyu Xiong, Rose Yu, Jie Zhang. I want to especially thank Haggai Maron for being
a great mentor and continuing to push the limit of equariance research. I want to thank Wujie Wang for turning a serendipity meetup in NeurIPS into a fruitful collaboration and Mingkai Xu for a long-term discussion on AI4Science. I want to thank Tess Smidt and Mario Geiger for the help with e3nn.  

I would like to pay special tributes to the researchers I look up to.
I want to thank Taco Cohen, Max Welling, Risi Kondor, Leonidas Guibas, Jure Leskovec, Haggai Maron, Joan Bruna, and Michael M. Bronstein for their pioneering work on graph neural networks and geometric deep learning. I want to thank Sanjeev Arora and Mikhail Belkin for their original perspectives on deep learning theory. 

I want to thank the friends and mentors I met during my internships. Thank you Stephan Eismann and Raphael Townshend for the great opportunity to work on RNA structural prediction in Atomic.ai. Thank Nick Boyd, Brandon Anderson, Paul Raccuglia, Ruth Brillman,
Yuan Zhang and Walid Chaabene for mentorship in Atomic.ai, Google, and Amazon.

I want to thank my friends Shuaicheng Chang, Justin Eldridge, Tao Hou, Like Hui, Wuwei Lan, Tianqi Li, Rui Li, Albert Liang, Lucas Magee, Dingkang Wang, Jiayuan Wang, Cheng Xin, Ryan Slechta, Jiankai Sun, Dayu Shi, Longhua Wu, Hao Zhang, Qi Zhao and Shi Zong
in the Ohio State University and
Tristan Brugere, Sam Chen, Truong Son Hy, Akbar Rafiey, Puoya Tabaghi, Xingyue Xia, Rui Wang, Libin Zhu and Zhengchao Wan from UC San Diego.
I also want to thank Huang Fang, Xueyu Mao, Haiyan Yin, and Yingxue Zhou, whom I got to know during internships and kept in touch since.
I want to especially thank Siyuan Ma for many discussions and insights.
I want to thank my long-time friends Yiming Chen, Shashan Shu, and Tianyu Zuo who have been with me for more than ten years.
I enjoy the time with all my friends and thank you all for making my Ph.D. life so much fun.

I would like to thank Prof. Jingbo Shang, Prof. Rose Yu, and Prof. Gal Mishne who agree to serve on my committee and
provide insightful comments and suggestions.

Lastly, I want to thank my wife for her love and support and my parents for raising me as a curious and independent person. This thesis is dedicated to them.

\end{acknowledgements}

\begin{vita}
\noindent
\begin{cv}{}
\begin{cvlist}{}
\item[2015] Bachelor of Science, China Agricultural University
\item[2017] Master of Science, Stony Brook University
\item[2020] Master of Science, Ohio State University
\item[2023] Doctor of Philosophy, University of California San Diego
\end{cvlist}
\end{cv}

\publications
\begin{itemize}[label={}, leftmargin=*]
    \item Chen Cai, Yusu Wang., ``Convergence of invariant graph networks.",  \textit{International Conference on Machine Learning (ICML)}, 2022.

    \item Chen Cai, Truong Son Hy, Rose Yu, Yusu Wang., `` On the connection between MPNN and Graph Transformer'', \textit{International Conference on Machine Learning (ICML)}, 2023.
    
    \item Chen Cai, Dingkang Wang, Yusu Wang., ``Graph coarsening with neural networks'', \textit{International Conference on Learning Representations (ICLR)}, 2022.
\end{itemize}

\end{vita}

\begin{dissertationabstract}
Message Passing Neural Networks (MPNN) has been the leading architecture for machine learning on graphs. Its theoretical study focuses on increasing expressive power and overcoming over-squashing \& over-smoothing phenomena. The expressive power study of MPNN suggests that one needs to move from local computation to global modeling to gain expressive power in terms of the Weisfeiler-Lehman hierarchy. My dissertation centers around understanding the theoretical property of global GNN, its relationship to the local MPNN, and how to use local MPNN for coarse-graining. In particular, it consists of three parts:

\textbf{Convergence of Invariant Graph Network}. 
One type of global GNNs is the so-called Invariant graph networks (IGN). In the first part, we aim to study the convergence behavior of IGNs, where a similar understanding has already been provided for the local MPNNs.
We investigate the convergence of one powerful GNN, Invariant Graph Network (IGN) over graphs sampled from graphons. We first prove the stability of linear layers for general $k$-IGN (of order $k$) based on a novel interpretation of linear equivariant layers. Building upon this result, we prove the convergence of $k$-IGN under the model 
of \cite{ruiz2020graphon}, where we access the edge weight but the convergence error is measured for graphon inputs. Under the more natural (and more challenging) setting of \cite{keriven2020convergence} where one can only access 0-1 adjacency matrix sampled according to edge probability, we first show a negative result that the convergence of any IGN is not possible. We then obtain the convergence of a subset of IGNs, denoted as IGN-small, after the edge probability estimation. We show that IGN-small still contains functions rich enough to approximate spectral GNNs arbitrarily well. Lastly, we perform experiments on various graphon models to verify our statements.

\textbf{The Connection between MPNN and Graph Transformer}. In the second part, we study the connection between local GNN (MPNN) and global GNN (Graph Transformer). Previous work \cite{kim2022pure} shows that with proper position embedding, GT can approximate MPNN arbitrarily well, implying that GT is at least as powerful as MPNN. Here we study the inverse connection and show that MPNN with virtual node (VN), a commonly used heuristic with little theoretical understanding, is powerful enough to arbitrarily approximate the self-attention layer of GT.
In particular, we first show that if we consider one type of linear transformer, the so-called Performer/Linear Transformer, then MPNN + VN with only $O(1)$ depth and $O(1)$ width can approximate a self-attention layer in Performer/Linear Transformer. Next, via a connection between MPNN + VN and DeepSets, we prove the MPNN + VN with $O(n^d)$ width and $O(1)$ depth can approximate the self-attention layer arbitrarily well, where d is the input feature dimension. Lastly, under some (albeit rather strong) assumptions, we provide an explicit construction of MPNN + VN with $O(1)$ width and $O(n)$ depth approximating the self-attention layer in GT arbitrarily well. 

\textbf{Graph Coarsening with Neural Networks}. Finally, one way to obtain global information via local MPNN is through graph coarsening, where at a coarser level, edges among super-nodes represent more global connections. However, when performing graph coarsening, one hope to be able to preserve the original graph's properties. The specific property we aim to preserve is its spectral property, which can capture long-range interaction in graphs (e.g., the behavior of random walks). 
In the last part, we first propose a framework for measuring the quality of coarsening algorithm and show that depending on the goal, we need to carefully choose the Laplace operator on the coarse graph and associated projection/lift operators. Motivated by the observation that the current choice of edge weight for the coarse graph may be sub-optimal, we parametrize the weight assignment map with GNN and train it to improve the coarsening quality in an unsupervised way. Through extensive experiments on both synthetic and real networks, we demonstrate that our method significantly improves common graph coarsening methods under various metrics, reduction ratios, graph sizes, and graph types. It generalizes to graphs of larger size (25x of training graphs), is adaptive to different losses (differentiable and non-differentiable), and scales to much larger graphs than previous work.
\end{dissertationabstract}

\mainmatter
\chapter{Introduction}
\section{Background}
Graphs are flexible representations for modeling complex objects, such as road networks, protein interaction networks,
social networks, molecules, and so on. From the methodology perspective, modeling functions on general graphs naturally requires handling greater variability
compared to deep learning on images (2d grid) and sequences (1d line graph).
This implies that the design of graph neural networks is more challenging
than common techniques in processing images and sequences such as CNN, RNN, and Transformer.
The study of machine learning \& deep learning on graphs is therefore of great theoretical interest and practical significance, and has been extensively studied in recent years \cite{kipf2016semi,velivckovic2017graph,hamilton2017inductive,gurukar2019network,wang2022generative,cai2018simple,gurukarbenchmarking,cai2019group,cai2023equivariant,zhang2022composition}.

The purpose of this thesis is to provide a local-to-global perspective on the graph neural network (GNN), a leading machine learning architecture for processing graphs.  The local approach to GNN results in Message Passing Neural Network (MPNN) that is widely used in practice, which includes popular models like GAT \cite{velivckovic2017graph}, GCN \cite{kipf2016semi}, and GraphSAGE \cite{hamilton2017inductive}. However,
the theoretical study of MPNN reveals its limitations, such as limited expressive power, over-smoothing, and over-squashing.
Take the expressive power as an example, it is well known MPNN can not be more expressive in terms of distinguishing non-isomorphic graphs than the 1-WL (weisfeiler-lehman) test \cite{xu2018powerful}.
Increasing the expressive power of GNN beyond 1-WL has been extensively studied \cite{bevilacqua2021equivariant,maron2019provably,azizian2020expressive,geerts2020expressive,sato2020survey}. 
On the high level, most work requires some elements of global modeling, in the form of modeling high-order interaction or subgraph aggregation. This motivates the study of the global approach to learning on graphs such as Invariant Graph Network (IGN) \cite{maron2018invariant} and  Graph Transformer (GT) \cite{kreuzer2021rethinking,ying2021transformers}.

The first part of the thesis in \Cref{chapter:ign} introduces the Invariant Graph Network (IGN), a global GNN,
and provides a systematic study of its convergence property. Convergence is closely related to generalization,
a central topic in graph neural network research and machine learning in general.

The second part of the thesis in \Cref{chapter:mpnn_gt} studies the connection between local MPNNs and global Graph Transformers. It connects the local approach (MPNN)
and global approach (Graph Transformer), with DeepSets and Invariant Graph Network (IGN) serving as the conceptual bridge.

One common approach to model long-range interaction modeling on irregular domain is graph coarsening. In \Cref{chapter:cg}, the last part of the thesis, we study the
creative use of MPNN to perform graph coarsening. MPNN offers an alternative to classical optimization techniques,
with the advantage of generalizing to graphs of different sizes.

We next give a short introduction of models that appeared in the thesis and then provide the outline of the thesis in \Cref{sec:outline} and list my contributions in \Cref{sec:contributions}. 
\subsection{Message Passing Neural Network (MPNN)}
Message Passing Neural Networks (MPNNs) are a class of neural networks designed to handle structured data, specifically graph-like structures. MPNNs are capable of capturing the complex relations between nodes in a graph, making them ideal for a variety of tasks ranging from social network analysis to chemical structure prediction. They operate through a process known as ``message passing'', where nodes in the graph exchange and aggregate information iteratively, thereby enabling the network to learn a representation of the whole graph based on local node features and their connections. 

Specifically, the $t$-th iteration of messaging passing takes the following form
\begin{equation*}
\begin{aligned}
m_v^{t+1} & =\sum_{w \in N(v)} M_t\left(h_v^t, h_w^t, e_{v w}\right) \\
h_v^{t+1} & =U_t\left(h_v^t, m_v^{t+1}\right)
\end{aligned}
\end{equation*}
where $h_v^{t+1}$, the hidden representation for node $v$, are updated based on messages $m_v^{t+1}$. $M_t$ is the message functions, and $U_t$ is the vertex update functions. Many popular graph networks such as GCN \cite{kipf2016semi}, GAT \cite{velivckovic2017graph}, and GraphSage \cite{hamilton2017inductive} can be realized under the MPNN framework.

\subsection{Invariant Graph Network (IGN)}
Invariant Graph Network (IGN) is a class of global GNN that treats graphs as order 2 tensors.
Just as CNN interleaves the linear and nonlinear layers, IGN follows the same approach to building graph networks.
As there is no canonical node order, the linear layer needs to be both linear and permutation equivariant.
Such linear and permutation equivariance constraints dramatically reduce the degrees of freedom. \cite{maron2018invariant} characterizes the space of linear and permutation equivariant functions from tensor of order $l$ to
tensor of order $m$, i.e. all linear permutation equivariant functions from $\mathbb{R}^{n^l}$ to $\mathbb{R}^{n^m}$ is of
dimension $Bell(l+m)$. Note that the dimension of space is independent of graph nodes $n$, and it is the reason why IGN
can generalize to graphs of different sizes.

Depending on the order of intermediate tensor representation, IGN can be parameterized by $k$-IGN, where $k$ is the largest tensor order.
It is shown $2$-IGN can arbitrarily approximate MPNN, and therefore as least has the 1-WL expressive power.
In general, $k$-IGN has the expressive power of $k$-WL in terms of distinguishing non-isomorphic graphs, which implies that
IGN is a class of highly expressive GNN that deserve further investigation.

\subsection{Graph Transformer (GT)}
We next introduce another class of global GNN, Graph Transformer (GT).
Because of the great successes of Transformers in natural language processing (NLP) \cite{vaswani2017attention,wolf2020transformers} and recently in computer vision \cite{dosovitskiy2020image,d2021convit,liu2021swin}, there is great interest in extending transformers for graphs. In particular, it encodes the graph structure into the position embeddings and then applies the Transformer layer to mix the features.
One common belief of advantage of graph transformer over MPNN is its capacity in capturing long-range interactions while alleviating over-smoothing \cite{li2018deeper,oono2019graph,cai2020note} and over-squashing in MPNN \cite{alon2020bottleneck,topping2021understanding}. 

Fully-connected Graph transformer \cite{dwivedi2020generalization} was introduced with eigenvectors of graph Laplacian as the node positional encoding (PE). Various follow-up works proposed different ways of PE to improve GT, ranging from an invariant aggregation of Laplacian's eigenvectors in SAN \cite{kreuzer2021rethinking}, pair-wise graph distances in Graphormer \cite{ying2021transformers}, relative PE derived from diffusion kernels in GraphiT \cite{mialon2021graphit}, and recently Sign and Basis Net \cite{lim2022sign} with a principled way of handling sign and basis invariance. 
Other lines of research in GT include combining MPNN and GT  \cite{wu2021representing,rampavsek2022recipe}, encoding the substructures \cite{chen2022structure}, and efficient graph transformers for large graphs \cite{wunodeformer}.

\section{Outline of Thesis}
\label{sec:outline}
In \Cref{chapter:ign}, we introduce a class of GNN model named Invariant Graph Network (IGN). IGN is
parameterized by the $k$, order of intermediate tensor representation. The $k$ is a hyperparameter that can be tuned to balance
the expressive power (in terms Weisfeiler-Lehman hierarchy) and the computational complexity. We study the convergence property of IGN. The problem is stated
as follows: given a continuous graphon model where we can sample a sequence of graphs $G_i$, we are interested in whether the output of
IGN $\phi(G_i)$ will converge. Convergence relates the output of IGN given a sequence of graphs sampled from the same continuous models.
Convergence is easier to study than generalization, a central topic in GNN research because we restrict the variability of input
graphs to come from the same model. Therefore, studying the convergence may shed light on the generalization of GNN \cite{cai2023connection}. 

In \Cref{chapter:mpnn_gt}, we study the connection between local MPNN and global Graph Transformer. Local MPNN has been widely studied while
the global Graph Transformer is a new model that receives a lot of attention recently. One advantage of the Graph Transformer is that
its Transformer backbone allows more effective feature mixing than MPNN, which will require many layers to pass information when
the radius of the input graph is large. Previous work \cite{kim2022pure} shows that with proper position embedding, GT can approximate MPNN arbitrarily well, implying that GT is at least as powerful as MPNN. In this chapter, we study the inverse connection and show that MPNN with virtual node (VN), a commonly used heuristic with little theoretical understanding, is powerful enough to arbitrarily approximate the self-attention layer of GT.
Our work draws a tighter connection between local and global GNN \cite{cai2022convergence}.

In \Cref{chapter:cg}, we study the problem of graph coarsening, which aims to reduce the size of the graph while preserving the essential
property. Despite rich graph coarsening literature, there is only limited exploration of data-driven methods in the field.
In this chapter, we propose a novel data-driven graph coarsening method based on the message-passing neural network (MPNN) \cite{cai2021graph}.

\section{Contributions}
\label{sec:contributions}
The first family of popular global GNN we look into is the so-called Invariant graph networks (IGN).
In \Cref{chapter:ign}, we study the convergence of $k$-IGN under two models, the edge weight continuous model, and the edge probability
discrete model. We first provide a novel interpretation of the linear permutation equivariant basis of $k$-IGN for any $k$, which
is interesting on its own. Based on such interpretations, we prove the convergence of $k$-IGN under the edge weight continuous model.
Under the more challenging edge probability discrete model, we first show that convergence of $k$-IGN is not possible. We then
showed that under the preprocessing step of edge smoothing (used in Graphon estimation),
we can retain the convergence property of a subset of $k$-IGN,  named IGN-small, under the edge probability discrete model.
Lastly, we characterize that IGN-small in some sense is not too small as it still contains the rich class of functions that can approximate spectral GNN arbitrarily well.

Another popular class of global GNN is Graph Transformer (GT).
GT recently has emerged as a new paradigm of graph learning algorithms, outperforming the previously popular Message Passing Neural Network (MPNN) on multiple benchmarks. 
In \Cref{chapter:mpnn_gt}, we provide a systematic study of the approximation power of MPNN with virtual node (VN). In particular,
In particular, we first show that if we consider one type of linear transformer, the so-called Performer/Linear Transformer (Choromanski et al., 2020; Katharopoulos et al., 2020), then MPNN + VN with only $O(1)$ depth and $O(1)$ width can approximate a self-attention layer in Performer/Linear Transformer. Next, via a connection between MPNN + VN and DeepSets, we prove the MPNN + VN with $O(n^d)$ width and $O(1)$ depth can approximate the self-attention layer arbitrarily well, where d is the input feature dimension. Lastly, under some assumptions, we provide an explicit construction of MPNN + VN with $O(1)$ width and O(n) depth approximating the self-attention layer in GT arbitrarily well. On the empirical side, we demonstrate that 1) MPNN + VN is a surprisingly strong baseline, outperforming GT on the recently proposed Long Range Graph Benchmark (LRGB) dataset, 2) our MPNN + VN improves over early implementation on a wide range of OGB datasets.

Finally, one way to obtain global information via local MPNN is through graph coarsening, where at a coarser level, connections among super-nodes represent more global connections. However, when performing graph coarsening, one hopes to be able to preserve the original graph's properties. The specific property we aim to preserve is its spectral property, which can capture long-range interaction in graphs (e.g., the behavior of random walks). 
In \Cref{chapter:cg}, we explored the use of data-driven methods for graph coarsening. We leverage the recent progress of deep learning on graphs for graph coarsening. We first propose a framework for measuring the quality of coarsening algorithm and show that depending on the goal, we need to carefully choose the Laplace operator on the coarse graph and associated projection/lift operators. Motivated by the observation that the current choice of edge weight for the coarse graph may be sub-optimal, we parametrize the weight assignment map with graph neural networks and train it to improve the coarsening quality in an unsupervised way. Through extensive experiments on both synthetic and real networks, we demonstrate that our method significantly improves common graph coarsening methods under various metrics, reduction ratios, graph sizes, and graph types. It generalizes to graphs of larger size (25× of training graphs), is adaptive to different losses (differentiable and non-differentiable), and scales to much larger graphs than previous work.

\chapter{Convergence of Invariant Graph Networks}\label{chapter:ign}
\section{Introduction}

In this chapter, we focus on the the convergence property of a class of powerful global GNN called Invariant Graph Networks (IGN).
Although theoretical properties of GNN such as expressive power \cite{maron2019universality,keriven2019universal,maron2019provably,garg2020generalization,azizian2020expressive,geerts2020expressive,bevilacqua2021equivariant} and over-smoothing \cite{li2018deeper,oono2019graph, cai2020note, zhou2021dirichlet} of GNNs have received much attention, their convergence property is less understood. In this chapter, we systematically investigate the convergence of one of the most powerful families of GNNs, the \emph{Invariant Graph Network (IGN)} \cite{maron2018invariant}.
Different from message passing neural network (MPNN) \cite{gilmer2017neural}, it treats graphs and associated node/edge features as monolithic tensors and processes them in a permutation equivariant manner. \IGN{} can approximate the message passing neural network (MPNN) arbitrarily well on the compact domain. When allowing the use of high-order tensor as the intermediate representation, $k$-IGN is shown at least as powerful as $k$-WL test. As the tensor order $k$ goes to $O(n^4)$, it achieves the universality and can distinguish all graphs of size $n$ \cite{maron2019universality,keriven2019universal,azizian2020expressive}. %

The high level question we are interested in is the convergence and stability of GNNs. In particular, given a sequence of graphs sampled from some generative models, does a GNN performed on them also converge to a limiting object? This problem has been considered recently, however, so far, the studies \cite{ruiz2020graphon,keriven2020convergence} focus on the convergence of \emph{spectral GNNs}, which encompasses several models \cite{bruna2013spectral,defferrard2016convolutional} including GCNs with order-1 filters \cite{kipf2016semi}. However, it is known that the expressive power of GCN is limited.
Given that 2(k)-IGN is strictly more powerful than GCN \cite{xu2018powerful} in terms of 
separating graphs\footnote{In terms of separating graphs, $\text{$k$-IGN}>\text{$2$-IGN}=\text{GIN}>\text{GCN}$ for $k>2$. } and 
its ability to achieve universality, it is of great interest to study the convergence of such powerful GNN. In fact, it is posted as an open question in \cite{keriven2021universality} to study convergence for models more powerful than spectral GNNs and higher order GNNs. This is the question we aim to study in this chapter.

\textbf{Contributions.}
We present the first convergence study of the powerful $k$-IGNs (strictly more powerful than the Spectral GNN which previous work studied).
We first analyze the building block of IGNs: linear equivariant layers, and develop a stability result for such layers.
The case of $2$-IGN is proved via case analysis while the general case of $k$-IGN uses a novel interpretation of the linear equivariant layers which we believe is of independent interest.

There have been two existing models of convergence of \sGNN{}s for graphs sampled from graphons developed in \cite{ruiz2020graphon} and \cite{keriven2020convergence}, respectively.
Using the model of \cite{ruiz2020graphon} (denoted by the \textit{\edgeweight{}}) where we access the edge weight but the convergence error is measured between \textit{graphon inputs} (see \Cref{sec:convergence-ruiz} for details), we obtain analogous convergence results for $k$-IGNs. %
The results cover both deterministic and random sampling for {\bf $k$-IGN} while \cite{ruiz2020graphon}
only covers deterministic sampling for the much weaker {\bf Spectral GNN}s.

Under more natural (and more challenging) setting of \cite{keriven2020convergence} where one can only access 0-1 adjacency matrix sampled according to edge probability (called the \textit{\edgeprob{}}), we first show a negative result that in general the convergence of all IGNs is not possible.
Building upon our earlier stability result, we obtain the convergence of a subset of IGN, denoted as \smallIGN{}, after a step of edge probability estimation. We show that \smallIGN{} still contains rich function class that can approximate Spectral GNN arbitrarily well. %
Lastly, we perform experiments on various graphon models to verify our statements.

\begin{table}[htp]
\renewcommand{\arraystretch}{1.5}
\centering
\caption{Linear equivariant maps for $\mb{R}^{n \times n} \rightarrow \mb{R}^{n\times n}$ and $\mb{R}^{[0,1]^2} \rightarrow \mb{R}^{[0,1]^2}$. $\one$ is a all-one vector  of size $n\times 1$ and $\mathrm{I}_{u=v}$ is the indicator function. }
\label{tab:R2-R2}
\resizebox{1\textwidth}{!}{
\begin{tabular}[]{@{}llll@{}}
\toprule
\makecell[l]{Operations} & Discrete & Continuous & Partitions \\ \midrule
\makecell[l]{1-2: The identity and\\ transpose operations} &
\makecell[l]{$T(A) = A$\\ $T(A) = A^T$} &
\makecell[l]{$T(W) = W$ \\ $T(W) = W^T$} &
\makecell[l]{$\{\{1,3\},\{2,4\}\}$ \\ $\{\{1,4\},\{2,3\}\}$}
\\ \hline

3: The diag operation &
\makecell[l]{$T(A) = \text{Diag}(\text{Diag}^{*}(A))$} &
\makecell[l]{$T(W)(u, v) = W(u, v)\mathrm{I}_{u=v}$}  &
\makecell[l]{$\{\{1,2,3,4\}\}$}
\\ \hline

\makecell[l]{4-6: Average of rows replicated \\ on rows/ columns/ diagonal} &
\makecell[l]{$T(A) = \frac{1}{n}A\one\one^T$ \\ $T(A) = \frac{1}{n}\one(A\one)^T$\\ $T(A) = \frac{1}{n}\text{Diag}(A\one)$}
&
\makecell[l]{$T(W)(*, u) = \int W(u, v)dv$\\ $T(W)(u, *) = \int W(u, v)dv$ \\ $T(W)(u, v) = \mathrm{I}_{u=v}\int W(u, v')dv' $}  &
\makecell[l]{$\{\{1,4\},\{2\},\{3\}\}$ \\ $\{\{1,3\},\{2\},\{4\}\}$ \\ $\{\{1,3,4\},\{2\}\}$}
\\\hline

\makecell[l]{7-9: Average of columns replicated \\on rows/ columns/ diagonal } &
\makecell[l]{$T(A) = \frac{1}{n}A^T\one\one^T$ \\ $T(A) = \frac{1}{n}\one(A^T \one)^T$ \\ $T(A) = \frac{1}{n}\text{Diag}(A^T\one).$} &
\makecell[l]{$T(W)(*, v) = \int W(u, v)du$\\ $T(W)(v, *) = \int W(u, v)du$ \\ $T(W)(u, v) = \mathrm{I}_{u=v}\int W(u', v)du' $} &
\makecell[l]{$\{\{1\},\{2,4\},\{3\}\}$ \\ $\{\{1\},\{2,3\},\{4\}\}$ \\ $\{\{1\},\{2,3,4\}\}$}
\\ \hline

\makecell[l]{10-11: Average of all elements \\replicated on all matrix/ diagonal} &
\makecell[l]{$T(A)=\frac{1}{n^2}(\one^T A\one) \cdot \one \one^T$
\\ $T(A) = \frac{1}{n^2}(\one^T A\one) \cdot \text{Diag}(\one).$} &
\makecell[l]{$T(W)(*, *) = \int W(u, v)dudv$\\ $T(W)(u, v) = \mathrm{I}_{u=v}\int W(u', v')du'dv'$ } &
\makecell[l]{$\{\{1\},\{2\},\{3\},\{4\}\}$ \\ $\{\{1\},\{2\},\{3,4\}\}$}
\\ \hline

\makecell[l]{12-13: Average of diagonal elements \\ replicated on all matrix/diagonal} &
\makecell[l]{$T(A) = \frac{1}{n}(\one^T \text{Diag}^{*}(A)) \cdot \one \one^T$\\ $T(A) = \frac{1}{n}(\one^T \text{Diag}^{*}(A)) \cdot \text{Diag}(\one)$} &
\makecell[l]{$T(W)(*, *) = \int \mathrm{I}_{u=v}W(u, v) dudv$\\ $T(W)(u,v) = \mathrm{I}_{u=v}\int W(u', u') du'$ } &
\makecell[l]{$\{\{1,2\},\{3\},\{4\}\}$ \\ $\{\{1,2\},\{3,4\}\}$}
\\ \hline

\makecell[l]{14-15: Replicate diagonal elements \\ on rows/columns} &
\makecell[l]{$T(A) = \text{Diag}^{*}(A)\one^T$ \\ $T(A) = \one\text{Diag}^{*}(A)^T$} &
\makecell[l]{$T(W)(u, v)= W(u, u)$\\ $T(W)(u, v) = W(v, v)$} &
\makecell[l]{$\{\{1,2,4\},\{3\}\}$ \\ $\{\{1,2,3\},\{4\}\}$}
\\
\bottomrule
\end{tabular}
}
\end{table}

\begin{table}
\renewcommand{\arraystretch}{1.5}
\centering
\caption{Linear equivariant maps for $\mb{R}^{n} \rightarrow \mb{R}^{n\times n}$ and $\mb{R}^{[0,1]} \rightarrow \mb{R}^{[0,1]^2}$.}
\label{tab:R1-R2}
\resizebox{1\textwidth}{!}{
\begin{tabular}[]{@{}llll@{}}
\toprule
Operations & Discrete & Continuous & Partitions\\ \midrule
1-3: Replicate to diagonal/rows/columns &
\makecell[l]{$T(A) = \text{Diag}(A)$\\  $T(A)_{i, j} = A_i$\\$T(A)_{i, j}= A_j$} &
\makecell[l]{$T(W)(u, v) = \mathrm{I}_{u=v}W(u)$\\ $T(W)(u,v) = W(u)$\\$T(W)(u,v)=W(v)$} &
\makecell[l]{ \{\{1,2,3\}\} \\ \{\{1,3\},\{2\}\} \\ \{\{1,2\},\{3\}\}}
\\

\hline

4-5: Replicate mean to diagonal/all matrix &
\makecell[l]{$T(A)_{i, i} = \frac{1}{n}A \one$\\$T(A)_{i, j} = \frac{1}{n}A \one $} &
\makecell[l]{$T(W)(u, v) = \mathrm{I}_{u=v}\int W(u) du$\\$T(W)(u,v) = \int W(u) du$} &
\makecell[l]{ \{\{1\},\{2,3\}\} \\ \{\{1\},\{2\},\{3\}\} }
\\
\bottomrule
\end{tabular}
}
\end{table}

\begin{table}
\renewcommand{\arraystretch}{1.5}
\centering
\caption{Linear equivariant maps for $\mb{R}^{n \times n} \rightarrow \mb{R}^{n}$ and $\mb{R}^{[0,1]^2} \rightarrow \mb{R}^{[0,1]}$.}
\label{tab:R2-R1}
\resizebox{1\textwidth}{!}{
\begin{tabular}[]{@{}llll@{}}
\toprule
Operations & Discrete & Continuous & Partitions\\ \midrule
\makecell[l]{1-3: Replicate diagonal/row mean/\\columns mean} &
\makecell[l]{$T(A)=\text{Diag}^{*}(A)$\\$T(A)=\frac{1}{n}A\one$\\$T(A) = \frac{1}{n}A^T\one$} &
\makecell[l]{$T(W)(u) = W(u, u)$\\$T(W)(u) = \int W(u, v)dv$\\$T(W)(u) = \int W(u, v)du$} &
\makecell[l]{ \{\{1,2,3\}\} \\ \{\{1,2\},\{3\}\} \\ \{\{1,3\},\{2\}\} }
\\ \hline

\makecell[l]{4-5: Replicate mean of all elements/\\ mean of diagonal} &
\makecell[l]{$T(A)_i = \frac{1}{n^2}\one^T A \one $ \\$T(A)_i = \frac{1}{n}\one^T \text{Diag}(\dgdual{A}) \one$} &
\makecell[l]{$T(W)(u) = \int W(u, v) du dv$ \\ $T(W)(u) = \int \mathrm{I}_{u, v} W(u, v)dudv$} &
\makecell[l]{ \{\{1\},\{2\},\{3\}\} \\ \{\{1,2\},\{3\}\} }
\\
\bottomrule
\end{tabular}
}
\end{table}

\section{Preliminaries}
\label{sec:pre}
\subsection{Notations}
To talk about convergence/stability, we will consider graphs of different sizes sampled from a generative model. Similar to the earlier work in this direction, the specific general model we consider is a graphon model.

\textbf{Graphons.} A graphon is a bounded, symmetric and measurable function~$W: [0,1]^2 \to [0,1]$. We denote the space of graphon as $\mc{W}$. It can  be intuitively thought of as an undirected weighted graph with an uncountable number of nodes: roughly speaking, given $u_i, u_j \in [0,1]$, we can consider there is an edge $(i,j)$ with weight $W(u_i, u_j)$.
Given a graphon $W$, we can sample {\bf unweighted} graphs of any size from $W$, either in a deterministic or stochastic manner. We defer the definition of the sampling process until we introduce the \edgeweight{} in \Cref{sec:convergence-ruiz} and \edgeprob{} in \Cref{sec:EP-convergence}.

\textbf{Tensor.} Let $[n]$ denote $\{1, ..., n\}$. A tensor $X$ of order $k$, called a \emph{$k$-tensor}, is a map from $[n]^{\otimes k}$ to $\mb{R}^d$. If we specify a name $\name{i}$ for each axis, we then say $X$ is indexed by $(\name{1}, ..., \name{k})$. With slight abuse of notation, we also write that $X \in \dtensor{k}{d}$. %
We refer to $d$ as the \emph{feature dimensions} or the \emph{channel dimensions}. If $d = 1$, then we have a $k$-tensor $\mb{R}^{n^k\times 1}  = \mb{R}^{n^k}$.
Although the name for each axis acts as an identifier and can be given arbitrarily, we will use \textit{set} to name each axis in this chapter. For example, given a 3-tensor $X$, we use $\{1\}$ to name the first axis, $\{2\}$ for the second axis, and so on. The benefits of doing so will be clear in \Cref{subsec:stability-of-k-ign}.

\textbf{Partition.} A partition of $[k]$, denoted as $\gamma$, is defined to be a set of disjoint sets $\gamma:=\{\gamma_1, ..., \gamma_s \}$  with $s\leqslant k$ such that the following condition satisfies,
1) for all $i\in [s],  \gamma_{i} \subset [k]$, 2) $\gamma_i \cap \gamma_j = \emptyset, \forall$ $i, j\in [s]$, and 3) $\cup_{i=1}^{s} \gamma_i = [k]$.
We denote the space of all partitions of $[k]$ as $\parspace_k$. %
Its cardinality is called the $k$-th \emph{bell number} $\bell{k}=|\parspace_k|$.

\textbf{Other conventions.}
 By default, we use 2-norm (Frobenius norm) to refer \lnorm\
for all vectors/matrices and \Lnorm\ for functions on $[0, 1]$ and $[0,1]^2$. $\|\cdot\|_2$ or $\|\cdot\|$ denotes the 2 norm for discrete objects 
while $\|W\|_{L_2}:= \int \int W(u, v)dudv$ denotes the norm for continuous objects. Similarly, we use $\| \cdot \|_{{\infty}}$ and $\| \cdot \|_{L_\infty}$ to denotes the infinity norm.
When necessary, we use $\|\cdot \|_{L_2([0, 1])}$ to specify the support explicitly. We use $\specnorm{\cdot}$ to denote spectral norm. $\Phi_c$ and $\Phi_d$ refers to the continuous IGN and discrete IGN respectively. We sometimes call a function $f: [0, 1] \rightarrow \mb{R}^d$ a \emph{graphon signal}. Given $A\in \mb{R}^{n^k \times d_1}, B\in \mb{R}^{n^k \times d_2}$, $[A, B]$ is defined to be the concatenation of $A$ and $B$ along feature dimensions, i.e., $[A, B] \in \mb{R}^{n^k \times (d_1 + d_2)}$. See \Cref{table:symbol_notation} for the full symbol list.

\begin{table}[!htbp]
\caption{Summary of important notations.}
\begin{center}
\resizebox{\textwidth}{!}{
\begin{tabular}{@{}l|l@{}}
    \hline
    \toprule
    Symbol & Meaning \\
    \midrule
    \midrule
    $\one_n$ & all-one vector of size $n\times 1$ \\
    $\| \cdot \|_2/\| \cdot \|_{L_2}$ & 2-norm for matrix/
    graphon \\
    $\| \cdot \|_\infty / \| \cdot \|_{L_{\infty}}$ & infinity-norm for matrix/graphon \\
    $[\cdot, \cdot]$ & \makecell[l]{Given $A\in \mb{R}^{n^k \times d_1}, B\in \mb{R}^{n^k \times d_2}$, \\ $[A, B]$ is the concatenation of $A$ and $B$ along feature dimension. $[A, B] \in \mb{R}^{n^k \times (d_1 + d_2)}$.}
      \\
    $W: [0, 1]^2 \rightarrow [0,1]$ &  graphon \\
    $X\in \mb{R}^{[0,1]\times d}$ & 1D signal \\
    $\mc{W}$ & space of graphons \\
    $\newnorm{\cdot}$ & \textnewnorm{}. When the underlying norm is $L_{\infty}$ norm, we also use $\newnorminf{\cdot}$. \\
    $\mathrm{I}$ & indicator function \\
    $I$ & interval \\
    SGNN & spectral graph neural networks, defined in \Cref{eq:sgnn} \\
    $\LE_{\ell, m}$ & linear equivariant maps from $\ell$-tensor to $m$-tensor \\
    \hline
     & Notations related to sampling \\ \hline

    $W_n$ & Induced piecewise constant graphon from fixed grid \\
    $\widetilde{W_n}$ & Induced piecewise constant graphon from random grid \\
    $\inducedEW$ & \makecell[l]{Induced piecewise constant graphon from random grid, but resize the all \\ individual blocks to be of equal size (also called chessboard graphon in this chapter). \\ $\inducedEW(I_i \times I_j) \:= W(u_{(i)}, u_{(j)})$} \\
    $\sampleW$ & $n\times n$ matrix sampled from $W$; $\sampleW(i, j) = W(u_{i}, u_{j})$ \\
    $\widehat{W}_{n\times n} \in \mb{R}^{n \times n}$ & \makecell[l]{the estimated edge probability from graphs sampled according to \\ edge probability discrete model from \cite{zhang2015estimating}} \\
    $\widetilde{x_n} \in \mb{R}^{n \times d}$ & sampled signal $[\widetilde{x_n}]_i:=X(u_i)$ \\
    $X_n$ & induced 1D piecewise graphon signal from fixed grid \\
    $\inducedX$ & induced 1D piecewise graphon signal from random grid \\
    $S_U$ & normalized sampling operator for random grid. $S_Uf(i, j) = \frac{1}{n}(f(u_{(i)}), f(u_{(j)})$ \\
    $S_n$ & normalized sampling operator for fixed grid. $S_nf(i, j) = \frac{1}{n}(f(\frac{i}{n}), f(\frac{j}{n}))$\\
    $\MSE_U(x, f)$ & $\left(n^{-1} \sum_{i=1}^{n}\left\|x_{i}-f\left(u_{i}\right)\right\|^{2}\right)^{1 / 2}$ for 1D signal; $\left(n^{-2} \sum_{i}\sum_{j}\left\|x_{i, j}-f\left(u_{i}, u_{j}\right)\right\|^{2}\right)^{1 / 2}$ for 2D case\\
    $\alpha_n$ & a parameter that controls the sparsity of sample graphs. Set to be $1$ in the chapter. \\
    \hline
     & Notations related to IGN \\ \hline
    $\bell{k}$ & Bell number: number of partitions of $[k]$. $\bell{2} = 2, \bell{3} = 5, \bell{4} = 15, \bell{5} = 52...$ \\
    $\parspace_k$ & space of all partitions of $[k]$ \\
    $\mc{I}_k$ & the space of indices. $\mc{I}_k:=\{(i_1, ..., i_k)| i_1\in [n], ..., i_k\in [n]\}$. Elements of $\mc{I}_k$ is denoted as $\bs{a}$ \\
    $\gamma \in [k]$ & \makecell[l]{partition of $[k]$. For example $\set{\set{1,2}, \set{3}}$ is a partition of $[3]$. \\
    The total number of partitions of $[k]$ is $\bell{k}$.} \\
    $\bs{a} \in \gamma$ & $\bs{a}$ satisfies the equivalence pattern of $\gamma$. For example, $(x, x, y) \in \set{\set{1, 2}, \set{3}}$ where $x, y, z\in [n]$.  \\
    $\gamma < \beta $ & given two partitions $\gamma, \beta \in \parspace_k$, $\gamma < \beta$ if $\gamma$ is finer than $\beta$. For example, $\set{1,2,3}<\set{\set{1,2}, \set{3}}$. \\

    $\bs{B}_{\gamma}$ & \makecell[l]{$l+m$ tensor; tensor representation of $\LE_{l, m}$ maps. \\ we differentiate $T_{\gamma}$ (operators) from  $\bs{B}_{\gamma}$ (tensor representation of operators)}\\

    $\mc{B}$ & a basis of the space of linear equivariant operations from $\ell$-tensor to $m$-tensor. $\mc{B}=\set{T_{\gamma}|\gamma \in \parspace_{l+k}}$ \\
    
    $T_c/T_d$ & linear equivariant layers for graphon (continuous) and graphs (discrete) \\
    
    $\Phi_c/\Phi_d$ & IGN for graphon (continuous) and graphs (discrete) \\
    $L^{(i)}$ & i-th linear equivariant layer of IGN \\
    $L$ & normalized graph Laplacian \\
    $T_{i}$ & basis element of the space of linear equivariant maps; sometimes also written as $T_{\gamma}$. \\

    \bottomrule
\end{tabular}
}
\end{center}
\label{table:symbol_notation}
\end{table}

\subsection{Invariant Graph Network}

\begin{definition}\label{def:ign}
An Invariant Graph Network (IGN) is a function $\Phi: \mb{R}^{n^{2} \times d_{0}} \rightarrow \mb{R}^d$ of the following form:
\begin{equation}\label{eqn:ign}
F=%
h \circ L^{(T)} \circ \sigma \circ \cdots \circ \sigma \circ L^{(1)},
\end{equation}
where each $L^{(t)}$ is a linear equivariant (LE) layer \cite{maron2018invariant} from $\mb{R}^{n^{k_{t-1}} \times d_{t-1}} \text { to } \mb{R}^{n^{k_{t}} \times d_{t}}$ (i.e., mapping a $k_{t-1}$ tensor with $d_{t-1}$ channels to a $k_t$ tensor with $d_t$ channels), $\sigma$ is nonlinear activation function, $h$ is a linear invariant layer from $k_T$-tensor $\mb{R}^{n^{k_T} \times d_{T}} \text { to vector in } \mb{R}^d$. $d_t$ is the channel number, and $k_t$ is tensor order in $t$-th layer. %
\end{definition}

Let $\dg{\cdot}$ be the operator of constructing a diagonal matrix from vector and $\dgdual{\cdot}$ be the operation of extracting a diagonal from a matrix.
Under the IGN framework, we view a graph with $n$ nodes as a $2$-tensor: In particular,
given its adjacency matrix $A_n$ of size $n\times n$ with node features $X_{n}\in \mb{R}^{n\times d_{\text{node}}}$ and edge features $E_{n \times n} \in \mb{R}^{n^2\times d_{\text{edge}}}$, the input of IGN is the concatenation of $[A_n, \text{Diag}(X_n), E_{n\times n}]\in \mb{R}^{n^2\times (1 + d_{\text{node}}+ d_{\text{edge}})}$ along different channels. %
We drop the subscript when there is no confusion.
We use {\bf $2$-IGN} to denote the IGN whose largest tensor order within any intermediate layer is $2$,
while {\bf $k$-IGN} is one whose largest tensor order across all layers is $k$. We use IGN to refer to the general IGN for any order $k$. %

Without loss of generality, we consider input and output tensor to have a single channel. %
Consider all linear equivariant maps from $\mb{R}^{n^{\ell}}$ to $\mb{R}^{n^{m}}$, denoted as $\LE_{\ell + m}$. \cite{maron2018invariant} characterizes the basis of the space of $\LE_{\ell, m}$. It turns out that the cardinality of the basis equals to the bell number $\bell{\ell + m}$, thus depending only on the order of input/output tensor and independent from graph size $n$.
As an example, we list a specific basis of the space of \LE{} maps for \IGN{} (thus with tensor order at most $2$) in \Cref{tab:R2-R2,tab:R1-R2,tab:R2-R1} when input/output channel numbers are both 1. Extending the \LE{} layers to multiple input/output channels is straightforward, and can be achieved by parametrizing the \LE{} layers according to indices of input/output channel. See \Cref{remark:IGN}.
Note that one difference of the operators in  \Cref{tab:R2-R1,tab:R1-R2,tab:R2-R2} from those given in the original paper is that
here we normalize all operators appropriately w.r.t. the graph size $n$. (This normalization is also in the official implementation of the IGN paper.)
This is necessary when we consider the continuous limiting case.

\begin{remark}[multi-channel IGN contains MLP]
\label{remark:IGN}
For simplicity, in the main text, we focus on the case when the input and output tensor channel number is 1. The general case of multiple input and output channels is presented in Equation 9 of \cite{maron2018invariant}. The main takeaway is that permutation equivariance does not constrain the mixing over feature channels, i.e., the space of linear equivariant maps from $\mb{R}^{n^{\ell}\times d_1} \rightarrow \mb{R}^{n^{m}\times d_2}$ if of dimension $d_1d_2\bell{l+m}$. Therefore IGN contains MLP.  %
\end{remark}

To talk about convergence, one has to define the continuous analog of IGN for graphons. In \Cref{tab:R2-R2,tab:R2-R1,tab:R1-R2} we extend all \LE{} operators %
defined for graphs to graphons, %
resulting in the continuous analog of $2$-IGN, denoted as 2-\cIGN{} or $\Phi_c$ in the remaining text. Similar operation can be done in general for $k$-IGN as well, where the basis elements for $k$-IGNs will be described in \Cref{subsec:stability-of-k-ign}.
\begin{definition}[$2$-cIGN]\label{def:cign}
By extending all \LE{} layers for $2$-IGN to the graphon case as shown in \Cref{tab:R2-R2,tab:R2-R1,tab:R1-R2}, we can definite the corresponding 2-\cIGN{} %
via \cref{eqn:ign}.
\end{definition}

\subsection{Edge Probability Estimation from \cite{zhang2015estimating}}
\label{subsec:zhang}
We next restate the setting and theorem regarding the theoretical guarantee of the edge probability estimation algorithm.

\begin{definition}
For any $\delta, A_1>0$, let $\mathcal{F}_{\delta ; L}$ de note a family of piecewise Lipschitz graphon functions $f:[0,1]^{2} \rightarrow[0,1]$ such that $(i)$ there exists an integer $K \geq 1$ and a sequence $0=x_{0}<\cdots<x_{K}=1$ satisfying $\min _{0 \leqslant  s \leqslant  K-1}\left(x_{s+1}-\right.$ $\left.x_{s}\right) \geq \delta$, and (ii) both $\left|f\left(u_{1}, v\right)-f\left(u_{2}, v\right)\right| \leqslant  A_1\left|u_{1}-u_{2}\right|$ and $\left|f\left(u, v_{1}\right)-f\left(u, v_{2}\right)\right| \leqslant  A_1 \mid v_{1}-$ $v_{2} \mid$ hold for all $u, u_{1}, u_{2} \in\left[x_{s}, x_{s+1}\right], v, v_{1}, v_{2} \in\left[x_{t}, x_{t+1}\right]$ and $0 \leqslant  s, t \leqslant  K-1$
\end{definition}

Assume that $\alpha_n$ = 1. It is easy to see that the setup considered in \cite{zhang2015estimating} is slightly more general than the setup in \cite{keriven2020convergence}. The statistical guarantee of the edge smoothing algorithm is stated below.

\begin{theorem}[\cite{zhang2015estimating}]\label{thm:graphon-estimation}
Assume that $A_1$ is a global constant and $\delta=\delta(n)$ depends on $n$, satisfying $\lim_{n \rightarrow \infty} \delta /(n^{-1} \log n)^{1 / 2} \rightarrow \infty$. Then the estimator $\tilde{P}$ with neighborhood $\mathcal{N}_{i}$ defined in \cite{zhang2015estimating} and $h=C(n^{-1} \log n)^{1 / 2}$ for any global constant $C \in(0,1]$, satisfies
$
\max_{f \in \mathcal{F}_{\delta; A_1}} \operatorname{pr}\{d_{2, \infty}(\tilde{P}, P)^{2} \geq C_{1}(\frac{\log n}{n})^{1 / 2}\} \leqslant  n^{-C_{2}}
$
where $C_{1}$ and $C_{2}$ are positive global constants. Here, $d_{2, \infty}(P, Q) \:= n^{-1 / 2}\|P-Q\|_{2, \infty}=\max_{i} n^{-1 / 2} \|P_{i}-Q_{i} \|_{2}$.
\end{theorem}

\section{Stability of Linear Layers in IGN}
\label{sec:stability}
In this section, we first show a stability result for a single linear layer of IGN. That is, given two graphon $W_1, W_2$, we show that if $\|W_1 - W_2\|_{\tn{pn}}$ is small, then the distance between the objects after applying a single \LE{} layer remain close. Here $\| \cdot \|_{\tn{pn}}$ is a \textnewnorm{} that will be introduced in a moment.
Similar statements also hold for the discrete case when the input is a graph. We first describe how to prove stability for $2$-(c)IGN as a warm-up. We then prove it for $k$-(c)IGN, which is significantly more interesting and requires a new interpretation of the elements in a specific basis of the space of \LE{} operators in \cite{maron2018invariant}.

A the general \LE{} layer $T: \mb{R}^{n^\ell} \to \mb{R}^{n^m}$%
can be written as $T = \sum_{\gamma} c_{\gamma}T_{\gamma}$, where $T_{\gamma} \in \mc{B}:=\set{T_{\gamma}|\gamma \in \parspace_{\ell +m}}$ is the basis element of the space of $\LE_{\ell, m}$ and $c_{\gamma}$ are denoted as filter coefficients. Hence proving the stability of $T$ can be reduced to showing the stability for each element in $\mc{B}$, which we focus from now on.

\subsection{Stability of Linear Layers of $2$-IGN}

A natural way to show stability is by showing that the spectral norm of each \LE{} operator in a basis is bounded. However, even for 2-IGN, as we see some \LE{} operator %
requires replicating ``diagonal elements to all rows" (e.g., operator 14-15 in Table \ref{tab:R2-R2}), and has unbounded spectral norm. To address this challenge, we need a more refined analysis. In particular, below we will introduce a ``new" norm that treats the diagonal differently from non-diagonal elements for the $2$-tensor case. We term it \emph{\textnewnorm{}} as later when handling high order $k$-IGN, we will see that this norm arises naturally w.r.t. the partition of index set of tensors.

\begin{restatable}[\textNewnorm{}]{definition}{partitionnormdef}
\label{def:new-norm}
The \emph{\textnewnorm{}} of 2-tensor $A\in \mb{R}^{n^2}$ is defined as $\|A\|_{\tn{pn}}:=(\frac{\|\dgdual{A}\|_2}{\sqrt{n}}, \frac{\|A\|_2}{n})$.
The continuous analog of the \textnewnorm{} for graphon $W \in \mc{W}$ is defined as $\|W\|_{\tn{pn}} = \left(\sqrt{\int W^2(u,u)du}, \sqrt{\iint W^2(u, v) dudv} \right)$.

We refer to the first term as the \emph{normalized diagonal norm} and the second term as the \emph{normalized matrix norm}. Furthermore, we define operations like addition/comparison on the \textnewnorm{} simply as component-wise operations. For example, $\newnorm{A} \le \newnorm{B}$ if each of the two terms of $A$ is at most the corresponding term of $B$.
\end{restatable}

As each term in \textnewnorm{} is a norm on different parts of the input, the \textnewnorm{} is also a norm.
By summing over the finite feature dimension both for finite and infinite cases, the definition of the \textnewnorm{} can be extended to multi-channel tensors $\mb{R}^{n^2\times d}$ and its continuous version $ \mb{R}^{[0,1]^2 \times d}$. See \Cref{subsec:extending-new-norm} for details.  %

The following result shows that each basis operation for 2-IGN, shown in Tables \ref{tab:R2-R2}, \ref{tab:R1-R2} and \ref{tab:R2-R1}, is stable w.r.t. the \textnewnorm{}. Hence a \LE{} layer consisting of a finite combination of these operations will remain stable. The proof is via a case-by-case analysis and can be found in \Cref{app:2-IGN-linear-stability}.

\begin{restatable}[]{proposition}{RtwoRtwo}
\label{prop:R2-R2}
For all \LE{} operators $T_i:\mb{R}^{n^2} \rightarrow \mb{R}^{n^2}$ of discrete $2$-IGN listed in \Cref{tab:R2-R2}, $\|T_i(A)\|_{\tn{pn}} \leqslant  \newnorm{A}$  for any
$A \in \mb{R}^{n^2}$. Similar statements hold for $T_i: \mb{R}^{n} \rightarrow \mb{R}^{n^2}$ and $T_i: \mb{R}^{n^2} \rightarrow \mb{R}^{n}$ in \Cref{tab:R1-R2,tab:R2-R1}. In the case of continuous 2-cIGN, the stability also holds. %
\end{restatable}

\begin{remark}
Note that this also implies that given $W_1, W_2 \in \mc{W}$, we have that $\newnorm{T_i(W_1) - T_i(W_2)} \le \newnorm{W_1 - W_2}$. Similarly, given $A_1, A_2 \in \mb{R}^{n^2 \times 1}$ $= \mb{R}^{n^2}$, we have $\newnorm{T_i(A_1) - T_i(A_2)} \le \newnorm{A_1 - A_2}$.
\end{remark}

\subsection{Stability of Linear Layers of \kIGN{}}
\label{subsec:stability-of-k-ign}
We now consider the more general case of \kIGN{}. In principle, the proof of \IGN{} can still be extended to \kIGN{}, but going through all $\bell{k}$ number of elements of \LE{} basis of \kIGN{} one by one
 can be quite cumbersome. %
In the next two subsections, we provide a new interpretation of elements of the basis of space of $\LE_{\ell, m}$ in a unified framework so that we can avoid a case-by-case analysis. Such an interpretation, detailed in \Cref{subsec:interpretation}, is potentially of independent interest.
First, we need some notations. %

\begin{definition}[Equivalence pattern]\label{def:equivpattern}
Given a $k$-tensor $X$, denote the space of its indices $\{(i_1, ..., i_k) \mid i_1\in [n], ..., i_k\in [n]\}$ by $\mc{I}_{k}$. Given $X$, $\gamma =\{\gamma_1, ..., \gamma_d\}\in \parspace_k$ and an element $\bs{a}=(a_1, ..., a_k) \in \mc{I}_{k}$, we say $\bs{a} \in \gamma$ if $i, j \in \gamma_{l}$ for some $l\in[d]$ always implies $a_i = a_j$. %
  Alternatively, we also say $\bs{a}$ satisfies the equivalence pattern of $\gamma$ if $\bs{a} \in \gamma$.
\end{definition}
As an example, suppose $\gamma=\{\{1, 2\}, \{3\}\}$. Then $(x, x, y) \in \gamma$ %
while $(x, y, z) \notin \gamma$.
Equivalence patterns can induce ``slices''/sub-tensors of a tensor.
 \begin{definition}[Slice/sub-tensor of $X \in \dtensor{k}{1}$ for $\gamma \in \parspace_k$]
 \label{def:slice}
 Let $X \in \dtensor{k}{1}$ be a $k$-tensor indexed by $(\set{1}, ..., \set{k})$.
 Consider a partition $\gamma = \{\gamma_1, ..., \gamma_{k'}\} \in \parspace_k$ of cardinality $k' \leqslant k$.
 The \emph{slice (sub-tensor) of $X$ induced by $\gamma$} is a $k'$-tensor $X_{\gamma}$, indexed by $(\gamma_1, ..., \gamma_{k'})$, and defined to be
  $X_{\gamma}(j_1, ..., j_{k'}) := X(\iota_{\gamma}(j_1, ..., j_{k'}))$ where $j_{\cdot} \in [n]$ and $\iota_{\gamma}(j_1, ..., j_{k'})\in \gamma$.
  $\iota_{\gamma}: [n]^{k'} \rightarrow [n]^k$ is defined to be $\iota_{\gamma}(j_1, ..., j_{k'}) := (i_1, ..., i_k) $ such that $\{a, b\}\subseteq \gamma_c$   implies $i_a = i_b := j_c$. Here $a, b\in [k], c\in [k']$. %
As an example, we show five slices of a $3$-tensor in \Cref{fig:slices}.
 \end{definition}
Consider the \LE{} operators from $\mb{R}^{n^{\ell}}$ to $\mb{R}^{n^m}$. Each such map $T_{\gamma}$ can be represented by a matrix of size $n^{\ell}\times n^m$ which can further considered as a $(\ell+m)$-tensor $\bs{B}_{\gamma}$. %
\cite{maron2018invariant} showed that a specific basis for such operators can be characterized as follows: Each basis element will correspond to one of the $\bell{\ell+m}$ partitions in $\parspace_{\ell+m}$. In particular, given a partition $\gamma\in \parspace_{\ell+m}$, we have a corresponding basis \LE{} operator $T_{\gamma}$ and its tensor representation $\mathbf{B}_{\gamma}$ defined as follows:
\begin{equation}
\text{for any}~\bs{a} \in \mc{I}_{\ell +m}, ~~
\mathbf{B}_{\gamma}(\boldsymbol{a})= \begin{cases}1 & \boldsymbol{a} \in \gamma \\ 0 & \text { otherwise }\end{cases}
\end{equation}\label{eqn:kIGNbasis}
The collection $\mc{B} = \{ T_\gamma \mid \gamma \in \parspace_{\ell+m}\}$ form a basis for all $\LE_{\ell, m}$ maps. In \Cref{subsec:interpretation}, we will provide an interpretation of each element of $\mc{B}$, making it easy to reason its effect on an input tensor using a unified framework.

Before the main theorem, we also need to extend the \textnewnorm{} in \Cref{def:new-norm} from 2-tensor to high-order tensor. %
Intuitively, for $X \in \mb{R}^{n^{k}}$, $\newnorm{X}$ has $\bell{k}$ components, where each component corresponds to the normalized norm of $X_{\gamma}$, the slice of $X$ induced by $\gamma \in \parspace_{k}$. See \Cref{fig:slices} for examples of slices of a 3-tensor.  The \textnewnorm{} of input and output of a $\LE_{\ell, m}$ will be of dimension $\bell{\ell }$ and $\bell{m}$ respectively. See \Cref{subsec:extending-new-norm} for details.

\begin{figure}[htp]
  \centering
  \includegraphics[width=1\linewidth]{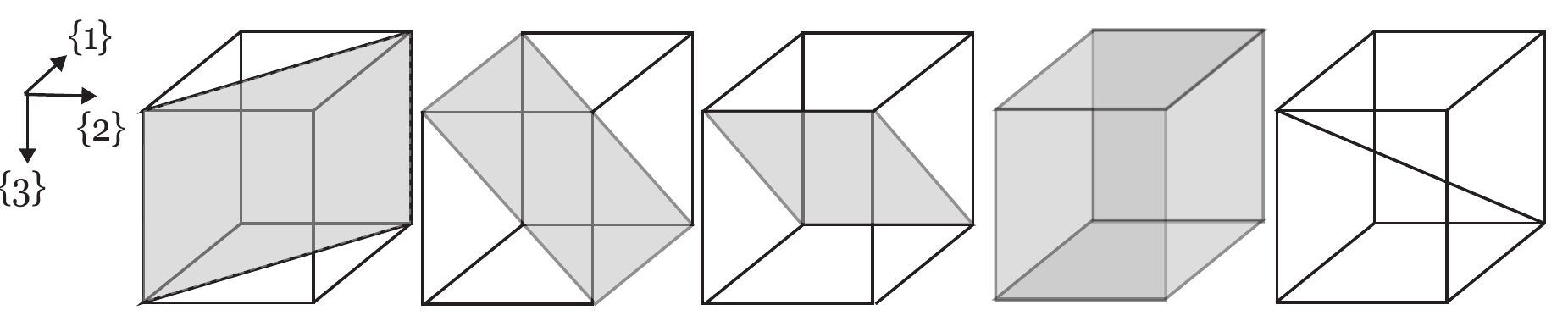}
\caption{Five possible ``slices'' of a 3-tensor, corresponding to $\bell{3}=5$ partitions of $[3]$. From left to right: a) $\{\{1, 2\}, \{3\}\}$ b) $\{\{1\}, \{2, 3\}\}$ c) $\{\{1, 3\}, \{2\}\}$ d) $\{\{1\}, \{2\}, \{3\}\}$ e) $\{\{1, 2, 3\}\}$.}
\label{fig:slices}
\end{figure}

 The following theorem characterizes the effect of each operator in $\mc{B}$ in terms of \textnewnorm{} of input and output, generalizing \Cref{prop:R2-R2} from matrix to high order tensor.

\begin{restatable}[Stability of \LE{} layers for $k$-IGN]{theorem}{kignlinearstability}
\label{thm:linear-layer-stability}
Let $T_{\gamma}: \mb{R}^{[0,1]^{\ell}} \rightarrow \mb{R}^{[0,1]^m}$ be a basis element of the space of $\LE_{\ell, m}$ maps where $\gamma \in \parspace_{\ell + m}$. %
If $\newnorm{X} \leqslant  \epsilon \one_{\bell{\ell}}$, then the \textnewnorm{} of $Y:=T_{\gamma}(X)$ satisfies  $\newnorm{Y}\leqslant  \epsilon \one_{\bell{m}}$ for all $\gamma \in \parspace_{\ell+m}$.
\end{restatable}

The proof relies on a new interpretation of elements of $\mc{B}$ in $k$-IGN. We give only an intuitive sketch using an example in the next subsection. See \Cref{subsec:linear-layer-stability-proof} for the proof.

\subsection{Interpretation of Basis Elements}
\label{subsec:interpretation}
For better understanding, we color the input axis %
 $\{1, ..., \ell\}$ as red and output axis $\{\ell +1, ..., \ell +m \}$ as blue. Each $T_{\gamma}$ corresponds to one partition $\gamma$ of $[\ell + m]$. %

For any partition $\gamma\in \parspace_{l+k}$, we can write this set as disjoint union $\gamma=S_1 \cup S_2 \cup S_3$ where $S_1$ is a set of set(s) of input axis, and $S_3$ is a set of set(s) of output axis. $S_2$ is a set of set(s) where each set contains both input and output axis. With slight abuse of notation, we omit the subscript $\gamma$ for $S_1, S_3, S_3$ when its choice is fixed or clear, and denote $\{\ell +1, ..., \ell +m \}$ as $\ell + [m]$.
As an example, one basis element of the space of $\LE_{3, 3}$ maps is $\gamma =\{ \{1,2\}, \{3,6\}, \{4\}, \{5\}\}$
\begin{equation}
\label{eq:partition-example}
\underbrace{S_1=\{\{\red{1},\red{2}\}\}}_{\text{Only has \red{input} axis}} \cup \underbrace{S_2=\{\{\red{3},\blue{6}\}\}}_{\substack{\text{has both} \\ \text{\red{input} and \blue{output} axis}}} \cup \underbrace{S_3=\{\{\blue{4}\}, \{\blue{5}\}\}}_{\text{only has \blue{output} axis}}
\end{equation}
where $1, 2, 3$ specifies the axis of input tensor and $4,5,6$ specifies the axis of the output tensor.
\begin{figure}[htp]
  \centering
  \vspace{-5pt}
  \includegraphics[width=.6\linewidth]{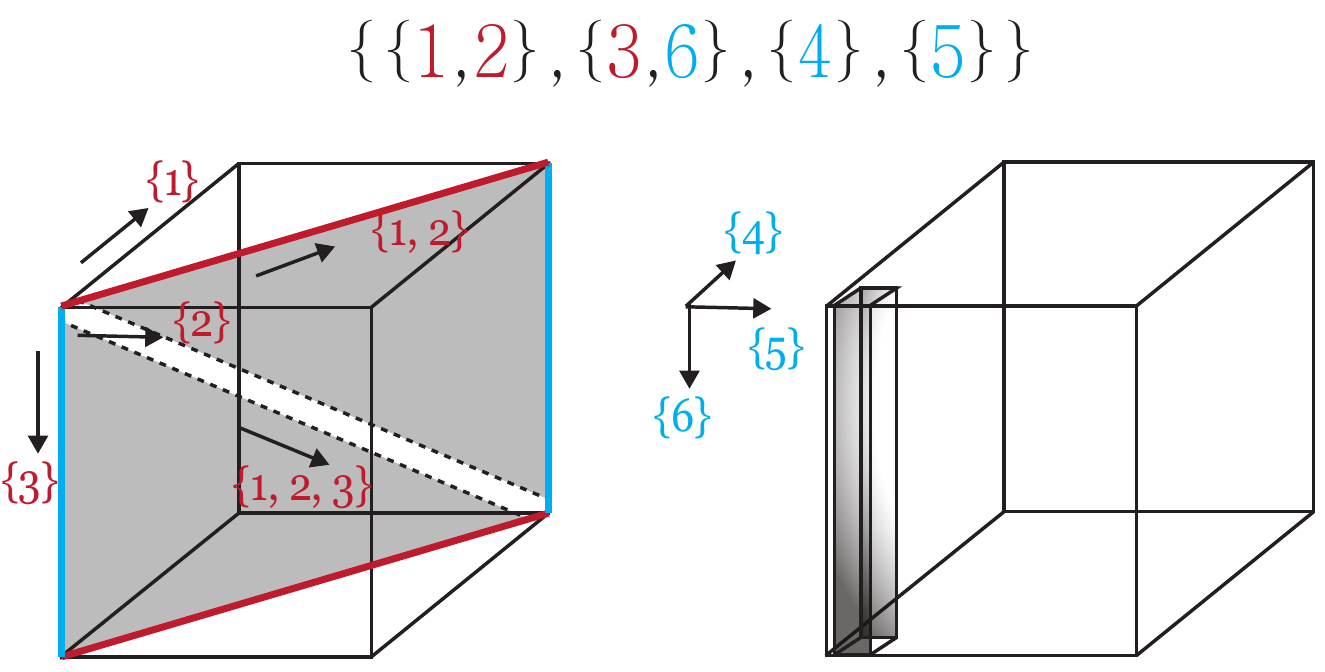}
\caption{An illustration of the one basis element of the space of $\LE_{3, 3}$. %
It selects area spanned by axis $\{1,2\}$ and $\{3\}$, average over the axis $\{1, 2\}$, and then align the resulting 1D tensor with axis $\{6\}$, and finally replicate the slices along axis $\{4\}$ and $\{5\}$ to fill in the whole cube on the right. }
\label{fig:partition}
\end{figure}
Recall that there is a one-to-one correspondence between the partitions over $[\ell+m]$ and the base elements in $\mathcal{B}$ as in Eqn (\ref{eqn:kIGNbasis}). The basis element $T_\gamma$ corresponding to $\gamma = S_1 \cup S_2 \cup S_3$ operates on an input tensor $X\in \mb{R}^{n^\ell}$ and produce an output tensor $Y \in \mb{R}^{n^m}$ as follows:  \vspace{-5pt}
\begin{quotation}
Given input $X$, (step 1) obtain its slice $X_{\gamma}$ on $\Pi_1$ (selection axis),
(step 2) average $X_{\gamma}$ over $\Pi_2$ (reduction axis), resulting in $\xred$.
(step 3) Align $\xred$ on $\Pi_3$ (alignment axis) with $Y_{\gamma}$  and
(step 4) replicate $Y_{\gamma}$ along $\Pi_4$ (replication axis), resulting $\yrep$, a slice of $Y$. Entries of $Y$ outside $\yrep$ will be set to be 0.
\end{quotation}
In general, $\Pi_i$ can be read off from $S_1$-$S_3$. See \Cref{subsec:linear-layer-stability-proof} for details.
As a running example, \Cref{fig:partition} illustrates the basis element corresponding to $\gamma = S_1\cup S_2 \cup S_3$ where $S_1 =\{\{\red{1},\red{2}\}\} \cup S_2=\{\{\red{3},\blue{6}\}\} \cup S_3=\{\{\blue{4}\}, \{\blue{5}\}\}$. In the first step, given 3-tensor $X$, indexed by $\set{\set{1}, \set{2}, \set{3}}$ we select slices of interest $X_{\gamma}$ on $\Pi_1 = \set{\set{1,2}, \set{3}}$, colored in grey in the left cube of \Cref{fig:partition}. In the second step, we average $X_{\gamma}$ over axis $\Pi_2 = \set{\set{1,2}}$ to reduce 2-tensor $X_{\gamma}$, indexed by $\set{\set{1,2}, \set{3}}$ to a 1-tensor $\xred$, indexed by $\set{\set{3}}$. In the third step, the $\xred$ is aligned with $\Pi_3 = \set{\set{6}}$,  resulting in the grey cuboid $Y_{\gamma}$ indexed by $\set{\set{6}}$, shown in the right cube in \Cref{fig:partition}. Here the only difference between $\xred$ and $Y_{\gamma}$ is the index name of two tensors. In the fourth step, we replicate the grey cuboid $Y_{\gamma}$ over axis $\Pi_4 = \set{\set{4}, \set{5}}$ to fill in the cube, resulting in $\yrep$, indexed by $\set{\set{3}, \set{4}, \set{5}}$. %
Note in general $\yrep$ is a slice of $Y$ and does have to be the same as $Y$. %

These steps are defined formally in the \Cref{sec:ign_missing_proofs}.
For each of the four steps, we can control the \textnewnorm{} of output for each step (shown in \Cref{lemma:property-of-partition-norm}), and therefore control the \textnewnorm{} of the final output for every basis element. See \Cref{subsec:linear-layer-stability-proof} for full proofs.

\section{Convergence of IGN in the Edge Weight Continuous Model}
\label{sec:convergence-ruiz}
\cite{ruiz2020graphon} consider  the convergence of $\|\Phi_c(W)-\Phi_c(W_n)\|_{L_2}$ in the \textit{graphon space}, where $W$ is the original graphon and $W_n$ is a piecewise constant graphon induced from graphs of size $n$ sampled from $W$ (to be defined soon). We call this model as the \textit{\edgeweight{}}. The main result of \cite{ruiz2021graph} is the convergence of \textit{continuous} \sGNN{} in the \textit{deterministic} sampling case where graphs are sampled from $W$ deterministically.
Leveraging our earlier stability result of linear layers of continuous IGNs in \Cref{thm:linear-layer-stability}, we can prove an analogous convergence result of cIGNs in the \edgeweight{} for both the deterministic and random sampling cases.

\textbf{Setup of the \edgeweight{}.} %
Given a graphon $W \in \mc{W}$ and a signal
$X \in \mb{R}^{[0, 1]\times d}$, the input of cIGN will be
$[W, \text{Diag}(X)] \in \mb{R}^{[0,1]^2 \times (1+d)}$. In the random sampling setting, we sample a graph of size $n$ from $W$ by setting the following edge weight matrix and discrete signal:
\begin{align}\label{eqn:det_gcn-random} \begin{split}
&[\widetilde{A_n}]_{ij} := W(u_i,u_j) \quad \mbox{and} \quad
[\widetilde{x_n}]_i := X(u_i)
\end{split}\end{align}
where $u_i$ is the $i$-th smallest point from $n$ i.i.d points sampled from uniform distribution on $[0,1]$.
We further lift the discrete graph $(\widetilde{A_n}, \widetilde{x_n})$ to a piecewise-constant graphon $\inducedW$ with signal $\inducedX$. Specifically, partition $[0,1]$ to be $I_1 \cup \ldots \cup I_n$ with $I_i = (u_i, u_{i+1}]$. We then define
\begin{align}\begin{split} \label{eqn:induced-cgcn-random}
&\inducedW(u,v) := {[\widetilde{A_n}]_{ij}} \times \mathrm{I}(u \in I_i)\mathrm{I}(v \in I_j) \quad \mbox{and} \\{}
&\inducedX(u) := [\widetilde{x_n}]_i \times \mathrm{I}(u \in I_i)
\end{split}\end{align}
 where $\mathrm{I}$ is the indicator function.
 Replacing the random sampling with fixed grid, i.e., let $u_i = \frac{i-1}{n}$,  we can get the deterministic \edgeweight{}, where $W_n$ and $X_n$ can be defined similarly as the lifting of a discrete sampled graph to a piecewise constant graphon. Note that $\inducedW$ is a piecewise constant graphon where each block is not of the same size, while all blocks $W_n$ are of size $\frac{1}{n} \times \frac{1}{n}$.
 We use $\widetilde{\cdot}$ to emphasize that $\widetilde{W_n}$/$\widetilde{X_n}$  are random variables, in contrast to the deterministic $W_n$/$X_n$.

We also need a few assumptions on the input and IGN.
\begin{assumption} \label{as:graphon-lip}
    The graphon $W$ is $A_1$-Lipschitz, i.e.
     $|W(u_2,v_2)-W(u_1,v_1)| \leqslant  A_1(|u_2-u_1|+|v_2-v_1|)$.
    \end{assumption}

    \begin{assumption} \label{as:filter-bound}
      The filter coefficients $c_{\gamma}$ are upper bounded by $\filterbound{}$. %
    \end{assumption}

    \begin{assumption} \label{as:signal-lip}
     The graphon signal $X$ is $A_3$-Lipschitz.
    \end{assumption}

    \begin{assumption} \label{as:activation-lip}
     The activation functions in IGNs are normalized Lipschitz, i.e.
     $|\rho(x)-\rho(y)| \leqslant  |x-y|$, and $\rho(0)=0$.
    \end{assumption}

Such four assumptions are quite natural and also adopted in \cite{ruiz2020graphon}. With AS 1-4, we have the following key proposition.
The proof leverages the stability of linear layers for $k$-IGN from \Cref{thm:linear-layer-stability}; see \Cref{app:proofs-EW} for details.

\begin{restatable}[Stability of $\Phi_c$]{proposition}{phistable}
\label{prop:Phi-stable}
If cIGN $\Phi_c: \mb{R}^{[0,1]^2 \times d_{\tn{in}}}   \rightarrow \mb{R}^{d_{\tn{out}}}$ satisfy AS\ref{as:filter-bound}, AS\ref{as:activation-lip} and $\newnorm{W_1-W_2}\leqslant  \epsilon \one_2$, then
 $\newnorm{\Phi_c(W_1) - \Phi_c(W_2)} = \| \Phi_c(W_1) - \Phi_c(W_2)\|_{L_2} \leqslant  C(A_2)\epsilon$ . The same statement still holds if we change the underlying norm of \textNewnorm{} from $L_2$ to $L_{\infty}$.
 \end{restatable}
 \begin{remark}
 Statements in \Cref{prop:Phi-stable} holds for discrete IGN $\Phi_d$ as well.
 \end{remark}

 From AS\ref{as:signal-lip} we can also bound the difference between the original signal $X$ and the induced signal ($X_n$ and $\inducedX$).

\begin{restatable}[]{lemma}{xdiffrandom}
\label{lem:x-diff}
Let $X \in \mb{R}^{[0,1]\times d}$ be an $A_3$-\lip{} graphon signal satisfying AS\ref{as:signal-lip}, and let $\inducedX$ and $X_n$ be the induced graphon signal as in \cref{eqn:det_gcn-random,eqn:induced-cgcn-random}. Then we have i) $\newnorm{X-X_n}$ converges to 0 and ii) $\newnorm{X-\inducedX}$ converges to 0 in probability.
\end{restatable}
We have the similar statements for $W$ as well. %
\begin{restatable}[]{lemma}{wdiffrandom}
\label{lem:w-diff}
If $W$ satisfies AS\ref{as:graphon-lip}, $\|W-W_n\|_{\tn{pn}} $ converges to 0. $\|W-\inducedW\|_{\tn{pn}} $ converges to 0 in probability.
\end{restatable}
The following main theorem (for $k$-cIGN of any order $k$) of this section can be shown by combining \Cref{prop:Phi-stable} with \Cref{lem:x-diff,lem:w-diff}; see  \Cref{app:proofs-EW} for details.

\begin{restatable}[Convergence of cIGN in the edge weight continuous model]{theorem}{EWconvergence}
\label{thm:EW-convergence}
Under the fixed sampling condition, IGN converges to cIGN, i.e., $\| \Phi_c\left([W, \dg{X}]\right) - \Phi_c([W_n, \dg{X_n}])\|_{L_2}$ converges to 0.

An analogous statement hold for the random sampling setting, where \\ $\| \Phi_c([W, \dg{X}]) - \Phi_c([\inducedW, \dg{\inducedX}])\|_{L_2}$ converges to 0 in probability.

\end{restatable}

\section{Convergence of IGN in the Edge Probability Discrete Model}
\label{sec:EP-convergence}

In this section, we will consider the convergence setup of  \cite{keriven2020convergence}, which we call the \emph{\edgeprob}.
The major difference from the \edgeweight{} of \cite{ruiz2020graphon} is that (1) we only access 0-1 adjacency matrix instead of full edge weights and (2) the convergence error is measured in the graph space (instead of graphon space).

This model is more natural. However, we will first show a  negative result that in general IGN does not converge in the \edgeprob{} in \Cref{subsec:negative-result}. This motivates us to consider a relaxed setting where we estimate the edge probability from data. With this extra assumption, we can prove the convergence of \smallIGN{}, a subset of IGN, in the edge probability discrete model in  \Cref{subsec:smallign-convergnce}. Although this is not entirely satisfactory, we show that nevertheless, the family of functions that can be represented by \smallIGN{} is still rich enough to for example approximate any \sGNN{} arbitrarily well.

\subsection{Setup: Edge Probability Continuous Model}
\label{subsec: EP}
We first state the setup and results of \cite{keriven2020convergence}. We keep the notation close to the original paper for consistency.
A random graph model $(P, W, f)$ is represented as a probability distribution $P$ uniform over latent space $\mc{U}=[0,1]$, a symmetric kernel $W: \mc{U} \times \mc{U} \rightarrow [0, 1]$ and a bounded function (graph signal) $f: \mc{U} \rightarrow \mathbb{R}^{d_{z}}$. A random graph $G_n$ with $n$ nodes is then generated from $(P,W,f)$ according to latent variables $U := \set{u_1, ..., u_n}$ as follows:
\begin{gather}
\forall j<i \leqslant n: \quad \mbox{graph node}~~ u_{i} \stackrel{i i d}{\sim} P, \quad z_{i}=f\left(u_{i}\right), \\
\quad  \text{graph edge } a_{i j} \sim \operatorname{Ber}\left(\alpha_{n} W(u_{i}, u_{j})\right) \numberthis
\label{equ:graph-generation-EP}
\end{gather}
where Ber is the Bernoulli distribution and $\alpha_n$ controls the sparsity of sampled graph.  %
Note that in our case, we assume that the sparsification factor $\alpha_n = 1$ (which is the classical graphon model).
We define a degree function by $d_{W, P}(\cdot):= \int W(\cdot, u) d P(u)$.
We assume the following
\begin{gather}
\|W(\cdot, u)\|_{L_\infty} \leqslant c_{\max }, \quad d_{W, P}(u) \geqslant c_{\min }, \\
\quad W(\cdot, u) \text { is }\left(c_{\text {Lip. }}, n_{\mathcal{U}}\right) \text {-piecewise Lipschitz. } %
\end{gather}
A function $f: \mathcal{U} \rightarrow \mathbb{R}$ is said to be $\left(c_{\text {Lip. }}, n_{\mathcal{U}}\right) \text {-piecewise Lipschitz}$ if there is 
a partition $\mathcal{U}_1, ..., \mathcal{U}_n$ of $\mathcal{U}$ such that, for all $u, u'$ in the same $\mathcal{U}_i$, 
we have $|f(u)-f(u')|< c_{Lip.} d(u, u')$.
We introduce two normalized sampling operator $S_U$ and $S_n$ that sample a continuous function to a discrete one over $n$ points. %
For a function %
 $W': \mc{U}^{\otimes k} \rightarrow \mb{R}^{d_{\tn{out}}}$, $S_UW'(i_1, ..., i_k) := (\frac{1}{\sqrt{n}})^k(W'(u_{(i_1)}), ..., W'(u_{(i_k)})$ where $u_{(i)}$ is the i-th smallest number over $n$ uniform random samples over $[0,1]$ and $i_{1}, ..., i_{k}\in[n]$.
Similarly, $S_nW'(i_1, ..., i_k) := (\frac{1}{\sqrt{n}})^k \left(W'(\frac{i_1}{n}), ..., W'(\frac{i_k}{n})\right)$ %
Note that the normalizing constant will depend on the dimension of the support of $W'$.
We have $\|S_UW'\|_2 \leqslant  \| W'\|_{L_{\infty}}$ and $\|S_nW'\|_2 \leqslant  \| W'\|_{L_{\infty}}$. %

To measure the convergence error, we consider root mean square error at the node level: for a signal $x \in \mb{R}^{n^2 \times d_{\tn{out}}}$ and latent variables $U$, we define
$\MSE_{U}(f, x)  := \| S_U f - \frac{x}{n}\|_2 = (n^{-2} \sum_{i=1}^{n}\sum_{j=1}^{n}\left\|f\left(u_i, u_j\right)-x(i, j)\right\|^{2})^{1 / 2}.$
Again, there is a dependency on the input dimension -- the normalization term $n^{-2}$ will need to be adjusted when the input order is different from 2.
\subsection{Negative Result}%
\label{subsec:negative-result}

\begin{restatable}[]{theorem}{convfailure}
\label{thm:conv-failure}
Given any graphon $W$ with $c_{\text{max}} < 1$ and an IGN architecture (fix hyper-parameters like number of layers), there exists a set of parameters $\theta$ such that convergence of $IGN_{\theta}$ to c$IGN_{\theta}$ is not possible, i.e., $\MSE_U(\Phi_c\left([W, \dg{X}]\right), \Phi_d([A_n, \dg{\widetilde{x_n}}]))$ does not converge to 0 as $n\to \infty$, where $A_n$ is 0-1 matrix generated according to \cref{equ:graph-generation-EP}, i.e., $A_n[i][j] = a_{i,j}$. %
\end{restatable}
The proof of \Cref{thm:conv-failure} hinges on the fact that the input to IGN in discrete case is 0-1 matrix while the input to cIGN in the continuous case has edge weight upper bounded by $c_{\tn{max}} < 1$. The margin between 1 and $c_{\tn{max}}$ makes it easy to construct counterexamples.  See \Cref{app:missing-proofs-neg} for details.

Theorem \ref{thm:conv-failure} states that we cannot expect every IGN will converge to its continuous version cIGN.
As the proof of this theorem crucially uses the fact that we can only access 0-1 adjacency matrix, a natural question is what if we can estimate the edge probability from the data?
Interestingly, we can obtain the convergence of for a subset of IGNs (which is still rich enough), called \smallIGN{}, in this case. %

\subsection{Convergence of \smallIGN{}}
\label{subsec:smallign-convergnce}

Let $\widehat{W}_{n \times n}$ be the estimated $n\times n$ edge probability matrix from $A_n$. $\widetilde{W_n}$ is the induced graphon defined in \cref{eqn:induced-cgcn-random}.
 To analyze the convergence error for general IGN after edge probability estimation, we first decompose the convergence error of the interest using triangle inequality. Assuming the output is 1-tensor, then

\begin{align*}
\label{}
& \MSE_U(\Phi_c(W), \Phi_d(\widehat{W}_{n \times n})) \\
& = \|S_U \Phi_c(W)-\frac{1}{\sqrt{n}}\Phi_d (\widehat{W}_{n \times n}) \| \\
&\leqslant  \underbrace{\|S_U \Phi_c(W)- S_U\Phi_c(\inducedW)\|}_\text{First term: discretization error} +
\underbrace{\|S_U\Phi_c(\inducedW) - \Phi_dS_U(\inducedW)\|}_\text{Second term: sampling error} \\
& + \underbrace{\|\Phi_dS_U(\inducedW) - \frac{1}{\sqrt{n}}\Phi_d (\widehat{W}_{n \times n})\|}_\text{Third term: estimation error} \numberthis
\end{align*}
\begin{figure*}[htp!]
  \centering
  \includegraphics[width=.33\linewidth]{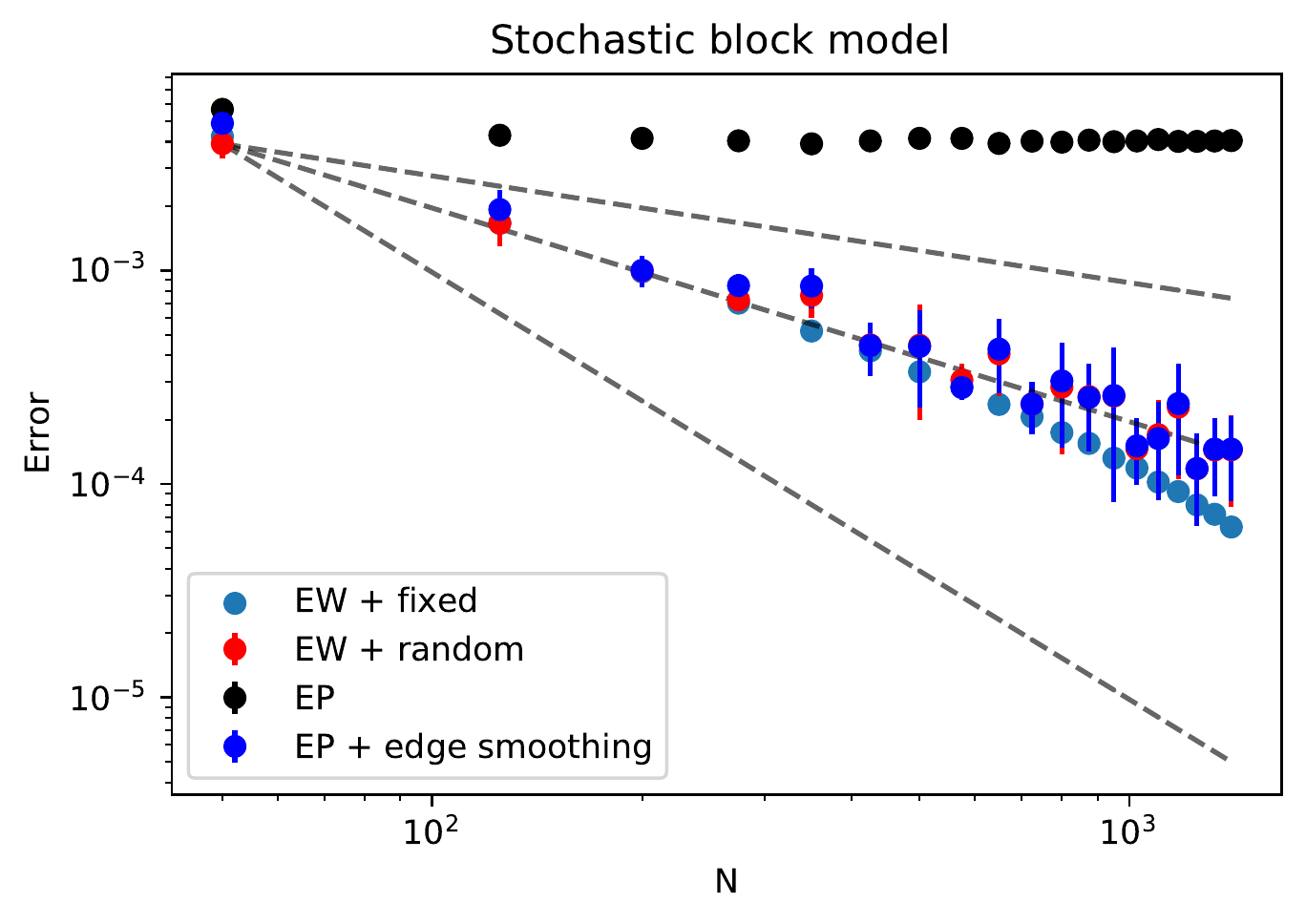}
  \includegraphics[width=.33\linewidth]{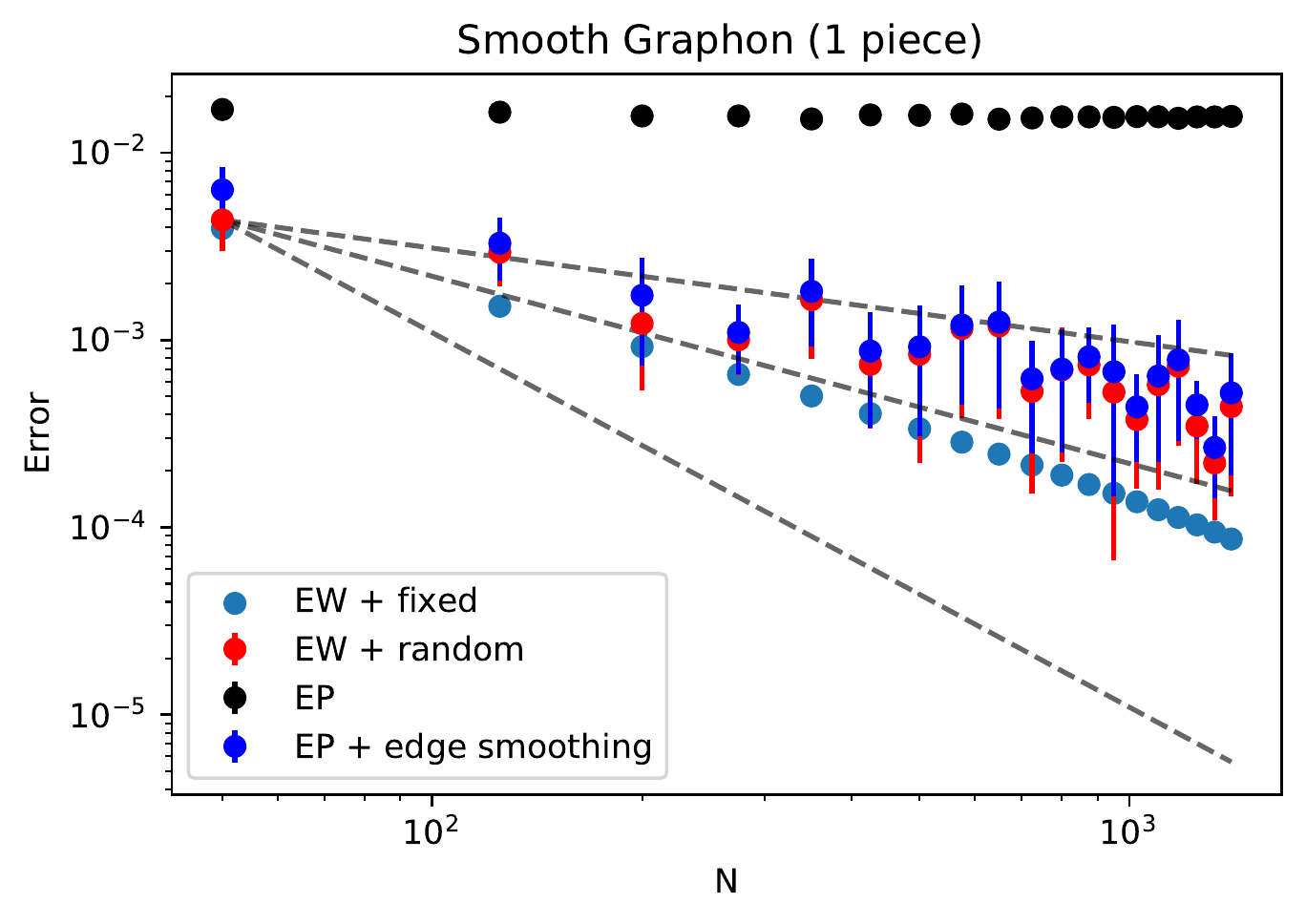}
  \includegraphics[width=.33\linewidth]{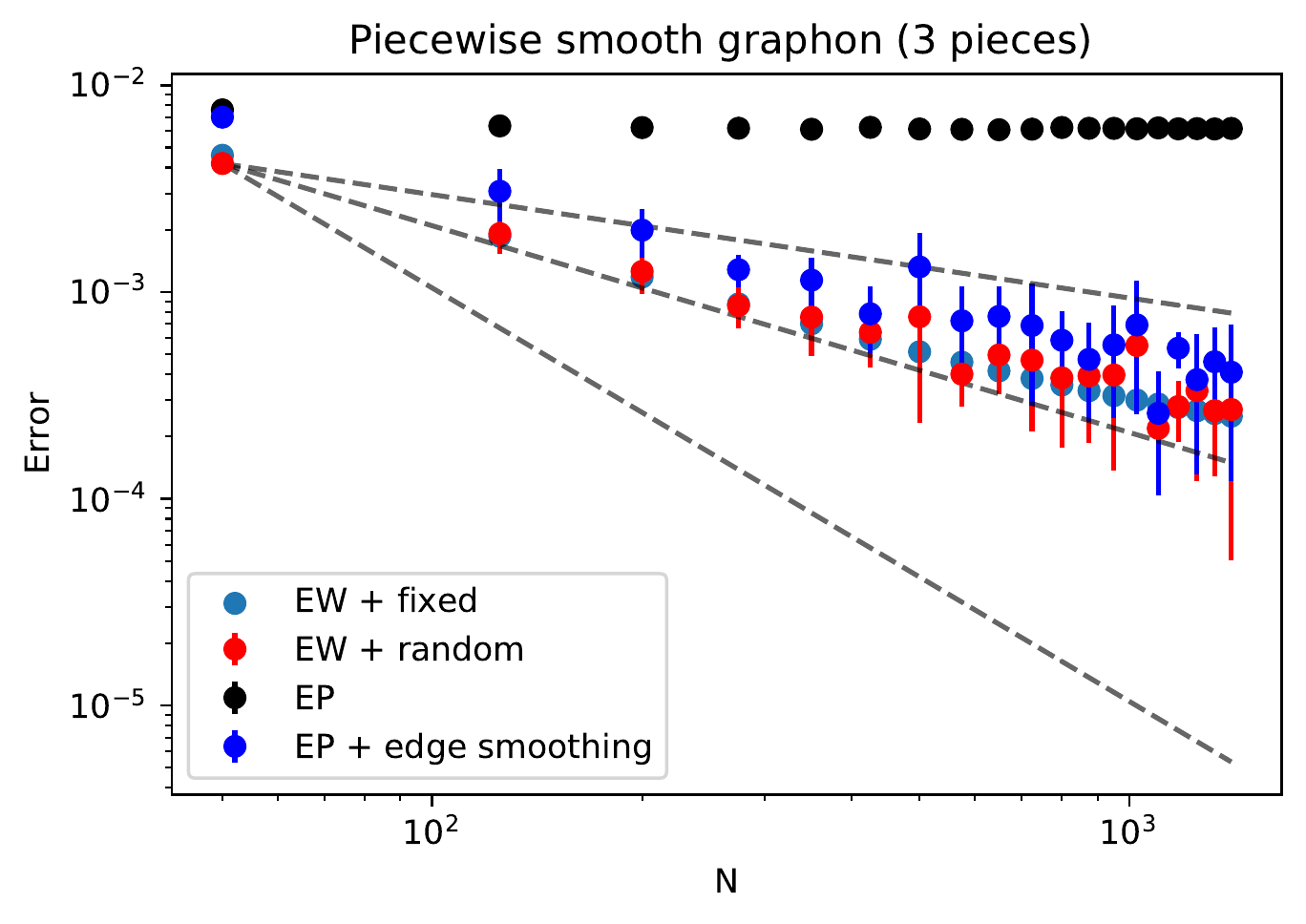}
  \label{fig:sfig1}
\vspace{-15pt}
\caption{The convergence error for three generative models: (left) stochastic block model, (middle) smooth graphon, (right) piece-wise smooth graphon. EW and EP stands for edge weight continuous model (\cref{eqn:det_gcn-random}) and edge probability discrete model (\cref{equ:graph-generation-EP}). %
}
\vspace{-5pt}
\label{fig:convergence}
\end{figure*}
The three terms measure the different sources of error. First-term is concerned with the discretization error, which can be controlled via a property of $S_U$ and \Cref{prop:Phi-stable}. The  Second term concerns the sampling error from the randomness of $U$. This term will vanish if we consider only $S_n$ instead of $S_U$ under the extra condition stated below. The third term concerns the edge probability estimation error, which can also be controlled by leveraging existing literature on the statistical guarantee of the \textit{edge probability estimation} algorithm from \cite{zhang2015estimating}. \footnote{For better readability, here we only use the $W$ as input instead of $[W, \dg{X}]$. Adding $\dg{X}$ into the input is easy and is included in the full proof in \Cref{app:smallign-convergnce}. }

Controlling the second term is more involved. This is also the place where we have to add an extra assumption to constrain the IGN space in order to achieve convergence after edge smoothing.

\begin{definition}[\smallIGN{}]
Let $\inducedEW$ be a graphon with ``chessboard pattern'' \footnote{See full definition in \Cref{def:chessboard}.}, i.e., it is a piecewise constant graphon where each block is of the same size. Similarly, define $\inducedEX$ as the 1D analog.
\smallIGN{} denotes a subset of IGN that satisfies $S_n\Phi_c([\inducedEW, \dg{\inducedEX}]) = \Phi_d S_n([\inducedEW, \dg{\inducedEX}])$.
\end{definition}

\begin{restatable}[convergence of \smallIGN{} in the edge probability discrete model]{theorem}{convergenceafterEM}
\label{thm:convergenceafterEM}
Assume AS 1-4, and let $\widehat{W}_{n \times n}$ be the estimated edge probability that satisfies $\frac{1}{n}\|W_{n \times n} - \widehat{W}_{n \times n}\|_2$ converges to 0 in probability. Let $\Phi_c, \Phi_d$ be continuous and discrete \smallIGN{}. Then  \\ $\MSE_U\left(\Phi_c\left([W, \dg{X}]\right), \Phi_d\left([\widehat{W}_{n \times n}, \dg{\widetilde{x_n}}]\right)\right)$ converges to 0 in probability.
\end{restatable}

We leave the detailed proofs in \Cref{app:smallign-convergnce} with some discussion on the challenges for achieving full convergence results in the \Cref{remark:difficulty}.
We note that Theorem \ref{thm:convergenceafterEM} has a practical implication: It suggests that in practice, for a given unweighted graph (potentially sampled from some graphon), it may be beneficial to first perform edge probability estimation before feeding into the general IGN framework, to improve the architecture's stability and convergence.

Finally, although the convergence of \smallIGN{} is not entirely satisfactory, it contains some interesting class of functions that can approximate any \sGNN{} arbitrarily well. See \Cref{app:approx} for proof details.  %

\begin{restatable}{theorem}{smalligngcn}
\smallIGN{} can approximates \sGNN{} (both discrete and continuous ones) arbitrarily well on the compact domain in the $\|\cdot \|_{L_{\infty}}$ sense.
\end{restatable}

\section{Experiments}
\label{sec:exp}

We experiment 2-IGN on three graphon models of increasing complexity: Erdoes Renyi graph with $p=0.1$, stochastic block model of 2 blocks of equal size and probability matrix $[[0.1, 0.25], [0.25, 0.4]]$, a \lip{} graphon model with $W(u, v) = \frac{u + v + 1}{4}$, and a piecewise \lip{} graphon with $W(u, v) = \frac{u\%\frac{1}{3} + v\%\frac{1}{3} + 1}{4}$ where $\%$ is modulo operation. Similar to \cite{keriven2020convergence}, we consider untrained IGN with random weights to assess how convergence depends on the choice of architecture rather than learning.
 We use a 5-layer IGN with hidden dimension 16. We take graphs of different sizes as input and plot the error in terms of the norm of the output difference. The results are plotted in \Cref{fig:convergence}. %

As suggested by the \Cref{thm:EW-convergence}, for both deterministic and random sampling, the error decreases as we increase the size of the sampled graph. %
Interestingly, if we take the 0-1 adjacency matrix as the input, the error does not decrease, which aligns with the negative result in \Cref{thm:conv-failure}. We further implement the edge smoothing algorithm \cite{eldridge2016graphons} and find that after the edge probability estimation, the error again decreases, as implied by \Cref{thm:convergenceafterEM}. We remark that although \Cref{thm:convergenceafterEM} works only for \smallIGN{}, our experiments for the general $2$-IGN with randomized initialized weights still show encouraging convergence results. Understanding the convergence of general IGN after edge smoothing is an important direction that we will leave for further investigation.

\section{Related Work}
One type of convergence in deep learning concerns the limiting behavior of neural networks when the width goes to infinity \cite{jacot2018neural,du2018gradient,arora2019fine,lee2019wide,du2019gradient}.  In that regime, the gradient flow on a normally initialized, fully connected neural network with a linear output layer in the infinite-width limit turns out to be equivalent to kernel regression with respect to the Neural Tangent Kernel \cite{jacot2018neural}.

Another type of convergence
concerns the limiting behavior of neural networks when the depth goes to infinity. In the continuous limit, models such as residual networks, recurrent neural network decoders, and normalizing flows can be seen as an Euler discretization of an ordinary differential equation \cite{weinan2017proposal, chen2018neural,lu2018beyond,ruthotto2020deep}. %

The type of convergence we consider in this chapter concerns when the input objects converge to a limit, does the output of some neural network over such sequence of objects also converge to a limit? In the context of GNNs, such convergence and related notion of stability and transferability have been studied in both graphon  \cite{ruiz2020graphon,keriven2020convergence,gama2020stability,ruiz2021graph} and manifold setting \cite{kostrikov2018surface, levie2021transferability}. In the manifold setting, the analysis is closely related to the literature on convergence of Laplacian operator \cite{xu2004discrete,wardetzky2008convergence, belkin2008discrete,belkin2009constructing,dey2010convergence}. %

Lastly, after ICML 2022 conference (where the papaer this chapter is based on is published) it is brought to our attention that the characterization of linear permutation equivariant layers in $k$-IGN bears similarity in \cite{albooyeh2019incidence}. The pooling and broadcasting operations in \cite{albooyeh2019incidence} are the same as what we call the "averaging" and "replication" operations in our paper. This is discussed in details in \Cref{remark:difference}.   

\section{Concluding Remarks}
in this chapter, we investigate the convergence property of a powerful GNN, Invariant Graph Network. We first prove a general stability result of linear layers in IGNs. We then prove a convergence result under the model of \cite{ruiz2020graphon} for both \IGN{} and high order \kIGN{}. %
Under the model of \cite{keriven2020convergence}
we first show a negative result that in general the convergence of every IGN is not possible.
Nevertheless, we pinpoint the major roadblock and prove that if we preprocess input graphs by edge smoothing \cite{zhang2015estimating}, the convergence of a subfamily of IGNs, called \smallIGN{}, can be obtained. As an attempt to quantify the size of \smallIGN{}, we also show that \smallIGN{} contains a rich class of functions that can approximate any \sGNN.

\section{Missing proofs}
\label{sec:ign_missing_proofs}
\subsection{Missing Proofs from \Cref{sec:stability}}
\subsubsection{Extension of \textNewnorm{}}
\label{subsec:extending-new-norm}
There are three ways of extending  \textNewnorm{} 1) extend the definition of \textnewnorm{} to multiple channels 2) changing the underlying norm from $L_2$ norm to $L_{\infty}$ norm, and 3) extend \textNewnorm{} defined for 2-tensor to $k$-tensor.

First recall the definition \textnewnorm{}.
\partitionnormdef*

To extend \textnewnorm{} to signal $A \in \mb{R}^{n^2 \times d}$ of multiple channels, we denote $A = [A_{\cdot, 1} \in \mb{R}^{n^2 \times 1}, ..., A_{\cdot, d} \in \mb{R}^{n^2 \times 1}]$ where $[\cdot, \cdot]$ is the concatenation along channels. $\newnorm{A}:=\sum_{i=1}^d \newnorm{A_{\cdot, i}}$.
both for multi-channel signal both for graphs and graphons.

Another way of generalizing \textNewnorm{} is to change the $L_2$ to $L_{\infty}$ norm. We denote the resulting norm as $\newnorminf{\cdot}$. For $W \in \mc{W}, \newnorminf{W} := (\max_{u\in[0, 1]} W(u, u), \max_{u\in[0, 1], v\in[0, 1]} W(u, v))$. The discrete case and high order tensor case can be defined similarly as the $L_2$ case.

The last way of extending \textNewnorm{} to $k$-tensor $X \in \dtensor{k}{1}$ is to define the norm for each slice of $X$, i.e., $\newnorm{X}:=((\frac{1}{\sqrt{n}})^{|\gamma_1|}\|X_{\gamma_1} \|_2, ..., \frac{1}{\sqrt{n}})^{|\gamma_{\bell{k}}|}\|X_{\gamma_{\bell{k}}} \|_2)$ where $\gamma_{\cdot} \in \parspace_k$. Note how we order $(\gamma_1, ..., \gamma_{\bell{k}})$ can be arbitrary as long as the order is used consistent.

\subsubsection{Proof of stability of linear layer for 2-IGN}
\label{app:2-IGN-linear-stability}

\RtwoRtwo*
\begin{proof}

The statements hold in both discrete and continuous cases. Without loss of generality, we only prove the continuous case by going over all linear equivariant maps $\mb{R}^{[0, 1]^2} \rightarrow \mb{R}^{[0, 1]^2}$ in \Cref{tab:R2-R2}.
\begin{itemize}
\item 1-3: It is easy to see that the \textnewnorm{} does not increase for all three cases.

\item 4-6:
It is enough to prove case 4 only. Since $T(W)(*, u) = \int W(u, v)dv$, \dnorm{} $\|\dg{T(W)}\|_{L_2}^2 = \int (\int W(u, v)dv)^2 du \leqslant  \iint W^2(u, v)dudv $.  For matrix norm: $\|T(W)\|^2_{L_2} = \|\dg{T(W)}\|_{L_2} \leqslant   \iint W^2(u, v)dudv$. 
Therefore the statement holds for this linear equivariant operation.

\item 7-9: same as case 4-6.

\item 10-11: It is enough to prove the first case: average of all elements replicated on the whole matrix. The \dnorm{} is the same as the \mnorm{}. Both norms are decreasing so we are done.

\item 12-13: It is enough to prove only case 12. Since \dnorm{} is equal to \mnorm{}, and \dnorm{} is decreasing by Jensen's inequality we are done.

\item 14-15: Since \mnorm{} is the same as \dnorm{}, which stays the same so we are done.
\end{itemize}
As shown in all cases for any $W \in \mc{W}$ with $\|W\|_{\tn{pn}} < (\epsilon, \epsilon)$, $\|T_i(W)\|_{\tn{pn}} < (\epsilon, \epsilon)$. Therefore we finish the proof for $\mb{R}^{[0,1]^2} \rightarrow \mb{R}^{[0,1]^2}$.
We next go over all linear equivariant maps $\mb{R}^{[0, 1]} \rightarrow \mb{R}^{[0, 1]^2}$ in \Cref{tab:R1-R2} and prove it case by case.

\begin{itemize}
\item 1-3: It is enough to prove the second case. It is easy to see diagonal norm is preserved and $\|T(W)\|_2 = \|W\|_2 \leqslant  \epsilon$. Therefore $\|T(W)\|_{\tn{pn}} \leqslant  (\epsilon, \epsilon)$.

\item 4-5: It is enough to prove the second case. Norm on diagonal is no larger than $\|W\|$ by Jensen's inequality. The matrix norm is the same as the diagonal norm therefore also no large than $\epsilon$. Therefore $\|T(W)\|_{\tn{pn}} \leqslant  (\epsilon, \epsilon)$.
\end{itemize}

Last, we prove the cases for $\mb{R}^{[0, 1]^2} \rightarrow \mb{R}^{[0, 1]}$.

For cases 1-3, it is enough to prove case 2. Since the norm of the output is no large than the matrix norm of input by Jensen's inequality, we are done. Similar reasoning applies to cases 4-5 as well.
\end{proof}

\subsubsection{Proof of \Cref{thm:linear-layer-stability}}
\label{subsec:linear-layer-stability-proof}
We need a few definitions and lemmas first.
\begin{definition}[axis of a tensor]
Given a k-tensor $X \in \dtensor{k}{1}$ indexed by $(\name{1}, ..., \name{k})$. The axis of $X$, denoted as $\axis{X}$, is defined to be $\axis{X}:=(\name{1}, ..., \name{k})$. %
\end{definition}
As an example, the aixs of the first grey sub-tensor in \Cref{fig-app:slices}a, which is a $2$-tensor, is $\set{\set{1,2}, \set{3}}$.

\begin{definition}[replication of a tensor]
\label{def:replication}
Given a k-tensor $X \in \dtensor{k}{1}$ indexed by $(1, ..., k)$, replicating $X$ over new axis $(k+1, ..., k+d)$ means that the resulting new tensor $X'$ of $k+d$ dimension is $X'(i_1, ..., i_k, *, ..., *) := X(i_1, ..., i_k)$. %
\end{definition}

\begin{definition}[partial order of partitions]
\label{def:partial-order}
Given two partitions of $[k]$, denoted as $\gamma = \{\gamma_1, ..., \gamma_{d_1}\}$ and $\beta = \{\beta_1, ..., \beta_{d_2}\}$, we say $\gamma$ is finer than $\beta$, denoted as $\gamma < \beta$, if and only if 1) $\gamma \neq \beta$ and 2) for any $\beta_j \in \beta$, there exists $\gamma_i \in \gamma$ such that $ \beta_j \subseteq \gamma_i$.
\end{definition}
For example, $\{\{1,2,3\}\}$ is finer than $\{\{1, 2\}, \{3\}\}$ but $\set{\set{1,2}, \set{3}}$ is not comparable with $\set{\set{1,3}, \set{2}}$. Note that space of partitions forms a Hasse diagram under the partial order defined above (each set of elements has a least upper bound and a greatest lower bound, so that it forms a lattice). See \Cref{fig:hasse-diagram} for an example.

\begin{figure}%
\centering
\begin{tikzpicture}[scale=.8]
  \node (one) at (0,2) {$\set{\set{1,2,3}}$};
  \node (a) at (-4,0) {$\set{\set{1,2}, \set{3}}$};
  \node (b) at (0,0) {$\set{\set{1,3}, \set{2}}$};
  \node (c) at (4,0) {$\set{\set{2, 3}, \set{1}}$};
  \node (zero) at (0,-2) {$\set{\set{1}, \set{2}, \set{3}}$};
  \draw (zero) -- (a) -- (one) -- (b) -- (zero) -- (c) -- (one) ;
\end{tikzpicture}
\caption{Space of partitions forms a Hasse diagram under the partial order defined in \Cref{def:partial-order}.} Top to bottom corresponds to coarse partition to finer partition. %
\label{fig:hasse-diagram}
\end{figure}
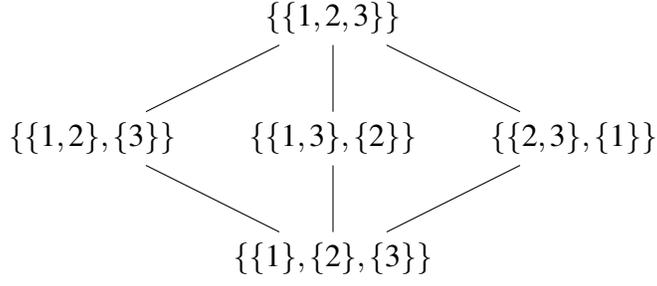

\begin{definition}[average a $k$-tensor $X$ over $\Pi$]\label{def:averaging}
Let $X \in \dtensor{k}{1}$ be a $k$-tensor indexed by $\set{\set{1}, ..., \set{k}}$. Without loss of generality, let $\Pi = \set{\set{1}, ..., \set{d}}$. Denote the resulting $(k-d)$-tensor $X'$, indexed by $\set{\set{d+1}, ..., \set{k}}$.
By averaging $X$ over $\Pi$, we mean %
\begin{equation*}
 X'(\cdot):=\frac{1}{n^d}\sum_{t\in \mc{I}_d} X(t, \cdot).
 \end{equation*}
The definition can be extended to $\mb{R}^{[0, 1]^k}$ by replacing average with integral.
 \end{definition}

\begin{lemma}[properties of \textnewnorm{}]
\label{lemma:property-of-partition-norm}
We list some properties of the \textnewnorm{}. Although all lemmas are stated in the discrete case, the continuous version also holds. The statements also holds for $\newnorminf{\cdot}$ as well.
\begin{enumerate}[label=(\alph*)]
\item Let $X \in \dtensor{k}{1}$ be a $k$-tensor and denote one of its slices $X' \in \dtensor{k'}{1}$ with $k'\leqslant k$.  If $\newnorm{X} \leqslant \epsilon \one_{\bell{k}}$, then $\newnorm{X'} \leqslant \epsilon \one_{\bell{k'}}$.

\item Let $k'<k$. Let $X \in \dtensor{k}{1}$ be a $k$-tensor and $X' \in \dtensor{k'}{1}$ be the resulting $k'$-tensor after averaging over $k-k'$ axis of $X$. If $\newnorm{X} \leqslant \epsilon \one_{\bell{k}}$, then $\newnorm{X'} \leqslant \epsilon \one_{\bell{k'}}$.

\item Let $k'>k$. Let $X \in \dtensor{k}{1}$ be a $k$-tensor and $X'$ be the resulting $k'$-tensor after replicating $X$ over $k'-k$ axis of $X'$. If $\newnorm{X} \leqslant \epsilon \one_{\bell{k}}$, then $\newnorm{X'} \leqslant \epsilon \one_{\bell{k'}}$. 

\item Let $k'<k$ and $X \in \dtensor{k}{1}$ be a $k$-tensor such that it has only one non-zero slice $X_{\gamma}$ of order $k'$, i.e., 
if $\bs{a}\in \mc{I}_k, X(\bs{a})\neq 0 $, it implies $\bs{a} \in \gamma$.
If $\newnorm{X_{\gamma}} \leqslant \epsilon \one_{\bell{k'}}$, then $\newnorm{X} \leqslant \epsilon \one_{\bell{k}}$.
\end{enumerate}

\end{lemma}
\begin{proof}
We prove statements one by one. Note that although the proof is done for $L_2$ norm, we do not make use of any specific property of $L_2$ norm and the same proof can be applied to $L_{\infty}$ as well. Therefore all statements in the lemma apply to $\newnorminf{\cdot}$ as well.
\begin{enumerate}
    \item By the definition of \textnewnorm{} and slice in \Cref{def:slice}, we know that any slice of $X'$ is also a slice of $X$, therefore any component of $\newnorm{X'}$ will be upper bounded by $\epsilon$, which concludes the proof.

    \item Without loss of generality, we can assume that $k' = k-1$ as the general case can be handled by induction. Let the axis of $X$ that is averaged over is axis $\set{1}$. To bound $\newnorm{X'}$, we need to bound the normalized norm of any slice of $X'$. Let $X'_{\gamma'}$ be arbitrary slice of $X'$. Since $X'$ is obtained by averaging over axis 1 of $X$, we know that $X'_{\gamma'}$ is the obtained by averaging over axis of 1 of $X_{\gamma}$, a slice of $X$, where $\gamma := \gamma' \cup \{\set{1}\}$. Since $\newnorm{X} \leqslant \epsilon \one_{\bell{k}}$, we know that $(\frac{1}{\sqrt{n}})^{|\gamma|}\| X_{\gamma}\| \leqslant \epsilon$. By Jensen's inequality, we have $(\frac{1}{\sqrt{n}})^{|\gamma'|}\| X'_{\gamma'}\| \leqslant (\frac{1}{\sqrt{n}})^{|\gamma|}\| X_{\gamma}\|$, and therefore $(\frac{1}{\sqrt{n}})^{|\gamma'|}\| X'_{\gamma'}\| \leqslant \epsilon $. Since $(\frac{1}{\sqrt{n}})^{|\gamma'|}\| X'_{\gamma'}\| \leqslant \epsilon $ holds for arbitrary slice of $X'$, we conclude that $\newnorm{X'} \leqslant \epsilon \one_{\bell{k'}}$.

    The proof above only handles the case of $k'=k-1$. The general case where $k-k'>1$ can be handled by evoking the proof above multiple times for different reduction axis.

    \item Similar to the \Cref{lemma:property-of-partition-norm} (b), we can handle general case by performing induction. Therefore without loss of generality,  we assume $X$ is indexed by $(\set{1}, ..., \set{k} )$ and $X'$ is indexed by $(\set{1}, ..., \set{k+1})$. Just as the last case, without loss of generality we assume that $X'$ is obtained by replicating $X$ over $1$ new axis, denoted as $\set{k+1}$. In other words, $\axis{X'} = \axis{X} \cup \set{\set{k+1}}$.

    To control $\newnorm{X'}$, we need to bound $(\frac{1}{\sqrt{n}})^{|\gamma|}\| X'_{\gamma}\|$ where $\gamma \in \parspace_{k+1}$. Since $X'$ is obtained from $X$ by replicating it over $\set{k+1}$,  $(\frac{1}{\sqrt{n}})^{|\gamma|}\| X'_{\gamma}\| = (\frac{1}{\sqrt{n}})^{|\beta|}\| X_{\beta}\|$ where $\beta = \gamma |_{[k]}$. As $\newnorm{X} \leqslant \epsilon \one_{\bell{k}}$, it implies that $(\frac{1}{\sqrt{n}})^{|\gamma|}\| X'_{\gamma}\| \leqslant \epsilon $ holds for any $\gamma \in \parspace_{k'}$. Therefore we conclude that $\newnorm{X'} \leqslant \epsilon \one_{\bell{k'}}$.

    \item %
    To bound $\newnorm{X}$, we need to bound the normalized norm of any slice of $X$. Let $X_{\beta}$ be arbitrarily slice of $X$ where $\beta \in \parspace_k$.  Since $\gamma$ and $\beta$ are partitions of $[k]$, there exist partitions that are finer than both $\beta$ and $\gamma$, where the notion of finer between two partitions is defined in \Cref{def:partial-order}. Among all partitions that satisfy such conditions, denote the most coarse one as $\alpha \in \parspace_k$. This can be done because the $\parspace_k$ is finite.
    Note that $|\alpha| < |\beta|$ and $|\alpha| < |\gamma|$.

    Since $X_{\alpha}$ is a slice of $X_{\gamma}$ and $\newnorm{X_{\gamma}} \leqslant \epsilon \one_{\bell{k'}}$,
    $(\frac{1}{\sqrt{n}})^{|\alpha|}\|X_{\alpha}\| \leqslant \epsilon$ according to \Cref{lemma:property-of-partition-norm} (a).
    As $X_{\alpha}$ is the slice of $X_{\beta}$ (implies $\| X_{\alpha} \leqslant X_{\beta}\|$ )
    and $\alpha$ is the most coarse partition that is finer than $\beta$ and $\gamma$ (implies $\|X_{\alpha}\|\geqslant \| X_{\beta}\|$ we have $\| X_{\beta}\| = \| X_{\alpha}\|$.
    This implies $(\frac{1}{\sqrt{n}})^{|\beta|}\| X_{\beta}\| \leqslant (\frac{1}{\sqrt{n}})^{|\alpha|}\| X_{\alpha}\| \leqslant \epsilon$.
    
    As $(\frac{1}{\sqrt{n}})^{k'}\| X_{\beta}\| \leqslant \epsilon $ holds for arbitrary slice $\beta$ of $X$, we conclude that $\newnorm{X}\leqslant \epsilon \one_{\bell{k}}$.
\end{enumerate}

\end{proof}

Now we are ready to prove the main theorem.
\kignlinearstability*
\begin{proof}
Without loss of generality, we first consider discrete cases of mapping from $X \in \mb{R}^{n^{\ell}}$ to $Y \in \mb{R}^{n^m}$. In general, each element $T_{\gamma}$ of linear permutation equivariant basis can be identified with the following operation on input/output tensors.
\begin{quotation}
Given input $X$, (step 1) obtain its subtensor $X_{\gamma}$ on a certain $\Pi_1$ (selection axis),
(step 2) average $X_{\gamma}$ over $\Pi_2$ (reduction axis), resulting in $\xred$.
(step 3) Align $\xred$ on $\Pi_3$ (alignment axis) with $Y_{\gamma}$  and
(step 4) replicate $Y_{\gamma}$ along $\Pi_4$ (replication axis), resulting $\yrep$, a slice of $Y$. Entries of $Y$ outside $\yrep$ will be set to be 0. In general, $\Pi_i$ can be read off from $S_1$-$S_3$. %
\end{quotation}
$\Pi_1$-$\Pi_4$ corresponds to different axis of input/output tensor and can be read off from different parts of $S_{\gamma} = S_1 \cup S_2 \cup S_3$ as we introduced in the main text. Note such operation can be naturally extended to the continuous case, as done in \Cref{tab:R1-R2,tab:R2-R2,tab:R2-R1} for $2$-IGN. We next give detailed explanations of each step.

\begin{figure}[htp]
  \centering
  \includegraphics[width=.8\linewidth]{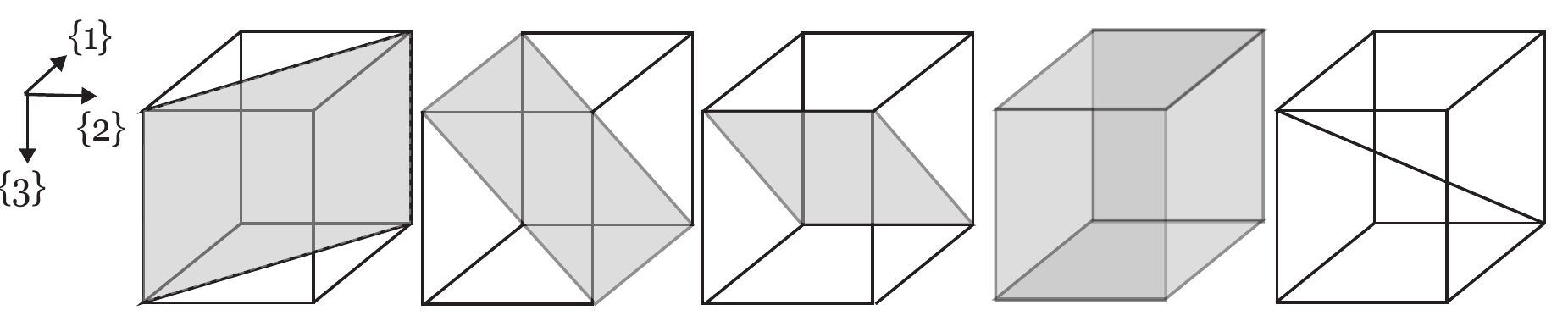}
\caption{Five ``slices'' of a 3-tensor, corresponding to $\bell{3}=5$ partitions of $[3]$. From left to right: a) $\{\{1, 2\}, \{3\}\}$ b) $\{\{1\}, \{2, 3\}\}$ c) $\{\{1, 3\}, \{2\}\}$ d) $\{\{1\}, \{2\}, \{3\}\}$ e) $\{\{1, 2, 3\}\}$.}
\label{fig-app:slices}
\end{figure}

\textbf{First step ($X \rightarrow X_\gamma$):
select $X_{\gamma}$ from $X$ via $\Pi_1$.}

$\Pi_1$ corresponds to $$S|_{[\ell]}\:=\{s\cap [l]\mid s\in S \text{ and } s\cap [l] \neq \emptyset\}. $$ It specifies the what parts (such as diagonal part for 2-tensor) of the input $\ell$-tensor is under consideration. We denote the resulting subtensor as $X_{\gamma}$. See \Cref{def:slice} for formal definition.
As an example in \Cref{eq:partition-example}, $\Pi_1$ corresponds to $\{\{1, 2\},\{3\}\}$, meaning we select a 2-tensor with axises $\{1, 2\}$ and $\{3\}$. Note that the cardinality $|S|_{[\ell]}| = |(S_1 \cup S_2)|_{[\ell]}|\leqslant  l$ encodes the order of $X_{\gamma}$.

\textbf{Second step ($X_\gamma \rightarrow \xred$): average of $X_{\gamma}$ over $\Pi_2$.}
 $\Pi_2$ corresponds axes in  $S_1\subset S|_{[\ell]}$, which tells us along what axis to average over $X_{\gamma}$.
 It will reduce the tensor $X_{\gamma}$ of order $|S_1| + |S_2|$, indexed by $S|_{[\ell]} $, to a tensor of order $|S|_{[\ell]}|-|S_1|=|S_2|$, indexed by $S_2|_{[l]}$. Recall the definition of "averaging" in Definition \ref{def:averaging}.

 In the example of \Cref{fig-app:partition}, this corresponds to averaging over axis $\{\{1, 2\}\}$ %
 , reducing 2-tensor (indexed by axis $\{1,2\}$ and $\{3\}$) to 1-tensor (indexed by axis $\{3\}$). The normalization factor in the discrete case is $n^{|S_1|}$. We denote the tensor after reduction as $\xred$.

As the second step performs tensor order reduction, we end up with a tensor $\xred$ of order $|S_2|$.
The next two steps will describe how to fill in the output tensor $Y$ using $\xred$. To fill in $Y$, we will first align $\xred$ with $Y_{\gamma}$, a subtensor of $Y$, in the third step. We then replicate $Y_{\gamma}$ on $\Pi_4$ in the fourth step, resulting in $\yrep$, a sub-tensor of $Y$.  Finally, we fill all entries of $Y$ outside the subtensor $Y_\gamma$ to be zero.

\textbf{Third step ($\xred \rightarrow Y_{\gamma}$): align $\xred$ with $Y_{\gamma}$}.
To fill in $Y_{\gamma}$, we need to specify how the resulting $|S_2|$-tensor $\xred$ is \textit{aligned} with a certain $|S_2|$-subtensor $Y_\gamma$ of $Y$. After all, there are many ways of selecting a $|S_2|$-tensor from $Y$, which is indexed by $\{ \{l+1\}, ..., \{\ell +m\}\}$.
Specifically, set $Y_{\gamma}$ be the $|S_2|$-tensor indexed by $S_2 |_{\ell+[m]}$. We next define the precise relationship between $\xred$ and $Y_{\gamma}$. $\xred$ is indexed by $S_2|_{[l]}$ while $Y_{\gamma}$ is indexed by $S_2|_{l + [m]}$ and defined to be $Y_{\gamma}(\cdot) = \xred(\cdot)$. %
In the example of \Cref{fig-app:partition}, $\xred$ is a 1D tensor indexed by $\{3\}$ and $Y_{\gamma}$ (the grey cuboid on the right cube of \Cref{fig-app:partition}) is indexed by $\{6\}$.

\textbf{Fourth step ($Y_{\gamma} \rightarrow \yrep$): replicating $Y_{\gamma}$ over $\Pi_4$}. $\Pi_4$ corresponds to axes in $S_3$. It will be used to specify along what axis (axes) we will replicate the $|S_2|$-tensor $Y_{\gamma}$ over. Recall that $Y_{\gamma}$ is indexed by $S_2|_{l + [m]}$. Let $\yrep$ be a subtensor of $Y \in \mb{R}^{n^{l}}$ indexed by $(S_2 \cup S_3)|_{l + [m]}$. Obviously, the tensor $Y_{\gamma}$ output from the Third step is a subtensor of $\yrep$. 
Without loss of generality, let the first $|S_2|$ component are indexed by $S_2|_{l + [m]}$ and the rest components are indexed by $S_3|_{l + [m]}$. The mathematical definition of the fourth step is then $\yrep(\cdot, t) := Y_{\gamma}(\cdot)$ for all $t\in [n]^{|S_3|}$. Note that the order of $\yrep$ can be smaller than order of $Y$.

The example in \Cref{eq:partition-example} has $S_3=\{\{4\}, \{5\}\}$, which means that we will replicate the 1-tensor along axis $\{4\}$ and $\{5\}$. Note that in general, we do not have to fill in the whole $m$-tensor (think about copy row average to diagonal in \Cref{tab:R2-R2}).

\begin{figure}[htp]
  \centering
  \vspace{-5pt}
  \includegraphics[width=.4\linewidth]{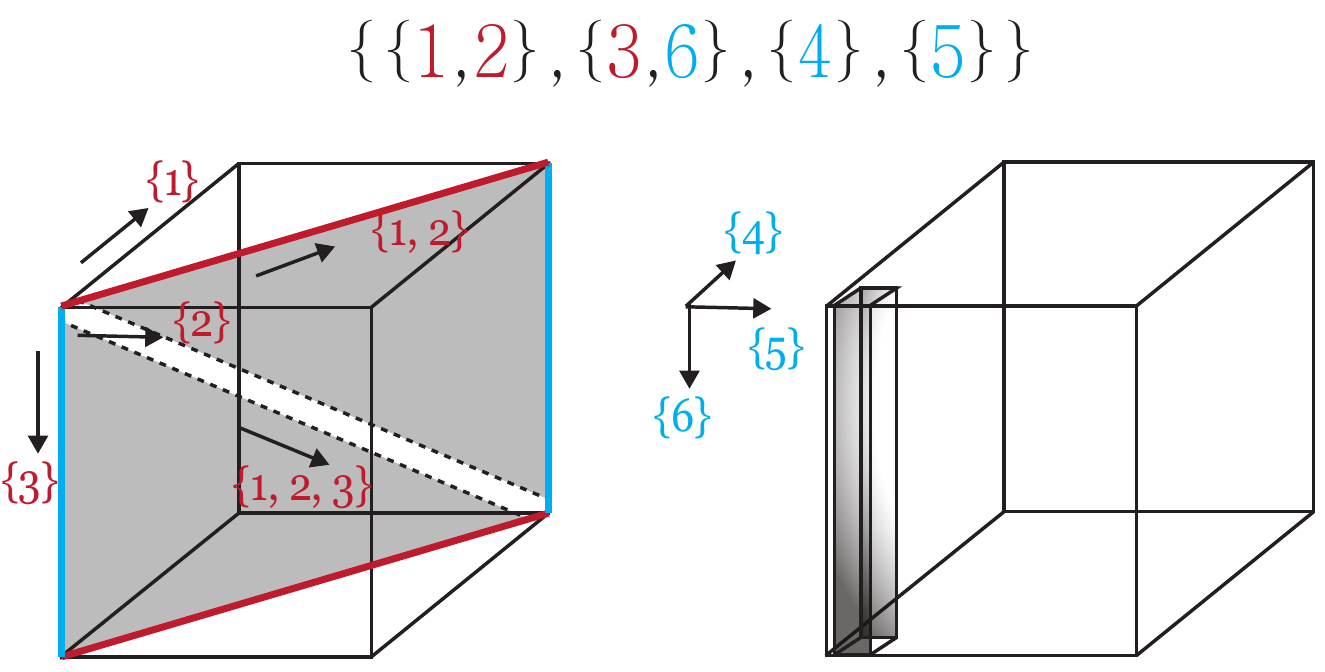}
\caption{An illustration of the one basis element of the space of $\LE_{3, 3}$. %
It selects area spanned by axis $\{1,2\}$ and $\{3\}$, average over the axis $\{1, 2\}$, and then align the resulting 1D tensor with axis $\{6\}$, and finally replicate the slices along axis $\{4\}$ and $\{5\}$ to fill in the whole cube on the right. }
\label{fig-app:partition}
\end{figure}

After the interpretation of general linear equivariant maps in $k$-IGN, We now show that %
if $\newnorm{X} \leqslant  \epsilon \one_{\bell{\ell}}$, then $T_{\gamma}(X)\leqslant  \epsilon \one_{\bell{m}}$ holds for all $\gamma$. This can be done easily with the use of \Cref{lemma:property-of-partition-norm}. %

For any partition of $[\ell +m]$ $\gamma$, according to the first step we are mainly concerned about the $\newnorm{X_{\gamma}}$ instead of $\newnorm{X}$.  Since $X_{\gamma}$ is a slice of $X$, then if $\newnorm{X} \leqslant \epsilon \one_{\bell{\ord{X}}}$, by \Cref{lemma:property-of-partition-norm} (a), then $\newnorm{X_{\gamma}} \leqslant \epsilon \one_{\bell{|S_1|+|S_2|}}$.

According to the second step and \Cref{lemma:property-of-partition-norm} (b), we can also conclude that $\newnorm{\xred} \leqslant \epsilon \one_{\bell{|S_2|}}$.

For the third step of align $\xred$ with $Y_{\gamma}$, it is quite obvious that $\newnorm{Y_{\gamma}} = \newnorm{\xred} \leqslant \epsilon \one_{\bell{|S_2|}}$.

For the fourth step of replicating $Y_{\gamma}$ over $\Pi_4$ to get $\yrep$, by \Cref{lemma:property-of-partition-norm} (c), we have $\newnorm{\yrep} \leqslant \epsilon \one_{\bell{|S_2|+|S_3|}}$.

Lastly, we evoke \Cref{lemma:property-of-partition-norm} (d) to get $\newnorm{Y} \leqslant \epsilon \one_{\bell{m}}$, which concludes our proof.

\end{proof}

\begin{remark}[On the difference from Incidence Networks for Geometric Deep Learning.]
\label{remark:difference}

A recent preprint Incidence Networks for Geometric Deep Learning \cite{albooyeh2019incidence} characterize the linear equivariant maps between incidence tensor, which encodes the combinatorial structure of graphs and its higher order analog simplicial complex and polytopes. \cite{albooyeh2019incidence} characterizes the linear permutation equivariant maps in terms of pooling and broadcasting operations. The pooling and broadcasting operations is the same as the averaging and replication operation defined in \Cref{def:averaging} and \Cref{def:replication}.

The main difference of \cite{albooyeh2019incidence} from our paper is 1) their motivation is to characterize the linear permutation equivariant maps between incidence tensors while in our paper, the similar characterization (in the case of linear permutation equivariant maps of $k$-IGN) serves as a building block for our convergence proof; 2) the characterization in \cite{albooyeh2019incidence} is slightly more general as incidence tensor can have different length for different axis while tensors considered in our case has the same length across all axis.
\end{remark}

\subsection{Missing Proofs from \Cref{sec:convergence-ruiz} (Edge Weight \\ Continuous Model)}
\label{app:proofs-EW}
First we need a lemma on the distribution of gaps between $n$ uniform sampled points on $[0,1]$.
\begin{lemma}
\label{lemma:lengh-distribution}
Let $u_{(i)}$ be $n$ points uniformly sampled on $[0,1]$, sorted from small to large with $u_{(0)}=0$ and $u_{(n+1)}=1$. Let $D_i = u_{(i)} - u_{(i-1)}$. All $D_i$s have same distribution, which is $\tn{Beta}(1, n)$. 
In particular, expectation of $D_i$ $\mathbb{E}(D_i) = \frac{1}{n+1}$, $\mathbb{E}(D_i^2) = \frac{2}{(n+1)(n+2)}$, $\mathbb{E}(D_i^3) = \frac{6}{(n+1)(n+2)(n+3)}$. 
\end{lemma}
\begin{proof}
By a symmetry argument, it is easy to see that all intervals follow the same distribution.
For the first interval, the probability all the $n$ points are above $x$ is $(1-x)^n$ so the density of the length of the first (and so each) interval is $n(1-x)^{n-1}$. This is a Beta distribution with parameters $\alpha=1$ and $\beta=n$ The expectation of higher moments follows easily. Note that although the intervals are identically distributed, they are not independently distributed, since their sum is 1.
\end{proof}

\xdiffrandom*
\begin{proof}
We first bound the $\|X -X_n\|_{L_2[0, 1]}$ and $\|X -\inducedX\|_{L_2[0, 1]}$. %
For the first case, partitioning the unit interval as $I_i = [(i-1)/n, i/n]$ for $1 \leqslant  i \leqslant   n$ (the same partition used to obtain $x_n$, and thus $X_n$, from $X$), we can use the Lipschitz property of $X$ to derive
\begin{align*}
\left\|X-X_n\right\|^2_{L_2(I_i)} \leqslant  A_3^2\int_0^{1/n} u^2 du =  \frac{A_3^2}{3n^3} %
\end{align*}
We can then write
$\|X-X_n\|^2_{{L_2([0,1])}} = \sum_{i}\left\|X-X_n\right\|^2_{L_2(I_i)} \leqslant  \frac{A_3}{3n^2} $, %
which implies that $\|X-X_n\|_{{L_2([0,1])}} \leqslant  \sqrt{\frac{A_3}{3n^2}}$.

For the second case, since
$
\|X-\inducedX\|^2_{{L_2([0,1])}} = \sum_{i}\|X-\inducedX\|^2_{L_2(I_i)} %
$
, we will bound the $\left\|X-\inducedX\right\|^2_{L_2(I_i)}$.
As
\begin{equation*}
\left\|X-\inducedX\right\|^2_{L_2(I_i)} \leqslant  A_3^2\int_0^{D_i} u^2 du = A_3 D_i^{3} / 3
\end{equation*}
therefore
\begin{equation*}
\left\|X-\inducedX\right\|^2_{L_2(I)} = \sum_i \left\|X-\inducedX\right\|^2_{L_2(I_i)} \leqslant  A_3/3 \sum_i D_i^3
\end{equation*}

where $D_i$ stands for the length of $I_i$, which is a random variable due to the random sampling.

According to \Cref{lemma:lengh-distribution}, all $D_i$ are identically distributed and follows the Beta distribution $B(1, n-1)$. The expectation $E(D_i^3) = \frac{6}{n(n+1)(n+2)}$. %
Since by Jensen's inequality $E(\sqrt{Y}) \leqslant  \sqrt{E(Y)} $ holds for any positive random variable $Y$, $E(\sqrt{\frac{A_3}{3} \sum_i D_i^3)} \leqslant  \sqrt{E(\frac{A_3}{3} \sum_i D_i^3)} = \sqrt{\frac{A_3}{3} \frac{1}{n(n+2)}} = \Theta(\frac{1}{n})$. Using Markov inequality, we can then upper bound the
\begin{equation}
 P(\|X-\inducedX\|_{L_2(I)} \geq \epsilon) \leqslant
 P(\sqrt{\frac{A_3}{3} \sum_i D_i^3} \geq \epsilon) \leq
 \frac{E(\sqrt{\frac{A_3}{3} \sum_i D_i^3})}{\epsilon} =
 \Theta(\frac{1}{n\epsilon})
\end{equation}
Since the $P(\|X-\inducedX\|_{L_2(I)} \geq \epsilon)$ goes to 0 as $n$ increases, we conclude that $\newnorm{X-\inducedX}$ converges to 0 in probability. %
\end{proof}

\wdiffrandom*
\begin{proof}
For the first case, partitioning the unit interval as $I_i = [(i-1)/n, i/n]$ for $1 \leqslant  i \leqslant   n$, we can use the graphon's Lipschitz property to derive
\begin{align*}
\left\|W-W_n\right\|_{L_1(I_i\times I_j)} \leqslant  A_1\int_0^{1/n} \int_0^{1/n} |u| du dv &+ A_1\int_0^{1/n} \int_0^{1/n} |v| dv du =  \frac{A_1}{2n^3} + \frac{A_1}{2n^3} = \frac{A_1}{n^3} \text{.}
\end{align*}
We can then write
$\|W-W_n\|_{{L_1([0,1]^2)}} = \sum_{i,j}\left\|W-W_n\right\|_{L_1(I_i\times I_j)} \leqslant  n^2 \frac{A_1}{n^3} = \frac{A_1}{n}
$
which, since $W-W_n: [0,1]^2 \to [-1,1]$, implies
$\|W-W_n\|_{{L_2([0,1]^2)}} \leqslant  \sqrt{\|W-W_n\|_{{L_1([0,1]^2)}}} \leqslant  \wbound \text{.}$ The second last inequality holds because all entries of $W-W_n$ lies in $[-1, 1]$.

Similarly, $\|\text{Diag}(W-W_n)\|_{L_2[0,1]} \leqslant  \sqrt{\|\text{Diag}(W-W_n)\|_{L_1[0,1]}} \leqslant  \sqrt{2nA_1\int_0^{1/n}udu} = \sqrt{\frac{A_1}{n}}$. Therefore we conclude the first part of the proof.

For the second case, \dnorm{} is similar to the proof of \Cref{lem:x-diff} so we only focus on the $\|W-W_n\|_{{L_2([0,1]^2)}}$. Since $W-\inducedW: [0,1]^2 \to [-1,1]$ implies
\begin{equation*}
\|W-\inducedW\|_{{L_2([0,1]^2)}} \leqslant
\sqrt{\|W-\inducedW\|_{{L_1([0,1]^2)}}} =
\sqrt{\sum_{i,j} \|W-\inducedW\|_{{L_1(I_i \times I_j)}}}
\end{equation*}
where
\begin{equation*}
\|W-\inducedW\|_{{L_1(I_i\times I_j)}} \leqslant  A_1\int_{I_v} \int_{I_u} |u| du dv + A_1\int_{I_u} \int_{I_v} |v| dv du =
\frac{A_1}{2} (D_iD_j^2+D_jD_i^2)
\end{equation*}
Therefore
\begin{equation}
\|W-\inducedW\|_{{L_2([0,1]^2)}} \leqslant
\sqrt{\|W-\inducedW\|_{{L_1([0,1]^2)}}} =
\sqrt{\sum_{i, j} \frac{A_1}{2} (D_jD_i^2 + D_iD_j^2)} =
\sqrt{A_1\sum_i D_i^2}
\end{equation}
where we use the $\sum_i D_i =1$ for the last equality.
Since by Jensen's inequality $E(\sqrt{Y}) \leqslant  \sqrt{E(Y)} $ for any positive random variable $Y$, $E(\sqrt{\sum_i D_i^2}) \leqslant  \sqrt{E(\sum_i D_i^2)} = \Theta(\frac{1}{\sqrt{n}})$ since $E(D_i^2) = \Theta(\frac{1}{n^2})$ by \Cref{lemma:lengh-distribution}. By Markov inequality, we then bound
\begin{equation*}
P(\|W-\inducedW\|_{{L_2([0,1]^2)}}>\epsilon) \leq
P(\sqrt{\|W-\inducedW\|_{{L_1([0,1]^2)}}} > \epsilon) \leqslant
\frac{E(\sqrt{\sum_i D_i^2})}{\epsilon} \leq
\Theta(\frac{1}{\sqrt{n}\epsilon})
\end{equation*}
\end{proof}
Therefore, we conclude that both $\newnorm{W-W_n}$ and $\newnorm{W-\inducedW}$ converges to 0. %

\phistable
\begin{proof}
Without loss of generality, it suffices to prove for $2$-IGN as $k$-IGN follows the same proof with the constant being slightly different.
Since we have proved stability of every linear layers of IGN in \Cref{thm:linear-layer-stability}, the general linear layer $T$ is just a linear combinations of individual linear basis, i.e. $T = \sum_{\gamma} c_{\gamma}T_{\gamma}$ where $c_i \leqslant  A_2$ for all $i$ according to AS\ref{as:filter-bound}. Without loss of generality, We can assume $T(X)$ is of order 2 and have
\begin{align*}
\|T(W_1)-T(W_2)\|_{\tn{pn}} & = \|\sum_i c_{\gamma}T_{\gamma}(W_1-W_2)\|_{\tn{pn}} \\
& \leqslant  \sum_i \|c_{\gamma}T_{\gamma}(W_1-W_2)\|_{\tn{pn}} \\
& \leqslant  (\sum |c_{\gamma}|\epsilon, \sum |c_{\gamma}|\epsilon) = (15\filterbound \epsilon, 15 \filterbound \epsilon)
\end{align*}

To extend the result to nonlinear layer, note that AS\ref{as:activation-lip} ensures the 2-norm shrinks after passing through nonlinear layers. Therefore $\newnorm{\sigma \circ T(X) - \sigma \circ T(Y)} \leqslant  \newnorm{T(X) - T(Y)} = \newnorm{T(X-Y)}\leqslant  15A_2\newnorm{X-Y}$.
Repeating such process across layers, we finish the proof of the $L_{2}$ case. %

The extension to $L_{\infty}$ is similar to the case of $L_2$ norm. The main modification is to change the definition of the \textnewnorm{} from $L_2$ norm on different slices (corresponding to different partitions of $[\ell]$ where $\ell$ is the order of input) to $L_{\infty}$ norm. The extension to the case where input and output tensor is of order $\ell$ and $m$ is also straightforward according to \Cref{thm:linear-layer-stability}.

\end{proof}

\EWconvergence*
\begin{proof}
By \Cref{prop:Phi-stable}, it suffices to prove that $\newnorm{[W, \dg{X}]) - [W_n, \dg{X_n}]}$ and $\newnorm{[W, \dg{X}]) - [\inducedW, \dg{\inducedX}]}$ goes to 0.

$\newnorm{[W, \dg{X}]) - [W_n, \dg{X_n}]}$ is upper bounded by $(\Theta(\frac{1}{n^{1.5}}), \Theta(\frac{1}{n^{1.5}}))$ according to \Cref{lem:w-diff,lem:x-diff}, which decrease to 0 as $n$ increases. Therefore we finish the proof of convergence for the deterministic case.

For the random sampling case, by \Cref{lem:w-diff,lem:x-diff}, we know that both $\|W-\inducedW\|_{{L_2([0,1]^2)}}$ and $\|X-\inducedX\|_{L_2(I)}$ goes to 0 as $n$ increases in probability at the rate of $\Theta(\frac{1}{n^{1.5}})$. Therefore we can also conclude that the convergence of IGN in probability according to \Cref{prop:Phi-stable}.
\end{proof}

\subsection{Missing Proof from \Cref{sec:EP-convergence} (Edge Probability Continuous Model)}
\subsubsection{Missing Proof for \Cref{subsec:negative-result}}
\label{app:missing-proofs-neg}

\convfailure*
\begin{proof}
Given a fixed IGN architecture $\Phi_c$ that maps input $\mathbb{R}^{n^2 \times d_1}$ to $ \mathbb{R}^{n^k \times d_2}$, it suffices to show the case of $k=1$ and $d_2 = 1$. 
Under the case of $k=1$ and $d_2 = 1$, it suffice to show that single layer IGN may not converge. Let $IGN =  \sigma \circ L^{(1)}$ have only one linear layer, and let the input to IGN be $A$ in the discrete case and $W$ in the continuous case. For simplicity, we assume that graphon $W$ is constant $p$ on $[0,1]^2$.
As $A$ consists of only 0 and 1 and all entries of $W$ is below $c_{\text{max}}$, we can set weights of IGN such that its first linear layer consists of only identity map and bias term. By choosing bias term to be any number between $[-1, -c_{\text{max}}]$, $L^{(1)}$ map any number no large than $c_\text{max}$ to negative and maps 1 to positive.

Therefore $L^{(1)}(W)=0$ and $L^{(1)}(A)$ is a positive number $c\in \mb{R}^{+}$ on entries $(i, j)$ where $A(i, j)=1$. Let $\sigma$ be ReLU and $L^{(2)}$ be average of all entries. We can see that c$IGN(W)=0$ for all $n$ while $IGN(A)$ converges to  $\sigma(c)p$ as $n$ increases.

As the construction above only relies on the fact that there is a separation between $c_{\text{max}}$ and $1$ (but not on size $n$), it can be extended to deeper IGNs 
, which means the gap between c$IGN(W)$ and $IGN(A)$ will not decrease as $n$ increases. In the general case of $W$ not being constant, the only difference is that $IGN(A)$ will converge to be $\sigma(c)p^*$ where $p^*$ is a different constant that depends on $W$.
Therefore we conclude the proof. 
\end{proof}
\begin{remark}
The reason that the same argument does not work for \sGNN{} is that \sGNN{} always maintains $Ax$ in the intermediate layer. In contrast, IGN keeps both $A$ and $\dg{x}$ in separate channels, which makes it easy to isolate them to construct counterexamples. 

\end{remark}

\subsubsection{Missing Proofs from \Cref{subsec:smallign-convergnce}}
\label{app:smallign-convergnce}
\textbf{Notation.} For any $P, Q \in \mb{R}^{n \times n}$, define $d_{2, \infty}$, the normalized $2, \infty$ matrix norm, by
$d_{2, \infty}(P, Q)=n^{-1 / 2}\|P-Q\|_{2, \infty}:=\max _{i} n^{-1 / 2}\left\|P_{i, \cdot}-Q_{i, \cdot} \right\|_{2}$ where $P_{i, \cdot}, Q_{i, \cdot}$ are $i$-th row of $P$ and $Q$, respectively. Note that $d_{2, \infty}(P, Q) \geq  \frac{1}{n}\| P - Q\|_2$. %

 Let $S_U$ be the sampling operator for $W$, i.e., $S_U(W) = \frac{1}{n}[W(U_i, U_j)]_{n\times n}$. Note that as $U$ is randomly sampled, $S_U$ is a random operator. Denote $\sampleE$ as sampling on a fixed equally spaced grid of size $n\times n$, i.e. $S_nW = \frac{1}{n}[W(\frac{i}{n}, \frac{j}{n})]_{n \times n}$. $S_n$ is a fixed operator when $n$ is fixed.

Let $\widehat{W}_{n \times n}$ be the estimated edge probability from graphs $A$ sampled from $W$. Let $\inducedW$ be the piece-wise constant graphon induced from sample $U$ as \cref{eqn:induced-cgcn-random}. Similarly, denote $\sampleW$ be the $n\times n$ matrix realized on sample $U$, i.e., $\sampleW[i, j] = W(u_i, u_j)$. It is easy to see that $S_U(W)  = \frac{1}{n}\sampleW$. Let $\inducedEW$ be the graphon induced by $\sampleW$ with $n\times n$ blocks of the same size. In particular, $\inducedEW(I_i \times I_j) \:= W(u_{(i)}, u_{(j)})$ where $I_i = [\frac{i-1}{n}, \frac{i}{n}]$. $E$ in the subscript is the shorthand for the ``blocks of equal size''. Similarly we can also define the 1D analog of $\inducedW$ and $\inducedEW$, $\inducedX$ and $\inducedEX$.

\textbf{Proof strategy.} We first state five lemmas that will be used in the proof of \Cref{thm:convergenceafterEM}.
\Cref{lem:property-snsx} concerns the property of normalized sampling operator $S_U$ and $S_n$.
\Cref{lem:inducedW-converges,lem:inducedEW-converges} concern the convergence of $\Linf{\inducedW-W}$ and $\Linf{\inducedEW-W}$. \Cref{lem:property-of-T} characterize the effects of linear equivariant layers $T$ and IGN $\Phi$ on \Linfnorm of the input and output. \Cref{lem:sxsn-sampling} bounds the \Linfnorm of the difference of stochastic sampling operator $S_U$ and the deterministic sampling operator $S_n$. \Cref{thm:convergenceafterEM} is built on the results from five lemmas and the existing result on the theoretical guarantee of edge probability estimation from \cite{zhang2015estimating}.

The convergence some lemmas states is almost surely convergence.  Convergence almost surely implies convergence in probability, and in this chapter, all theorems concern convergence in probability.  Note that proofs of \Cref{lem:property-snsx,lem:inducedW-converges,lem:inducedEW-converges,lem:sxsn-sampling} for the $W$ and $X$ are almost the same. Therefore without loss of generality, we mainly prove the case of $W$.

\begin{definition}[Chessboard pattern]
\label{def:chessboard}
Let $u_i = \frac{i-1}{n}$ for all $i\in [n]$. A graphon $W$ is defined to have chessboard pattern if and only if there exists a $n$ such that $W$ is a piecewise constant on $[u_i, u_{i+1}]\times [u_j, u_{j+1}]$ for all $i, j\in [n]$. Similarly, $f: [0, 1] \rightarrow \mb{R}$ has 1D chessboard pattern if there exists $n$ such that $f$ is a piecewise constant on $[u_i, u_{i+1}]$ for all $i\in [n]$.
\end{definition}
See \Cref{fig:chessboard} for examples and counterexamples.

\begin{figure}[htp!]
  \centering
  \includegraphics[width=.7\linewidth]{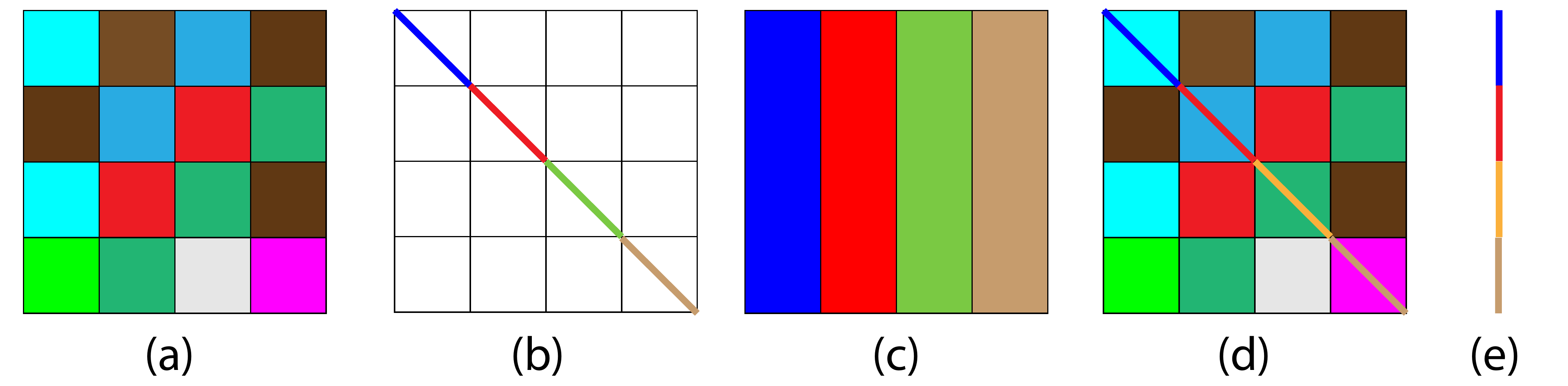}
\caption{(a) and (c) has chessboard pattern. (e) has 1D chessboard pattern. (d) does not has the chessboard pattern. (b) is of form $\dg{\inducedEf}$ and also does not have chessboard pattern, but in the case of IGN approximating Spectral GNN, (b) is represented in the form of c).} %
\label{fig:chessboard}
\end{figure}

\begin{lemma}[Property of $S_n$ and $S_U$]\label{lem:property-snsx}
We list some properties of sampling operator $S_U$ and $S_n$
\begin{enumerate}
\item $S_U \circ \sigma = \sigma \circ S_U$. Similar result holds for $S_n$ as well.

\item $\|S_Uf_{\text{1d}}\| \leqslant  \| f_{\text{1d}} \|_{L_\infty}$ 

where $f_{\text{1d}}: [0,1] \rightarrow \mb{R}$. Similar result holds for $f_{\text{2d}}: [0,1]^2 \rightarrow \mb{R}$ and $S_n$ as well.

\end{enumerate}
\end{lemma}

\begin{lemma}\label{lem:inducedW-converges}
Let $W$ be $[0,1]^2\rightarrow \mb{R}$ and $X$ be $[0,1]\rightarrow \mb{R}$.
If $W$ is \lip{}, $\|\inducedW - W \|_{L_\infty}$ converges to 0 in probability. If $X$ is \lip{}, $\|\inducedX - X \|_{L_\infty}$ converges to 0 in probability.

\end{lemma}
\begin{proof}
Without loss of generality, we only prove the case for $W$. By the \lip{} condition of $W$, if suffices to bound the $Z_n\:=\text{max}_{i=1}^{n}D_i$ where $D_i$ is the length of i-th interval $|u_{(i)} - u_{(i-1)}|$. Characterizing the distribution of the length of largest interval is a well studied problem \cite{renyi1953theory,pyke1965spacings,holst1980lengths}.  It can be shown that $Z_n$ follows  $P\left(Z_{n} \leqslant  x\right)=\sum_{j=0}^{n+1}\left(\begin{array}{c}n+1 \\ j\end{array}\right)(-1)^{j}(1-j x)_{+}^{n}$ with the expectation $E(Z_k) = \frac{1}{n+1}\sum_{i=1}^{n+1}\frac{1}{i} = \Theta(\frac{\log n}{n})$. By Markov inequality, we conclude that $\|\inducedW - W \|_{L_\infty}$ converges to 0 in probability. %

\end{proof}

\begin{lemma}\label{lem:inducedEW-converges}
Let $W$ be $[0,1]^2\rightarrow \mb{R}$ and $X$ be $[0,1]\rightarrow \mb{R}$.
If $W$ is \lip{}, $\| \inducedEW - W\|_{L_\infty}$ converges to 0 almost surely. If $X$ is \lip{}, $\| \inducedEX - X\|_{L_\infty}$ converges to 0 almost surely.
\end{lemma}

\begin{proof}
As $\inducedEW$ is a piecewise constant graphon and $W$ is \lip{}, we only need to examine $\text{max}_{i, j}\|(W - \inducedEW)(\frac{i}{n}, \frac{j}{n})\|$.

It is easy to see that $(W - \inducedEW)(\frac{i}{n}, \frac{j}{n}) = W(\frac{i}{n}, \frac{j}{n}) - W(u_{(i)}, u_{(j)})$ where $u_{(i)}$ stands for the i-th smallest random variable from uniform \iid samples from $[0,1]$. By the \lip{} condition of $W$, if suffices to bound $\|\frac{i}{n} - u_{(i)} \| + \|\frac{j}{n} - u_{(j)} \|$.
Glivenko-Cantelli theorem tells us that the $L_\infty$ of empirical distribution $F_n$ and  cumulative distribution function $F$ converges to 0 almost surely, i.e., $\text{sup}_{u\in [0,1]} |F(u) - F_n(u)| \rightarrow 0$ almost surely. Since $\text{max}_i \|u_{(i)} - \frac{i}{n} \| = \text{sup}_{u\in \{u_{(1)}, ..., u_{(n)}\}} |F(u) - F_n(u)|  \leqslant  \text{sup}_{u\in [0,1]} |F(u) - F_n(u)|$ when $F(u)=u$ (cdf of uniform distribution), we conclude that $\| \inducedEW - W\|_{L_\infty}$ converges to 0 almost surely.

\end{proof}

We also need a lemma on the property of the linear equivariant layers $T$.
\begin{lemma}[Property of $T_c$ and $\sigma$]
\label{lem:property-of-T}
Let $\sigma$ be nonlinear layer. Let $T_c$ be a linear combination of elements of basis of the space of linear equivariant layers of cIGN, with coefficients upper bounded. We have the following property about $T_c$ and $\sigma$
\begin{enumerate}
\item If $W$ is \lip{}, $T_c(W)$ is piecewise \lip{} on diagonal and off-diagonal. Same statement holds for $\Phi_c(W)$.

\item $S_n \circ \sigma (\inducedEW) = \sigma \circ S_n (\inducedEW)$.
\end{enumerate}
\end{lemma}

\begin{proof}
We prove two statements one by one.
\begin{enumerate}
\item
We examine the linear equivariant operators from $\mb{R}^{[0,1]^2}$ to $\mb{R}^{[0,1]^2}$ in \Cref{tab:R2-R2}. There are some operations such as ``average of rows replicated on diagonal'' will destroy the \lip{} condition of $T_c(W)$ but $T_c(W)$ will still be piecewise \lip{} on diagonal and off-diagonal.
Since $\sigma$ will preserve the \lip{}ness, $\Phi_c(W)$ is piecewise \lip{} on diagonal and off-diagonal.

\item This is easy to see as $\sigma$ acts on input pointwise.
\end{enumerate}
\end{proof}

\begin{lemma}\label{lem:sxsn-sampling}
Let $W$ be $[0,1]^2\rightarrow \mb{R}$
\begin{enumerate}
\item If $W$ is \lip{},  $\| S_UW - S_nW\|$ converges to 0 almost surely. Similarly, if $X$ is \lip{},  $\| S_U\dg{X} - S_n\dg{X}\|$ converges to 0 almost surely.

\item If $W$ is piecewise \lip{} on $S_1$ and $S_2$ where $S_1$ is the diagonal and $S_2$ is off-diagonal,  then $\| S_UW - S_nW\|$ converges to 0 almost surely. %

\end{enumerate}
\end{lemma}

\begin{proof}
Since the case of $X$ is essentially the same with that of $W$, we only prove the case of $W$.
\begin{enumerate}
\item As $n\| S_UW - S_nW\|_{\infty} \geq \| S_UW - S_nW\|$, it suffices to prove that $n\| S_UW - S_nW\|_{\infty} = \text{max}_{i, j}|W(u_{(i)}, u_{(j)}) - W(\frac{i}{n}, \frac{j}{n})|$  converges to 0 almost surely. Similar to \Cref{lem:inducedEW-converges}, using \lip{} condition of $W$ and Glivenko-Cantelli theorem concludes the proof.

\item This statement is stronger than the one above. The proof of the last item can be adapted here. As $W$ is $A_1$ \lip{} on off-diagonal region and $A_2$ \lip{} on diagonal,
\begin{align*}
 n\| S_UW - S_nW\|_{\infty} & =
  \text{max}_{i, j}\left|W(u_{(i)}, u_{(j)}) - W(\frac{i}{n}, \frac{j}{n})\right| \\
  = & \max\left(\text{max}_{i\neq j} \left|W(u_{(i)}, u_{(j)}) - W(\frac{i}{n}, \frac{j}{n})\right|, \text{max}_{i = j}\left|W(u_{(i)}, u_{(j)}) - W(\frac{i}{n}, \frac{j}{n})\right|\right).
\end{align*}
Using \lip{} condition on diagonal and off-diagonal part of $W$ and Glivenko-Cantelli theorem concludes the proof.
\end{enumerate}
\end{proof}
With all lemmas stated, we are ready to prove the main theorem.

\convergenceafterEM*
\begin{proof}

Using the triangle inequality %

\begin{align*}
\small
& \MSE_U(\Phi_c \left([W, \dg{X}] \right), \Phi_d\left([\widehat{W}_{n \times n}, \dg{\widetilde{x_n}}]\right)) \\
& = \left\|S_U \Phi_c\left([W, \dg{X}]\right)-\frac{1}{\sqrt{n}}\Phi_d \left([\widehat{W}_{n \times n}, \dg{\widetilde{x_n}}]\right) \right\| \\
& = \|S_U \Phi_c\left([W, \dg{X}]\right)- S_U\Phi_c\left([\inducedW, \dg{\inducedX}]\right) \\
& + S_U\Phi_c\left([\inducedW, \dg{\inducedX}]\right) - \Phi_dS_U([\inducedW, \dg{\widetilde{x_n}}])  \\
&+ \Phi_dS_U([\inducedW, \dg{\inducedX}]) - \frac{1}{\sqrt{n}}\Phi_d ([\widehat{W}_{n \times n}, \dg{\inducedX}])\|\\
&\leqslant  \underbrace{\left\|S_U \Phi_c\left([W, \dg{X}]\right)- S_U\Phi_c\left([\inducedW, \dg{\inducedX}]\right)\right\|}_\text{First term: discretization error}  \\
&+\underbrace{\left\|S_U\Phi_c\left([\inducedW, \dg{\inducedX}]\right) - \Phi_dS_U([\inducedW, \dg{\inducedX}])\right\|}_\text{Second term: sampling error} \\
& + \underbrace{\left\|\Phi_dS_U([\inducedW, \dg{\inducedX}]) - \frac{1}{\sqrt{n}}\Phi_d \left([\widehat{W}_{n \times n}, \dg{\widetilde{x_n}}]\right)\right\|}_\text{Third term: estimation error} \numberthis \label{equ:threeterms} %
\end{align*}
The three terms measure the different sources of error. The first term is concerned with the discretization error. The second term concerns the sampling error from the randomness of $U$. This term will vanish if we consider only $S_n$ instead of $S_U$ for \smallIGN{}. The third term concerns the edge probability estimation error.

For the first term, it is similar to the sketch in \Cref{subsec:smallign-convergnce}. $\|S_U \Phi_c([W, \dg{X}]) - S_U \Phi_c ([\inducedW, \dg{\inducedX}]) \| = \| S_U(\Phi_c([W, \dg{X}]) - \Phi_c([\inducedW, \dg{\inducedX}]))\|$, if suffices to upper bound $\|\Phi_c([W, \dg{X}])-\Phi_c([\inducedW, \dg{\inducedX}])\|_{L_\infty}$ according to property of $S_U$ in \Cref{lem:property-snsx}. Since $\| \Phi_c ([W, \dg{X}]) - \Phi_c ([\inducedW, \dg{\inducedX}])\|_{L_\infty} \leqslant  C (\|W - \inducedW\|_{L_\infty} + \|\dg{X} - \dg{\inducedX}\|_{L_\infty})$ by \Cref{prop:Phi-stable}, and $\|W - \inducedW\|_{L_\infty}$ converges to 0 in probability according to \Cref{lem:inducedW-converges}, we conclude that the first term will converges to 0 in probability.

For the third term $\| \Phi_d S_U ([\inducedW, \dg{\inducedX}]) - \frac{1}{\sqrt{n}}\Phi_d ([\widehat{W}_{n \times n}, \dg{\widetilde{x_n}}])\| $$ \\ =\| \frac{1}{\sqrt{n}}( \Phi_d ([\sampleW, \dg{\widetilde{x_n}}]) - \Phi_d ([\widehat{W}_{n \times n}, \dg{\widetilde{x_n}}]) )\| $$=\newnorm{\Phi_d([\sampleW, \dg{\widetilde{x_n}}]) - \\ \Phi_d([\widehat{W}_{n \times n}, \dg{\widetilde{x_n}}])}$,

it suffices to control the $\newnorm{[W_{n \times n}, \dg{\widetilde{x_n}}]-[\widehat{W}_{n \times n}, \dg{\widetilde{x_n}}]} = \frac{1}{n}\| \sampleW-\widehat{W}_{n\times n} \|_2 \leqslant  \|\sampleW-\widehat{W}_{n\times n}\|_{2, \infty}$, which will also goes to 0 in probability as $n$ increases according to the statistical guarantee of edge probability estimation of neighborhood smoothing algorithm \cite{zhang2015estimating}, stated in \Cref{thm:graphon-estimation}. 
Therefore by \Cref{prop:Phi-stable}, the third term also goes to 0 in probability.

Therefore the rest work is to control the second term $\|S_U \Phi_c\left([\inducedW, \dg{\inducedX}]\right) - \\ \Phi_d S_U\left([\inducedW, \dg{\inducedX}]\right)\|$. Again, we use the triangle inequality

{\fontsize{10}{12}\selectfont
\begin{align*}
& \tn{Second term} \\
& = \left\|S_U \Phi_c\left([\inducedW, \dg{\inducedX}]\right) - \Phi_d S_U\left([\inducedW, \dg{\inducedX}]\right)\right\| \\
& \leqslant  \left\|S_U \Phi_c\left([\inducedW, \dg{\inducedX}]\right) - \sampleE \Phi_c\left([\inducedEW, \dg{\inducedEX}]\right) \right\| + \left\| \sampleE \Phi_c\left([\inducedEW, \dg{\inducedEX}]\right) - \Phi_d S_U\left([\inducedW, \dg{\inducedX}]\right)\right\| \\
& =  \left\|S_U \Phi_c\left([\inducedW, \dg{\inducedX}]\right) - \sampleE \Phi_c\left([\inducedEW, \dg{\inducedEX}]\right) \right\| + \left\| \sampleE \Phi_c\left([\inducedEW, \dg{\inducedEX}]\right) - \Phi_d \sampleE([\inducedEW, \dg{\inducedEX}])\right\|\\
& = \left\|S_U \Phi_c\left([\inducedW, \dg{\inducedX}]\right) - \sampleE \Phi_c\left([\inducedEW, \dg{\inducedEX}]\right) \right\|\\
& \leqslant  \left\|S_U \Phi_c\left([\inducedW, \dg{\inducedX}]\right) - S_U \Phi_c\left([\inducedEW, \dg{\inducedEX}]\right)\right\| + \left\|S_U \Phi_c\left([\inducedEW, \dg{\inducedEX}]\right) - \sampleE \Phi_c\left([\inducedEW, \dg{\inducedEX}]\right)  \right\| \\
& = \underbrace{\left\|S_U \left(\Phi_c ([\inducedW, \dg{\inducedX}]\right) - \Phi_c \left([\inducedEW, \dg{\inducedEX}]\right)\right\|}_\text{term $a$} +
\underbrace{\left\|(S_U - \sampleE)\Phi_c\left([\inducedEW, \dg{\inducedEX}]\right) \right\|}_\text{term $b$}
\end{align*}
}

The second equality holds because $S_U([\inducedW, \dg{\inducedX}]) = S_n([\inducedEW, \inducedEX])$ by definition of $\inducedEW$ and \smallIGN{} (See \Cref{remark:difficulty} for more discussion).  The third equality holds by the definition of \smallIGN{}.
We will bound the term a) $\|S_U (\Phi_c ([\inducedW, \dg{\inducedX}]) - \Phi_c ([\inducedEW, \inducedEX]))\|$ and b) $\|(S_U - \sampleE)\Phi_c([\inducedEW, \dg{\inducedEX}]) \|$ next.

For term a) $\|S_U (\Phi_c ([\inducedW, \dg{\inducedX}]) - \Phi_c ([\inducedEW, \inducedEX]))\|$, if suffices to prove that \\ $\|\Phi_c ([\inducedW, \dg{\inducedX}]) -\Phi_c ([\inducedEW, \inducedEX]))\|_{L_\infty}$ converges to 0 in probability.
According to \Cref{prop:Phi-stable}, it suffices to bound the $\| [\inducedW, \inducedX] - [\inducedEW, \inducedEX] \|_{L_\infty}$. Because $[\inducedW, \inducedX] - [\inducedEW, \inducedEX] \|_{L_\infty}$ \\ $ =  \|\inducedW - \inducedEW \|_{L_\infty} + \|\dg{\inducedX} - \dg{\inducedEX} \|_{L_\infty}) \leqslant  \|\inducedW - W \|_{L_\infty} + \|\inducedEW - W \|_{L_\infty} + \|\dg{\inducedX} - \dg{X} \|_{L_\infty} + \|\dg{\inducedEX} - \dg{X} \|_{L_\infty}$, we only need to upper bound $\|\inducedW - W \|_{L_\infty}$, $\|\inducedEW - W \|_{L_\infty}$, $\|\dg{\inducedX} - \dg{X} \|_{L_\infty})$ and $\|\dg{\inducedEX} - \dg{X} \|_{L_\infty})$, which are proved by \Cref{lem:inducedW-converges} and \Cref{lem:inducedEW-converges} respectively.

For term b) $\|(S_U - \sampleE)\Phi_c\left([\inducedEW, \dg{\inducedEX}]\right) \|$
{\fontsize{10}{12}\selectfont
\begin{align*}
&\left\|(S_U - \sampleE)\Phi_c\left([\inducedEW, \dg{\inducedEX}]\right) \right\| \\
&= \left\|(S_U \Phi_c\left([\inducedEW, \dg{\inducedEX}]\right) - \sampleE \Phi_c\left([\inducedEW, \dg{\inducedEX}]\right)\right\|\\
& \leqslant  \left\|(S_U \Phi_c\left([\inducedEW, \dg{\inducedEX}]\right) - S_U \Phi_c\left([W, \dg{X}]\right) \right\| + \left\|S_U \Phi_c\left([W, \dg{X}]\right) - S_n \Phi_c\left([W, \dg{X}]\right)\right\|  \\
& + \left\|S_n \Phi_c\left([W, \dg{X}]\right) - \sampleE \Phi_c\left([\inducedEW, \dg{\inducedEX}]\right)\right\|\\
& = \left\|(S_U (\Phi_c\left([\inducedEW, \dg{\inducedEX}]\right)-\Phi_c ([W, \dg{X}])) \right\| + \left\|S_U \Phi_c\left([W, \dg{X}]\right) - S_n \Phi_c\left([W, \dg{X}]\right)\right\| \\
& + \left\|\sampleE (\Phi_c\left([\inducedEW, \dg{\inducedEX}]\right)-\Phi_c ([W, \dg{X}]))\right\|
\end{align*}
}
For the first and last term, by the property of $S_U, S_n$ and $\Phi_c$, it suffices to bound $\|W-\inducedEW \|_{L_\infty}$ and $\|\dg{X}-\dg{\inducedEX} \|_{L_\infty}$. Without loss of generality, We only prove the case for $W$. As $\|W-\inducedEW \|_{L_\infty}$ converges to 0 almost surely by \Cref{lem:inducedEW-converges}, we conclude that the first and last term converges to 0 almost surely (therefore in probability).
For the second term $\|S_U \Phi_c\left([W, \dg{X}]\right) - S_n \Phi_c\left([W, \dg{X}]\right)\|$, $\Phi_c\left([W, \dg{X}]\right)$ is piecewise \lip{} on diagonal and off-diagonal according to \Cref{lem:property-of-T} 
, and it converges to 0 almost surely according to the second part of \Cref{lem:sxsn-sampling}.

As all terms converge to 0 in the probability or almost surely, we conclude that \\  $\|S_U \Phi_c\left([W, \dg{X}]\right)-\Phi_d ([\widehat{W}_{n \times n}, \dg{\inducedX}]) \|$ converges to 0 in probability.
\end{proof}

\begin{remark}
\label{remark:difficulty}
Note that we can not prove $S_n \cdot \Phi_c (\inducedEW) = \Phi_d \cdot S_n (\inducedEW)$ in general.
The difficulty is that starting with $\inducedEW$ of chessboard pattern, after the first layer, pattern like \Cref{fig:chessboard}(e) may appear in $\sigma \circ T_1(\inducedW)$. If $T_2$ is just a average/integral to map $\dtensor{2}{1}$ to $\mb{R}$, then $S_n \circ T_2 \circ \sigma \circ T_1(\inducedW) = T_2 \circ \sigma \circ T_1(\inducedW)$ will not be equal to $T_2 \circ \sigma \circ T_1(S_n \inducedW)$. The reason is that both $\sigma \circ T_1(\inducedW)$ and $\sigma \circ T_1(S_n \inducedW)$ will no longer be of chessboard pattern (\Cref{fig:chessboard}(e) may occur). The diagonal in the $\sigma \circ T_1(\inducedW)$ has no effect after taking integral in $T_2$ as it is of measure 0. On the other hand, the diagonal in the matrix $\sigma \circ T_1(S_n \inducedW)$ will affect the average. Therefore in general,  $S_n \Phi_c (\inducedEW) = \Phi_d S_n (\inducedEW)$ does not hold.

\end{remark}

\subsection{\smallIGN{} can Approximate Spectral GNN}
\label{app:approx}
\textbf{Definition of Spectral GNN.}
The \sGNN{} (SGNN) here stands for GNN with multiple layers of the following form $\forall j=1, \ldots d_{\ell+1}$,
\begin{align}
\label{eq:sgnn}
{\quad }z_{j}^{(\ell+1)}=\sigma \left(\sum_{i=1}^{d_{\ell}} h_{i j}^{(\ell)}(L) z_{i}^{(\ell)}+b_{j}^{(\ell)} 1_{n}\right) \in \mb{R}^{n}
\end{align}
where $L = D(A)^{-\frac{1}{2}} A D(A)^{-\frac{1}{2}}$ stands for normalized adjacency,\footnote{We follow the same notation as \cite{keriven2020convergence}, which is different from the conventional notation.} $z_j^{\ell}, b_j^{\ell} \in \mb{R}$ denotes the embedding and bias at layer $\ell$. $d_\ell$ stands for the number of output channels in $\ell$-th layer. 
 $h: \mb{R} \rightarrow \mb{R}, h(\lambda) = \sum_{k\geq 0} \beta_k \lambda^k, h(L)=\sum_{k} \beta_{k} L^{k}$, i.e., we apply $h$ to the eigenvalues of $L$ when it is diagonalizable. Extending $h$ to multiple input output channels which are indexed in $i$ and $j$, we have $h_{i j}^{(\ell)}(\lambda)=\sum_{k} \beta_{i j k}^{(\ell)} \lambda^{k}$. By defining all components of \sGNN{} for graphon, the continuous version of \sGNN{} can also be defined. %
See \cite{keriven2020convergence} for details.

We first prove IGN can approximate \sGNN{} arbitrarily well, both for discrete SGNN and continuous SGNN. Next, we show that such IGN belongs to \smallIGN{}.
We need the following simple assumption to ensure the input lies in a compact domain.
    \begin{assumption}\label{as:x-compact}
    There exists an upper bound on $\|x\|_{L_\infty}$ for the discrete case and $\|X\|_{L_\infty}$ in the continuous case.
    \end{assumption}

    \begin{assumption}\label{as:min-deg-lower-bound}
    $\text{min}(\degmean) \geq c_{\tn{min}}$ where $\degmean$ is defined to be $\frac{1}{n}\dg{A\one}$. The same lower bound holds for graphon case.
    \end{assumption}

\begin{restatable}[]{lemma}{ignmultiplication}
\label{lem:ign-multiplication}
Assume AS1-AS6 and $DMD$ arbitrarily well in $L_\infty$ sense on a compact domain %
\end{restatable}

    \begin{proof}
    Given diagonal matrix $D$ and matrix $M$,
    to implement $DMD$ with linear equivariant layers of \IGN, we first use operation 14-15 in \Cref{tab:R2-R2} to copy diagonal elements in $D$ to rows and columns of two matrix $D_\tn{row}$ and $D_\tn{col}$. Then calculating $DMD$ becomes entry-wise multiplication of three matrix $D_\tn{row}, M, D_\tn{col}$. Assuming all entries of $D$ and $M$ lies in a compact domain, we can use MLP (which is part of IGN according to \Cref{remark:IGN}) to approximate multiplication arbitrarily well \cite{cybenko1989approximation,hornik1989multilayer}.  for illustration.

    To implement $\frac{1}{n}Mx$ with linear equivariant layers of \IGN,
    first map $x$ into a diagonal matrix $\text{Diag}(x)$ and concatenate it with $M$ as the input
    $[\text{Diag}(x), M] \in \mb{R}^{n\times n \times 2}$ to \IGN. Apply ``copy diagonal to all columns'' to the first channel and use MLP to uniformly approximates up to arbitrary precision $\epsilon$ the multiplication of first channel with the second channel. Then use operation ``copy row mean'' to map $\mb{R}^{n\times n } \rightarrow \mb{R}^{n}$ to get the $\frac{1}{n}Mx$ within $\epsilon$ precision. See \Cref{fig:approx}.
    \end{proof}

    \begin{remark}\label{remark:IGN-mm-multiplication}
    Linear layers in \IGN{} can not implement matrix-matrix multiplication in general. When we introduce the matrix multiplication component, the expressive power of GNN in terms of WL test provably increases from 2-WL to 3-WL \cite{maron2019provably}). 
    \end{remark}

    \begin{restatable}[]{theorem}{ignapproxsgnn}
    \label{thm:ign-approx-sgnn}
    Given $n$, $\epsilon$, and $\tn{SGNN}_{\theta_1}(n)$, there exists a \tn{\IGN{}} $\tn{IGN}_{\theta_2}(n)$ such that it approximates $\tn{SGNN}_{\theta_1}(n)$ on a compact set (support of input feature $x_n$) arbitrarily well in $L_\infty$ sense. %
    \end{restatable}

\begin{figure}[htp!]
  \centering
  \includegraphics[width=.8\linewidth]{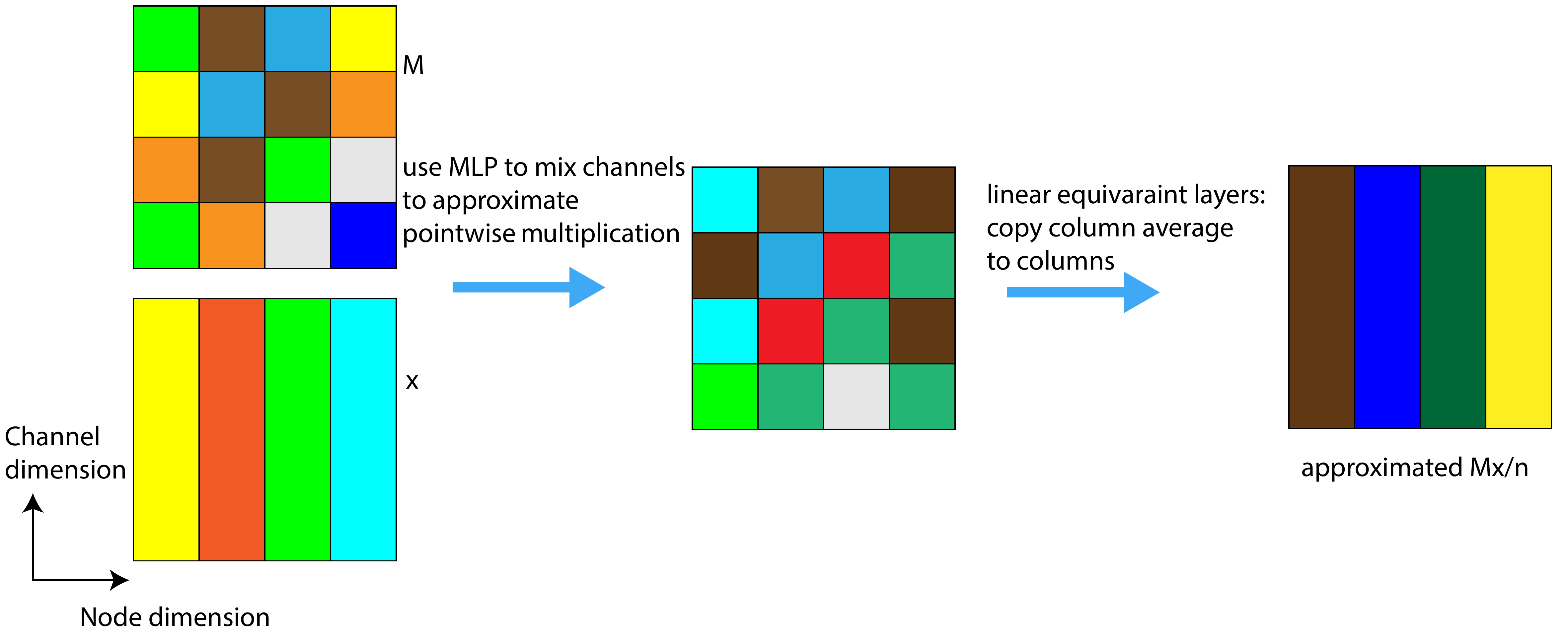}
\caption{An illustration of how we approximate the major building blocks of SGNN: $\frac{1}{n}Ax$. }
  \label{fig:approx}
\end{figure}

    \begin{proof}
    Since IGN and SGNN has the same non-linearity. To show that IGN can approximate SGNN, it suffices to show that IGN can approximate linear layer of SGNN, which further boils down to prove that IGN can approximate $Lx$.

    Here we assume the input of \IGN{} is $A\in \mb{R}^{n \times n}$ and $x\in \mb{R}^{n \times d}$. We need to first show how $L = D(A)^{-\frac{1}{2}} A D(A)^{-\frac{1}{2}}$ can be implemented by linear layers of IGN. This is achieved by noting that $L = \frac{1}{n} \degmean^{-\frac{1}{2}} A \degmean^{-\frac{1}{2}}$ where $\degmean$ is normalized degree matrix $\frac{1}{n}\dg{A\one}$.
    Representing $L$ as $\frac{1}{n} \degmean^{-\frac{1}{2}} A \degmean^{-\frac{1}{2}}$ ensures that all entries in $A$ and $\degmean$ lies in a compact domain, which is crucial when we extending the approximation proof to the graphon case.

    Now we show how $Lx = \frac{1}{n} \degmean^{-\frac{1}{2}} A \degmean^{-\frac{1}{2}} x$ is implemented.
    First, it is easy to see that \IGN{} can calculate exactly $\degmean$ using equivariant layers. Second, as approximating a) $f(a, b) = ab$ and b) $f(a) = \frac{1}{\sqrt{a}}$ can achieved by MLP on compact domain, approximating $\degmean^{-\frac{1}{2}} A \degmean^{-\frac{1}{2}}$ can also achieved by \IGN{} layers according to \Cref{lem:ign-multiplication}. Third, we need to show $\frac{1}{n} \degmean^{-\frac{1}{2}} A \degmean^{-\frac{1}{2}}x$ can also be implemented. This is proved in \Cref{lem:ign-multiplication}.

    There are two main functions we need to approximate with MLP:
    a) $f(x) = 1/\sqrt{a}$ and b) $ f(a, b) = ab$. 

    For a) the input is entries of $\degmean$ which lie in $[0,1]$. By classical universal approximation theorem \cite{cybenko1989approximation,hornik1989multilayer}, we know MLP can approximate a) arbitrarily well.

    For b) the input is $(\degmean^{-1/2}, A)$ for normalized adjacency matrix calculation, and $(L, x)$ for graph signal convolution.
    
    To ensure the uniform approximation, we need to ensure all of them lie in a compact domain. This is indeed the case as all entries in $\degmean, A, x$ are all upper bounded

    \begin{enumerate}
    \item  every entry in $A$ is either 0 or 1 therefore lies in a compact domain.
    \item similarly, all entries $\degmean$ lies in $[c_\text{min}, 1]$ by AS\ref{as:min-deg-lower-bound}, and therefore $\degmean^{-\frac{1}{2}}$ also lies in a compact domain. As $L(A)$ is the multiplication of $\degmean^{-1/2}, A, \degmean^{-1/2}$, every  entry of $L(A)$ also lies in compact domain.
    \item input signal $x$ has bounded $l_\infty$-norm by assumption AS\ref{as:x-compact}.
    \item  all coefficient for operators is upper bounded and independent from $n$ by AS\ref{as:filter-bound}.
    \end{enumerate}

    Since we showed the $L(A)x$ can be approximated arbitrarily well by IGN, repeating such processes and leveraging the fact that $L$ has bounded spectral norm, we can then approximate $L^k(A)x$ up to $\epsilon$ precision. The errors $\epsilon$ depend on the approximation error of the MLP to the relevant function, the previous errors, and uniform bounds as well as uniform continuity of the approximated functions.
    \end{proof}

    \begin{restatable}[]{theorem}{cignapproxcsgnn}
    \label{thm:cign-approx-csgnn}
    Given $\epsilon$, and a \sGNN{} c$\tn{SGNN}_{\theta_1}$, there exists a continuous \IGN{} c$\tn{IGN}_{\theta_2} $such that it approximates c$\tn{SGNN}_{\theta_1}$ on a compact set (input feature $X$) arbitrarily well.
    \end{restatable}

    \begin{proof}

In the continuous case, $Lx = \frac{1}{n} \degmean^{-\frac{1}{2}} A \degmean^{-\frac{1}{2}} x$ in the discrete case will be replaced with $D(W)^{-\frac{1}{2}} W D(W)^{-\frac{1}{2}} X$ where $D(W)$ is a diagonal graphon defined to be $D(W)(i, i) = \int_0^1 W(i, j)dj$. 

We show that all items listed in proof of \Cref{thm:ign-approx-sgnn} still holds in the continuous case
    \begin{itemize}
      \item we consider the $W$ instead in the continuous case, where all entries still lies in a compact domain $[0,1]$.
      \item similarly all entries of the continuous analog of $\degmean, \degmean^{-\frac{1}{2}}$, and $T(W)$ also lies in a compact domain according to AS\ref{as:min-deg-lower-bound}.
      \item the statements about input signal $X$ and the coefficient for linear equivariant operators also holds in the continuous setting.
    \end{itemize}
    Therefore we conclude the proof. Now we are ready to prove that those IGN that can approximate SGNN well is a subset of \smallIGN{}.
    \end{proof}

\begin{lemma}
\label{lemma:sgnn-linear-layer-commutes}
With slight abuse of notation, let $\inducedEW$ be graphon of chessboard pattern. Let $\inducedEX$ be a graphon signal with 1D chessboard pattern.
$S_n \circ \inducedEW \inducedEX = (S_n \inducedEW) (S_n \inducedEX)$. %
\end{lemma}
\begin{proof}
 Since $ S_n \circ \inducedEW \inducedEX = S_n \circ\int_{j\in [0, 1]} \inducedEW(i, j) \inducedEX(j) dj = \left(..., \frac{1}{\sqrt{n}}\int_{j\in [0, 1]} \inducedEW(\frac{i}{n}, j) \inducedEX(j), ...\right)$, it suffices to analyze $i$-th component $\frac{1}{\sqrt{n}}\int_{j\in [0, 1]} \inducedEW(\frac{i}{n}, j) \inducedEX(j)$.

 Since $\inducedEW, \inducedEX$ are of chessboard pattern, we can replace integral with summation.
 \begin{align*}
 S_n \circ \inducedEW \inducedEX(i)
 & =  \frac{1}{\sqrt{n}}\int_{j\in [0, 1]} \inducedEW(\frac{i}{n}, j) \inducedEX(j) \\
  & = \frac{1}{\sqrt{n}} \frac{1}{n} \sum_{j \in [n]} \inducedEW(\frac{i}{n}, \frac{j}{n}) \inducedEX(\frac{j}{n}) \\
  & = \sum_{j \in [n]} \frac{1}{n} \inducedEW(\frac{i}{n}, \frac{j}{n}) ( S_n \inducedEX ) (j) \\
  & = \sum (S_n \inducedEW )(i, j) ( S_n \inducedEX ) (j) \\
  & = \left((S_n \inducedEW ) ( S_n \inducedEX ) \right)(i)
\end{align*}
Which concludes the proof. Note that our proof does make use of the property of multiplication between two numbers.
\end{proof}
\begin{remark}
\label{remark: replace-multiplication}
The whole proof only relies on that $\inducedEW$ and $\inducedEX$ have checkerboard patterns. Therefore replacing the multiplication with other operations (such as a MLP) will still hold.
\end{remark}

\smalligngcn*
\begin{proof}
To prove this, we only need to show that $S_n\Phi_{c, \tn{approx}}([\inducedEW, \inducedEf]) = \\ \Phi_{d, \tn{approx}} S_n([\inducedEW, \inducedEf])$. Here  $\Phi_{c, \tn{approx}}$ and $\Phi_{d, \tn{approx}}$ denotes those specific IGN in \Cref{thm:ign-approx-sgnn,thm:cign-approx-csgnn} constructed to approximate SGNN. 

To build up some intuition, let $\Phi_{\tn{SGNN}}$ denotes the \sGNN{} that $\Phi_{\tn{approx}}$ approximates. it is easy to see that $S_n\Phi_{c, \tn{SGNN}}([\inducedEW, \inducedEf]) = \Phi_{d, \tn{SGNN}}S_n([\inducedEW, \inducedEf])$  due to \Cref{lemma:sgnn-linear-layer-commutes} and \Cref{lem:property-of-T}.2. 
To show the same holds for $\Phi_{\tn{approx}}$, note that the only difference between $\inducedEW\inducedEf$ implemented by SGNN and approximated by  $\Phi_{\tn{approx}}$ is that $\Phi_{\tn{approx}}$ use MLP to simulate multiplication between numbers. According to \Cref{remark: replace-multiplication}, the approximated version of $\inducedEW\inducedEf$ still commutes with $S_n$.    %

Since nonlinear layer $\sigma$ in $\Phi_{\tn{approx}}$ also commutes with $S_n$ according to \Cref{lem:property-of-T}.2, we can combine the result above and conclude that $\Phi_{\tn{approx}}$ commutes with $S_n$.  Therefore $\Phi_{\tn{approx}}$ belongs to \smallIGN{}, which finishes the proof.
\end{proof}

\chapter{On the Connection Between MPNN and Graph Transformer}\label{chapter:mpnn_gt}
\setcounter{assumption}{0}
\section{Introduction}
In this chapter, we study the connection between MPNN and Graph Transformer. 
MPNN (Message Passing Neural Network) \cite{gilmer2017neural} has been the leading architecture for processing graph-structured data. Recently, transformers in natural language processing \cite{vaswani2017attention,kalyan2021ammus} and vision \cite{d2021convit,han2022survey} have extended their success to the domain of graphs. There have been several pieces of work \cite{ying2021transformers,wu2021representing,kreuzer2021rethinking,rampavsek2022recipe, kim2022pure} showing that with careful position embedding \cite{lim2022sign}, graph transformers (GT) can achieve compelling empirical performances on large-scale datasets and start to challenge the dominance of MPNN. 
\begin{figure}[t!]
  \centering
  \includegraphics[width=1\linewidth]{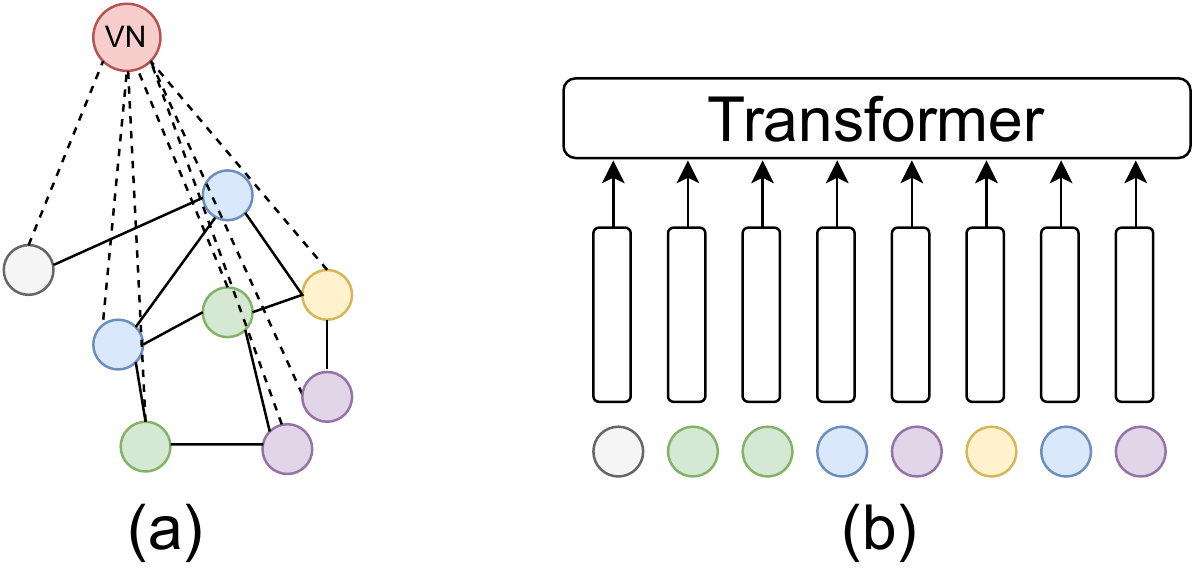}
\caption{MPNN + VN and Graph Transformers.}
\label{fig:mpnn+gt}
\end{figure}

\begin{table}[th!]
\centering
\caption{Summary of approximation result of MPNN + VN on self-attention layer. $n$ is the number of nodes and $d$ is the feature dimension of node features. The dependency on $d$ is hidden. }
\tabtopvspace
\label{tab:theoretical-result}
\resizebox{1\textwidth}{!}{
\begin{tabular}{@{}lllll@{}}
\toprule
 & Depth & Width & Self-Attention & Note \\ \midrule
 \Cref{thm:constant-depth-constant-width} & $O(1)$ & $O(1)$ & Approximate & Approximate self attention in Performer \cite{choromanski2020rethinking} \\ 
 \Cref{thm:constant-depth} & $O(1)$ & $O(n^d)$ & Full & Leverage the universality of equivariant \DS{} \\
\Cref{thm:constant-width}  & $O(n)$ & $O(1)$ & Full & Explicit construction, strong assumption on $\mc{X}$ \\
\Cref{prop:gat-v2-selection} & $O(n)$ & $O(1)$ & Full & Explicit construction, more relaxed (but still strong) assumption on $\mc{X}$ \\ \bottomrule
\end{tabular}
}
\end{table}

MPNN imposes a sparsity pattern on the computation graph and therefore enjoys linear complexity. It however suffers from well-known over-smoothing \cite{li2018deeper,oono2019graph,cai2020note} and over-squashing \cite{alon2020bottleneck,topping2021understanding} issues, limiting its usage on long-range modeling tasks where the label of one node depends on features of nodes far away. GT relies purely on position embedding to encode the graph structure and uses vanilla transformers on top. \footnote{GT in this chapter refers to the practice of tokenizing graph nodes and applying standard transformers on top \cite{ying2021transformers,kim2022pure}. There exists a more sophisticated GT \cite{kreuzer2021rethinking} that further conditions attention on edge types but it is not considered in this chapter. } %
It models all pairwise interactions directly in one layer, making it computationally more expensive. Compared to MPNN, GT shows promising results on tasks where modeling long-range interaction is the key, but the quadratic complexity of self-attention in GT limits its usage to graphs of medium size. Scaling up GT to large graphs remains an active research area \cite{wunodeformer}. %

Theoretically, it has been shown that graph transformers can be powerful graph learners \cite{kim2022pure}, i.e., graph transformers with appropriate choice of token embeddings have the capacity of approximating linear permutation equivariant basis, and therefore can approximate 2-IGN (Invariant Graph Network), a powerful architecture that is at least as expressive as MPNN \cite{maron2018invariant}. This raises an important question that \textit{whether GT is strictly more powerful than MPNN}. Can we approximate GT with MPNN?

One common intuition of the advantage of GT over MPNN is its ability to model long-range interaction more effectively. However, from the MPNN side, one can resort to a simple trick to escape locality constraints for effective long-range modeling: the use of an additional \emph{virtual node (VN)} that connects to all input graph nodes. On a high level, MPNN + VN augments the existing graph with one virtual node, which acts like global memory for every node exchanging messages with other nodes. Empirically this simple trick has been observed to improve the MPNN and has been widely adopted \cite{gilmer2017neural,hu2020open,hu2021ogb} since the early beginning of MPNN \cite{gilmer2017neural, battaglia2018relational}. However, there is very little theoretical study of MPNN + VN \cite{hwanganalysis}. 

In this work, we study the theoretical property of MPNN + VN, and its connection to GT. We systematically study the representation power of MPNN + VN, both for certain approximate self-attention and for the full self-attention layer, and provide a depth-width trade-off, summarized in \Cref{tab:theoretical-result}. In particular, 
\begin{itemize}
\item With $O(1)$ depth and $O(1)$ width, MPNN + VN can approximate one self-attention layer of Performer \cite{choromanski2020rethinking} and Linear Transformer \cite{katharopoulos2020transformers}, a type of linear transformers \cite{tay2020efficient}. %
\item %
Via a link between MPNN + VN with \DS{} \cite{zaheer2017deep}, we prove MPNN + VN with $O(1)$ depth and $O(n^d)$ width ($d$ is the input feature dimension) is permutation equivariant universal, implying it can approximate self-attention layer and even full-transformers.  %

\item %
Under certain assumptions on node features, we prove an explicit construction of $O(n)$ depth $O(1)$ width MPNN + VN approximating 1 self-attention layer arbitrarily well on graphs of size $n$. 
Unfortunately, the assumptions on node features are rather strong, and whether we can alleviate them will be an interesting future direction to explore. %

\item Empirically, we show 1) that MPNN + VN works surprisingly well on the recently proposed LRGB (long-range graph benchmarks) datasets \cite{dwivedi2022long}, which arguably require long-range interaction reasoning to achieve strong performance 2) our implementation of MPNN + VN is able to further improve the early implementation of MPNN + VN on OGB datasets and 3) MPNN + VN outperforms Linear Transformer \cite{katharopoulos2020transformers} and MPNN on the climate modeling task. %
\end{itemize}

\section{Preliminaries}
We denote $\mX\in \R{n \times d}$ the concatenation of graph node features and positional encodings, where node $i$ has feature $\vx_i \in \R{d}$. When necessary, we use $\vx^{(l)}_j$ to denote the node $j$'s feature at depth $l$.  Let $\mc{M}$ be the space of multisets of vectors in $\R{d}$. We use $\mc{X} \subseteq \R{n\times d} $ to denote the space of node features and the $\mc{X}_i$ be the projection of $\mc{X}$ on $i$-th coordinate. $\norm{\cdot}$ denotes the 2-norm. $[\vx, \vy, \vz]$ denotes the concatenation of $\vx, \vy, \vz$. $[n]$ stands for the set $\{1, 2, ..., n\}$.

\begin{definition}[attention]
We denote key and query matrix as $\mW_K, \mW_Q\in \R{d \times d'}$, and value matrix as $\mW_V \in \R{d \times d}$ \footnote{For simplicity, we assume the output dimension of self-attention is the same as the input dimension. All theoretical results can be extended to the case where the output dimension is different from $d$.}. Attention score between two vectors $\vu, \vv \in \R{d \times 1}$ is defined as $\alpha(\vu, \vv) = \text{softmax}(\vu^T \mW_Q(\mW_K)^T\vv)$. We denote $\mc{A}$ as the space of attention $\alpha$ for different $\mW_Q, \mW_K, \mW_V$. We also define unnormalized attention score $\alpha'(\cdot, \cdot)$ to be $\alpha'(\vu, \vv) = \vu^T \mW_Q(\mW_K)^T\vv$.
Self attention layer is a matrix function $\mL: \R{n\times d} \rightarrow \R{n\times d}$  of the following form: $\mL(\mX) = \text{softmax}(\mX\mW_Q(\mX\mW_K)^T)\mX \mW_V$. %
\end{definition}

\subsection{MPNN Layer}
\begin{definition}[MPNN layer \cite{gilmer2017neural}]
An MPNN layer on a graph $G$ with node features $\vx^{(k)}$ at $k$-th layer and edge features $\ve$ is of the following form
\begin{equation*}
\vx_i^{(k)}=\gamma^{(k)}\left(\vx_i^{(k-1)}, \pool_{j \in \mc{N}(i)} \phi^{(k)}\left(\vx_i^{(k-1)}, \vx_j^{(k-1)}, \ve_{j, i}\right)\right)
\end{equation*}

Here $\gamma: \mb{R}^d \times \mb{R}^{d'} \rightarrow \mb{R}^d$ is update function,
$\phi:\mb{R}^d \times \mb{R}^{d} \times \mb{R}^{d_e} \rightarrow \mb{R}^{d'}$ is message function where $d_e$ is the edge feature dimension, 
$\pool: \mc{M}  \rightarrow \mb{R}^d$ is permutation invariant aggregation function 
and $\mathcal{N}(i)$ is the neighbors of node $i$ in $G$.
Update/message/aggregation functions are usually parametrized by neural networks. For graphs of different types of edges and nodes, one can further extend MPNN to the heterogeneous setting.  We use $1, ..., n$ to index graph nodes and $\vn$ to denote the virtual node. 
\end{definition}

\begin{definition}[heterogeneous MPNN + VN layer]\label{def-hetero-mpnn-vn-layer} 
The heterogeneous MPNN + VN layer operates on two types of nodes: 1) virtual node and 2) graph nodes, denoted as \text{vn} and \text{gn}, and three types of edges: 1) \text{vn}-\text{gn} edge and 2) \text{gn}-\text{gn} edges and 3) \text{gn}-\text{vn} edges. It has the following form

\begin{equation}
\vx_{\vn}^{(k)}=\gamma_{\vn}^{(k)}\left(\vx_i^{(k-1)}, \pool_{j \in [n]}  \phi^{(k)}_{\vngn}\left(\vx_i^{(k-1)}, \vx_j^{(k-1)}, \ve_{j, i}\right)\right)  
\end{equation}
for the virtual node, and 
\begin{equation}
\begin{split}
\vx_i^{(k)}&=\gamma_{\gn}^{(k)}(\vx_i^{(k-1)}, \pool_{j \in \mathcal{N}_1(i)}  \phi^{(k)}_{\gnvn}\left(\vx_i^{(k-1)}, \vx_j^{(k-1)}, \ve_{j, i}\right)  \\
& +  \pool_{j \in \mathcal{N}_2(i)} \phi^{(k)}_{\gngn}\left(\vx_i^{(k-1)}, \vx_j^{(k-1)}, \ve_{j, i})\right)  
\end{split}
\end{equation}
for graph node. Here $\mc{N}_1(i)$ for graph node $i$ is the virtual node and $\mc{N}_2(i)$ is the set of neighboring graph nodes.
\end{definition}
Our proof of approximating self-attention layer $\mL$ with MPNN layers does not use the graph topology. Next, we introduce a simplified heterogeneous MPNN + VN layer, which will be used in the proof. It is easy to see that setting $\phi^{(k)}_{\gngn}$ to be 0 in \Cref{def-hetero-mpnn-vn-layer} recovers the simplified heterogeneous MPNN + VN layer.

\begin{definition}[simplified heterogeneous MPNN + VN layer]
\label{def:simplified-hetero-mpnn-vn}
A simplified heterogeneous MPNN + VN layer is the same as a heterogeneous MPNN + VN layer in \Cref{def-hetero-mpnn-vn-layer} except we set $\theta_{\gngn}$ to be 0. I.e., we have
\begin{equation*}
\vx_{\vn}^{(k)}=\gamma_{\vn}^{(k)}\left(\vx_i^{(k-1)}, \pool_{j \in [n]}  \phi^{(k)}_{\vngn}\left(\vx_i^{(k-1)}, \vx_j^{(k-1)}, \ve_{j, i}\right)\right)  
\end{equation*}
for the virtual node, and
\begin{equation*}
\vx_i^{(k)}=\gamma_{\gn}^{(k)}\left(\vx_i^{(k-1)}, \pool_{j \in \mathcal{N}_1(i)}  \phi^{(k)}_{\gnvn}\left(\vx_i^{(k-1)}, \vx_j^{(k-1)}, \ve_{j, i}\right)\right)  
\end{equation*}
for graph nodes. 
\end{definition}

Intuitively, adding the virtual node (VN) to MPNN makes it easy to compute certain quantities, for example, the mean of node features (which is hard for standard MPNN unless the depth is proportional to the diameter of the graph). Using VN thus makes it easy to implement for example the mean subtraction, which helps reduce over-smoothing and improves the performance of GNN \cite{yang2020revisiting,zhao2019pairnorm}. See more connection between MPNN + VN and over-smoothing in \Cref{subsec:over-smoothing}.

\subsection{Assumptions}
We have two mild assumptions on feature space $\mathcal{X}\subset \mb{R}^{n \times d}$ and the regularity of target function $\mL$.

\begin{assumption}\label{AS-2} $ \forall i\in [n], \vx_i \in \mc{X}_i,  \norm{\vx_i} < C_1$. This implies $\mc{X}$ is compact.  
\end{assumption} 

\begin{assumption}\label{AS-3}  $\norm{\mW_Q}< C_2, \norm{\mW_K} < C_2, \norm{\mW_V} < C_2$ for target layer $\mL$. Combined with AS\ref{AS-2} on $\mathcal{X}$, this means $\alpha'(\vx_i, \vx_j)$ is both upper and lower bounded, which further implies $\sum_j e^{\alpha'(\vx_i, \vx_j)}$ be both upper bounded and lower bounded. 
\end{assumption}

\subsection{Notations} 
We provide a notation table for references. 
\begin{table}[!htbp]
\caption{Summary of important notations.} 
\begin{center}
\resizebox{\textwidth}{!}
{
\begin{tabular}{@{}l|l@{}}
    \hline
    \toprule
    Symbol & Meaning \\
    \midrule
    \midrule
    $\mX \in \mc{X} \subset \R{n \times d}$ & graph node features \\
    $\vx_i \in \R{1\times d}$ & graph node $i$'s feature \\
    $\tilde{\vx}_i \in \R{1\times d}$ & approximated graph node $i$'s feature via attention selection \\
    $\mc{M}$ & A multiset of vectors in $\R{d}$  \\
    $\mW_Q^{(l)}, \mW_K^{(l)}, \mW_V^{(l)} \in \R{d\times d'}$ & attention matrix of $l$-th self-attention layer in graph transformer \\
    $\mc{X}$ & feature space \\ %
    $\mc{X}_i$ & projection of feature space onto $i$-th coordinate \\
    $\mL^{\ds}_i$ & $i$-th linear permutation equivariant layer in \DS{} \\ 
    $\mL, \mL'$ & full self attention layer; approximate self attention layer in Performer \\  
    $\vz_{\vn}^{(l)}, \vz_{i}^{(l)}$ &  virtual/graph node feature at layer $l$ of heterogeneous MPNN + VN  \\
    $\alpha_{\vn}$ & attention score in MPNN + VN \\
    
    $\alpha(\cdot, \cdot)$ & normalized attention score \\
    $\alpha_{\gat}(\cdot, \cdot)$ & normalized attention score with \gat{} \\
    $\alpha'(\cdot, \cdot)$ & unnormalized attention score. $\alpha'(\vu, \vv) = \vu \mW_Q(\mW_K)^T\vv^T$ \\
    $\alpha'_{\gat}(\cdot, \cdot)$ & unnormalized attention score with \gat{}. $\alpha'_{\text{GATv2}}(\vu, \vv) :=  \va^T \operatorname{LeakyReLU}\left(\mW \cdot\left[\vu \| \vv\right] + \vb \right)$ \\
    $\mc{A}$ & space of attentions, where each element $\alpha \in \mc{A}$ is of form $\alpha(\vu, \vv) = \text{softmax}(\vu \mW_Q(\mW_K)^T\vv^T)$ \\
    $C_1$ & upper bound on norm of all node features $\norm{\vx_i}$ \\
    $C_2$ & upper bound on the norm of $\mW_Q, \mW_K, \mW_V$ in target $\mL$ \\
    $C_3$ & upper bound on the norm of attention weights of $\alpha_{\vn}$ when selecting $\vx_i$ \\
    \midrule
    $\gamma^{(k)}(\cdot, \cdot)$ & update function \\
    $\theta^{(k)}(\cdot, \cdot)$ & message function \\
    $\pool(\cdot)$ & aggregation function \\
    \bottomrule
\end{tabular}
}
\end{center}
\label{table:symbol_notation}
\end{table}

\section{$O(1)$-Depth $O(1)$-Width MPNN + VN for Unbiased Approximation of Attention}
\label{sec:shallow-narrow-attention}
The standard self-attention takes $O(n^2)$ computational time, therefore not scalable for large graphs. Reducing the computational complexity of self-attention in Transformer is active research \cite{tay2020efficient}. 
In this section, we consider self-attention in a specific type of efficient transformers, Performer \cite{choromanski2020rethinking} and Linear Transformer \cite{katharopoulos2020transformers}.

One full self-attention layer $\mL$ is of the following form
\begin{equation}
\label{equ:attention-kernel-repr}
\vx_i^{(l+1)}=\sum_{j=1}^n \frac{\kappa\left(\mW_Q^{(l)} \vx_i^{(l)}, \mW_K^{(l)} \vx_j^{(l)}\right)}{\sum_{k=1}^n \kappa\left(\mW_Q^{(l)} \vx_i^{(l)}, \mW_K^{(l)} \vx_k^{(l)}\right)} \cdot\left(\mW_V^{(l)} \vx_j^{(l)}\right)
\end{equation}

where $\kappa: \mb{R}^d \times \mb{R}^d \rightarrow \mb{R}$ is the softmax kernel $\kappa(\vx, \vy):=\exp(\vx^T\vy)$. The kernel function can be approximated via $\kappa(\vx, \vy) = \ip{\Phi(\vx), \Phi(\vy)}_{\mc{V}} \approx \phi(\vx)^T\phi(\vy)$ where the first equation is by Mercer's theorem and $\phi(\cdot): \mb{R}^d \rightarrow \mb{R}^m $ is a low-dimensional feature map with random transformation. For Performer \cite{choromanski2020rethinking}, the choice of $\phi$ is taken as $\phi(\vx)=\frac{\exp \left(\frac{-\|\vx\|_2^2}{2}\right)}{\sqrt{m}}\left[\exp \left(\vw_1^{T} \vx\right), \cdots, \exp \left(\vw_m^{T} \vx\right)\right]$ where $\vw_k \sim \mathcal{N}\left(0, I_d\right)$ is i.i.d sampled random variable. For Linear Transformer \cite{katharopoulos2020transformers}, $\phi(\vx)=\operatorname{elu}(\vx)+1$. 

By switching $\kappa(\vx, \vy)$ to be $\phi(\vx)^T\phi(\vy)$, and denote $\vq_i=\mW_Q^{(l)} \vx_i^{(l)}, \vk_i=\mW_K^{(l)} \vx_i^{(l)} \text { and } \vv_i=\mW_V^{(l)} \vx_i^{(l)}$, the approximated version of \Cref{equ:attention-kernel-repr} by Performer and Linear Transformer becomes 
\begin{equation}
\begin{split}
\label{equ:modified-layer}
\vx_i^{(l+1)}&=\sum_{j=1}^n \frac{\phi\left(\vq_i\right)^{T} \phi\left(\vk_j\right)}{\sum_{k=1}^n \phi\left(\vq_i\right)^{T} \phi\left(\vk_k\right)} \cdot \vv_j \\
& =\frac{\left(\phi\left(\vq_i\right)^T \sum_{j=1}^n \phi\left(\vk_j\right) \otimes \vv_j\right)^T}{\phi\left(\vq_i\right)^{T} \sum_{k=1}^n \phi\left(\vk_k\right)}. \\
\end{split}
\end{equation}
where we use the matrix multiplication association rule to derive the second equality. 

The key advantage of \Cref{equ:modified-layer} is that $\sum_{j=1}^n \phi\left(\vk_j\right)$ and $\sum_{j=1}^n \phi(\vk_j) \otimes \vv_j$ can be approximated by the virtual node, and shared for all graph nodes, using only $O(1)$ layers of MPNNs.  
We denote the self-attention layer of this form in \Cref{equ:modified-layer} as $\mL_{\text{Performer}}$. Linear Transformer differs from Performer by choosing a different form of $\phi(\vx)=\operatorname{Relu}(\vx)+1$ in its self-attention layer $\mL_{\text{Linear-Transformer}}$.  

In particular, the VN will approximate $\sum_{j=1}^n \phi\left(\vk_j\right)$ and $\sum_{j=1}^n \phi\left(\vk_j\right) \otimes \vv_j$, and represent it as its feature. Both $\phi\left(\vk_j\right)$ and $\phi\left(\vk_j\right) \otimes \vv_j$ can be approximated arbitrarily well by an MLP with constant width (constant in $n$ but can be exponential in $d$) and depth. Note that $\phi(\vk_j) \otimes \vv_j \in \mb{R}^{dm}$ but can be reshaped to 1 dimensional feature vector.

More specifically, the initial feature for the virtual node is $\bm{1}_{(d+1)m}$, where $d$ is the dimension of node features and $m$ is the number of random projections $\omega_i$.  
Message function + aggregation function for virtual node 
$\pool \phi_{\vngn}: \mb{R}^{(d+1)m} \times \mc{M} \rightarrow \mb{R}^{(d+1)m} $ is
\begin{equation}
\label{eq:vn-gn}
\begin{split}
 & \pool_{j \in [n]} \phi_{\vngn}^{(k)}(\cdot, \{\vx_i\}_i)  =  [\sum_{j=1}^n \phi\left(\vk_j\right), \\ & \tooned{\sum_{j=1}^n \phi\left(\vk_j\right) \otimes \vv_j}]
 \end{split}
\end{equation}
 where $\tooned{\cdot}$ flattens a 2D matrix to a 1D vector in raster order. This function can be arbitrarily approximated  by MLP. Note that the virtual node's feature dimension is $(d+1)m$ (where recall $m$ is the dimension of the feature map $\phi$ used in the linear transformer/Performer), which is larger than the dimension of the graph node $d$. This is consistent with the early intuition that the virtual node might be overloaded when passing information among nodes. %
 The update function for virtual node $\gamma_{\vn}: $ $\mb{R}^{(d+1)m} \times \mb{R}^{(d+1)m} \rightarrow \mb{R}^{(d+1)m}$ is just coping the second argument, which can be exactly implemented by MLP. 

VN then sends its message back to all other nodes, where each graph node $i$ applies the update function $\gamma_{\gn}: \mb{R}^{(d+1)m} \times \mb{R}^d \rightarrow \mb{R}^d$ of the form
\begin{equation}
\label{eq:gn}
\begin{split}
 & \gamma_{\gn} (\vx_i,  [\sum_{j=1}^n \phi\left(\vk_j\right), \tooned{\sum_{j=1}^n \phi\left(\vk_j\right) \otimes \vv_j}])\\ 
 & = \frac{\left(\phi\left(\vq_i\right) \sum_{j=1}^n \phi\left(\vk_j\right) \otimes \vv_j\right)^T}{\phi\left(\vq_i\right)^{T} \sum_{k=1}^n \phi\left(\vk_k\right)} 
 \end{split}
\end{equation}
 to update the graph node feature. 

As the update function $\gamma_{\gn}$ can not be computed exactly in MLP, what is left is to show that error induced by using MLP to approximate $\pool \phi_{\vngn}$ and $\gamma_{\gn}$ in \Cref{eq:vn-gn} and \Cref{eq:gn} can be made arbitrarily small. %

\begin{restatable}{theorem}{doubleconstant}
\label{thm:constant-depth-constant-width}
Under the AS\ref{AS-2} and AS\ref{AS-3}, MPNN + VN of $O(1)$ width and $O(1)$ depth can approximate $\mL_{\text{Performer}}$ and $\mL_{\text{Linear-Transformer}}$ arbitrarily well. 
\end{restatable}

\begin{proof}
We first prove the case of $\mL_{\text{Performer}}$.
We can decompose our target function as the composition of $\pool_{j \in [n]} \phi_{\vngn}^{(k)}$, $\gamma_{\gn}$ and $\phi$. 
By the uniform continuity of the functions, it suffices to show that 1) we can approximate $\phi$, 2) we can approximate operations in $\gamma_{\gn}$ and $\pool \phi_{\vngn}$ arbitrarily well on the compact domain, and 3) the denominator $\phi\left(\vq_i\right)^{T} \sum_{k=1}^n \phi\left(\vk_k\right)$ is uniformly lower bounded by a positive number for any node features in $\mc{X}$.

For 1), each component of $\phi$ is continuous and all inputs $\vk_j, \mathbf{q}_j$ lie in the compact domain so $\phi$ can be approximated arbitrarily well by MLP with $O(1)$ width and $O(1)$ depth \cite{cybenko1989approximation}. 

For 2), we need to approximate the operations in $\gamma_{\gn}$ and $\pool \phi_{\vngn}$, i.e., approximate multiplication, and vector-scalar division arbitrarily well.  As all those operations are continuous, it boils down to showing that all operands lie in a compact domain. By assumption AS\ref{AS-2} and AS\ref{AS-3} on $\mW_Q, \mW_K, \mW_V$ and input feature $\mc{X}$, we know that $\vq_i, \vk_i, \vv_i$ lies in a compact domain for all graph nodes $i$. As $\phi$ is continuous, this implies that $\phi(\vq_i), \sum_{j=1}^n \phi(\vk_j) \otimes \vv_j$ lies in a compact domain ($n$ is fixed), therefore the numerator lies in a compact domain. Lastly, since all operations do not involve $n$, the depth and width are constant in $n$.  

For 3), it is easy to see that $\phi\left(\vq_i\right)^{T} \sum_{k=1}^n \phi\left(\vk_k\right)$ is always positive.  We just need to show that the denominator is bound from below by a positive constant. For Performer, $\phi(\vx)=\frac{\exp \left(\frac{-\|\vx\|_2^2}{2}\right)}{\sqrt{m}}\left[\exp \left(\vw_1^{T} \vx\right), \cdots, \exp \left(\vw_m^{T} \vx\right)\right]$ where $\vw_k \sim \mathcal{N}\left(0, I_d\right)$. As all norm of input $\vx$ to $\phi$ is upper bounded by AS\ref{AS-2}, $\exp(\frac{-\|\vx\|_2^2}{2})$ is lower bounded. As $m$ is fixed, we know that $\norm{\vw^T_i \vx} \leq \norm{\vw_i} \norm{\vx}$, which implies that $\vw_i^T \vx$ is lower bounded by $-\norm{\vw_i} \norm{\vx}$ %
which further implies that $\exp(\vw_i^T \vx)$ is lower bounded. This means that $\phi\left(\vq_i\right)^{T} \sum_{k=1}^n \phi\left(\vk_k\right)$ is lower bounded. 

For Linear Transformer, the proof is essentially the same as above. We only need to show that $\phi(\vx)=\operatorname{elu}(\vx)+1$ is continuous and positive, which is indeed the case. 
\end{proof}

Besides Performers, there are many other different ways of obtaining linear complexity. In \Cref{subsec:ohter-linear-transformer}, we discuss the limitation of MPNN + VN on approximating other types of efficient transformers such as Linformer \cite{wang2020linformer} and Sparse Transformer \cite{child2019generating}.  

\section{$O(1)$ Depth $O(n^d)$ Width MPNN + VN}
\label{sec:shadow-wide-mpnn}

\begin{figure}[t!]
  \centering
  \includegraphics[width=1\linewidth]{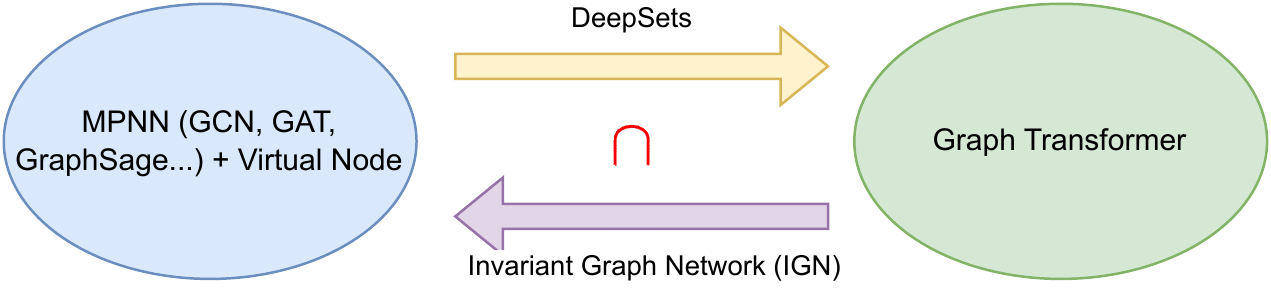}
\caption{The link between MPNN and GT is drawn via \DS{} in \Cref{sec:shadow-wide-mpnn} of this chapter and Invariant Graph Network (IGN) in \cite{kim2022pure}. Interestingly, IGN is a generalization of \DS{} \cite{maron2018invariant}. }
\label{fig:mpnn+gt}
\end{figure}

We have shown that the MPNN + VN can approximate self-attention in Performer and Linear Transformer using only $O(1)$ depth and $O(1)$ width. 
One may naturally wonder whether MPNN + VN can approximate the self-attention layer in the \textit{full} transformer. In this section,
we show that MPNN + VN with $O(1)$ depth (number of layers), but with $O(n^d)$ width, can approximate 1 self-attention layer (and full transformer) arbitrarily well. 

The main observation is that MPNN + VN is able to exactly simulate (not just approximate) equivariant \DS{} \cite{zaheer2017deep}, which is proved to be universal in approximating any permutation invariant/equivariant maps \cite{zaheer2017deep,segol2019universal}. Since the self-attention layer is permutation equivariant, this implies that MPNN + VN can approximate the self-attention layer (and full transformer) with $O(1)$ depth and $O(n^d)$ width following a result on \DS{} from \cite{segol2019universal}.

We first introduce the permutation equivariant map, equivariant \DS{}, and permutation equivariant universality.

\begin{definition}[permutation equivariant map]
A map $\mF: \mb{R}^{n \times k} \rightarrow \mb{R}^{n \times l}$ satisfying $\mF(\sigma \cdot \mX)=\sigma \cdot \mF(\mX)$ for all $\sigma \in S_n$ and $\mX \in \mb{R}^{n \times d}$ is called permutation equivariant.
\end{definition}

\begin{definition}[equivariant \DS{} of  \cite{zaheer2017deep}]
Equivariant \DS{} has the following form $\mF(\mX)=\mL_m^{\ds} \circ \nu \circ \cdots \circ \nu \circ \mL^{\ds}_1(\mX)$, where $\mL^{\ds}_i$ is a linear permutation equivariant layer and $\nu$ is a nonlinear layer such as ReLU. 
The linear permutation equivariant layer in \DS{} has the following form $\mL^{\ds}_i(\mX)=\mX \boldsymbol{A}+\frac{1}{n} \mathbf{1 1}^T \mX \mB+\mathbf{1} \boldsymbol{c}^T$, where $\mA, \mB \in \mb{R}^{d_{i} \times d_{i+1}}$, $\vc \in \mb{R}^{d_{i+1}}$ is the weights and bias in layer $i$, and $\nu$ is ReLU. %
\end{definition}

\begin{definition}[permutation equivariant universality]
\label{def:ds-universality}
Given a compact domain $\mc{X}$ of $\mb{R}^{n \times d_{\text{in}}}$, permutation equivariant universality of a model $\mF: \mb{R}^{n \times d_{\text{in}}} \rightarrow \mb{R}^{n \times d_{\text{out}}}$ means that for every permutation equivariant continuous function $\mH: \mb{R}^{n \times d_{\text{in}}} \rightarrow \mb{R}^{n \times d_{\text{out}}}$ defined over $\mc{X}$, and any $\epsilon>0$, there exists a choice of $m$ (i.e., network depth), $d_i$ (i.e., network width at layer $i$) and the trainable parameters of $\mF$ so that $\|\mH(\mX)-\mF(\mX)\|_{\infty}<\epsilon$ for all $\mX \in \mc{X}$.
\end{definition}

The universality of equivariant \DS{} is stated as follows. %

\begin{theorem}[\cite{segol2019universal}]
\label{thm:ds-universality}
\DS{} with constant layer is universal. Using ReLU activation the width $\omega:=\text{max}_i d_i $ ($d_i$ is the width for $i$-th layer of \DS{}) required for universal permutation equivariant
network satisfies $\omega \leq d_{\text{out}}+d_{\text{in}}+\left(\begin{array}{c} n+d_{\text{in}} \\ d_{\text{in}} \end{array}\right) = O(n^{d_\text{in}})$.
\end{theorem}

We are now ready to state our main theorem. 

\begin{restatable}{theorem}{constantdepth}
\label{thm:constant-depth} 
MPNN + VN can simulate (not just approximate) equivariant \DS{}: $\mb{R}^{n \times d} \rightarrow \mb{R}^{n \times d}$. The depth and width of MPNN + VN needed to simulate \DS{} is up to a constant factor of the depth and width of \DS{}. 
This implies that MPNN + VN of $O(1)$ depth and $O(n^d)$ width is permutation equivariant universal, and can approximate self-attention layer and transformers arbitrarily well. 
\end{restatable}

\begin{table}[t!]
    \caption{Baselines for \pepfunc (graph classification) and \pepstruct (graph regression). The performance metric is Average Precision (AP) for classification and MAE for regression. \textbf{Bold}: Best score.} %
    \label{tab:experiments_peptides}
    \tabtopvspace
    \begin{adjustwidth}{-2.5 cm}{-2.5 cm}\centering
   \resizebox{1\textwidth}{!}{
    \setlength\tabcolsep{4pt} %
    \begin{tabular}{l c c c c c}\toprule
    \multirow{2}{*}{\textbf{Model}} & \multirow{2}{*}{\textbf{\# Params.}} & \multicolumn{2}{c}{\pepfunc} & \multicolumn{2}{c}{\pepstruct} \\ \cmidrule(lr){3-4} \cmidrule(lr){5-6}
                                    &                                      & \textbf{Test AP before VN}  & \textbf{Test AP after VN $\uparrow$}                                    & \textbf{Test MAE before VN} & \textbf{Test MAE after VN $\downarrow$} \\\midrule
    GCN                             & 508k                                 & 0.5930$\pm$0.0023            & 0.6623$\pm$0.0038                                                       & 0.3496$\pm$0.0013           & \first{0.2488$\pm$0.0021} \\
    GINE                            & 476k                                 & 0.5498$\pm$0.0079            & 0.6346$\pm$0.0071                                                       & 0.3547$\pm$0.0045           & 0.2584$\pm$0.0011 \\
    GatedGCN                        & 509k                                 & 0.5864$\pm$0.0077            & 0.6635$\pm$0.0024                                                       & 0.3420$\pm$0.0013           & 0.2523$\pm$0.0016 \\
    GatedGCN+RWSE                   & 506k                                 & 0.6069$\pm$0.0035            & \first{0.6685$\pm$0.0062}                                               & 0.3357$\pm$0.0006           & 0.2529$\pm$0.0009 \\ \midrule
    Transformer+LapPE               & 488k                                 & 0.6326$\pm$0.0126            & -                                                                       & 0.2529$\pm$0.0016           & - \\
    SAN+LapPE                       & 493k                                 & 0.6384$\pm$0.0121            & -                                                                       & 0.2683$\pm$0.0043           & - \\
    SAN+RWSE                        & 500k                                 & 0.6439$\pm$0.0075            & -                                                                       & 0.2545$\pm$0.0012           & - \\
    \bottomrule
    \end{tabular}
    }
    \end{adjustwidth}
\end{table}

\begin{proof}
Equivariant \DS{} has the following form $\mF(\mX)=\mL^{\ds}_m \circ \nu \circ \cdots \circ \nu \circ \mL^{\ds}_1(\mX)$, where $\mL^{\ds}_i$ is the linear permutation equivariant layer and $\nu$ is an entrywise nonlinear activation layer.
Recall that the linear equivariant layer has the form $\mL^{\ds}_i(\mX)=\mX \boldsymbol{A}+\frac{1}{n} \mathbf{1 1}^T \mX \mB+\mathbf{1} \boldsymbol{c}^T$. 
As one can use the same nonlinear entrywise activation layer $\nu$ in MPNN + VN, it suffices to prove that MPNN + VN can compute linear permutation equivariant layer $\mL^{\ds}$. Now we show that 2 layers of MPNN + VN can exactly simulate any given linear permutation equivariant layer $\mL^{\ds}$.

Specifically, at layer 0, we initialized the node features as follows: The VN node feature is set to 0, while the node feature for the $i$-th graph node is set up as $\vx_i \in \mathbb{R}^d$.  

At layer 1: VN node feature is $\frac{1}{n} \mathbf{1 1}^T \mX$, average of node features. The collection of features over $n$ graph node feature is $\mX \mA$. 
We only need to transform graph node features by a linear transformation, and set the VN feature as the average of graph node features in the last iteration. Both can be exactly implemented in \Cref{def:simplified-hetero-mpnn-vn} of simplified heterogeneous MPNN + VN. %

At layer 2: VN node feature is set to be 0, and the graph node feature is $\mX \boldsymbol{A}+\frac{1}{n} \mathbf{1 1}^T \mX \mB+\mathbf{1} \boldsymbol{c}^T$. Here we only need to perform the matrix multiplication of the VN feature with $\mB$, as well as add a bias $\vc$. This can be done by implementing a linear function for $\gamma_{\gn}$. %

It is easy to see the width required for MPNN + VN to simulate \DS{} is constant.
Thus, one can use 2 layers of MPNN + VN to compute linear permutation equivariant layer $\mL^{\ds}_i$, which implies that MPNN + VN can simulate 1 layer of \DS{} exactly with constant depth and constant width (independent of $n$). Then by the universality of \DS{}, stated in \Cref{thm:ds-universality}, we conclude that MPNN + VN is also permutation equivariant universal, which implies that the constant layer of MPNN + VN with $O(n^d)$ width is able to approximate any continuous equivariant maps. As the self-attention layer $\mL$ and full transformer are both continuous and equivariant, they can be approximated by MPNN + VN arbitrarily well.      
\end{proof}
Thanks to the connection between MPNN + VN with \DS{}, there is no extra assumption on $\mc{X}$ %
except for being compact. The drawback on the other hand is that the upper bound on the computational complexity needed to approximate the self-attention with wide MPNN + VN is worse than directly computing self-attention when $d>2$. 

\section{$O(n)$ Depth $O(1)$ Width MPNN + VN}
\label{sec:deep-vn}

The previous section shows that we can approximate a full attention layer in Transformer using MPNN with $O(1)$ depth but $O(n^d)$ width where $n$ is the number of nodes and $d$ is the dimension of node features. 
In practice, it is not desirable to have the width depend on the graph size. %

In this section, we hope to study MPNN + VNs with $O(1)$ width and their ability to approximate a self-attention layer in the Transformer. However, this appears to be much more challenging. Our result in this section only shows that for a rather restrictive family of input graphs (see Assumption \ref{AS-1} below), we can approximate a full self-attention layer of transformer with an MPNN + VN of $O(1)$ width and $O(n)$ depth. We leave the question of MPNN + VN's ability in approximate transformers for more general families of graphs for future investigation.

We first introduce the notion of $(\mv,\delta)$ separable node features. This is needed to ensure that VN can approximately select one node feature to process at each iteration with attention $\alpha_{\vn}$, the self-attention in the virtual node.

 \begin{definition}[$(\mv,\delta)$ separable by $\alpha$]
\label{vdeltaseparable}
Given a graph $G$ of size $n$ and a fixed $\mV \in \mb{R}^{n\times d} = [\vv_1, ..., \vv_n]$ 
and $\bar{\alpha} \in \mc{A}$, we say node feature $\mX\in \mb{R}^{n\times d}$ of $G$ is 
$(\mv,\delta)$ separable by some $\bar{\alpha}$ if the following holds. For any node feature $\vx_i$, there exist weights $\mW^{\bar{\alpha}}_K, \mW^{\bar{\alpha}}_Q$ in attention score $\bar{\alpha}$ 
such that $\bar{\alpha}(\vx_i, \vv_i) > \max_{j\neq i} \bar{\alpha}(\vx_j, \vv_i) + \delta$. We say set $\mathcal{X}$ is $(\mv,\delta)$ 
separable by $\bar{\alpha}$ if every element $\mX \in \mathcal{X}$ is $(\mv,\delta)$ separable by $\bar{\alpha}$.
 \end{definition}

\begin{table}[t]
    \caption{Test performance in graph-level OGB benchmarks \cite{hu2020open}. Shown is the mean~$\pm$~s.d.~of 10 runs. %
    }
    \label{tab:results_ogb}
    \tabtopvspace
    \centering
    \fontsize{10pt}{10pt}\selectfont
    \begin{tabular}{lccccc}\toprule
    \multirow{2}{*}{\textbf{Model}} &\textbf{ogbg-molhiv} &\textbf{ogbg-molpcba} &\textbf{ogbg-ppa} &\textbf{ogbg-code2} \\\cmidrule{2-5}
                               & \textbf{AUROC $\uparrow$}    & \textbf{Avg.~Precision $\uparrow$} & \textbf{Accuracy $\uparrow$} & \textbf{F1 score $\uparrow$} \\\midrule
    GCN                        & 0.7606 $\pm$ 0.0097          & 0.2020 $\pm$ 0.0024          & 0.6839 $\pm$ 0.0084          & 0.1507 $\pm$ 0.0018 \\
    GCN+virtual node           & 0.7599 $\pm$ 0.0119          & 0.2424 $\pm$ 0.0034          & 0.6857 $\pm$ 0.0061          & 0.1595 $\pm$ 0.0018 \\
    GIN                        & 0.7558 $\pm$ 0.0140          & 0.2266 $\pm$ 0.0028          & 0.6892 $\pm$ 0.0100          & 0.1495 $\pm$ 0.0023 \\
    GIN+virtual node           & 0.7707 $\pm$ 0.0149          & 0.2703 $\pm$ 0.0023          & 0.7037 $\pm$ 0.0107          & 0.1581 $\pm$ 0.0026 \\
    \midrule
    SAN                        & 0.7785 $\pm$ 0.2470          & 0.2765 $\pm$ 0.0042          & --                           & -- \\
    GraphTrans (GCN-Virtual)   & --                           & 0.2761 $\pm$ 0.0029          & --                           & 0.1830 $\pm$ 0.0024 \\
    K-Subtree SAT              & --                           & --                           & 0.7522 $\pm$ 0.0056          & 0.1937 $\pm$ 0.0028 \\%\midrule
    \method                    & 0.7880 $\pm$ 0.0101          & 0.2907 $\pm$ 0.0028          & 0.8015 $\pm$ 0.0033          & 0.1894 $\pm$ 0.0024 \\ \midrule
    MPNN + VN + NoPE           & 0.7676 $\pm$ 0.0172          & 0.2823 $\pm$ 0.0026          & 0.8055 $\pm$ 0.0038          & 0.1727 $\pm$ 0.0017             & \\
    MPNN + VN + PE             & 0.7687 $\pm$ 0.0136          & 0.2848 $\pm$ 0.0026          & 0.8027 $\pm$ 0.0026          & 0.1719 $\pm$ 0.0013             & \\
    \bottomrule
    \end{tabular}
\end{table}

The use of $(\mv,\delta)$ separability is to approximate hard selection function arbitrarily well, which is stated below and proved in \Cref{subsec:assumptions}. 
\begin{restatable}[approximate hard selection]{lemma}{UniformSelection}
\label{lemma-uniform-selection}
Given $\mathcal{X}$ is $(\mv,\delta)$ separable by $\bar{\alpha} $ for some fixed $\mv\in \mb{R}^{n\times d}$, $\bar{\alpha} \in \mc{A}$ and $\delta > 0$, the following holds. For any $\epsilon>0$ and $i\in [n]$, there exists a set of attention weights $\mW_{i, Q}, \mW_{i, K}$ in $i$-th layer of MPNN + VN such that $\alpha_{\vn}(\vx_i, \vv_i ) > 1-\epsilon$ for any $\vx_i \in \mathcal{X}_i$. In other words, we can approximate a hard selection function $f_i(\vx_1, ..., \vx_n) = \vx_i$ arbitrarily well on $\mathcal{X}$ by setting $\alpha_{\vn} = \bar{\alpha}$.  %
\end{restatable}

With the notation set up, We now state an extra assumption needed for deep MPNN + VN case and the main theorem. 
\begin{assumption}\label{AS-1} 
$\mathcal{X}$ is $(\mv,\delta)$ separable by $\bar{\alpha}$ for some fixed $\mv \in \mb{R}^{n\times d}$, $\bar{\alpha}\in \mc{A}$ and $\delta>0$. 
\end{assumption}

\begin{restatable}{theorem}{mainthm}
\label{thm:constant-width}
Assume AS 1-3 hold for the compact set $\mc{X}$ and $\mL$. Given any graph $G$ of size $n$ with node features $\mX\in \mc{X}$, and a self-attention layer $\mL$ on $G$ (fix $\mW_K, \mW_Q, \mW_V$ in $\alpha$), there exists a $O(n)$ layer of heterogeneous MPNN + VN with the specific aggregate/update/message function that can approximate $\mL$ on $\mc{X}$ arbitrarily well. 
\end{restatable}

The proof is presented in the \Cref{sec:approximate-full-self-attention}. On the high level, we can design an MPNN + VN where the $i$-th layer will select $\tilde{\vx}_i$, an approximation of $\vx_i$ via attention mechanism, enabled by \Cref{lemma-uniform-selection}, and send $\tilde{\vx}_i$ to the virtual node. Virtual node will then pass the $\tilde{\vx}_i$ to all graph nodes and computes the approximation of $e^{\alpha(\vx_i, \vx_j)}, \forall j\in [n]$. Repeat such procedures $n$ times for all graph nodes, and finally, use the last layer for attention normalization. A slight relaxation of AS\ref{AS-1} is also provided in \Cref{sec:connection_missing_proofs}.

\begin{table}[th!]
    \caption{Evaluation on PCQM4Mv2~\cite{hu2021ogb} dataset.
    For \method evaluation, we treated the \emph{validation} set of the dataset as a test set, since the \emph{test-dev} set labels are private.
    }
    \label{tab:results_pcqm4m}
    \centering
    \fontsize{10.0pt}{10.0pt}\selectfont
    \begin{tabular}{lccccc}\toprule
    \label{tab:pcqm}
    \multirow{2}{*}{\textbf{Model}} &\multicolumn{3}{c}{\textbf{PCQM4Mv2}} \\\cmidrule{2-5}
    &\textbf{Test-dev MAE $\downarrow$} &\textbf{Validation MAE $\downarrow$} &\textbf{Training MAE} &\textbf{\# Param.} \\\midrule
    GCN &0.1398 &0.1379 & n/a &2.0M \\
    GCN-virtual &0.1152 &0.1153 & n/a &4.9M \\
    GIN &0.1218 &0.1195 & n/a &3.8M \\
    GIN-virtual &0.1084 &0.1083 & n/a &6.7M \\\midrule
    GRPE~\cite{park2022grpe} &0.0898 &0.0890 & n/a &46.2M \\
    EGT~\cite{hussain2022global} &0.0872 &0.0869 & n/a &89.3M \\
    Graphormer~\cite{shi2022benchmarking} &n/a &0.0864 &0.0348 & 48.3M \\% \midrule
    \method-small &n/a &0.0938 &0.0653 &6.2M \\
    \method-medium &n/a &0.0858 &0.0726  &19.4M \\ \midrule
    MPNN + VN + PE (small)  &n/a &0.0942 &0.0617 &5.2M \\
    MPNN + VN + PE (medium)  &n/a &0.0867 &0.0703 &16.4M \\
    MPNN + VN + NoPE (small)  &n/a &0.0967 &0.0576 &5.2M \\
    MPNN + VN + NoPE (medium)  &n/a &0.0889 &0.0693 &16.4M \\
    \bottomrule
    \end{tabular}
\end{table}

\section{Experiments}
We benchmark MPNN + VN for three tasks, long range interaction modeling in \Cref{subsec:lrgb} and OGB regression tasks in \Cref{subsec:ogb}. The code is available \url{https://github.com/Chen-Cai-OSU/MPNN-GT-Connection}. 

\subsection{Dataset Description}
\textbf{ogbg-molhiv} and \textbf{ogbg-molpcba} \cite{hu2020open} are molecular property prediction datasets
adopted by OGB from MoleculeNet. These datasets use a common node (atom) and edge (bond)
featurization that represent chemophysical properties. 
The prediction task of ogbg-molhiv is a binary
classification of molecule's fitness to inhibit HIV replication. The ogbg-molpcba, derived from
PubChem BioAssay, targets to predict the results of 128 bioassays in the multi-task binary classification
setting.

\textbf{ogbg-ppa} \cite{wu2021representing} consists of protein-protein association (PPA) networks derived from
1581 species categorized into 37 taxonomic groups. Nodes represent proteins and edges encode the
normalized level of 7 different associations between two proteins. The task is to classify which of the
37 groups does a PPA network originate from.

\textbf{ogbg-code2} \cite{wu2021representing} consists of abstract syntax trees (ASTs) derived from the source
code of functions written in Python. The task is to predict the first 5 subtokens of the original
function's name. %

\textbf{OGB-LSC PCQM4Mv2} \cite{hu2021ogb} is a large-scale molecular dataset that shares the
same featurization as ogbg-mol* datasets. It consists of 529,434 molecule graphs. The task is to predict the HOMO-LUMO gap, a quantum physical property originally calculated using Density Functional Theory. True labels for original
test-dev and test-challange dataset splits are kept private by the OGB-LSC challenge organizers.
Therefore for the purpose of this chapter, we used the original validation set as the test set, while we
left out random 150K molecules for our validation set.

\subsection{MPNN + VN for LRGB Datasets}
\label{subsec:lrgb}
We experiment with MPNN + VN for Long Range Graph Benchmark (LRGB) datasets. %
Original paper \cite{dwivedi2022long} observes that GT outperforms MPNN on 4 out of 5 datasets, among which GT shows significant improvement over MPNN on \pepfunc and \pepstruct for all MPNNs. To test the effectiveness of the virtual node, we take the original code and modify the graph topology by adding a virtual node and keeping the hyperparameters of all models unchanged. 

Results are in \Cref{tab:experiments_peptides}. Interestingly, such a simple change can boost MPNN + VN by a large margin on \pepfunc and \pepstruct. Notably, with the addition of VN, GatedGCN + RWSE (random-walk structural encoding) after augmented by VN {\bf outperforms all transformers} on \pepfunc, and GCN outperforms transformers on \pepstruct. 

\subsection{Stronger MPNN + VN Implementation}
\label{subsec:ogb}
Next, by leveraging the modularized implementation from GraphGPS \cite{rampavsek2022recipe}, we implemented a version of MPNN + VN with/without extra positional embedding. Our goal is not to achieve SOTA but instead to push the limit of MPNN + VN and better understand the source of the performance gain for GT. 
In particular, we replace the GlobalAttention Module in GraphGPS with \DS{}, which is equivalent to one specific version of MPNN + VN. We tested this specific version of MPNN + VN on 4 OGB datasets, both with and without the use of positional embedding. %
The results are reported in Table \ref{tab:results_ogb}. Interestingly, even without the extra position embedding, our MPNN + VN is able to further improve over the previous GCN + VN \& GIN + VN implementation. 
The improvement on \textbf{ogbg-ppa} is particularly impressive, which is from 0.7037 to 0.8055. 
Furthermore, it is important to note that while MPNN + VN does not necessarily outperform GraphGPS, which is a state-of-the-art architecture using both MPNN, Position/structure encoding and Transformer, the difference is quite small -- this however, is achieved by a simple MPNN + VN architecture. %

We also test MPNN + VN on large-scale molecule datasets PCQMv2, which has 529,434 molecule graphs. 
We followed \cite{rampavsek2022recipe}
and used the original validation set as the test set, while we left out random 150K molecules for our validation set. 
As we can see from \Cref{tab:pcqm}, MPNN + VN + NoPE performs significantly better than the early MPNN + VN implementation: GIN + VN and GCN + VN. 
The performance gap between GPS on the other hand is rather small: 0.0938 (GPS) vs. 0.0942 (MPNN + VN + PE) for the small model and 0.0858 (GPS) vs. 0.0867 (MPNN + VN + PE) for the medium model.

\subsection{Connection to Over-Smoothing Phenomenon}
\label{subsec:over-smoothing}
Over-smoothing refers to the phenomenon that deep GNN will produce same features at different nodes after too many convolution layers. Here we draw some connection between VN and common ways of reducing over-smoothing. We think that using VN can potentially help alleviate the over-smoothing problem. In particular, we note that the use of VN can simulate some strategies people use in practice to address over-smoothing. We give two examples below. 

Example 1: In \cite{zhao2019pairnorm}, the two-step method (center \& scale) PairNorm is proposed to reduce the over-smoothing issues. In particular, PairNorm consists of 1) Center and 2) Scale

$$\tilde{\mathbf{x}}_i^c =\tilde{\mathbf{x}}_i-\frac{1}{n} \sum_i \tilde{\mathbf{x}}_i$$

$$\dot{\mathbf{x}}_i = s \cdot \frac{\tilde{\mathbf{x}}_i^c}{\sqrt{\frac{1}{n} \sum_i \left||\tilde{\mathbf{x}}_i^c\right||_2^2}}$$

Where $\tilde{\mathbf{x}}$ is the node features after graph convolution and $s$ is a hyperparameter. The main component for implementing PairNorm is to compute the mean and standard deviation of node features. For the mean of node features, this can be exactly computed in VN. For standard deviation, VN can arbitrarily approximate it using the standard universality result of MLP [5]. If we further assume that the standard deviation is lower bounded by a constant, then MPNN + VN can arbitrarily approximate the PairNorm on the compact set. 

Example 2: In \cite{yang2020revisiting} mean subtraction (same as the first step of PairNorm) is also introduced to reduce over-smoothing. As mean subtraction can be trivially implemented in MPNN + VN, arguments in \cite{yang2020revisiting} (with mean subtraction the revised power Iteration in GCN will lead to the Fiedler vector) can be carried over to MPNN + VN setting. 

In summary, introducing VN allows MPNN to implement key components of \cite{yang2020revisiting,zhao2019pairnorm}, we think this is one reason why we observe encouraging empirical performance gain of MPNN + VN. 

\section{On the Limitation of MPNN + VN}
\label{sec:limitation-of-approximate-transformer}
Although we showed that in the main part of this chapter, MPNN + VN of varying depth/width can approximate the self-attention of full/linear transformers, this does not imply that there is no difference in practice between MPNN + VN and GT. Our theoretical analysis mainly focuses on approximating self-attention without considering computational efficiency. In this section, we mention a few limitations of MPNN + VN compared to GT.  

\subsection{Representation Gap}
The main limitation of deep MPNN + VN approximating full self-attention is that we require a quite strong assumption: we restrict the variability of node features in order to select one node feature to process each iteration. Such assumption is relaxed by employing stronger attention in MPNN + VN but is still quite strong. 

For the large width case, the main limitation is the computational complexity: even though the self-attention layer requires $O(n^2)$ complexity, to approximate it in wide MPNN + VN framework, the complexity will become $O(n^d)$ where $d$ is the dimension of node features.

We think such limitation shares a similarity with research in universal permutational invariant functions. Both \DS{} \cite{zaheer2017deep} and Relational Network \cite{santoro2017simple} are universal permutational invariant architecture but there is still a representation gap between the two \cite{zweig2022exponential}. Under the restriction to analytic activation functions, one can construct a symmetric function acting on sets of size $n$ with elements in dimension $d$, which can be efficiently
approximated by the Relational Network, but provably requires width exponential in $n$ and $d$ for the \DS{}. We believe a similar representation gap also exists between GT and MPNN + VN and leave the characterization of functions lying in such gap as the future work. 

\subsection{On the Difficulty of Approximating Other Linear Transformers}
\label{subsec:ohter-linear-transformer}
In \Cref{sec:shallow-narrow-attention}, we showed MPNN + VN of $O(1)$ width and depth can approximate the self-attention layer of one type of linear transformer, Performer. The literature on efficient transformers is vast \cite{tay2020efficient} and we do not expect MPNN + VN can approximate many other efficient transformers. Here we sketch a few other linear transformers that are hard to  approximate by MPNN + VN of constant depth and width. %

Linformer \cite{wang2020linformer} projects the $n\times d$ dimension keys and values to $k\times d$ suing additional projection layers, which in graph setting is equivalent to graph coarsening. As MPNN + VN still operates on the original graph, it fundamentally lacks the key component to approximate Linformer.

We consider various types of efficient transformers effectively generalize the virtual node trick. By first switching to a more expansive model and reducing the computational complexity later on, efficient transformers effectively explore a larger model design space than MPNN + VN, which always sticks to the linear complexity.

\subsection{Difficulty of Representing SAN Type Attention}
In SAN \cite{kreuzer2021rethinking}, different attentions are used conditional on whether an edge is presented in the graph or not, detailed below. One may wonder whether we can approximate such a framework in MPNN + VN. 

In our proof of using MPNN + VN to approximate regular GT, we mainly work with \Cref{def:simplified-hetero-mpnn-vn} where we do not use any \gngn{}   edges and therefore not leverage the graph topology. It is straightforward to use \gngn{} edges and obtain the different message/update/aggregate functions for \gngn{} edges non-\gngn{} edges. Although we still achieve the similar goal of SAN to condition on the edge types, it turns out that we can not arbitrarily approximate SAN. 

Without loss of generality, SAN uses two types of attention depending on whether two nodes are connected by the edge. Specifically, 
\begin{equation}
\begin{aligned}
& \hat{\boldsymbol{w}}_{i j}^{k, l}=\left\{\begin{array}{lr}
\frac{\boldsymbol{Q}^{1, k, l} \boldsymbol{h}_i^l \circ \boldsymbol{K}^{1, k, l} \boldsymbol{h}_j^l \circ \boldsymbol{E}^{1, k, l} \boldsymbol{e}_{i j}}{\sqrt{d_k}} & \text { if } i \text { and } j \text { are connected in sparse graph } \\
\frac{\boldsymbol{Q}^{2, k, l} \boldsymbol{h}_i^l \circ \boldsymbol{K}^{2, k, l} \boldsymbol{h}_j^l \circ \boldsymbol{E}^{2, k, l} \boldsymbol{e}_{i j}}{\sqrt{d_k}} & \text { otherwise }
\end{array}\right\} \\
& w_{i j}^{k, l}=\left\{\begin{array}{cc}
\frac{1}{1+\gamma} \cdot \operatorname{softmax}\left(\sum_{d_k} \hat{\boldsymbol{w}}_{i j}^{k, l}\right) & \text { if } i \text { and } j \text { are connected in sparse graph } \\
\frac{\gamma}{1+\gamma} \cdot \operatorname{softmax}\left(\sum_{d_k} \hat{\boldsymbol{w}}_{i j}^{k, l}\right) & \text { otherwise }
\end{array}\right\}
\end{aligned}
\end{equation}
where $\circ$ denotes element-wise multiplication and $\boldsymbol{Q}^{1, k, l}, \boldsymbol{Q}^{2, k, l}, \boldsymbol{K}^{1, k, l}, \boldsymbol{K}^{2, k, l}, \boldsymbol{E}^{1, k, l}, \boldsymbol{E}^{2, k, l} \in$ $\mathbb{R}^{d_k \times d}$. $\gamma \in \mathbb{R}^{+}$is a hyperparameter that tunes the amount of bias towards full-graph attention, allowing flexibility of the model to different datasets and tasks where the necessity to capture long-range dependencies may vary. 

To reduce the notation clutter, we remove the layer index $l$, and edge features, and also consider only one-attention head case (remove attention index $k$). The equation is then simplified to
\begin{equation}
\label{equ:san-simplified-attention}
\begin{aligned}
& \hat{\boldsymbol{w}}_{i j}=\left\{\begin{array}{lr}
\frac{\boldsymbol{Q}^{1} \boldsymbol{h}_i^l \circ \boldsymbol{K}^{1} \boldsymbol{h}_j^l }{\sqrt{d_k}} & \text { if } i \text { and } j \text { are connected in sparse graph } \\
\frac{\boldsymbol{Q}^{2} \boldsymbol{h}_i^l \circ \boldsymbol{K}^{2} \boldsymbol{h}_j^l}{\sqrt{d_k}} & \text { otherwise }
\end{array}\right\} \\
& w_{i j}=\left\{\begin{array}{cc}
\frac{1}{1+\gamma} \cdot \operatorname{softmax}\left(\sum_{d} \hat{\boldsymbol{w}}_{i j}\right) & \text { if } i \text { and } j \text { are connected in sparse graph } \\
\frac{\gamma}{1+\gamma} \cdot \operatorname{softmax}\left(\sum_{d} \hat{\boldsymbol{w}}_{i j}\right) & \text { otherwise }
\end{array}\right\}
\end{aligned}
\end{equation}
We will then show that \Cref{equ:san-simplified-attention} can not be expressed (up to an arbitrary approximation error) in MPNN + VN framework. To simulate SAN type attention, our MPNN + VN framework will have to first simulate one type of attention for all edges, and then simulate the second type of attention between \gngn{} edges by properly offset the contribution from the first attention. This seems impossible (although we do not have rigorous proof) as we cannot express the difference between two attention in the new attention mechanism.

\section{Related Work}
\textbf{Virtual node in MPNN.}
The virtual node augments the graph with an additional node to facilitate the information exchange among all pairs of nodes. It is a heuristic proposed in \cite{gilmer2017neural} and has been observed to improve the performance in different tasks \cite{hu2021ogb,hu2020open}. Surprisingly, its theoretical properties have received little study. To the best of our knowledge, only a recent paper \cite{hwanganalysis} analyzed the role of the virtual node in the link prediction setting in terms of 1) expressiveness
of the learned link representation and 2) the potential impact on under-reaching and over-smoothing. 

\textbf{Graph transformer.}
Because of the great successes of Transformers in natural language processing (NLP) \cite{vaswani2017attention,wolf2020transformers} and recently in computer vision \cite{dosovitskiy2020image,d2021convit,liu2021swin}, there is great interest in extending transformers for graphs. One common belief of advantage of graph transformer over MPNN is its capacity in capturing long-range interactions while alleviating over-smoothing \cite{li2018deeper,oono2019graph,cai2020note} and over-squashing in MPNN \cite{alon2020bottleneck,topping2021understanding}. 

Fully-connected Graph transformer \cite{dwivedi2020generalization} was introduced with eigenvectors of graph Laplacian as the node positional encoding (PE). Various follow-up works proposed different ways of PE to improve GT, ranging from an invariant aggregation of Laplacian's eigenvectors in SAN \cite{kreuzer2021rethinking}, pair-wise graph distances in Graphormer \cite{ying2021transformers}, relative PE derived from diffusion kernels in GraphiT \cite{mialon2021graphit}, and recently Sign and Basis Net \cite{lim2022sign} with a principled way of handling sign and basis invariance. 
Other lines of research in GT include combining MPNN and GT  \cite{wu2021representing,rampavsek2022recipe}, encoding the substructures \cite{chen2022structure} and efficient graph transformers for large graphs \cite{wunodeformer}.

\textbf{Deep Learning on Sets.} Janossy pooling \cite{murphy2018janossy} is a framework to build permutation invariant architecture for sets using permuting \& averaging paradigm while limiting the number of elements in permutations to be $k < n$.  Under this framework, \DS{} \cite{zaheer2017deep} and PointNet \cite{qi2017pointnet} are recovered as the case of $k=1$. For case $k=2$, self-attention and Relation Network \cite{santoro2017simple} are recovered \cite{wagstaff2022universal}. Although \DS{} and Relation Network \cite{santoro2017simple} are both shown to be universal permutation invariant, recent work \cite{zweig2022exponential} provides a finer characterization on the representation gap between the two architectures.  %

\section{Concluding Remarks}
in this chapter, we study the expressive power of MPNN + VN under the lens of GT. If we target the self-attention layer in Performer and Linear Transformer, one only needs $O(1)$-depth $O(1)$ width for arbitrary approximation error. 
For self-attention in full transformer, we prove that heterogeneous MPNN + VN of either $O(1)$ depth $O(n^d)$ width or $O(n)$ depth $O(1)$ width (under some assumptions) can approximate 1 self-attention layer arbitrarily well. 
Compared to early results \cite{kim2022pure} showing GT can approximate MPNN, our theoretical result draws the connection from the inverse direction. %

On the empirical side, we demonstrate that MPNN + VN remains a surprisingly strong baseline. Despite recent efforts, we still lack good benchmark datasets where GT can outperform MPNN by a large margin. Understanding the inductive bias of MPNN and GT remains challenging. For example, can we mathematically characterize tasks that require effective long-range interaction modeling, and provide a theoretical justification for using GT over MPNN (or vice versa) for certain classes of functions on the space of graphs? We believe making processes towards answering such questions is an important future direction for the graph learning community.    

\section{Missing Proofs}
\label{sec:connection_missing_proofs}
In this section, we show the missing proofs of $O(n)$ heterogeneous MPNN + VN Layer with $O(1)$ width can approximate $1$ self attention layer arbitrarily well.
\label{sec:approximate-full-self-attention}

\subsection{Assumptions}
\label{subsec:assumptions}

A special case of $(\mv,\delta)$ separable is when $\delta=0$, i.e., $\forall i, \bar{\alpha}(\vx_i, \vv_i) > \max_{j\neq i} \bar{\alpha}(\vx_j, \vv_i)$. We provide a geometric characterization of $\mX$ being $(\mv, 0)$ separable. 
\begin{restatable}{lemma}{Geometriccharacterization}
\label{lemma:Geometric-characterization} Given $\bar{\alpha}$ and $\mV$, $\mX$ is $(\mv, 0)$ separable by $\bar{\alpha}$  $\Longleftrightarrow$ $\vx_i$ is not in the convex hull spanned by $\{\vx_j\}_{j\neq i}$. $\Longleftrightarrow$ there are no points in the convex hull of $\{\vx_i\}_{i\in[n]}$.
\end{restatable}

\begin{proof}
The second equivalence is trivial so we only prove the first equivalence. By definition, $\mX$ is $(\mv, 0)$ separable by $\bar{\alpha}$ $\Longleftrightarrow$ $\bar{\alpha}(\vx_i, \vv_i) > \max_{j\neq i} \bar{\alpha}(\vx_j, \vv_i) \forall i \in [n]$ $\Longleftrightarrow$ $\ip{\vx_i, \mW_Q^{\bar{\alpha}} \mW_K^{{\bar{\alpha}, T}} \vv_i} > \max_{j\neq i} \ip{\vx_j, \mW_Q^{\bar{\alpha}} \mW_K^{{\bar{\alpha}}, T} \vv_i} \forall i \in [n]$.

By denoting the $\vv'_i := \mW_Q^{\bar{\alpha}}\mW_K^{{\bar{\alpha}}, T} \vv_i \in \mb{R}^{d}$, we know that $\ip{\vx_i, \vv'_i} > \max_{j\neq i} \ip{\vx_j, \vv'_i} \forall i \in [n]$, which implies that  %
$\forall i \in [n], \vx_i$ can be linearly seprated from $\{\vx_j\}_{j\neq i}$ $\Longleftrightarrow$ $\vx_i$ is not in the convex hull spanned by $\{\vx_j\}_{j\neq i}$, which concludes the proof. 
\end{proof}

\UniformSelection*
\begin{proof}
Denote $\bar{\alpha}'$ as the unnormalized $\bar{\alpha}$. As $\mathcal{X}$ is $(\mv,\delta)$ separable by $\bar{\alpha}$, by definition we know that $\bar{\alpha}(\vx_i, \vv_i) > \max_{j\neq i }\bar{\alpha}(\vx_j, \vv_i) + \delta$ holds for any $i\in [n]$ and $\vx_i \in \mc{M}$. We can amplify this by multiple the weight matrix in $\bar{\alpha}$ by a constant factor $c$ to make $\bar{\alpha}'(\vx_i, \vv_i)> \max_{j\neq i}\bar{\alpha}'(\vx_j, \vv_i) + c\delta$. This implies that $e^{\bar{\alpha}'(\vx_i, \vv_i)}> e^ {c\delta} \max_{j\neq i}e^{\bar{\alpha}'(\vx_j, \vv_i)}$. This means after softmax, the attention score $\bar{\alpha}(\vx_i, \vv_i)$ will be at least $\frac{e^{c \delta}}{e^{c \delta}+n-1}$. We can pick a large enough $c(\delta, \epsilon)$ such that $\bar{\alpha}(\vx_i, \vv_i) > 1-\epsilon$ for any $\vx_i \in \mathcal{X}_i$ and $\epsilon > 0$.
\end{proof}

\textbf{Proof Intuition and Outline.} On the high level, $i$-th MPNN + VN layer will select $\tilde{\vx}_i$, an approximation $i$-th node feature $\vx_i$ via attention mechanism, enabled by \Cref{lemma-uniform-selection}, and send $\tilde{\vx}_i$ to the virtual node. Virtual node will then pass the $\tilde{\vx}_i$ to all graph nodes and computes the approximation of $e^{\alpha(\vx_i, \vx_j)}, \forall j\in [n]$. Repeat such procedures $n$ times for all graph nodes, and finally, use the last layer for attention normalization. 

The main challenge of the proof is to 1) come up with message/update/aggregation functions for heterogeneous MPNN + VN layer, which is shown in \Cref{subsec-mpnn-form},
and 2) ensure the approximation error, both from approximating Aggregate/Message/Update function with MLP and the noisy input, can be well controlled, which is proved in \Cref{subsec-controling-error}.

We will first instantiate the Aggregate/Message/Update function for virtual/graph nodes in \Cref{subsec-mpnn-form}, and prove that each component can be either exactly computed or approximated to an arbitrary degree by MLP. Then we go through an example in \Cref{subsec-a-running-example} of approximate self-attention layer $\mL$ with $O(n)$ MPNN + VN layers. The main proof is presented in \Cref{subsec-controling-error}, where we show that the approximation error introduced during different steps is well controlled. Lastly, in \Cref{subsec:relax-assumption} we show assumption on node features can be relaxed if a more powerful attention mechanism \gatii \cite{brody2021attentive} is allowed in MPNN + VN.  

\subsection{Aggregate/Message/Update
Functions}
\label{subsec-mpnn-form}

Let $\mc{M}$ be a multiset of vectors in $\R{d}$. 
The specific form of Aggregate/Message/Update for virtual and graph nodes are listed below. Note that ideal forms will be implemented as MLP, which will incur an approximation error that can be controlled to an arbitrary degree. We use $\vz_{\vn}^{(k)}$ denotes the virtual node's feature at $l$-th layer, and $\vz_{i}^{(k)}$ denotes the graph node $i$'s node feature. Iteration index $k$ starts with 0 and the node index starts with 1. 

\subsubsection{virtual node}\label{subsubsec-vn}
At $k$-th iteration, virtual node $i$'s feature $\vz_i^{(k)}$ is a concatenation of three component $[\tilde{\vx}_i, \vv_{k+1}, 0]$ where the first component is the approximately selected node features $\vx_i\in \R{d}$, the second component is the $\vv_i\in \R{d}$ that is used to select the node feature in $i$-th iteration. The last component is just a placeholder to ensure the dimension of the virtual node and graph node are the same. It is introduced to simplify notation.

\emph{Initial feature} is $\vz_{\vn}^{(0)} = [\bm{0}_d, \vv_1, 0]$. 

\emph{Message function + Aggregation function} $\pool_{j \in [n]} \phi_{\vngn}^{(k)}: \mb{R}^{2d+1} \times \mathcal{M} \rightarrow \mb{R}^{2d+1}$ has two cases to discuss depending on value of $k$. For $k=1, 2, ..., n$,
\begin{equation}\label{equ:vn-gn}
\begin{split}
 & \pool_{j \in [n]} \phi_{\vngn}^{(k)}(\vz_{\vn}^{(k-1)}, \{\vz_{i}^{(k-1)}\}_i) = \\
  & \begin{cases} 
  \sum_i \alpha_{\vn}(\vz_{\vn}^{(k-1)}, \vz_i^{(k-1)})\vz_i^{(k-1)} & k = 1, 2, ..., n \\
  \bm{1}_{2d+1} & k = n+1, n+2 \\ 
  \end{cases}
\end{split}
\end{equation}

 where $\vz_{\vn}^{(k-1)} = [\tilde{\vx}_{k-1}, \vv_{k}, 0]$. $\vz_i^{(k-1)} = [\overbrace{\underbrace{\vx_i}_{d \text{ dim}}, ..., ...}^{2d+1 \text{ dim}}]$ is the node $i$'s feature, where the first $d$ coordinates remain fixed for different iteration $k$. 
$\pool_{j \in [n]} \phi_{\vngn}^{(k)}$ use attention $\alpha_{\vn}$ to approximately select $k$-th node feature $[\overbrace{\underbrace{\vx_k}_{d \text{ dim}}, ..., ...}^{2d+1 \text{ dim}}]$. 
Note that the particular form of attention $\alpha_{\vn}$ needed for soft selection is not important as long as we can approximate hard selection arbitrarily well. As the $\vz^{(k-1)}_{\vn}$ contains $\vv_k$ and $\vz_i^{(k-1)}$ contains $\vx_i$ (see definition of graph node feature in \Cref{subsubsec-gn}),  this step can be made as close to hard selection as possible, according to \Cref{lemma-approximation-node-feature}. 

In the case of $k=n+1$,
$\pool_{j \in [n]} \phi_{\vngn}^{(k)}: \underbrace{\mb{R}^{2d+1}}_{\vn} \times \underbrace{\mathcal{M}}_{\text{set of } \gn} \rightarrow \mb{R}^d$ simply returns $\bm{1}_{2d+1}$. This can be exactly implemented by an MLP. 

\emph{Update function $\gamma_{\vn}^{(k)}: \underbrace{\mb{R}^{2d+1}}_{\vn} \times \underbrace{\mb{R}^{2d+1}}_{\gn} \rightarrow \mb{R}^{2d+1}$}:
Given the virtual node's feature in the last iteration, and the selected feature in virtual node $\vy = [\vx_{k}, ..., ...]$ with $\alpha_{\vn}$,
\begin{equation}
\gamma_{\vn}^{(k)}(\cdot, \vy) = 
\begin{cases}
[\vy_{0:d}, \vv_{k+1}, 0] & k=1, ..., n-1 \\
[\vy_{0:d}, \bm{0}_d, 0] & k=n \\
\bm{1}_{2d+1}                         & k=n+1, n+2  \\
\end{cases}
\end{equation}
where $\vy_{0:d}$ denotes the first $d$ channels of $\vy\in \mb{R}^{2d+1}$. $\vy$ denotes the selected node $\vz_i$'s feature in Message/Aggregation function. 
$\gamma_{\vn}^{(k)}$ can be exactly implemented by an MLP for any $k=1, ..., n+2$.

\subsubsection{Graph node}\label{subsubsec-gn}
Graph node $i$'s feature $\vv_i \in \R{2d+1}$ can be thought of as a concatenation of three components $[\underbrace{\vx_i}_{d \text{ dim}}, \underbrace{\tmp{}}_{d \text{ dim}}, \underbrace{\ps{}}_{1 \text{ dim}}]$, where $\vx_i, \in \R{d}, \tmp{} \in \R{d}$ \footnote{\tmp{} technicially denotes the dimension of projected feature by $W_V$ and does not has to be in $\R{d}$. We use $\R{d}$ here to reduce the notation clutter.}, and $\ps{}\in \R{}$. 

In particular, $\vx_i$ is the initial node feature. The first $d$ channel will stay the same until the layer $n+2$. $\tmp{} = \sum_{j \in \text{subset of} [n] } e^{\alpha'_{ij}}\vx_j$ stands for the unnormalized attention contribution up to the current iteration. $\ps{} \in \R{}$ is a partial sum of the unnormalized attention score, which will be used for normalization in the $n+2$-th iteration. 

\emph{Initial feature} $\vz_{\gn}^{(0)} =[\vx_i, \bm{0}_{d}, 0]$.

\emph{Message function + Aggregate function:
$\pool_{j \in [n]} \phi_{\gnvn}^{(k)}: \mb{R}^{2d+1}\times \mb{R}^{2d+1} \rightarrow \mb{R}^{2d+1}$}
is just ``copying the second argument'' since there is just one incoming message from the virtual node, i.e., $\pool_{j \in [n]} \phi_{\gnvn}^{(k)}(\vx, \{\vy\}) = \vy$. This function can be exactly implemented by an MLP. 

\emph{Update function}
$\gamma_{\gn}^{(k)}: \underbrace{\mb{R}^{2d+1}}_{\gn} \times \underbrace{\mb{R}^{2d+1}}_{\vn} \rightarrow \mb{R}^{2d+1}$ is of the following form. 
\begin{equation}\label{eqn-gn-update-function}
\begin{split}
& \gamma_{\gn}^{(k)}([\vx, \tmp{}, \ps{}], \vy) = \\
& \begin{cases}
[\vx, \tmp{}, \ps{}] &k = 1 \\
[\vx, \tmp{} + e^{\alpha'(\vx, \vy_{0:d})}\mW_V \vy_{0:d}, \\ \ps{}+e^{\alpha'(\vx, \vy_{0:d})}] &k = 2, ..., n+1  \\
[\frac{\tmp{}}{\ps{}}, \bm{0}_d, 0] & k=n+2 \\
\end{cases}
\end{split}
\end{equation}
where $\alpha'(\vx, \vy_{0:d})$ is the usual unnormalized attention score. Update function $\gamma_{\gn}^{(k)}$ can be arbitrarily approximated by an MLP, which is proved below. 

\begin{restatable}{lemma}{LemmaApproximateUpdateFunction}
\label{lemma-approximate-update-function}
Update function $\gamma_{\gn}^{(k)}$ can be arbitrarily approximated by an MLP from $\R{2d+1} \times \R{2d+1}$ to $\R{2d+1}$ for all $k=1,..., n+2$.
\end{restatable}
\begin{proof}
We will show that for any $k = 1, ..., n+2$, the target function
$\gamma_{\gn}^{(k)}: \mb{R}^{2d+1}\times \mb{R}^{2d+1} \rightarrow \mb{R}^{2d+1}$ is continuous and the domain is compact. By the universality of MLP in approximating continuous function on the compact domain, we know $\gamma_{\gn}^{(k)}$ can be approximated to arbitrary precision by an MLP. 

Recall that
\begin{equation*}
\begin{split}
& \gamma_{\gn}^{(k)}([\vx, \tmp{}, \ps{}], \vy) = \\
& \begin{cases}
[\vx, \tmp{}, \ps{}] &k = 1 \\
[\vx, \tmp{} + e^{\alpha'(\vx, \vy_{0:d})}\mW_V \vy_{0:d}, \\ \ps{}+e^{\alpha'(\vx, \vy_{0:d})}] &k = 2, ..., n+1  \\
[\frac{\tmp{}}{\ps{}}, \bm{0}_d, 0] & k=n+2 \\
\end{cases}
\end{split}
\end{equation*}
it is easy to see that $k=1$, $\gamma_{\gn}^{(1)}$ is continuous. We next show for $k=2, ..., n+2$, $\gamma_{\gn}^{(1)}$ is also continuous and all arguments lie in a compact domain. 

$\gamma_{\gn}^{(k)}$ is continuous because to a) $\alpha'(\vx, \vy)$ is continuous b) scalar-vector multiplication, sum, and exponential are all continuous. Next, we show that four component $\vx, \tmp{}, \ps{}, \\ \vy_{0:d}$ all lies in a compact domain.

$\vx$ is the initial node features, and by AS\ref{AS-2} their norm is bounded so $\vx$ is in a compact domain.

$\tmp{}$ is an approximation of $e^{\alpha'_{i, 1}}\mW_V \vx_1 + e^{\alpha'_{i, 2}}\mW_V \vx_2+...$. As $\alpha'(\vx_i, \vx_j)$ is both upper and lower bounded by AS\ref{AS-3} for all $i, j \in [n]$ and $\vx_i$ is bounded by AS\ref{AS-2}, $e^{\alpha'_{i, 1}}\mW_V \vx_1 + e^{\alpha'_{i, 2}}\mW_V \vx_2+...$ is also bounded from below and above. $\tmp{}$ will also be bounded as we can control the error to any precision. 

$\ps{}$ is an approximation of $e^{\alpha'_{i, 1}} + e^{\alpha'_{i, 2}} + ...$. For the same reason as the case above, $\ps{}$ is also bounded both below and above.

$\vy_{0:d}$ will be $\tilde{\vx}_i$ at $i$-th iteration so it will also be bounded by AS\ref{AS-2}.

Therefore we conclude the proof. 
\end{proof}

\subsection{A Running Example}
\label{subsec-a-running-example}
We provide an example to illustrate how node features are updated in each iteration. 

\textbf{Time $0$}: 
All nodes are initialized as indicated in \Cref{subsec-mpnn-form}. Virtual node feature  $\vz_{\vn}^{(0)} = [\bm{0}_d, \vv_1, 0]$. Graph node feature $\vz_{i}^{(0)} = [\vx_i, \bm{0}_{d}, 0]$ for all $i\in[n]$.

\textbf{Time $1$}:

For virtual node, according to the definition of $\pool_{j \in [n]} \phi_{\vngn}^{(1)}$ in \Cref{equ:vn-gn}, it will pick an approximation of $\vx_1$, i.e. $\tilde{\vx}_1$. Note that the approximation error can be made arbitrarily small. VN's node feature $\vz_{\vn}^{(1)} = [\tilde{\vx}_1, \vv_2, 0]$. 

For $i$-th graph node feature,  $\vz_{\vn}^{(0)} = \bm{1}_{d}$, and $\vz_i^{(0)} = [\vx_i, \bm{0}_{d}, 0]$. According to $\gamma_{\gn}^{(k)}$ in  \Cref{eqn-gn-update-function}, $\vz_i^{(1)} = [\vx_i, \bm{0}_{d}, 0]$.

\textbf{Time $2$}:

For the virtual node feature: similar to the analysis in time 1, VN's feature $\vz_{\vn}^{(2)} = [\tilde{\vx}_2, \vv_3, 0]$ now. Note that the weights and bias in $\pool_{j \in [n]} \phi_{\vngn}^{(2)}$ will be different from those in $\pool_{j \in [n]} \phi_{\vngn}^{(1)}$.

For $i$-th graph node feature, as $\vz_{\vn}^{(1)} = [\tilde{\vx}_1, \vv_2, 0]$ and $\vz^{(1)}_i = [\vx_i, \bm{0}_{d}, 0]$, according to $\gamma_{\gn}^{(k)}$ in  \Cref{eqn-gn-update-function}, $\vz^{(2)}_i =[\vx_i, e^{\widetilde{\alpha'_{i, 1}}}\mW_V \tilde{\vx}_1, e^{\widetilde{\alpha'_{i, 1}}}]$. Here $\widetilde{\alpha'_{i, 1}}:=\alpha'(\vx_i, \tilde{\vx}_1)$. We will use similar notations in later iterations. 
\footnote{To reduce the notation clutter and provide an intuition of the proof, we omit the approximation error introduced by using MLP to approximate
aggregation/message/update function, and assume the aggregation/message/update can be exactly implemented by neural networks. In the proofs, approximation error by MLP is handled rigorously. } 

\textbf{Time $3$}:

Similar to the analysis above, $\vz_{\vn}^{(3)} = [\widetilde{\vx_3}, \vv_4, 0]$.

$\vz_{i}^{(3)} = [\vx_i, e^{\widetilde{\alpha'_{i, 1}}}\mW_V \tilde{\vx}_1 + e^{\widetilde{\alpha'_{i, 2}}}\mW_V \tilde{\vx}_2, e^{\widetilde{\alpha'_{i, 1}}}+e^{\widetilde{\alpha'_{i, 2}}}]$.

\textbf{Time $n$}:

$\vz_{\vn}^{(n)} = [\tilde{\vx}_n, \bm{0}_d, 0]$. 

$\vz_{i}^{(n)} = \vx_i, 
\underbrace{e^{\widetilde{\alpha'_{i, 1}}}\mW_V \tilde{\vx}_1 + ... + e^{\widetilde{\alpha'_{i, n-1}}}\mW_V \widetilde{\vx_{n-1}}}_{n-1 \text{ terms}}, \\
\underbrace{e^{\widetilde{\alpha'_{i, 1}}}+e^{\widetilde{\alpha'_{i, 2}}}+... + e^{\widetilde{\alpha'_{i, n-1}}}]}_{n-1 \text{ terms}}
$.

\textbf{Time $n+1$}:

According to \Cref{subsubsec-vn}, in $n+1$ iteration, the virtual node's feature will be $\bm{1}_{d}$. 

$\vz_{i}^{(n+1)} = [\vx_i, \sum_{k\in [n]} e^{\widetilde{\alpha'_{ik}}}\mW_V\tilde{\vx}_k, \sum_{k\in [n]} e^{\widetilde{\alpha'_{ik}}}]$

\textbf{Time $n+2$ (final layer)}:

For the virtual node, its node feature will stay the same.

For the graph node feature, the last layer will serve as a normalization of the attention score (use MLP to approximate vector-scalar multiplication), and set the last channel to be 0 (projection), resulting in an approximation of $[\vx_i, \frac{\sum_{k\in [n]} e^{\widetilde{\alpha'_{ik}}} \mW_V \tilde{\vx}_k}{\sum_{k\in [n]} e^{\widetilde{\alpha'_{ik}}}}, 0]$. Finally, we need one more linear transformation to make the node feature become $[\frac{\sum_{k\in [n]}  e^{\widetilde{\alpha'_{ik}}} \mW_V \tilde{\vx}_k}{\sum_{k\in [n]} e^{\widetilde{\alpha'_{ik}}}}, \bm{0}_d, 0]$. The first $d$ channel is an approximation of the output of the self-attention layer for node $i$ where the approximation error can be made as small as possible. This is proved in \Cref{sec:approximate-full-self-attention}, and we conclude that heterogeneous MPNN + VN can approximate the self-attention layer $\mL$ to arbitrary precision with $O(n)$ MPNN layers.

\subsection{Controlling Error}\label{subsec-controling-error}

On the high level, there are three major sources of approximation error: 1) approximate hard selection with self-attention and 2) approximate equation $\gamma_{\gn}^{(k)}$ with MLPs, and 3) attention normalization in the last layer. %
In all cases, we aim to approximate the output of a continuous map $\mL_c(\vx)$. However, our input is usually not exact $\vx$ but an approximation of $\tilde{\vx}$. We also cannot access the original map $\mL_c$ but instead, an MLP approximation of $\mL_c$, denoted as $\mL_{\MLP}$. The following lemma allows to control the difference between $\mL_c(\vx)$ and $\mL_{\MLP}(\tilde{\vx})$. 

\begin{lemma}
\label{lemma:approximation-meta-lemma}
Let $\mL_c$ be a continuous map from compact set to compact set in Euclidean space. Let $\mL_{\MLP}$ be the approximation of $\mL_c$ by MLP. If we can control $\norm{\vx - \tilde{\vx} }$ to an arbitrarily small degree, we can then control the error $\norm{\mL_c(\vx)-\mL_{\MLP}(\tilde{\vx})}$ arbitrarily small. 
\end{lemma}
\begin{proof}
By triangle inequality $\norm{\mL_c(\vx)-\mL_{\MLP}(\tilde{\vx})} \leq \norm{\mL_c(\vx) - \mL_{\MLP}(\vx))} + \norm{\mL_{\MLP}(\vx) - \mL_{\MLP}(\tilde{\vx})}$. 

For the first term $\norm{\mL_c(\tilde{\vx}) -\mL_{\MLP}(\tilde{\vx})}$, by the universality of MLP, we can control the error $\norm{\mL_c(\tilde{\vx}) -\mL_{\MLP}(\tilde{\vx})}$ in arbitrary degree. 

For the second term $\norm{\mL_{\MLP}(\vx) - \mL_{\MLP}(\tilde{\vx})}$, as $\mL_{\MLP}$ is continuous on a compact domain, it is uniformly continuous by Heine-Cantor theorem. This means that we can control the $\norm{\mL_{\MLP}(\vx) - \mL_{\MLP}(\tilde{\vx})}$ as long as we can control $\norm{\vx - \tilde{\vx}}$, independent from different $\vx$. By assumption, this is indeed the case so we conclude the proof. 
\end{proof}

\begin{remark}
The implication is that when we are trying to approximate the output of a continuous map $\mL_c$ on the compact domain by an MLP $\mL_{\MLP}$, it suffices to show the input is 1) $\norm{\mL_c - \mL_{\MLP}}_{\infty}$ and 2) $\norm{\tilde{\vx}-\vx}$ can be made arbitrarily small. The first point is usually done by the universality of MLP on the compact domain \cite{cybenko1989approximation}. The second point needs to be shown case by case. 

In the \Cref{subsec-a-running-example}, to simplify the notations we omit the error introduced by using MLP to approximate aggregation/message/update functions (continuous functions on the compact domain of $\R{d}$.) in MPNN + VN. \Cref{lemma:approximation-meta-lemma} justify such reasoning.  
\end{remark}

\begin{lemma}[$\tilde{\vx}_i$ approximates $\vx_i$. $\widetilde{\alpha'_{i, j}}$ approximates $\alpha'_{i, j}$.]\label{lemma-approximation-node-feature}
For any $\epsilon>0$ and $x\in \mathcal{X}$, there exist a set of weights for message/aggregate functions of the virtual node such that $||\vx_i - \tilde{\vx}_i||<\epsilon$ and $|\alpha'_{i, j} -\widetilde{\alpha'_{i, j}}| < \epsilon$.
\end{lemma}

\begin{proof}
By \Cref{lemma-uniform-selection} We know that $\widetilde{\alpha_{i, j}} := \widetilde{\alpha}(\vx_i,\vx_j) \rightarrow \delta(i-j)$ as $C_3(\epsilon)$ goes to infinity. Therefore we have
\begin{equation}
||\tilde{\vx}_i - \vx_i|| = ||\sum_j \widetilde{\alpha_{i, j}}\vx_j-\vx_i|| = ||\sum (\widetilde{\alpha}_{i, j} - \delta(i-j))\vx_j|| < \epsilon \sum||\vx_j|| < n C_1 \epsilon
\end{equation}
As $n$ and $C_1$ are fixed, we can make the upper bound as small as we want by increasing $C_3$.

$|\alpha'_{i, j} -\widetilde{\alpha'_{i, j}}| =  |\alpha'(\vx_i, \vx_j) - \alpha'_{\MLP}(\tilde{\vx}_i, \vx_j) | = |\alpha'(\vx_i, \vx_j) - \alpha'(\tilde{\vx}_i, \vx_j) |  + |\alpha'(\tilde{\vx}_i, \vx_j) - \alpha'_{\MLP}(\tilde{\vx}_i, \vx_j)|=   |\alpha'(\vx_i - \tilde{\vx}_i, \vx_j)| = (\vx_i- \tilde{\vx}_i)^T\vx_jC_2^2 + \epsilon < nC_1\epsilon C_1C_2^2 + \epsilon = (nC_1^2C_2^2+1)\epsilon$. As $\alpha'_{i, j}, \widetilde{\alpha'_{i, j}}$ is bounded from above and below, it's easy to see that $|e^{\alpha'_{i, j}} -e^{\widetilde{\alpha'_{i, j}}}| = |e^{\alpha'_{i, j}}(1-e^{\alpha'_{i, j}- \widetilde{\alpha'_{i, j}}})| < C(1-e^{\alpha'_{i, j}- \widetilde{\alpha'_{i, j}}})$ can be controlled to arbitrarily degree. 
\end{proof}

\mainthm*
\begin{proof}
$i$-th MPNN + VN layer will select $\tilde{\vx}_i$, an arbitrary approximation $i$-th node feature $\vx_i$ via attention mechanism. This is detailed in the message/aggregation function of the virtual node in \Cref{subsubsec-vn}. Assuming the regularity condition on feature space $\mc{X}$, detailed in AS\ref{AS-1}, the approximation error can be made as small as needed, as shown in \Cref{lemma-uniform-selection,lemma-approximation-node-feature}. 

Virtual node will then pass the $\tilde{\vx}_i$ to all graph nodes, which computes an approximation of $e^{\alpha'(\tilde{\vx}_i, \vx_j)}, \forall j\in [n]$. This step is detailed in the update function $\gamma_{\gn}^{(k)}$ of graph nodes, which can also be approximated arbitrarily well by MLP, proved in \Cref{lemma-approximate-update-function}. By \Cref{lemma:approximation-meta-lemma}, we have an arbitrary approximation of $e^{\alpha'(\tilde{\vx}_i, \vx_j)}, \forall j\in [n]$, which itself is an arbitrary approximation of $e^{\alpha'(\vx_i, \vx_j)}, \forall j\in [n]$. 

Repeat such procedures $n$ times for all graph nodes, we have an arbitrary approximation of $\sum_{k\in [n]} e^{\alpha'_{ik}}\mW_V \vx_k \in \R{d}$ and $\sum_{k\in [n]} e^{\alpha'_{ik}} \in \R{}$. Finally, we use the last layer to approximate attention normalization $L_c(\vx, y) = \frac{\vx}{y}$, where $\vx \in \R{d}, y\in \R{}$. As inputs for attention normalization are arbitrary approximation of $\sum_{k\in [n]} e^{\alpha'_{ik}}\mW_V \vx_k$ and ${\sum_{k\in [n]} e^{\alpha'_{ik}}}$, both of them are lower/upper bounded according to AS\ref{AS-2} and AS\ref{AS-3}. Since the denominator is upper bounded by a positive number, this implies that the target function $L_c$ is continuous in both arguments. By evoking \Cref{lemma:approximation-meta-lemma} again, we conclude that we can approximate its output $\frac{\sum_{k\in [n]} e^{\alpha'_{ik}}\mW_V \vx_k}{{\sum_{k\in [n]} e^{\alpha'_{ik}}}}$ arbitrarily well. This concludes the proof. 

\end{proof}

\subsection{Relaxing Assumptions with More Powerful Attention}
\label{subsec:relax-assumption}
One limitation of \Cref{thm:constant-width} are assumptions on node features space $\mc{X}$: we need to 1) restrict the variability of node feature so that we can select one node feature to process each iteration. 2) The space of the node feature also need to satisfy certain configuration in order for VN to select it.  For 2), we now consider a different attention function for $\alpha_{\vn}$ in MPNN + VN that can relax the assumptions AS\ref{AS-1} on $\mc{X}$.

\textbf{More powerful attention mechanism.} From proof of \Cref{thm:constant-width}, we just need $\alpha(\cdot, \cdot)$ uniformly select every node in $\mX\in \mathcal{X}$. The unnormalized bilinear attention $\alpha'$ is weak in the sense that $f(\cdot) = \ip{\vx_i\mW_Q\mW_K^T, \cdot}$ has a linear level set. Such a constraint can be relaxed via an improved attention module \gatii. Observing the ranking of the attention scores given by \texttt{GAT} \cite{velivckovic2017graph} is unconditioned on the query node, \cite{brody2021attentive} proposed \gatii, a more expressive attention mechanism. 
In particular, the unnormalized attention score $\alpha'_{\text{GATv2}}(\vu, \vv) :=  \va^T \operatorname{LeakyReLU}\left(\mW \cdot\left[\vu \| \vv\right] + \vb \right)$, where $[\cdot || \cdot ]$ is concatenation. We will let $\alpha_{\vn} = \alpha_{\gat}$ to select features in $\pool_{j \in [n]} \phi_{\vngn}^{(k)}$. 

\begin{figure}[hbtp]
  \centering
  \includegraphics[width=.35\linewidth]{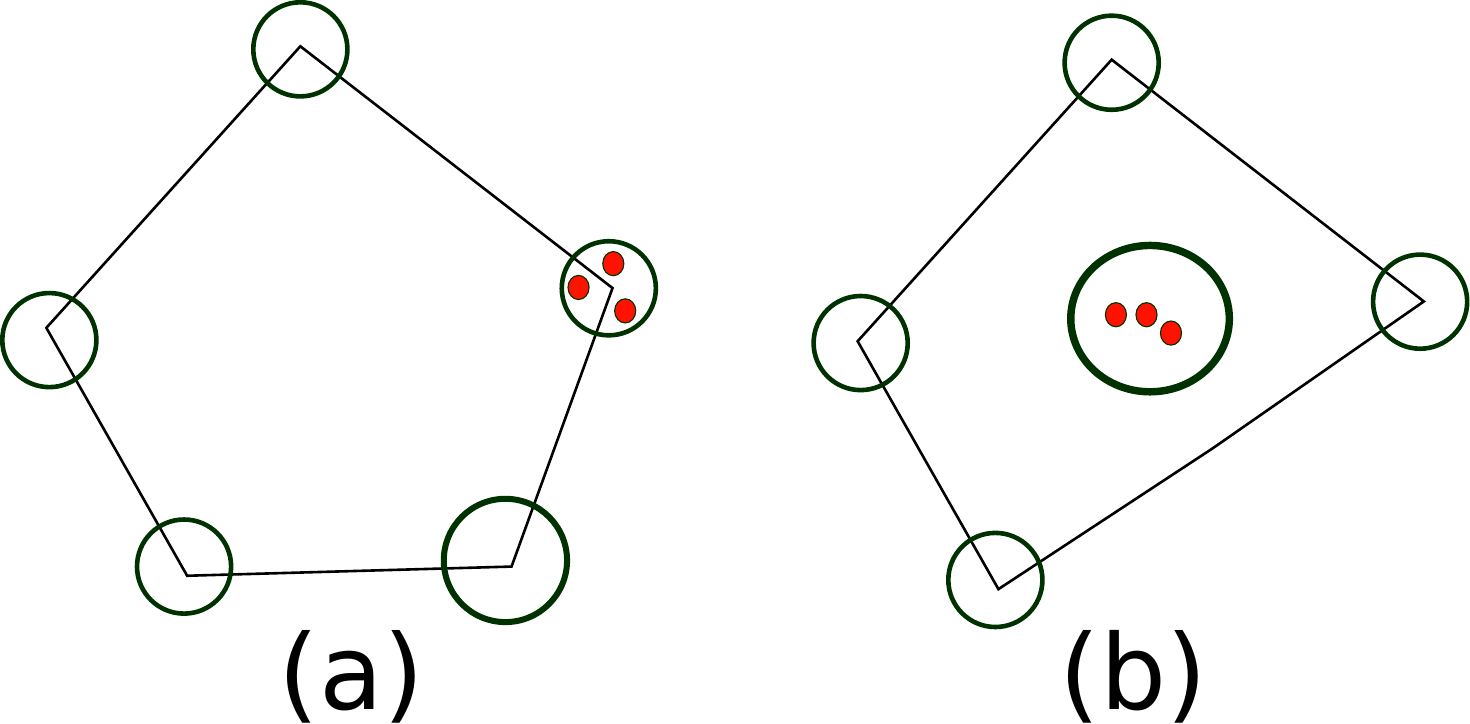}
\caption{In the left figure, we have one example of $\mc{X}$ being $(\mv,\delta)$ separable, for which $\alpha$ can uniformly select any point (marked as red) $\vx_i \in \mc{X}_i$. In the right figure, we change $\alpha_{\vn}$ in MPNN + VN to $\alpha_{\gat}$, which allows us to select more diverse feature configurations. 
}
\label{fig:convergence}
\end{figure}

\begin{restatable}{lemma}{gatvtwoUniversality}
\label{lemma-gatv2-universality}
$\alpha'_{\gat}(\cdot, \cdot)$ can approximate any continuous function from $\R{d} \times \R{d} \rightarrow \R{}$. For any $\vv\in \R{d}$, a restriction of $\alpha'_{\gat}(\cdot, \vv)$ can approximate any continuous function from $\R{d} \rightarrow \R{}$.
\end{restatable}
\begin{proof}
Any function continuous in both arguments of $\alpha'_{\gat}$ is also continuous in the concatenation of both arguments. As any continuous functions in $\R{2d}$ can be approximated by  $\alpha'_{\gat}$ on a compact domain according to the universality of MLP \cite{cybenko1989approximation}, we finish the proof for the first statement.

For the second statement, we can write $W$ as $2\times 2$ block matrix and restrict it to cases where only $\mW_{11}$ is non-zero. Then we have 
{\fontsize{10}{12}\selectfont
\begin{equation}
\alpha'_{\gat}(\vu, \vv) 
= a^T \operatorname{LeakyReLU}\left(
  \left[ {\begin{array}{cc}
   \mW_{11} & \mW_{12} \\
   \mW_{21} & \mW_{22} \\
  \end{array} } \right] \cdot \left[\begin{array}{c}
   \vu  \\
   \vv \\
  \end{array} \right]+ \vb\right) 
=\va^T \operatorname{LeakyReLU}\left(\mW_{11}\vu + \vb\right) 
\end{equation} 
}
  which gives us an MLP on the first argument $\vu$. By the universality of MLP, we conclude the proof for the second statement. 

\end{proof}

\begin{definition}
\label{delta-sepration}
Given $\delta>0$, We call $\mathcal{X}$ is $\delta$ nonlinearly separable if and only if $\min_{i \neq j} d(\mc{X}_i, \mc{X}_j) > \delta$. 
\end{definition}

\begin{assumption}
\label{AS-4} 
$\mathcal{X}$ is $\delta$ nonlinearly separable for some $\delta>0$. \todo{set counter}
\end{assumption}

\begin{restatable}{proposition}{gatvtwoselection}
\label{prop:gat-v2-selection}
If $\mc{X} \subset \R{n \times d}$ satisfies that $\mc{X}_i$ is $\delta$-separated from $\mc{X}_j$ for any $i, j \in [n]$, the following holds. For any $\mX\in \mc{X}$ and $i\in [n]$, there exist a $\alpha_{\gat}$ to select any $\vx_i \in \mc{X}_i$. This implies that we can arbitrarily approximate the self-attention layer $\mL$ after relaxing AS3 to AS3'. 
\end{restatable}
\begin{proof}
For any $i \in [n]$, as $\mc{X}_i$ is $\delta$-separated from other $\mc{X}_j, \forall j\neq i$, we can draw a region $\Omega_i \subset \R{d}$ that contains $\mc{X}_i$ and separate $\mc{X}_i$ from other $\mc{X}_j (j\neq i)$, where the distance from $\mc{X}_i$ from other $\mc{X}_j$ is at least $\delta$ according to the definition of \Cref{delta-sepration}. Next, we show how to construct a continuous function $f$ whose value in $\mc{X}_i$ is at least 1 larger than its values in any other $\mc{X}_j$  $\forall j\neq i$.

We set the values of $f$ in $\mc{X}_i$ to be 1.5 and values of $f$ in $\mc{X}_j, \forall j\neq i$ to be 0. We can then interpolate $f$ in areas outside of $\cup \mc{X}_i$ (one way is to set the values of $f(x)$ based on $d(x, \mc{X}_i$), which results in a continuous function that satisfies our requirement.   
By the universality of $\alpha_{\gat}$, we can approximate $f$ to arbitrary precision, and this will let us select any $\mc{X}_i$. 
\end{proof}

\chapter{Graph Coarsening with Neural Networks}\label{chapter:cg}
\section{Introduction}

In this chapter, we look into another common way of global modeling for large graphs. As large scale-graphs become increasingly ubiquitous in various applications, they pose significant computational challenges to process, extract and analyze information. It is therefore natural to look for ways to simplify the graph while preserving the properties of interest. 

There are two major ways to simplify graphs. First, one may reduce the number of edges, known as graph edge sparsification. It is known that pairwise distance (spanner), graph cut (cut sparsifier), eigenvalues (spectral sparsifier) can be approximately maintained via removing edges. A key result \cite{spielman2004nearly} in the spectral sparsification is that any dense graph of size $N$ can be sparsified to $O(Nlog^cN/\epsilon^2)$ edges in nearly linear time using a simple randomized algorithm based on the effective resistance. 

Alternatively, one could also reduce the number of nodes to a subset of the original node set. The first challenge here is how to choose the topology (edge set) of the smaller graph spanned by the sparsified node set. 
On the extreme, one can take the complete graph spanned by the sampled nodes. However, its dense structure prohibits easy interpretation and poses computational overhead for setting the $\Theta(n^2)$ weights of edges. this chapter focuses on \emph{graph coarsening}, which reduces the number of nodes 
by contracting disjoint sets of connected vertices.

The original idea dates back to the algebraic multigrid literature \cite{ruge1987algebraic} and has found various applications in graph partitioning \cite{hendrickson1995multi, karypis1998fast, kushnir2006fast}, visualization \cite{harel2000fast, hu2005efficient, walshaw2000multilevel} and machine learning \cite{lafon2006diffusion, gavish2010multiscale, shuman2015multiscale}. 

However, most existing graph coarsening algorithms come with two restrictions. First, they are \emph{prespecified} and not adapted to specific data nor different goals. 

Second, most coarsening algorithms set the edge weights of the coarse graph equal to the sum of weights of crossing edges in the original graph. This means the weights of the coarse graph is determined by the coarsening algorithm (of the vertex set), leaving no room for adjustment. 

With the two observations above, we aim to develop a data-driven approach to better assigning weights for the coarse graph depending on specific goals at hand. 

We will leverage the recent progress of deep learning on graphs to develop a framework to learn to assign edge weights \emph{in an unsupervised manner} from a collection of input (small) graphs. This learned weight-assignment map can then be applied to new graphs (of potentially much larger sizes). 

In particular, our contributions are threefold. 

\begin{itemize} %
\item First, depending on the quantity of interest $\mathcal{F}$ (such as the quadratic form w.r.t. Laplace operator), one has to carefully choose projection/lift operator to relate quantities defined on graphs of different sizes. 
We formulate this as the invariance of $\mathcal{F}$ under lift map, and provide three cases of projection/lift map as well as the corresponding operators on the coarse graph. Interestingly, those operators all can be seen as the special cases of doubly-weighted Laplace operators on coarse graphs \cite{horak2013spectra}. 

\item Second, we are the first to propose and develop a framework to learn the edge weights of the coarse graphs via graph neural networks (GNN) in an unsupervised manner. We show convincing results both theoretically and empirically that changing the weights is crucial to improve the quality of coarse graphs. 
many existing graph coarsening algorithms. 

\item Third, through extensive experiments on both synthetic graphs and real networks, we demonstrate that our method \nn significantly improves common graph coarsening methods under different evaluation metrics, reduction ratios, graph sizes, and graph types. It generalizes to graphs of larger size (than the training graphs), adapts to different losses (so as to preserve different properties of original graphs), and scales to much larger graphs than what previous work can handle. Even for losses that are not differentiable w.r.t the weights of the coarse graph, we show training networks with a differentiable auxiliary loss still improves the result. 
\end{itemize}

\section{Proposed Approach: Learning Edge Weight with GNN}
\label{sec:method}

\subsection{High-level overview}
\label{subsec-high-level-overview}

\begin{wrapfigure}{r}{0.28\textwidth}
\vspace{-20pt}
 \begin{center}
 \label{fig-toy-example}
  \includegraphics[width=0.26\textwidth]{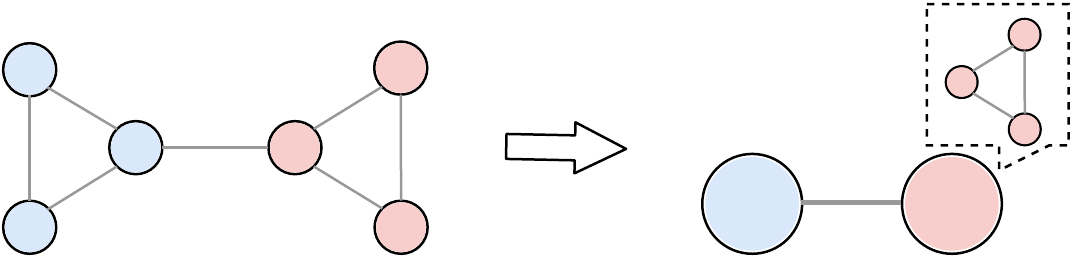}
 \end{center}

\end{wrapfigure}
Our input is a non-attributed (weighted or unweighted) graph $G = (V,E)$. 
Our goal is to construct an appropriate ``coarser" graph $\Ghat = (\Vhat, \Ehat)$ that preserves certain properties of $G$. 
Here, by a ``coarser" graph, we assume that $|\Vhat| << |V|$ and there is a surjective map $\vmap: V \to \Vhat$ that we call the \emph{\vertexmap}. Intuitively, (see figure on the right), for any node $\vhat \in \Vhat$, all nodes $\vmap^{-1}(\vhat) \subset V$ are mapped to this \emph{super-node} $\vhat$ in the coarser graph $\Ghat$. 
We will later propose a GNN based framework that can be trained using a collection of existing graphs \emph{in an unsupervised manner}, so as to construct such a coarse graph $\Ghat$ for a future input graph $G$ (presumably coming from the same family as training graphs) that can preserve properties of $G$ effectively. 

We will in particular focus on preserving properties of the \emph{Laplace operator $\genL$} 
of $G$, which is by far the most common operator associated to graphs, and forms the foundation for spectral methods. 
Specifically, given $G = (V = \{v_1, \ldots, v_{\N}\}, E)$ with $w: E\to \reals$ being the weight function for $G$ (all edges have weight 1 if $G$ is unweighted), let $W$ the corresponding $\N \times \N$ edge-weight matrix where $W[i][j] = w(v_i, v_j)$ if edge $(v_i,v_j)\in E$ and $0$ otherwise. 
Set $D$ to be the $N \times \N$ diagonal matrix with $D[i][i]$ equal to the sum of weights of all edges incident to $v_i$. 
The standard \emph{(un-normalized) combinatorial Laplace operator} of $G$ is then defined as $\combL = D - W$. The \emph{normalized Laplacian} is defined as $\mathcal{L} = D^{-1/2}\combL D^{-1/2} = I - D^{-1/2}WD^{-1/2}$. 

However, to make this problem as well as our proposed approach concrete, various components need to be built appropriately. 
We provide an overview here, and they will be detailed in the remainder of this section. 

\begin{itemize}[itemsep=0pt,topsep=0pt,leftmargin=10pt]
  \item \vspace*{-0.05in}Assuming that the set of super-nodes $\Vhat$ as well as the map $\vmap: V \to \Vhat$ are given, one still need to decide how to set up the connectivity (i.e, edge set $\Ehat$) for the coarse graph $\Ghat = (\Vhat, \Ehat)$. We introduce a natural choice in Section \ref{subsec:coarsegraph}, and provide some justification for this choice. 
  \item As the graph $G$ and the coarse graph $\Ghat$ have the different number of nodes, their Laplace operators $\genL $ and $\sgenL$ of two graphs are not directly comparable. Instead, we will compare $\Fcal(\genL , f)$ and $\Fcal(\sgenL, \fhat)$, where $\Fcal$ is a functional intrinsic to the graph at hand (invariant to the permutation of vertices), such as the quadratic form or Rayleigh quotient. 
  However, it turns out that depending on the choice of $\Fcal$, we need to choose the precise form of the Laplacian $\sgenL$, as well as the (so-called lifting and projection) maps relating these two objects, carefully, so as they are comparable. We describe these in detail in Section \ref{subsec:coarseLaplace}. 
  \item In Section \ref{subsec:weights} we show that adjusting the weights of the coarse graph $\Ghat$ can significantly improve the quality of $\Ghat$. This motivates a learning approach to learn a strategy (a map) to assign these weights from a collection of given graphs. We then propose a GNN-based framework to do so in an unsupervised manner. Extensive experimental studies will be presented in Section \ref{sec:exp}.  
\end{itemize}

\subsection{Construction of Coarse graph } 
\label{subsec:coarsegraph}

Assume that we are already given the set of super-nodes $\Vhat = \{\vhat_1, \ldots, \vhat_{\n}\}$ for the coarse graph $\Ghat$ together with the \vertexmap{} $\pi: V\to \Vhat$ -- There has been much prior work on computing the sparsified set $\Vhat \subset V$ and $\pi$ \cite{loukas2018spectrally, loukas2019graph}; and if
the \vertexmap{} $\pi$ is not given, then we can simply define it by setting $\pi(v)$ for each $v\in V$ to be the nearest neighbor of $v$ in $\Vhat$ in terms of graph shortest path distance in $G$ \cite{dey2013graph}. 

To construct edges for the coarse graph $\Ghat = (\Vhat, \Ehat)$ together with the edge weight function $\what: \Ehat\to \reals$, instead of using a complete weighted graph over $\Vhat$, which is too dense and expensive, we set $\Ehat$ to be those edges ``induced" from $G$ when collapsing each \emph{cluster $\vmap^{-1}(\vhat)$} to its corresponding super-node $\vhat \in \Vhat$: Specifically, $(\vhat, \vhat') \in \Ehat$ if and only if there is an edge $(v, v') \in E$ such that $\vmap(v) = \vhat$ and $\vmap(v') = \vhat'$. 

The weight of this edge is  
$\what(\vhat, \vhat') := \sum_{(v, v')\in E\big(\vmap^{-1}(\vhat), \vmap^{-1}(\vhat') \big) } w(v, v')$

where $E(A, B) \subseteq E$ stands for the set of edges crossing sets $A, B \subseteq V$; 

i.e., $\what(\vhat,\vhat')$ is the total weights of all crossing edges in $G$ between clusters $\vmap^{-1}(\vhat)$ and $\vmap^{-1}(\vhat')$ in $V$. 
We refer to $\Ghat$ constructed this way the \emph{$\Vhat$-induced coarse graph}. 

As shown in \cite{dey2013graph}, if the original graph $G$ is the $1$-skeleton of a hidden space $X$, then this induced graph captures the topological of $X$ at a coarser level if $\Vhat$ is a so-called $\delta$-net of the original vertex set $V$ w.r.t. the graph shortest path metric. 

Let $\What$ be the edge weight matrix, and $\Dhat$ be the diagonal matrix encoding the sum of edge weights incident to each vertex as before. Then the standard combinatorial Laplace operator w.r.t. $\Ghat$ is simply $\scombL = \Dhat - \What$. 

{{\bf Relation to the operator of \cite{loukas2019graph}.~}}
Interestingly, this construction of the coarse graph $\Ghat$ coincides with the coarse Laplace operator for a sparsified vertex set $\Vhat$ constructed by \cite{loukas2019graph}. 
We will use this view of the Laplace operator later; hence we briefly introduce the construction of \cite{loukas2019graph} (adapted to our setting): 

Given the \vertexmap{} $\pi: V\to \Vhat$, we set a $\n \times \N$ matrix $P$ by $P[r, i] = \left\{\begin{array}{ll}\frac{1}{\left|\vmap^{-1}(\vhat_r)\right|} & \text { if } v_{i} \in \vmap^{-1}(\vhat_r) \\ 0 & \text { otherwise }\end{array}\right.$.
In what follows, we denote $\gamma_r:= \left|\vmap^{-1}(\vhat_r)\right|$ for any $r\in [1,\n]$, which is the size of the cluster of $\vhat_r$ in $V$. 
$P$ can be considered as the weighted projection matrix of the vertex set from $V$ to $\Vhat$. 
Let $P^+$ denote the Moore-Penrose pseudoinverse of $P$, which can be intuitively viewed as a way to lift a function on $\Vhat$ (a vector in $\reals^\n$) to a function over $V$ (a vector in $\reals^\N$). 
As shown in \cite{loukas2019graph}, $P^+$ is the $\N \times \n$ matrix where $P^+[i, r] = 1$ if and only if $\vmap(v_i) = \vhat_r$. See \Cref{appendix-missing-proof} for a toy example.

Finally, \cite{loukas2019graph} defines an operator for the coarsened vertex set $\Vhat$ to be
$\tilde{L}_\Vhat = (P^+)^T \combL P^+$. 
Intuitively, $\LtildeVhat$ operators on $\n$-vectors. For any $\n$-vector $\fhat \in \reals^\n$, $\tilde{L}_\Vhat \fhat$ first lifts $\fhat$ to a $\N$-vector $f = P^+ \fhat$, and then perform $\combL $ on $f$, and then project it down to $\n$-dimensional via $(P^+)^T$. 

\begin{proposition}\cite{loukas2019graph}
The combinatorial graph Laplace operator $\Lhat = \Dhat - \What$ for the $\Vhat$-induced coarse graph $\Ghat$ constructed above equals to the operator $\tilde{L}_\Vhat = (P^+)^T \combL P^+$. 

\end{proposition}

\subsection{Laplace Operator for the Coarse Graph } 
\label{subsec:coarseLaplace} 

We now have an input graph $G = (V, E)$ and a coarse graph $\Ghat$ induced from the sparsified node set $\Vhat$, and we wish to compare their corresponding Laplace operators. However, as $\genL $ operates on $\reals^\N$ (i.e, functions on the vertex set $V$ of $G$) and $\sgenL$ operates on $\reals^\n$, we will compare them by their effects on ``corresponding" objects. 

\cite{loukas2018spectrally, loukas2019graph} proposed to use the quadratic form to measure the similarity between the two linear operators. In particular, given a linear operator $A$ on $\reals^N$ and any $x\in \reals^N$, $\Qform_A(x) = x^T A x$. The quadratic form has also been used for measuring spectral approximation under edge sparsification. 
The proof of the following result is in \Cref{appendix-missing-proof}. 

\begin{proposition}\label{prop:qform}
For any vector $\sx \in \reals^\n$, we have that $\Qform_{\LtildeVhat}(\sx) = \Qform_{\combL }(P^+ \sx)$, where $\Lhat$ is the combinatorial Laplace operator for the $\Vhat$-induced coarse graph $\Ghat$ constructed above. 
That is, set $\bx:= P^+ \sx$ as the lift of $\sx$ in $\reals^\N$, then $\sx^T \LtildeVhat \sx = \bx^T \combL \bx$. 

\end{proposition}
Intuitively, this suggests that if later, we measure the similarity between $\combL $ and some Laplace operator for the coarse graph $\Ghat$ based on a loss from quadratic form difference, then we should choose the Laplace operator $\sgenL$ to be $\LtildeVhat$ and compare $\Qform_{\scombL}(P \bx)$ with $\Qform_{\combL} (\bx)$. 
We further formalize this by considering the \emph{lifting map} $\lift: \reals^\n \to \reals^\N$ as well as a projection map $\proj: \reals^\N \to \reals^\n$, where $\proj \cdot \lift = Id_{\n}$. Proposition \ref{prop:qform} suggests that for quadratic form-based similarity, the choices are $\lift = P^+, \proj = P$, and $\sgenL = \LtildeVhat$. See the first row in Table \ref{tab:operator}. 

\begin{table}[htp!]
\centering
\caption{Depending on the choice of $\mathcal{F}$ (quantity that we want to preserve) and $\genL$, we have different projection/lift operators and resulting $\sgenL$ on the coarse graph.}
\label{tab:operator}
\resizebox{1\textwidth}{!}{
\begin{tabular}{@{}llllll@{}}
\toprule
Quantity $\mathcal{F}$ of interest & $\genL$ & Projection $\proj$ & Lift $\lift$ & $\sgenL$  & Invariant under $\lift$ \\ \midrule
Quadratic form $\Qform$ & $L$ & $P$  & ${P^{+}}$ & Combinatorial Laplace $\widehat{L}$ & $\Qform_L(\lift \sx) = \Qform_{\widehat{L}}(\sx)$\\
 Rayleigh quotient $\RQ$ & $L$ & $\Gamma^{-1/2}{{(P^{+})}^T}$ & ${P^{+}}\Gamma^{-1/2}$ & Doubly-weighted Laplace $\sdwL$ & $\RQ_L(\lift \sx) = \RQ_{\sdwL} (\sx)$\\
 Quadratic form $\Qform$ & $\mathcal{L}$ & $ \widehat{D}^{1/2}PD^{-1/2} $ & $ D^{1/2}{(P^{+})} \widehat{D}^{-1/2}$ & Normalized Laplace $\widehat{\mathcal{L}}$ & $\Qform_\mathcal{L}(\lift \sx) = \Qform_{\widehat{\mathcal{L}}}(\sx)$\\ \bottomrule
\end{tabular}
}
\vspace{-5pt}
\end{table}

On the other hand, eigenvectors and eigenvalues of a linear operator $A$ are more directly related, via Courant-Fischer Min-Max Theorem, to its Rayleigh quotient $\RQ_A (x) = \frac{x^T A x}{x^T x}$. 
Interestingly, in this case, to preserve the Rayleigh quotient, we should change the choice of $\sgenL$ to be the following \emph{doubly-weighted Laplace operator} for a graph that is both edge and vertex weighted. 

Specifically, for the coarse graph $\Ghat$, we assume that each vertex $\vhat \in \Vhat$ is weighted by $\gamma_\vhat := |\vmap^{-1}(\vhat)|$, the size of the cluster from $G$ that got collapsed into $\vhat$. Let $\Gamma$ be the vertex matrix, which is the $\n \times \n$ diagonal matrix with $\Gamma[r][r] = \gamma_{\vhat_r}$. 
The \emph{doubly-weighted Laplace operator} for a vertex- and edge-weighted graph $\Ghat$ is then defined as: 
\vspace{-3pt}
\begin{align*}\label{eqn:doublyweightedL}
  \sdwL = \Gamma^{-1/2} (\Dhat - \What) \Gamma^{-1/2} = \Gamma^{-1/2} \LtildeVhat \Gamma^{-1/2} = (P^+ \Gamma^{-1/2})^T \combL (P^+ \Gamma^{-1/2}). 
\end{align*}
The concept of doubly-weighted Laplace for a vertex- and edge-weighted graph is not new, see e.g \cite{chung1996combinatorial, horak2013spectra, xu2019weighted}. In particular, \cite{horak2013spectra} proposes a general form of combinatorial Laplace operator for a simplicial complex where all simplices are weighted, and our doubly-weighted Laplace has the same eigenstructure as their Laplacian when restricted to graphs. 
See \Cref{simplicial-laplace-appendix} for details. 
Using the doubly-weighted Laplacian for Rayleigh quotient based similarity measurement between the original graph and the coarse graph is justified by the following result (proof in \Cref{simplicial-laplace-appendix}). 
\begin{proposition}\label{prop:Rayleyquotient}
For any vector $x \in \reals^\n$, we have that $\RQ_{\sdwL}(\sx) = \RQ_{\combL }(P^+ \Gamma^{-1/2} \sx)$. That is, set the lift of $\sx$ in $\reals^\N$ to be $\bx = P^+ \Gamma^{-1/2} \sx$, then we have that 
$\frac{\sx^T \sdwL \sx}{\sx^T \sx} = \frac{\bx^T \combL \bx}{\bx^T \bx}$. 
\end{proposition} 

Finally, if using the normalized Laplace $\Lcal$ for the original graph $G$, then the appropriate Laplace operator for the coarse graph and corresponding projection/lift maps are listed in the last row of Table \ref{tab:operator}, with proofs in \Cref{appendix-missing-proof}. 

\subsection{A GNN-based Framework for Constructing the Coarse Graph } \label{subsec:weights}

\begin{figure}[htbp]
\begin{center}

\centering
\includegraphics[width=0.8\linewidth]{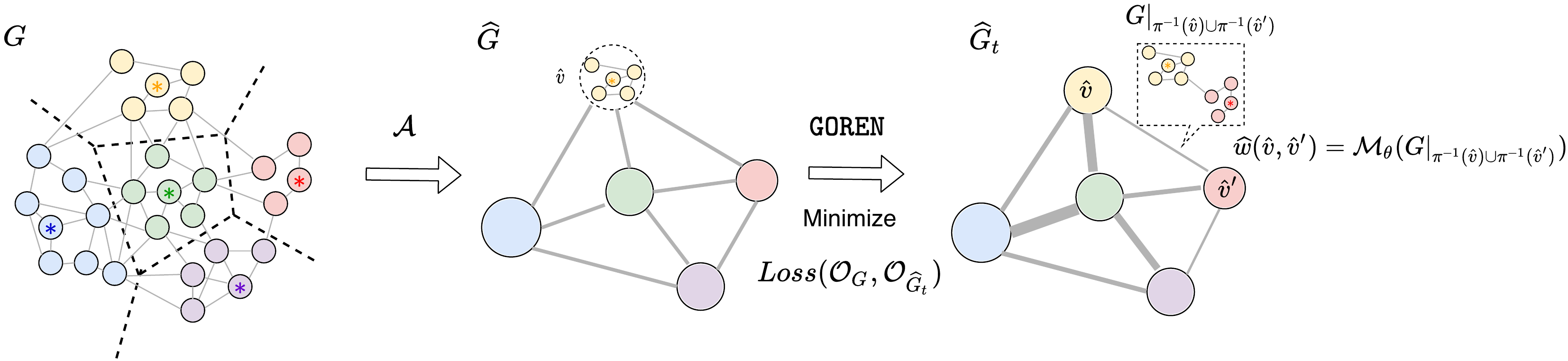}
\vspace*{-0.05in}\caption{An illustration of learnable coarsening framework. Existing coarsening algorithm determines the topology of coarse graph $\Ghat$, while \nn resets the edge weights of the coarse graph. }
\label{framework-fig}
\end{center}
\end{figure}

In the previous section, we argued that depending on what similarity measures we use, appropriate Laplace operator $\sgenL$ for the coarse graph $\Ghat$ should be used. 
Now consider the specific case of Rayleigh quotient, which can be thought of as a proxy to measure similarities between the low-frequency eigenvalues of the original graph Laplacian and the one for the coarse graph. As described above, here we set $\sgenL$ as the doubly-weighted Laplacian $\sdwL = \Gamma^{-1/2} (\Dhat - \What) \Gamma^{-1/2}$. 

\paragraph{The effect of weight adjustments. }
\begin{wrapfigure}{r}{0.20\textwidth}
 \begin{center}
\includegraphics[width=0.19\textwidth]{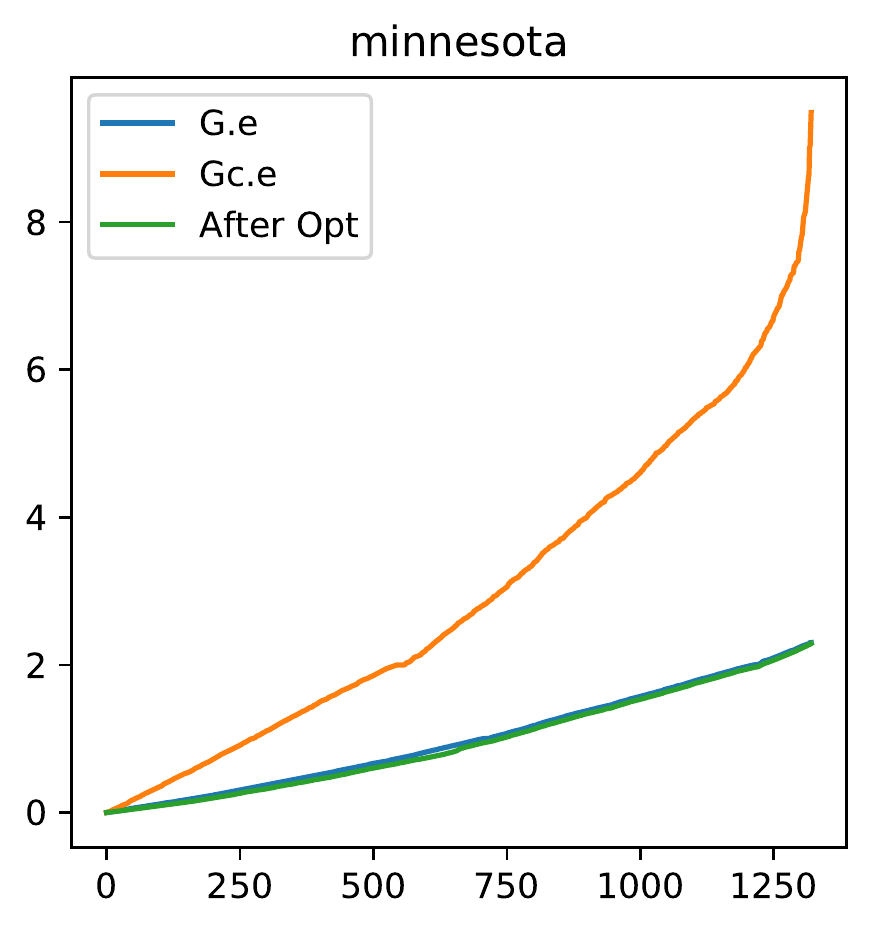}
 \end{center}
\end{wrapfigure}

 We develop an iterative algorithm with convergence guarantee (to KKT point in \ref{convergence-appendix}) for optimizing over edge weights of $\Ghat$ for better spectrum alignment. 
As shown in the figure on the right, after changing the edge weight of the coarse graph, the resulting graph Laplacian has eigenvalues much closer (almost identical) to the first $\n$ eigenvalues of the original graph Laplacian. More specifically, in this figure, $G.e$ and $Gc.e$ stand for the eigenvalues of the original graph $G$ and coarse graph $\Ghat$ constructed by the so-called Variation-Edge coarsening algorithm \cite{loukas2019graph}. ``After-Opt" stands for the eigenvalues of coarse graphs when weights are optimized by our iterative algorithm. See \Cref{iterative-algo-appendix} for the description of our iterative algorithm, its convergence results, and full experiment results.

{{\bf {A GNN-based framework for learning weight assignment map.}~}
The discussions above indicate that we can obtain better Laplace operators for the coarse graph by using better-informed weights than simply summing up the weights of crossing edges from the two clusters. 

More specifically, suppose we have a fixed strategy to generate $\Vhat$ from an input graph $G = (V, E)$. 
Now given an edge $(\vhat, \vhat') \in \Ehat$ in the induced coarse graph $\Ghat = (\Vhat, \Ehat)$, we model its weight $\what(\vhat, \vhat')$ by a \emph{weight-assignment function} $\mu( G|_{\vmap^{-1}(\vhat) \cup \vmap^{-1}(\vhat')} )$, where $G|_A$ is the subgraph of $G$ induced by a subset of vertices $A$. 
However, it is not clear how to setup this function $\mu$. Instead, we will learn it from a collection of input graphs in an \emph{unsupervised} manner. 
Specifically, we will parametrize the weight-assignment map $\mu$ by a learnable neural network $\myMu$. 

See Figure \ref{framework-fig} for an illustration.

In particular, we use Graph Isomorphism Network (GIN) \cite{xu2018powerful} to represent $\myMu$. 
We initialize the model by setting the edge attribute of the coarse graph to be 1. 
Our node feature is set to be a 5-dimensional vector based on LDP (Local Degree Profile) \cite{cai2018simple}. We enforce the learned weight of the coarse graph to be positive by applying one extra ReLU layer to the final output. 
All models are trained with Adam optimizer with a learning rate of 0.001. 
See \Cref{model-detail} for more details. We name our model as \textbf{G}raph c\textbf{O}arsening \textbf{R}efinem\textbf{E}nt \textbf{N}etwork (\nntight). 

Given a graph $G$ and a coarsening algorithm $\mathcal{A}$, the general form of loss is
\begin{align}
Loss (\genL, \mathcal{O}_{\widehat{G_t}}) = \frac{1}{k} \sum_{i=1}^{k} | \mathcal{F}( \mathcal{O}_G ,f_i) - \mathcal{F} (\mathcal{O}_{\widehat{G_t}} , \proj f_i) |, 
\end{align}
where $f_i$ is signal on the original graph (such as eigenvectors) and $\proj f_i$ is its projection. We use $\mathcal{O}_{\widehat{G_t}}$ to denote the operator of the coarse graph \textit{during training}, while $\mathcal{O}_{\widehat{G}}$ standing for the operator defined w.r.t. the coarse graph output by coarsening algorithm $\mathcal{A}$. That is, we will start with $\mathcal{O}_{\Ghat}$ and modify it to $\mathcal{O}_{\widehat{G_t}}$ during the training.

The loss can be instantiated for different cases in Table \ref{tab:operator}. For example, a loss based on quadratic form means that we choose $\mathcal{O}_G, \mathcal{O}_{\widehat{G_t}}$ to be the combinatorial Laplacian of $G$ and $\widehat{G_t}$, and the resulting \emph{quadratic loss} has the form:  
\begin{align} \label{loss-equ}
Loss (L, \hL_t) = \frac{1}{k} \sum_{i=1}^{k} |f_i^T L f_i - (Pf_i)^T \hL_t (Pf_i)|. 
\end{align}
It can be seen as a natural analog of the loss for spectral sparsification in the context of graph coarsening, which is also adopted in \cite{loukas2019graph}. 

Similarly, one can use a loss based on the Rayleigh quotient, by choosing $\mathcal{F}$ from the second row of Table \ref{tab:operator}. 
Our framework for graph coarsening is flexible. Many different loss functions can be used as long as it is differentiable in the weights of the coarse graph. we will demonstrate this point in Section \ref{flex}. 

Finally, given a collection of training graphs $G_1, \ldots, G_m$, we will train for parameters in the module $\myMu$ to minimize the total loss on training graphs. 
When a test graph $G_{test}$ is given, we simply apply $\myMu$ to set up weight for each edge in $\widehat{G_{test}}$, obtaining a new graph $\widehat{G_{test, t}}$. We compare $Loss (\mathcal{O}_{G_{test}}, \mathcal{O}_{\widehat{G_{test, t}}})$ against $Loss (\mathcal{O}_{G_{test}}, \mathcal{O}_{\widehat{G_{test}}})$ and expect the former loss is smaller. 

\section{Experiments Setup}
\label{sec:exp}

In the following experiments, we apply six existing coarsening algorithms to obtain the coarsened vertex set $\Vhat$, which are Affinity \cite{livne2012lean}, Algebraic Distance \cite{chen2011algebraic}, Heavy edge matching \cite{dhillon2007weighted,ron2011relaxation}, as well as two local variation methods based on edge and neighborhood respectively \cite{loukas2019graph}, and a simple baseline (BL); See \Cref{methods} for detailed descriptions. The two local variation methods are considered to be state-of-the-art graph coarsening algorithms \cite{loukas2019graph}. We show that our \nn framework can improve the qualities of coarse graphs produced by these methods. 

\subsection{Dataset}
\subsubsection{Synthetic Graphs}
\label{syn_graphs}
 \er graphs (ER). $G(n, p)$ where $p = \frac{0.1 * 512}{n}$

 Random geometric graphs (GEO). The random geometric graph model places $n$ nodes uniformly at random in the unit cube. Two nodes are joined by an edge if the distance between the nodes is at most radius $r$. We set $r = \frac{5.12}{\sqrt{n}}$.

 Barabasi-Albert Graph (BA). A graph of $n$ nodes is grown by attaching new nodes each with $m$ edges that are preferentially attached to existing nodes with high degrees. We set $m$ to be $4$.

 Watts-Strogatz Graph (WS). It is first created from a ring over $n$ nodes. Then each node in the ring is joined to its $k$ nearest neighbors (or $k - 1$ neighbors if $k$ is odd). Then shortcuts are created by replacing some edges as follows: for each edge $(u, v)$ in the underlying "$n$-ring with $k$ nearest neighbors" with probability $p$ replace it with a new edge $(u, w)$ with a uniformly random choice of existing node $w$. We set $k, p$ to be $10$ and  $0.1$. 

\subsubsection{Dataset from Loukas's paper}
\label{loukas-data}
Yeast. Protein-to-protein interaction network in budding yeast, analyzed by \cite{jeong2001lethality}. The network has $N = 1458$ vertices and $M = 1948$ edges. %

Airfoil. Finite-element graph obtained by airow simulation \cite{preis1997party}, consisting of $N = 4000$ vertices and $M = 11,490$ edges. %

Minnesota \cite{gleich2008matlabbgl}. Road network with $N = 2642$ vertices and $M = 3304$ edges. %

Bunny \cite{turk1994zippered}. Point cloud consisting of $N = 2503$ vertices and $M = 65,490$ edges. The point cloud has been sub-sampled from its original size. %

\subsubsection{Real Networks}
\label{real_graphs}
 Shape graphs (Shape). Each graph is KNN graph formed by 1024 points sampled from shapes from ShapeNet where each node is connected 10 nearest neighbors. 

Coauthor-CS (CS) and Coauthor-Physics (Physics) are co-authorship graphs based on the Microsoft Academic Graph from the KDD Cup 2016 challenge. Coauthor CS has $N = 18,333$ nodes and $M = 81,894$ edges. Coauthor Physics has $N = 34,493$ nodes and $M = 247,962$ edges.

PubMed \cite{sen2008collective} has $N = 19,717$ nodes and $M = 44,324$ edges. Nodes are documents and edges are citation links.

Flickr \cite{zeng2019graphsaint} has $N=89,250$ nodes and $M=899,756$ edges. One node in the graph represents one image uploaded to Flickr. If two images share some common properties (e.g., same geographic location, same gallery, comments by the same user, etc.), there is an edge between the nodes of these two images.

\subsection{Existing Graph Coarsening Methods}
\label{methods}
\textbf{Heavy Edge Matching}. 
At each level of the scheme, the contraction family is obtained
by computing a maximum-weight matching with the weight of each contraction set $(v_i, v_j)$ calculated as $w_{ij} / \text{max}\{d_i, d_j\}$. In this manner, heavier edges connecting vertices that are well separated from the rest of the graph are contracted first. %

\textbf{Algebraic Distance}.
This method differs from heavy edge matching in that the weight of each candidate set $(v_i, v_j) \in E$ is calculated as $\left(\sum_{q=1}^{Q}\left(x_{q}(i)-x_{q}(j)\right)^{2}\right)^{1 / 2}$, where $x_k$ is an $N$-dimensional test vector computed by successive sweeps of Jacobi relaxation. The complete method is described by \cite{ron2011relaxation}, see also \cite{chen2011algebraic}. %

\textbf{Affinity}.
This is a vertex proximity heuristic in the spirit of the algebraic distance that was proposed by \cite{livne2012lean} in the context of their work on the lean algebraic multigrid. As per the author suggests, the $Q = k$ test vectors are here computed by a single sweep of a Gauss-Seidel iteration.

\textbf{Local Variation}. There are two variations of local variation methods, edge-based local variation, and neighborhood-based local variation. They differ in how the contraction set is chosen. Edge-based variation is constructed for each edge, while the neighborhood-based variant takes every vertex and its neighbors as contraction set. What two methods have common is that they both optimize an upper bound of the restricted spectral approximation objective. In each step, they greedily pick the sets whose local variation is the smallest. See \cite{loukas2019graph} for more details. 

\textbf{Baseline}. We also implement a simple baseline that randomly chooses a collection of nodes in the original graph as landmarks and contract other nodes to the nearest landmarks. If there are multiple nearest landmarks, we randomly break the tie. The weight of the coarse graph is set to be the sum of the weights of the crossing edges. 
\subsection{Details of the Experimental Setup}
\label{model-detail}
\textbf{Feature Initialization.} We initialize the the node feature of subgraphs as a 5 dimensional feature based on a simple heuristics local degree profile (LDP) \cite{cai2018simple}. For each node $v \in G(V)$, let $DN(v)$ denote the multiset of the degree of all the neighboring nodes of $v$, i.e., $DN(v) = \{\text{degree}(u) | (u, v) \in E \}$. We take five node features, which are (degree($v$), min(DN($v$)), max(DN($v$)),mean(DN($v$)), std(DN($v$))). In other words, each node feature summarizes the degree information of this node and its 1- neighborhood. We use the edge weight as 1 dimensional edge feature.

\textbf{Optimization.} All models are trained with Adam optimizer \cite{kingma2014adam} with a learning rate of 0.001 and batch size 600. We use Pytorch \cite{paszke2017automatic} and Pytorch Geometric \cite{fey2019fast}
for all of our implementation. We train graphs one by one where for each graph we train the model to minimize the loss for certain epochs (see hyper-parameters for details) before moving to the next graph. We save the model that performs best on the validation graphs and test it on the test graphs.  

\textbf{Model Architecture.}
The building block of our graph neural networks is based on the modification of Graph Isomorphism Network (GIN) that can handle both node and edge features. In particular, we first linear transform both node feature and edge feature to be vectors of the same dimension. At the $k$-th layer, GNNs update node representations by

\begin{equation}
h_{v}^{(k)}=\operatorname{ReLU}\left(\operatorname{MLP}^{(k)}\left(\sum_{u \in \mathcal{N}(v) \cup\{v\}} h_{u}^{(k-1)}+\sum_{e=(v, u): u \in \mathcal{N}(v) \cup\{v\}} h_{e}^{(k-1)}\right)\right)
\end{equation}

where $\mathcal{N}(v)$ is a set of nodes adjacent to $v$, and $e = (v; v)$ represents the self-loop edge. Edge features $h_e^{(k-1)}$ is the same across the layers. 

We use average graph pooling to obtained the graph representation from node embeddings, i.e., $h_{G}=\operatorname{MEAN}\left(\left\{h_{v}^{(K)} | v \in G\right\}\right)$.
 The final prediction of weight is $ 1 + \operatorname{ReLu}(\Phi(h_G) )$ where $\Phi$ is a linear layer. We set the number of layers to be 3 and the embedding dimension to be 50.

\textbf{Time Complexity.}
In the preprocessing step, we need to compute the first $k$ eigenvectors of Laplacian (either combinatorial or normalized one) of the original graph as test vectors. Those can be efficiently computed by Restarted Lanczos Method \cite{lehoucq1998arpack} to find the eigenvalues and eigenvectors.

In the training time, our model needs to recompute the term in the loss involving the coarse graph to update the weights of the graph neural networks for each batch. For loss involving Laplacian (either combinatorial or normalized Laplacian), the time complexity to compute the $x^TLx$ is $O(|E|k)$ where $|E|$ is the number of edges in the coarse graph and $k$ is the number of test vectors. For loss involving conductance, computing the conductance of one subset $S \subset E$ is still $O(|E|)$ so in total the time complexity is also $O(|E|k)$. 
In summary, the time complexity for each batch is linear in the number of edges of training graphs.  All experiments are performed on a single Intel Xeon CPU E5-2630 v4@ 2.20GHz $\times$ 40 and 64GB RAM machine.

More concretely, for synthetic graphs, it takes a few minutes to train the model. For real graphs like CS, Physics, PubMed, it takes around 1 hour. For the largest network Flickr of 89k nodes and 899k edges, it takes about 5 hours for most coarsening algorithms and reduction ratios.

\textbf{Hyperparameters.}
We list the major hyperparameters of \nn below.
\begin{itemize}
\item epoch: 50 for synthetic graphs and 30 for real networks.
\item walk length: 5000 for real networks. Note the size of the subgraph is usually around 3500 since the random walk visits some nodes more than once.
\item number of eigenvectors $k$: 40 for synthetic graphs and 200 for real networks.
\item embedding dimension: 50
\item batch size: 600
\item learning rate: 0.001
\end{itemize}

\section{Proof of Concept}
\label{sec:proof-of-concept}

\begin{table}[ht]
\centering
\renewcommand{\arraystretch}{1.2}
\caption{The error reduction after applying \nntight. %
}
\label{better-fit-short}
\begin{tabular}{@{}llllll@{}}
\toprule
 Dataset  & Affinity & \makecell[l]{Algebraic \\ Distance} &  \makecell[l]{Heavy \\ Edge} & \makecell[l]{Local var \\ (edges)} & \makecell[l]{Local var \\ (neigh.)} \\
\midrule 
   Airfoil &    91.7\% &  88.2\% &  86.1\%  & 43.2\% & 73.6\% \\
  Minnesota &    49.8\% &  57.2\% &  30.1\%  & 5.50\%  & 1.60\% \\ 
    Yeast &    49.7\% &  51.3\% &  37.4\% & 27.9\% & 21.1\% \\ 
    Bunny &    84.7\% &  69.1\% &  61.2\%  & 19.3\% & 81.6\%\\
\bottomrule
\end{tabular}
\end{table}

As proof of concept, we show that \nn can improve common coarsening methods on multiple graphs (see \ref{loukas-data} for details).
Following the same setting as \cite{loukas2019graph}, we use the relative eigenvalue error as evaluation metric. It is defined as
$ \frac{1}{k} \sum_{i=1}^{k} \frac{\left|\widehat{\lambda}_{i}-\lambda_{i}\right|}{\lambda_{i}}$, where $\lambda_i, \widehat{\lambda}_i$ denotes eigenvalues of combinatorial Laplacian $L$ for $G$ and doubly-weighted Laplacian $\mathsf{\Lhat}$ for $\Ghat$ respectively, and $k$ is set to be 40. For simplicity, this error is denoted as \textit{Eigenerror} in the remainder of the chapter. Denote the Eigenerror of graph coarsening method as $l_1$ and Eigenerror obtained by \nn as $l_2$. In Table \ref{better-fit-short}, we show the \textit{error-reduction ratio}, defined as $\frac{l_1-l_2}{l_1}$. The ratio is upper bounded by 100.

Since it is hard to directly optimize Eigenerror, the loss function we use in our \nn set to be the \emph{Rayleigh loss} $Loss (\genL, \mathcal{O}_{\widehat{G_t}}) = \frac{1}{k} \sum_{i=1}^{k} | \mathcal{F}( \mathcal{O}_G ,f_i) - \mathcal{F} (\mathcal{O}_{\widehat{G_t}} , \proj f_i) |$ where $\mathcal{F}$ is Rayleigh quotient, $\proj = \Gamma^{-1/2} \ppt$ and $\mathcal{O}_{\Ghat_t}$ being doubly-weighted Laplacian $\sdwL_t$. 
In other words, We use Rayleigh loss as a differentiable proxy for the Eigenerror. As we can see in Table \ref{better-fit-short}, \nn reduces the Eigenerror by a large margin for \textit{training} graphs, which serves as a sanity check for our framework, as well as for using Rayleigh loss as a proxy for Eigenerror. See \Cref{better-fit} for full results where we reproduce the results in \cite{loukas2019graph} up to small differences. 

In Table \ref{nondiff-table}, we will demonstrate this training strategy also generalizes well to unseen graphs. 

\begin{table}
\renewcommand{\arraystretch}{1.2}
\caption{Loss: quadratic loss. Laplacian: combinatorial Laplacian for both original and coarse graphs. Each entry $x (y)$ is: $x =$ loss w/o learning, and $y =$ improvement percentage.} 
\label{loss_quadratic_lap_none_eigen_False_ratio_5}
\begin{center}
\resizebox{1\textwidth}{!}{
\begin{tabular}{@{}lllllllll@{}}
\addlinespace[-\aboverulesep] 
  \cmidrule[\heavyrulewidth]{2-8}
  & Dataset &   BL & Affinity & \makecell[l]{Algebraic \\ Distance} &  \makecell[l]{Heavy \\ Edge} & \makecell[l]{Local var \\ (edges)} & \makecell[l]{Local var \\ (neigh.)} \\ 
  \cmidrule{2-8}
  \parbox[t]{2mm}{\multirow{4}{*}{\rotatebox[origin=c]{90}{Synthetic}}} & BA &  0.44 (16.1\%) &  0.44 (4.4\%) &   0.68 (4.3\%) &  0.61 (3.6\%) &    0.21 (14.1\%) &    0.18 (72.7\%) \\
  & ER &   0.36 (1.1\%) &  0.52 (0.8\%) &   0.35 (0.4\%) &  0.36 (0.2\%) &    0.18 (1.2\%) &     0.02 (7.4\%) \\
  & GEO &  0.71 (87.3\%) &  0.20 (57.8\%) &  0.24 (31.4\%) & 0.55 (80.4\%) &    0.10 (59.6\%) &    0.27 (65.0\%) \\
  & WS &  0.45 (62.9\%) & 0.09 (82.1\%) &  0.09 (60.6\%) &      0.52 (51.8\%) &    0.09 (69.9\%) &    0.11 (84.2\%) \\
  \cmidrule{2-8}
  \parbox[t]{2mm}{\multirow{5}{*}{\rotatebox[origin=c]{90}{Real}}} & CS &  0.39 (40.0\%) & 0.21 (29.8\%) &  0.17 (26.4\%) &  0.14 (20.9\%) &    0.06 (36.9\%) &     0.0 (59.0\%) \\
    & Flickr &  0.25 (10.2\%) &  0.25 (5.0\%) &   0.19 (6.4\%) &   0.26 (5.6\%) &    0.11 (11.2\%) &    0.07 (21.8\%) \\
  & Physics &  0.40 (47.4\%) & 0.37 (42.4\%) &  0.32 (49.7\%) &  0.14 (28.0\%) &    0.15 (60.3\%) &     0.0 (-0.3\%) \\
  & PubMed &  0.30 (23.4\%) & 0.13 (10.5\%) &  0.12 (15.9\%) &  0.24 (10.8\%) &    0.06 (11.8\%) &    0.01 (36.4\%) \\
      & Shape &  0.23 (91.4\%) & 0.08 (89.8\%) &  0.06 (82.2\%) & 0.17 (88.2\%) &    0.04 (80.2\%) &    0.08 (79.4\%) \\
  \cmidrule[\heavyrulewidth]{2-8} \addlinespace[-\belowrulesep] 
\end{tabular}
}
\end{center}
\end{table}

\subsection{Synthetic Graphs}
\label{subsec:syn_graph}
We train the \nn on synthetic graphs from common graph generative models and test on larger unseen graphs from the same model. We randomly sample 25 graphs of size $\{512, 612, 712, ..., 2912\}$ from different generative models. If the graph is disconnected, we keep the largest component. We train \nn on the first 5 graphs, use the 5 graphs from the rest 20 graphs as the validation set and the remaining 15 as test graphs. We use the following synthetic graphs: \er graphs (ER), Barabasi-Albert Graph (BA), Watts-Strogatz Graph (WS), random geometric graphs (GEO). See \Cref{syn_graphs} for datasets details. 

We report both the loss $Loss(L, \widehat{L})$ of different algorithms (w/o learning) and the \emph{relative improvement percentage} defined as $\frac{Loss(L, \widehat{L}) - Loss(L, \widehat{L_t})}{Loss(L, \widehat{L})}$ when \nn is applied, shown in parenthesis. 

Erdős-Rényi graphs (ER). $G(n, p)$ where $p = \frac{0.1 * 512}{n}$

As we can see in Table \ref{loss_quadratic_lap_none_eigen_False_ratio_5}, for most methods, trained on small graphs, \nn also performs well on test graphs of larger size across different algorithms and datasets -- Again, the larger improvement percentage is, the larger the improvement by our algorithm is, and a negative value means that our algorithm makes the loss worse. 
Note the size of test graphs are on average $2.6\times$ the size of training graphs. 

For ER and BA graphs, the improvement is relatively smaller compared to GEO and WS graphs. This makes sense since ER and BA graphs are rather homogenous graphs, leaving less room for further improvement. 

\subsection{Real Networks}
\label{subsec:real_graph}
We test on five real networks: Shape, PubMed, Coauthor-CS (CS), Coauthor-Physics (Physics), and Flickr (largest one with 89k vertices), which are much larger than datasets used in \cite{hermsdorff2019unifying} ($\leq$ 1.5k) and \cite{loukas2019graph} ($\leq$ 4k). Since it is hard to obtain multiple large graphs (except for the Shape dataset, which contains meshes from different surface models) coming from similar distribution, we bootstrap the training data in the following way. For the given graph, we randomly sample a collection of landmark vertices and take a random walk of length $l$ starting from selected vertices. We take subgraphs spanned by vertices of random walks as training and validation graphs and the original graph as the test graph. See \Cref{real_graphs} for dataset details. 

As shown in the bottom half of Table \ref{loss_quadratic_lap_none_eigen_False_ratio_5}, across all six different algorithms, \nn significantly improves the result among all five datasets in most cases. For the largest graph Flickr, the size of test graphs is more than $25\times$ of the training graphs, which further demonstrates the strong generalization. 

\begin{table}
\label{better-fit}
\renewcommand{\arraystretch}{1.2}
\caption{Relative eigenvalue error (Eigenerror) by different coarsening algorithm and the improvement (in percentage) after applying \nntight.}
\begin{center}
\resizebox{\textwidth}{!}{
\begin{tabular}{@{}lllllll@{}}
\toprule
  Dataset &    Ratio  & Affinity & \makecell[l]{Algebraic \\ Distance} &    \makecell[l]{Heavy \\ Edge} & \makecell[l]{Local var \\ (edges)} & \makecell[l]{Local var \\ (neigh.)} \\
\midrule 
     &   0.3 &  0.262 (82.1\%) &   0.208 (64.9\%) &  0.279 (80.3\%) &     0.102 (-67.6\%) &            0.184 (69.6\%) \\
     Airfoil &   0.5 &   0.750 (91.7\%) &   0.672 (88.2\%) &  0.568 (86.1\%) &      0.336 (43.2\%) &            0.364 (73.6\%) \\
      &   0.7 &  2.422 (96.4\%) &   2.136 (93.5\%) &  1.979 (96.7\%) &      0.782 (78.8\%) &            0.876 (87.8\%) \\ \cmidrule{2-7}
    &   0.3 &  0.322 (-5.0\%) &    0.206 (0.5\%) &  0.357 (-4.5\%) &      0.118 (-5.9\%) &           0.114 (-14.0\%) \\
   Minnesota &   0.5 &  1.345 (49.8\%) &   1.054 (57.2\%) &  0.996 (30.1\%) &       0.457 (5.5\%) &             0.382 (1.6\%) \\ 
    &   0.7 &   4.290 (70.4\%) &   3.787 (76.6\%) &  3.423 (58.9\%) &      2.073 (55.0\%) &                   1.572 (38.1\%) \\ \cmidrule{2-7}
        &   0.3 &  0.202 (10.4\%) &    0.108 (5.6\%) &   0.291 (1.4\%) &       0.113 (6.2\%) &           0.024 (-58.3\%) \\
       Yeast &   0.5 &  0.795 (49.7\%) &   0.485 (51.3\%) &   1.080 (37.4\%) &      0.398 (27.9\%) &            0.133 (21.1\%) \\ 
        &   0.7 &   2.520 (60.4\%) &   2.479 (72.4\%) &  3.482 (52.9\%) &      2.073 (58.9\%) &            0.458 (45.9\%) \\ \cmidrule{2-7}
        &   0.3 &     0.046 (32.6\%) & 0.217 (50.0\%) & 0.258 (74.4\%) & 0.007 (-328.5\%) & 0.082 (74.8\%) \\
       Bunny & 0.5 &   0.085 (84.7\%) &  0.372 (69.1\%) & 0.420 (61.2\%) & 0.057 (19.3\%) & 0.169 (81.6\%)\\
        &   0.7 &   0.182 (84.6\%) &  0.574 (78.6\%) & 0.533 (75.4\%) & 0.094 (45.7\%) & 0.283 (73.9\%)\\
\bottomrule
\end{tabular}
}
\end{center}
\end{table}

\begin{table}
\renewcommand{\arraystretch}{1.2}
\caption{Loss: quadratic loss. Laplacian: \textit{combinatorial} Laplacian for both original and coarse graphs. Each entry $x (y)$ is: $x =$ loss w/o learning, and $y =$ improvement percentage. BL stands for the baseline. } 
\label{loss_quadratic_lap_none_eigen_False}
\begin{center}
\scalebox{.8}{
\begin{tabular}{@{}llllllll@{}}
\toprule
 Dataset &     Ratio & BL & Affinity & \makecell[l]{Algebraic \\ Distance} &    \makecell[l]{Heavy \\ Edge} & \makecell[l]{Local var \\ (edges)} & \makecell[l]{Local var \\ (neigh.)} \\ \midrule %
  &  0.3 &   0.36 (6.8\%) &   0.22 (2.9\%) &     0.56 (1.9\%) &   0.49 (1.7\%) &       0.06 (16.6\%) &        0.17 (73.1\%) \\
    BA &  0.5 &  0.44 (16.1\%) &   0.44 (4.4\%) &     0.68 (4.3\%) &   0.61 (3.6\%) &       0.21 (14.1\%) &        0.18 (72.7\%) \\
     &  0.7 &  0.21 (32.0\%) &  0.43 (16.5\%) &    0.47 (17.7\%) &   0.4 (19.3\%) &        0.2 (48.2\%) &        0.11 (11.1\%) \\
\midrule
       &  0.3 &  0.25 (28.7\%) &  0.08 (24.8\%) &    0.05 (21.5\%) &    0.09 (15.6\%) &     0.0 (-254.3\%) &         0.0 (60.6\%) \\
      CS &  0.5 &  0.39 (40.0\%) &  0.21 (29.8\%) &    0.17 (26.4\%) &    0.14 (20.9\%) &       0.06 (36.9\%) &         0.0 (59.0\%) \\
       &  0.7 &  0.46 (55.5\%) &  0.57 (36.8\%) &    0.33 (36.6\%) &    0.28 (29.3\%) &       0.18 (44.2\%) &        0.09 (26.5\%) \\
\midrule
  &  0.3 &  0.26 (35.4\%) &  0.36 (36.6\%) &     0.2 (29.7\%) &     0.1 (18.6\%) &    0.0 (-42.0\%) &          0.0 (2.5\%) \\
 Physics &  0.5 &   0.4 (47.4\%) &  0.37 (42.4\%) &    0.32 (49.7\%) &    0.14 (28.0\%) &       0.15 (60.3\%) &         0.0 (-0.3\%) \\
  &  0.7 &  0.47 (60.0\%) &  0.53 (55.3\%) &    0.42 (61.4\%) &    0.27 (34.4\%) &       0.25 (67.0\%) &        0.01 (-4.9\%) \\
 \midrule
   &  0.3 &   0.16 (5.3\%) &   0.17 (2.0\%) &     0.08 (4.3\%) &     0.18 (2.7\%) &       0.01 (16.0\%) &        0.02 (33.7\%) \\
  Flickr &  0.5 &  0.25 (10.2\%) &   0.25 (5.0\%) &     0.19 (6.4\%) &     0.26 (5.6\%) &       0.11 (11.2\%) &        0.07 (21.8\%) \\
   &  0.7 &  0.28 (21.0\%) &  0.31 (12.4\%) &    0.37 (18.7\%) &    0.33 (11.3\%) &        0.2 (17.2\%) &         0.2 (21.4\%) \\
  \midrule
   &  0.3 &  0.17 (13.6\%) &   0.06 (6.2\%) &     0.03 (9.5\%) &      0.1 (4.7\%) &       0.01 (18.8\%) &         0.0 (39.9\%) \\
  PubMed &  0.5 &   0.3 (23.4\%) &  0.13 (10.5\%) &    0.12 (15.9\%) &    0.24 (10.8\%) &       0.06 (11.8\%) &        0.01 (36.4\%) \\
   &  0.7 &  0.31 (41.3\%) &  0.23 (22.4\%) &     0.14 (8.3\%) &  0.14 (-491.6\%) &       0.16 (12.5\%) &        0.05 (21.2\%) \\
\midrule
     &  0.3 &   0.25 (0.5\%) &   0.41 (0.2\%) &      0.2 (0.5\%) &   0.23 (0.2\%) &        0.01 (4.8\%) &         0.01 (5.9\%) \\
    ER &  0.5 &   0.36 (1.1\%) &   0.52 (0.8\%) &     0.35 (0.4\%) &   0.36 (0.2\%) &        0.18 (1.2\%) &         0.02 (7.4\%) \\
     &  0.7 &   0.39 (3.2\%) &   0.55 (2.5\%) &     0.44 (2.0\%) &   0.43 (0.8\%) &        0.23 (2.9\%) &        0.29 (10.4\%) \\
\midrule
    &  0.3 &  0.44 (86.4\%) &  0.11 (65.1\%) &    0.12 (81.5\%) &  0.34 (80.7\%) &        0.01 (0.3\%) &        0.14 (70.4\%) \\
   GEO &  0.5 &  0.71 (87.3\%) &   0.2 (57.8\%) &    0.24 (31.4\%) &  0.55 (80.4\%) &        0.1 (59.6\%) &        0.27 (65.0\%) \\
    &  0.7 &  0.96 (83.2\%) &   0.4 (55.2\%) &    0.33 (54.8\%) &  0.72 (90.0\%) &       0.19 (72.4\%) &        0.41 (61.0\%) \\
\midrule
  &  0.3 &  0.13 (86.6\%) &  0.04 (79.8\%) &    0.03 (69.0\%) &  0.11 (69.7\%) &         0.0 (1.3\%) &        0.04 (73.6\%) \\
 Shape &  0.5 &  0.23 (91.4\%) &  0.08 (89.8\%) &    0.06 (82.2\%) &  0.17 (88.2\%) &       0.04 (80.2\%) &        0.08 (79.4\%) \\
  &  0.7 &  0.34 (91.1\%) &  0.17 (94.3\%) &     0.1 (74.7\%) &  0.24 (95.9\%) &       0.09 (64.6\%) &        0.13 (84.8\%) \\
\midrule
     &  0.3 &  0.27 (46.2\%) &  0.04 (65.6\%) &   0.04 (-26.9\%) &  0.43 (32.9\%) &       0.02 (68.2\%) &        0.06 (75.2\%) \\
    WS &  0.5 &  0.45 (62.9\%) &  0.09 (82.1\%) &    0.09 (60.6\%) &            0.52 (51.8\%) &       0.09 (69.9\%) &        0.11 (84.2\%) \\
     &  0.7 &  0.65 (73.4\%) &  0.15 (78.4\%) &    0.14 (66.7\%) &  0.67 (76.6\%) &       0.15 (80.8\%) &        0.16 (83.2\%) \\
\bottomrule
\end{tabular}
}
\end{center}
\end{table}

\section{Other Losses}
\label{flex}
{\bf {Other differentiable loss.}} To demonstrate that our framework is flexible, we adapt \nn to the following two losses. The two losses are both differentiable w.r.t the weights of coarse graph.

(1) Loss based on normalized graph Laplacian: $
Loss (\mathcal{L}, \mathcal{\hL}_t) = \frac{1}{k} \sum_{i=1}^k |f_i^T \mathcal{L} f_i - (\proj f_i)^T \mathcal{\hL}_t (\proj f_i)| $. Here $\{ f_i\}$ are the set of first $k$ eigenvectors of the normalized Laplacian $\mathcal{L}$ of original grpah $G$, and $ \proj = \widehat{D}^{1/2}PD^{-1/2}$. (2) Conductance difference between original graph and coarse graph. $Loss = \frac{1}{k} \sum_{i=1}^k | \varphi(S_i) - \varphi(\pi(S_i)) | $. $\varphi(S)$ is the conductance $\varphi(S):=\frac{\sum_{i \in S, j \in \bar{S}} a_{i j}}{\min (a(S), a(\bar{S}))}$ where $a(S):=\sum_{i \in S} \sum_{j \in V} a_{i j}$. We randomly sample $k$ subsets of nodes $S_0, ..., S_k \subset V$ where 

$|S_i|$ is set to be a random number sampled from the uniform distribution $U(|V|/4, |V|/2)$. Due to space limits, we present the result for conductance in \Cref{conductance-loss-appendix}. 

Following the same setting as before, we perform experiments to minimize two different losses. As shown in \Cref{loss_quadratic_lap_sym_eigen_False_ratio_5} and \Cref{conductance-loss-appendix}, for most graphs and methods, \nn still shows good generalization capacity and improvement for both losses. Apart from that, we also observe the initial loss for normalized Laplacian is much smaller than that for standard Laplacian, which might be due to that the fact that eigenvalues of normalized Laplacian are in $[0, 2]$.

\begin{table}
\caption{Loss: quadratic loss. Laplacian: normalized Laplacian for original and coarse graphs. Each entry $x (y)$ is: $x =$ loss w/o learning, and $y =$ improvement percentage.}
\label{loss_quadratic_lap_sym_eigen_False_ratio_5}
\renewcommand{\arraystretch}{1.2}
\begin{center}
\resizebox{1\textwidth}{!}{
\begin{tabular}{@{}lllllllll@{}}
\addlinespace[-\aboverulesep] 
  \cmidrule[\heavyrulewidth]{2-8}
  & Dataset &   BL & Affinity & \makecell[l]{Algebraic \\ Distance} &  \makecell[l]{Heavy \\ Edge} & \makecell[l]{Local var \\ (edges)} & \makecell[l]{Local var \\ (neigh.)} \\ 
  \cmidrule{2-8}
  \parbox[t]{2mm}{\multirow{4}{*}{\rotatebox[origin=c]{90}{Synthetic}}} & BA &  0.13 (76.2\%) & 0.14 (45.0\%) &  0.15 (51.8\%) &  0.15 (46.6\%) &    0.14 (55.3\%) &    0.06 (57.2\%) \\
  & ER &  0.10 (82.2\%) &   0.10 (83.9\%) &  0.09 (79.3\%) &  0.09 (78.8\%) &    0.06 (64.6\%) &    0.06 (75.4\%) \\
  & GEO &  0.04 (52.8\%) &  0.01 (12.4\%) &  0.01 (27.0\%) &  0.03 (56.3\%) &   0.01 (-145.1\%) &    0.02 (-9.7\%) \\
  & WS &  0.05 (83.3\%) &  0.01 (-1.7\%) &  0.01 (38.6\%) & 0.05 (50.3\%)       &    0.01 (40.9\%) &  0.01 (10.8\%) \\
  \cmidrule{2-8}
  \parbox[t]{2mm}{\multirow{5}{*}{\rotatebox[origin=c]{90}{Real}}} & CS & 0.08 (58.0\%) &  0.06 (37.2\%) &  0.04 (12.8\%) &  0.05 (41.5\%) &    0.02 (16.8\%) &    0.01 (50.4\%) \\
  & Flickr &  0.08 (-31.9\%)&       0.06 (-27.6\%) &  0.06 (-67.2\%) &  0.07 (-73.8\%) &   0.02 (-440.1\%) &    0.02 (-43.9\%) \\
  & Physics &  0.07 (47.9\%) &  0.06 (40.1\%) &  0.04 (17.4\%) &  0.04 (61.4\%) &   0.02 (-23.3\%) &    0.01 (35.6\%) \\
  & PubMed &   0.05 (47.8\%) &  0.05 (35.0\%) &  0.05 (41.1\%) &  0.12 (46.8\%) &   0.03 (-66.4\%) &   0.01 (-118.0\%) \\
  & Shape &  0.02 (84.4\%) &  0.01 (67.7\%) &  0.01 (58.4\%) &  0.02 (87.4\%) &    0.0 (13.3\%) &    0.01 (43.8\%) \\
  \cmidrule[\heavyrulewidth]{2-8} \addlinespace[-\belowrulesep] 
\end{tabular}
}
\end{center}
\end{table}

{\bf {Non-differentiable loss}.} In Section \ref{sec:proof-of-concept}, we use Rayleigh loss as a proxy for training but the Eigenerror for validation and test.
Here we train \nn with Rayleigh loss but evaluate \emph{Eigenerror} on \textit{test} graphs, which is more challenging. Number of vectors $k$ is 40 for synthetic graphs and 200 for real networks. 

 As shown in Table \ref{nondiff-table}, our training strategy via Rayleigh loss can improve the eigenvalue alignment between original graphs and coarse graphs in most cases. Reducing Eigenerror is more challenging than other losses, possibly because we are minimizing a differentiable proxy (the Rayleigh loss). Nevertheless, improvement is achieved in most cases.

\begin{table}
\caption{Loss: Eigenerror. Laplacian: combinatorial Laplacian for original graphs and doubly-weighted Laplacian for coarse ones. Each entry $x (y)$ is: $x =$ loss w/o learning, and $y =$ improvement percentage. $\dagger$ stands for out of memory.}
\label{nondiff-table}
\renewcommand{\arraystretch}{1.2}
\begin{center}
\resizebox{1\textwidth}{!}{
\begin{tabular}{@{}lllllllll@{}}
\addlinespace[-\aboverulesep] 
  \cmidrule[\heavyrulewidth]{2-8}
  & Dataset &   BL & Affinity & \makecell[l]{Algebraic \\ Distance} &  \makecell[l]{Heavy \\ Edge} & \makecell[l]{Local var \\ (edges)} & \makecell[l]{Local var \\ (neigh.)} \\ 
  \cmidrule{2-8}
  \parbox[t]{2mm}{\multirow{4}{*}{\rotatebox[origin=c]{90}{Synthetic}}} & BA &   0.36 (7.1\%) &  0.17 (8.2\%) &   0.22 (6.5\%) &  0.22 (4.7\%) &    0.11 (21.1\%) &    0.17 (-15.9\%) \\
  & ER &  0.61 (0.5\%) &  0.70 (1.0\%) &   0.35 (0.6\%) &  0.36 (0.2\%) &    0.19 (1.2\%) &     0.02 (0.8\%) \\
  & GEO &  1.72 (50.3\%) & 0.16 (89.4\%) &  0.18 (91.2\%) & 0.45 (84.9\%) &    0.08 (55.6\%) &     0.20 (86.8\%) \\
  & WS &  1.59 (43.9\%) & 0.11 (88.2\%) &  0.11 (83.9\%) &  0.58 (23.5\%) &    0.10 (88.2\%) &    0.12 (79.7\%) \\
  \cmidrule{2-8}
  \parbox[t]{2mm}{\multirow{5}{*}{\rotatebox[origin=c]{90}{Real}}} & CS & 1.10 (18.0\%) & 0.55 (49.8\%) &  0.33 (60.6\%) &  0.42 (44.5\%) &    0.21 (75.2\%) &    0.0 (-154.2\%) \\
  & Flickr & 0.57 (55.7\%) & $\dagger$ &  0.33 (20.2\%) &  0.31 (55.0\%) &    0.11 (67.6\%) &    0.07 (60.3\%) \\
  & Physics &  1.06 (21.7\%) & 0.58 (67.1\%) &  0.33 (69.5\%) &  0.35 (64.6\%) &    0.20 (79.0\%)   & 0.0 (-377.9\%) \\ %
  & PubMed &  1.25 (7.1\%) &  0.50 (15.5\%) &  0.51 (12.3\%) & 1.19 (-110.1\%) &    0.35 (-8.8\%) &    0.02 (60.4\%) \\
  & Shape &  2.07 (67.7\%) & 0.24 (93.3\%) &  0.17 (90.9\%) & 0.49 (93.0\%) &    0.11 (84.2\%) &     0.20 (90.7\%) \\
  \cmidrule[\heavyrulewidth]{2-8} \addlinespace[-\belowrulesep] 
\end{tabular}
}
\end{center}
\end{table}

\begin{table}
\renewcommand{\arraystretch}{1.2}

\caption{Loss: quadratic loss. Laplacian: \textit{normalized} Laplacian for both original and coarse graphs. Each entry $x (y)$ is: $x =$ loss w/o learning, and $y =$ improvement percentage. BL stands for the baseline. } 
\label{loss_quadratic_lap_sym_eigen_False}

\begin{center}
\resizebox{\textwidth}{!}{
\begin{tabular}{@{}llllllll@{}}
\toprule
 Dataset &     Ratio & BL & Affinity & \makecell[l]{Algebraic \\ Distance} &    \makecell[l]{Heavy \\ Edge} & \makecell[l]{Local var \\ (edges)} & \makecell[l]{Local var \\ (neigh.)} \\ \midrule %
      &  0.3 &    0.06 (68.6\%) &  0.07 (73.9\%) &    0.08 (80.6\%) &                               0.08 (79.6\%) &       0.06 (79.4\%) &       0.01 (-15.8\%) \\
      BA &  0.5 &    0.13 (76.2\%) &  0.14 (45.0\%) &    0.15 (51.8\%) &                               0.15 (46.6\%) &       0.14 (55.3\%) &        0.06 (57.2\%) \\
       &  0.7 &    0.22 (17.0\%) &   0.23 (5.5\%) &    0.24 (10.8\%) &                                0.24 (9.7\%) &        0.23 (5.4\%) &        0.17 (36.8\%) \\
\midrule
       &  0.3 & 0.04 (50.2\%) &   0.03 (44.1\%) &    0.01 (-7.0\%) &    0.03 (50.1\%) &      0.0 (-135.0\%) &       0.01 (-11.7\%) \\
      CS &  0.5 & 0.08 (58.0\%) &   0.06 (37.2\%) &    0.04 (12.8\%) &    0.05 (41.5\%) &       0.02 (16.8\%) &        0.01 (50.4\%) \\
       &  0.7 & 0.13 (57.8\%) &    0.1 (36.3\%) &    0.09 (21.4\%) &    0.09 (29.3\%) &       0.05 (11.6\%) &        0.04 (10.8\%) \\
\midrule
 &  0.3 & 0.05 (32.3\%) &    0.04 (5.4\%) &   0.02 (-16.5\%) &    0.03 (69.3\%) &     0.0 (-1102.4\%) &        0.0 (-59.8\%) \\
 Physics &  0.5 & 0.07 (47.9\%) &   0.06 (40.1\%) &    0.04 (17.4\%) &    0.04 (61.4\%) &      0.02 (-23.3\%) &        0.01 (35.6\%) \\
  &  0.7 &  0.14 (60.8\%)&    0.1 (52.0\%) &    0.06 (20.9\%) &    0.07 (29.9\%) &       0.04 (11.9\%) &        0.02 (39.1\%) \\
 \midrule
  &  0.3 &0.05 (-29.8\%)&             0.05 (-31.7\%) &   0.05 (-21.8\%) &   0.05 (-66.8\%) &      0.0 (-293.4\%) &        0.01 (13.4\%) \\
  Flickr &  0.5 & 0.08 (-31.9\%)&             0.06 (-27.6\%) &   0.06 (-67.2\%) &   0.07 (-73.8\%) &     0.02 (-440.1\%) &       0.02 (-43.9\%) \\
  &  0.7 &0.08 (-55.3\%)&             0.07 (-32.3\%) &  0.04 (-316.0\%) &  0.07 (-138.4\%) &     0.03 (-384.6\%) &      0.04 (-195.6\%) \\
\midrule
  &  0.3 & 0.03 (13.1\%) &  0.03 (-15.7\%) &   0.01 (-79.9\%) &    0.04 (-3.2\%) &     0.01 (-191.7\%) &        0.0 (-53.7\%) \\
  PubMed &  0.5 & 0.05 (47.8\%) &   0.05 (35.0\%) &    0.05 (41.1\%) &    0.12 (46.8\%) &      0.03 (-66.4\%) &      0.01 (-118.0\%) \\
  &  0.7 & 0.09 (58.0\%) &   0.09 (34.7\%) &    0.07 (68.7\%) &    0.07 (21.2\%) &       0.08 (67.2\%) &        0.03 (43.1\%) \\
\midrule
     &  0.3 &    0.06 (84.3\%) &    0.06 (82.0\%) &    0.05 (76.8\%) &    0.06 (80.5\%) &       0.03 (65.2\%) &        0.04 (80.8\%) \\
    ER &  0.5 &     0.1 (82.2\%) &     0.1 (83.9\%) &    0.09 (79.3\%) &    0.09 (78.8\%) &       0.06 (64.6\%) &        0.06 (75.4\%) \\
     &  0.7 &    0.12 (59.0\%) &    0.14 (52.3\%) &    0.12 (55.7\%) &    0.13 (57.1\%) &       0.08 (25.1\%) &        0.09 (50.3\%) \\
\midrule
    &  0.3 &    0.02 (73.1\%) &   0.01 (-37.1\%) &    0.01 (-4.9\%) &    0.02 (64.8\%) &      0.0 (-204.1\%) &       0.01 (-22.0\%) \\
   GEO &  0.5 &    0.04 (52.8\%) &    0.01 (12.4\%) &    0.01 (27.0\%) &    0.03 (56.3\%) &     0.01 (-145.1\%) &        0.02 (-9.7\%) \\
    &  0.7 &    0.05 (66.5\%) &    0.02 (39.8\%) &    0.02 (42.6\%) &    0.04 (66.0\%) &      0.01 (-56.2\%) &         0.02 (0.9\%) \\
\midrule
  &  0.3 &    0.01 (82.6\%) &     0.0 (41.9\%) &     0.0 (25.6\%) &    0.01 (87.3\%) &       0.0 (-73.6\%) &         0.0 (11.8\%) \\
 Shape &  0.5 &    0.02 (84.4\%) &    0.01 (67.7\%) &    0.01 (58.4\%) &    0.02 (87.4\%) &        0.0 (13.3\%) &        0.01 (43.8\%) \\
  &  0.7 &    0.03 (85.2\%) &    0.01 (78.9\%) &    0.01 (58.2\%) &    0.02 (87.9\%) &       0.01 (43.6\%) &        0.01 (59.4\%) \\
\midrule
  &  0.3 &    0.03 (78.9\%) &     0.0 (-4.4\%) &     0.0 (-7.2\%) &    0.04 (73.7\%) &      0.0 (-253.3\%) &        0.01 (60.8\%) \\
    WS &  0.5 &    0.05 (83.3\%) &    0.01 (-1.7\%) &    0.01 (38.6\%) &  0.05 (50.3\%)             &       0.01 (40.9\%) &        0.01 (10.8\%) \\
  &  0.7 &    0.07 (84.1\%) &    0.01 (56.4\%) &    0.01 (65.7\%) &    0.07 (89.5\%) &       0.01 (62.6\%) &        0.02 (68.6\%) \\
\bottomrule
\end{tabular}
}
\end{center}
\end{table}

\begin{table}
\label{conductance-loss-appendix}
\renewcommand{\arraystretch}{1.2}

\caption{Loss: conductance difference. Each entry $x (y)$ is: $x =$ loss w/o learning, and $y =$ improvement percentage. $\dagger$ stands for out of memory error.}
\label{loss_conductance_lap_none_eigen_False}

\begin{center}
\resizebox{\textwidth}{!}{
\begin{tabular}{@{}llllllll@{}}
\toprule
 Dataset &     Ratio & BL & Affinity & \makecell[l]{Algebraic \\ Distance} &    \makecell[l]{Heavy \\ Edge} & \makecell[l]{Local var \\ (edges)} & \makecell[l]{Local var \\ (neigh.)} \\ \midrule %
    &  0.3 &  0.11 (82.3\%) &    0.08 (78.5\%) &     0.10 (74.8\%) &   0.10 (74.3\%) &       0.09 (79.3\%) &        0.11 (83.6\%) \\
    BA &  0.5 &  0.14 (69.6\%) &   0.13 (31.5\%) &    0.14 (37.4\%) &  0.14 (33.9\%) &       0.13 (34.3\%) &        0.13 (56.2\%) \\
     &  0.7 &  0.22 (48.1\%) &   0.20 (11.0\%) &    0.21 (22.4\%) &  0.21 (20.0\%) &        0.20 (13.2\%) &        0.21 (47.8\%) \\
\midrule
     &  0.3 &   0.10 (81.0\%) &   0.09 (74.7\%) &     0.10 (74.3\%) &   0.10 (72.5\%) &       0.09 (76.4\%) &        0.12 (79.0\%) \\
    ER &  0.5 &  0.13 (64.0\%) &   0.14 (33.8\%) &    0.14 (33.6\%) &  0.14 (32.5\%) &       0.14 (31.9\%) &         0.12 (1.4\%) \\
     &  0.7 &   0.20 (43.4\%) &    0.19 (10.1\%) &     0.20 (17.6\%) &   0.20 (17.2\%) &       0.19 (23.4\%) &        0.17 (15.7\%) \\
\midrule
    &  0.3 &   0.10 (91.2\%) &   0.09 (87.0\%) &     0.10 (84.8\%) &   0.10 (85.5\%) &        0.10 (84.6\%) &        0.11 (92.3\%) \\
   GEO &  0.5 &  0.12 (88.1\%) &  0.13 (33.9\%) &    0.13 (32.6\%) &  0.13 (37.6\%) &       0.13 (35.3\%) &        0.13 (90.1\%) \\
    &  0.7 &  0.21 (86.7\%) &   0.17 (21.9\%) &    0.19 (25.2\%) &  0.19 (27.3\%) &       0.19 (27.8\%) &        0.11 (72.4\%)  \\
\midrule
  &  0.3 &        0.10 (82.3\%) &   0.10 (86.8\%) &    0.09 (85.8\%) &  0.09 (86.3\%) &       0.09 (84.8\%) &        0.09 (92.0\%) \\
 Shape &  0.5 &        0.14 (33.2\%) &   0.13 (34.7\%) &    0.13 (34.6\%) &  0.13 (37.7\%) &       0.13 (40.8\%) &        0.12 (89.8\%) \\
  &  0.7 &  0.17 (41.4\%) &   0.19 (23.4\%) &     0.20 (27.7\%) &   0.20 (34.0\%) &        0.20 (34.3\%) &        0.11 (76.8\%) \\
\midrule
  &  0.3 &   0.10 (86.7\%) &   0.09 (82.1\%) &     0.10 (84.3\%) &   0.10 (82.9\%) &       0.09 (81.9\%) &         0.10 (90.5\%) \\
    WS &  0.5 &  0.13 (80.8\%) &   0.13 (31.2\%) &    0.13 (33.1\%) &  0.13 (27.7\%) &       0.13 (34.0\%) &        0.13 (86.5\%) \\
  &  0.7 &  0.19 (45.3\%) &    0.19 (19.3\%) &    0.19 (27.0\%) &  0.19 (26.6\%) &        0.20 (27.1\%) &        0.11 (12.8\%) \\
\midrule
 & 0.3 & 0.11 (75.8\%) &  0.08 (86.8\%) &   0.12 (71.4\%) &   0.11 (62.6\%) &      0.11 (76.7\%) &       0.14 (87.9\%) \\
CS & 0.5 &  0.14 (48.3\%) &  0.12 (16.7\%) &   0.15 (50.0\%) &   0.11 (-7.2\%)  &      0.11 (6.7\%)   &       0.09 (9.6\%) \\
 & 0.7 &  0.26 (40.1\%) &  0.22 (29.0\%) &   0.24 (35.0\%) &   0.24 (41.0\%) &      0.23 (35.2\%) &       0.17 (28.8\%) \\
\midrule
 & 0.3 & 0.10 (81.7\%)  &  0.07 (79.2\%) &   0.11 (73.6\%)  &  0.10 (73.7\%) &      0.11 (79.0\%)   &      0.13 (4.4\%)  \\   
Physics & 0.5 & 0.13 (20.5\%) &  0.19 (39.7\%) &   0.15 (27.8\%)  & 0.16 (31.7\%)&       0.15 (25.4\%)  &     0.11 (-22.3\%) \\
 & 0.7 & 0.24 (60.2\%) &   0.16 (26.1\%) &   0.23 (15.3\%) &  0.24 (16.5\%)&       0.23 (11.2\%)  &       0.20 (35.9\%) \\
\midrule
 & 0.3 & 0.12 (42.8\%) &  0.10 (0.4\%)&      0.18 (3.6\%) &  0.18 (-0.2\%)&        0.19 (0.9\%)&        0.11 (26.4\%)\\
PubMed & 0.5 & 0.15 (19.7\%) &   0.19 (1.3\%)&   0.24 (-12.9\%)&    0.39 (3.7\%)&       0.39 (11.8\%)&        0.16 (16.0\%)\\
 & 0.7 &          0.25 (27.3\%) &            0.33 (0.8\%) &    0.36 (0.0\%) &            0.31 (33.2\%) &                 0.28 (35.3\%) &       0.23 (14.1\%)\\
\midrule
 & 0.3 &             0.11 (62.6\%) &            $\dagger$ &   0.13 (52.5\%)&   0.13 (54.7\%)&       0.12 (74.2\%)&        0.16 (58.3\%)\\
Flickr & 0.5 & 0.09 (-34.5\%)&             $\dagger$ &    0.15 (3.1\%)&    0.16 (3.4\%)&       0.15 (19.9\%)&        0.13 (-6.7\%) \\
 & 0.7 &   0.19 (35.6\%) &            $\dagger$ &      0.20 (6.0\%)&   0.28 (-3.1\%)&        0.29 (5.3\%)&       0.12 (-25.4\%)\\
\bottomrule
\end{tabular}
}
\end{center}
\end{table}

\begin{table}
\renewcommand{\arraystretch}{1.2}
\caption{Loss: Eigenerror. Laplacian: combinatorial Laplacian for original graphs and \textit{doubly-weighted Laplacian} for coarse graphs.  Each entry $x (y)$ is: $x =$ loss w/o learning, and $y =$ improvement percentage. $\dagger$ stands for out of memory error.}
\label{loss_quadratic_lap_none_eigen_True}
\begin{center}
\resizebox{\textwidth}{!}{
\begin{tabular}{@{}llllllll@{}}
\toprule
 Dataset &    Ratio   & BL & Affinity & \makecell[l]{Algebraic \\ Distance} &    \makecell[l]{Heavy \\ Edge} & \makecell[l]{Local var \\ (edges)} & \makecell[l]{Local var \\ (neigh.)} \\ \midrule
   &  0.3 &   0.19 (4.1\%) &    0.1 (5.4\%) &     0.12 (5.6\%) &   0.12 (5.0\%) &       0.03 (25.4\%) &        0.1 (-32.2\%) \\
    BA &  0.5 &   0.36 (7.1\%) &   0.17 (8.2\%) &     0.22 (6.5\%) &   0.22 (4.7\%) &       0.11 (21.1\%) &       0.17 (-15.9\%) \\
     &  0.7 &     0.55 (9.2\%)   &  0.32 (12.4\%) &    0.39 (10.2\%) &  0.37 (10.9\%) &       0.21 (33.0\%) &       0.28 (-29.5\%) \\
\midrule
       &  0.3 &  0.46 (16.5\%) &   0.3 (56.9\%) &    0.11 (59.1\%) &    0.23 (38.9\%) &       0.0 (-347.6\%) &      0.0 (-191.8\%) \\ %
      CS &  0.5 &   1.1 (18.0\%) &  0.55 (49.8\%) &    0.33 (60.6\%) &    0.42 (44.5\%) &       0.21 (75.2\%) &       0.0 (-154.2\%) \\
       &  0.7 &  2.28 (16.9\%) &  0.82 (57.0\%) &    0.66 (53.3\%) &    0.73 (38.9\%) &       0.49 (73.4\%) &        0.34 (63.3\%) \\
 \midrule
  &  0.3 & 0.48 (19.5\%) &  0.35 (67.2\%) &    0.14 (65.2\%) &     0.2 (57.4\%) &     0.0 (-521.6\%) &        0.0 (20.7\%) \\ %
 Physics &  0.5 &  1.06 (21.7\%) &  0.58 (67.1\%) &    0.33 (69.5\%) &    0.35 (64.6\%) &        0.2 (79.0\%)     &  0.0 (-377.9\%) \\ %
  &  0.7 &  2.11 (19.1\%) &  0.88 (72.9\%) &    0.62 (66.7\%) &    0.62 (64.9\%) &       0.31 (70.3\%) &      0.01 (-434.0\%) \\
 \midrule
   &  0.3 &  0.33 (20.4\%) & $\dagger$ &     0.16 (7.8\%) &     0.16 (9.1\%) &                 0.02 (63.0\%) &       0.04 (-88.9\%) \\
  Flickr &  0.5 &  0.57 (55.7\%) & $\dagger$ &    0.33 (20.2\%) &    0.31 (55.0\%) &       0.11 (67.6\%) &        0.07 (60.3\%) \\
   &  0.7 &  0.86 (85.2\%) & $\dagger$ &     0.6 (32.6\%) &    0.57 (38.7\%) &       0.23 (92.2\%) &        0.21 (40.7\%) \\
  \midrule
   &  0.3 &   0.56 (5.6\%) &  0.27 (13.8\%) &    0.13 (17.4\%) &    0.34 (10.6\%) &       0.06 (-0.4\%) &         0.0 (31.1\%) \\
  PubMed &  0.5 &   1.25 (7.1\%) &   0.5 (15.5\%) &    0.51 (12.3\%) &  1.19 (-110.1\%) &       0.35 (-8.8\%) &        0.02 (60.4\%) \\
   &  0.7 &   2.61 (8.9\%) &  1.12 (19.4\%) &  2.24 (-149.8\%) &   4.31 (-238.6\%) &     1.51 (-260.2\%) &        0.27 (75.8\%) \\
  \midrule
   &  0.3 &  0.27 (-0.1\%) &   0.35 (0.4\%) &     0.15 (0.6\%) &   0.18 (0.5\%) &        0.01 (5.7\%) &       0.01 (-10.4\%) \\
    ER &  0.5 &   0.61 (0.5\%) &    0.7 (1.0\%) &     0.35 (0.6\%) &   0.36 (0.2\%) &        0.19 (1.2\%) &         0.02 (0.8\%) \\
   &  0.7 &   1.42 (0.8\%) &   1.27 (2.1\%) &      0.7 (1.4\%) &   0.68 (0.3\%) &        0.29 (3.5\%) &        0.33 (10.2\%) \\
\midrule
   &  0.3 &  0.78 (43.4\%) &  0.08 (80.3\%) &    0.09 (77.1\%) &  0.27 (82.2\%) &     0.01 (-524.6\%) &         0.1 (82.5\%) \\
   GEO &  0.5 &  1.72 (50.3\%) &  0.16 (89.4\%) &    0.18 (91.2\%) &  0.45 (84.9\%) &       0.08 (55.6\%) &         0.2 (86.8\%) \\
   &  0.7 &  3.64 (30.4\%) &  0.33 (86.0\%) &    0.25 (86.7\%) &  0.61 (93.0\%) &       0.15 (88.7\%) &        0.32 (79.3\%) \\
\midrule
  &  0.3 &  0.87 (55.4\%) &  0.12 (88.6\%) &    0.07 (56.7\%) &  0.29 (80.4\%) &       0.01 (33.1\%) &        0.09 (84.5\%) \\
 Shape &  0.5 &  2.07 (67.7\%) &  0.24 (93.3\%) &    0.17 (90.9\%) &  0.49 (93.0\%) &       0.11 (84.2\%) &         0.2 (90.7\%) \\
  &  0.7 &  4.93 (69.1\%) &  0.47 (94.9\%) &    0.27 (68.5\%) &  0.71 (95.7\%) &       0.25 (79.1\%) &        0.34 (87.4\%) \\
\midrule
  &  0.3 &   0.7 (32.3\%) &  0.05 (84.7\%) &    0.04 (58.9\%) &  0.44 (37.3\%) &       0.02 (75.0\%) &        0.06 (83.4\%) \\
    WS &  0.5 &  1.59 (43.9\%) &  0.11 (88.2\%) &    0.11 (83.9\%) &   0.58 (23.5\%) &        0.1 (88.2\%) &        0.12 (79.7\%) \\
  &  0.7 &  3.52 (45.6\%) &  0.18 (77.7\%) &    0.17 (78.2\%) &  0.79 (82.8\%) &       0.17 (90.9\%) &        0.19 (65.8\%) \\

\bottomrule
\end{tabular}
}
\end{center}
\end{table}

\section{On the Use of GNN as Weight-Assignment Map.}
Recall that we use GNN to represent a edge-weight assignment map for an edge $(\hat{u}, \hat{v})$ between two super-nodes $\hat{u}, \hat{v}$ in the coarse graph $\Ghat$. The input will be the subgraph $G_{\hat{u},\hat{v}}$ in the original graph $G$ spanning the
clusters $\pi^{-1}(\hat{u})$, $\pi^{-1}(\hat{v})$, and the crossing edges among them; while the goal is to compute the weight of edge $(\hat{u}, \hat{v})$ based on this subgraph $G_{\hat{u}, \hat{v}}$. 
Given that the input is a local graph $G_{\hat{u}, \hat{v}}$, a GNN will be a natural choice to parameterize this edge-weight assignment map. Nevertheless, in principle, any architecture applicable to graph regression can be used for this purpose. 
To better understand if it is necessary to use the power of GNN, we replace GNN with the following baseline for graph regression. In particular, the baseline is a composition of mean pooling of node features in the original graph and a 4-layer MLP with embedding dimension 200 and ReLU nonlinearity. We use mean-pooling as the graph regression component needs to be permutation invariant over the set of node features. However, this baseline ignores the detailed graph structure which GNN will leverage. The results for different reduction ratios are presented in the \Cref{BLComparasion}. We have also implemented another baseline where the MLP module is replaced by a simpler linear regression module. The results are worse than those of MLP (and thus also GNN) as expected, and therefore omitted from this chapter.

As we can see, MLP works reasonably well in most cases, indicating that learning the edge weights is indeed useful for improvement. On the other hand, we see using GNN to parametrize the map generally yields a larger improvement over the MLP, which ignores the topology of subgraphs in the original graph. A systematic understanding of how different models such as various graph kernels \cite{kriege2020survey, vishwanathan2010graph} and graph neural networks affect the performance is an interesting question that we will leave for future work.

\begin{table}
\renewcommand{\arraystretch}{1.2}
\caption{Model comparison between MLP and \nn. Loss: quadratic loss. Laplacian: combinatorial Laplacian for both original and coarse graphs. Each entry $x (y)$ is: $x =$ loss w/o learning, and $y =$ improvement percentage.} 
\label{BLComparasion}
\begin{center}
\resizebox{1\textwidth}{!}{
\begin{tabular}{@{}llllllllll@{}}
\addlinespace[-\aboverulesep] 
  \cmidrule[\heavyrulewidth]{2-9}
  & Dataset &   Ratio &  BL & Affinity & \makecell[l]{Algebraic \\ Distance} &  \makecell[l]{Heavy \\ Edge} & \makecell[l]{Local var \\ (edges)} & \makecell[l]{Local var \\ (neigh.)} \\ 
  \cmidrule{2-9}
  
  &  &  0.3 &   0.27 (46.2\%) & 0.04 (4.1\%)   & 0.04 (-38.0\%)  & 0.43 (31.2\%)     & 0.02 (-403.3\%)        & 0.06 (67.0\%) \\
  & WS + MLP &  0.5 &  0.45 (62.9\%) & 0.09 (64.1\%) &    0.09 (15.9\%) &           0.52 (31.2\%)    &   0.09 (31.6\%) &       0.11 (58.5\%) \\
  &  &  0.7 &    0.65 (70.4\%) & 0.15 (57.6\%) &   0.14 (31.6\%) & 0.67 (76.6\%) &      0.15 (43.6\%) &       0.16 (54.0\%) \\
    \cmidrule{2-9}
&     &  0.3 & 0.27 (46.2\%) & 0.04 (65.6\%) &   0.04 (-26.9\%) &  0.43 (32.9\%) &       0.02 (68.2\%) &        0.06 (75.2\%) \\
&    WS + \nn &  0.5 & 0.45 (62.9\%)  &  0.09 (82.1\%) &    0.09 (60.6\%) &            0.52 (51.8\%) &       0.09 (69.9\%) &        0.11 (84.2\%) \\
&     &  0.7  & 0.65 (73.4\%) &  0.15 (78.4\%) &    0.14 (66.7\%) &  0.67 (76.6\%) &       0.15 (80.8\%) &        0.16 (83.2\%) \\
    \cmidrule{2-9}
    &  &  0.3 & 0.13 (76.6\%) &  0.04 (-53.4\%) & 0.03 (-157.0\%) & 0.11 (69.3\%) &     0.0 (-229.6\%) &       0.04 (-7.9\%) \\
    & Shape + MLP &  0.5 & 0.23 (78.4\%) &  0.08 (-11.6\%)&    0.06 (67.6\%) & 0.17 (83.2\%) &      0.04 (44.2\%) &       0.08 (-1.9\%) \\
      &  &  0.7 & 0.34 (69.9\%) &  0.17 (85.1\%) &     0.1 (73.5\%) & 0.24 (65.8\%) &      0.09 (74.3\%) &       0.13 (85.1\%) \\ 
    \cmidrule{2-9}       
    &        &  0.3 & 0.13 (86.8\%) &   0.04 (79.8\%) &    0.03 (69.0\%) &  0.11 (69.7\%) &         0.0 (1.3\%) &        0.04 (73.6\%) \\
 & Shape + \nn &  0.5 &  0.23 (91.4\%) &  0.08 (89.8\%) &    0.06 (82.2\%) &  0.17 (88.2\%) &       0.04 (80.2\%) &        0.08 (79.4\%) \\
 & &  0.7 &   0.34 (91.1\%) &  0.17 (94.3\%) &     0.1 (74.7\%) &  0.24 (95.9\%) &       0.09 (64.6\%) &        0.13 (84.8\%) \\

  \cmidrule[\heavyrulewidth]{2-9} \addlinespace[-\belowrulesep] 
\end{tabular}
}
\end{center}
\tabbottomvspace
\end{table}

\section{Visualization}
\label{viz-appendix}

\begin{figure}[h]
\begin{center} 
\includegraphics[scale=.25]{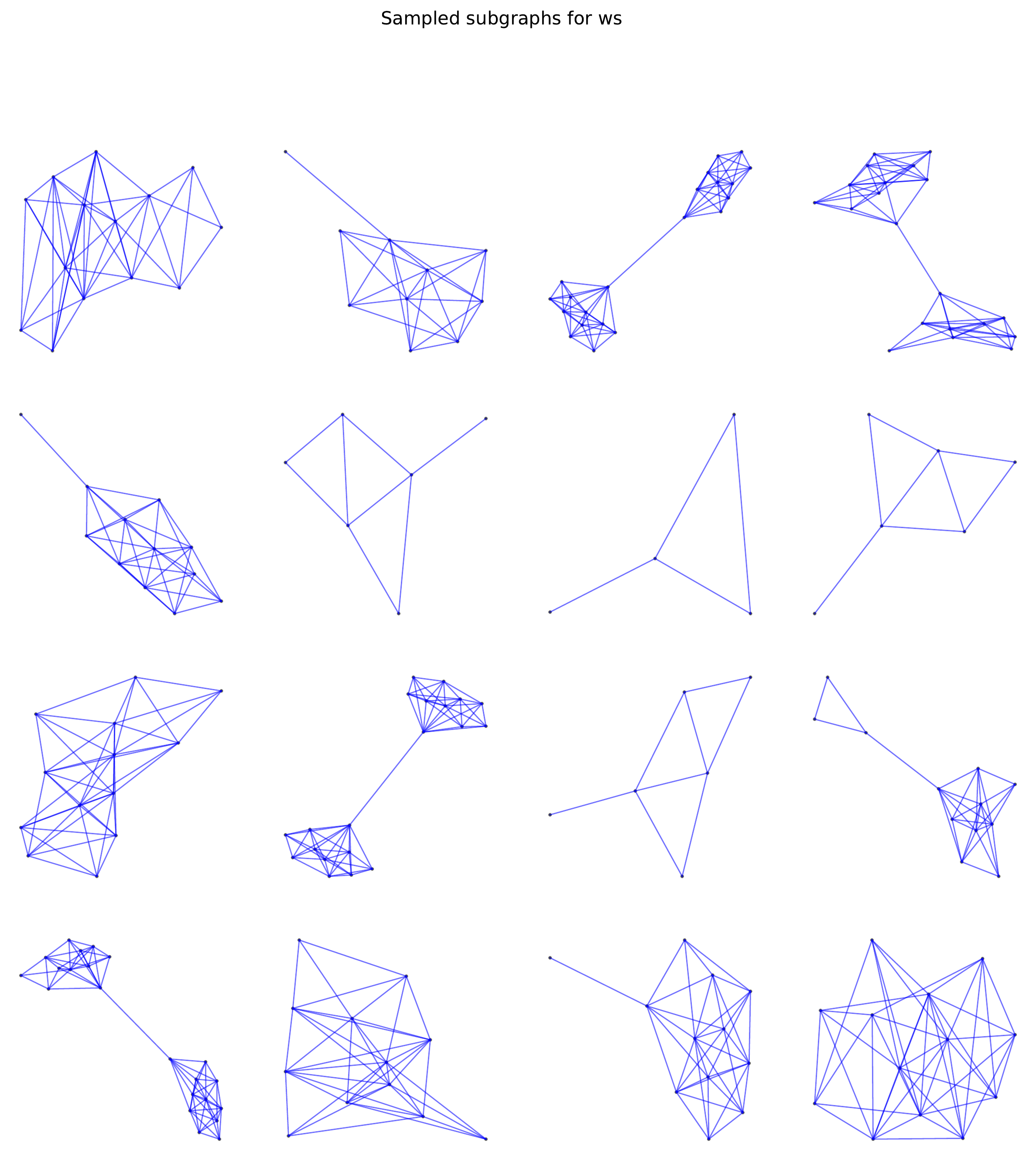}
\includegraphics[scale=.25]{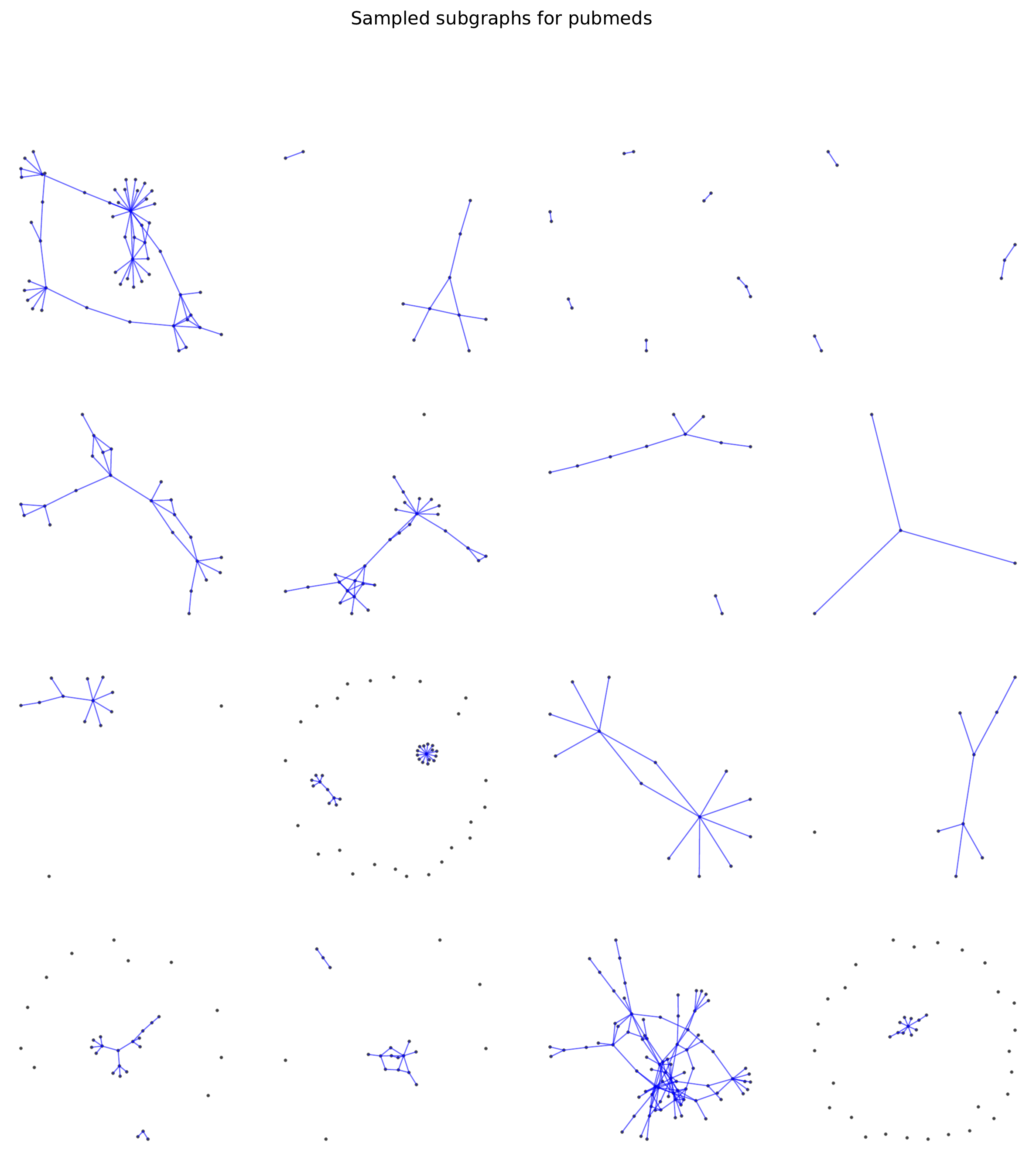}
\end{center}
\caption{A collection of subgraphs corresponding to edges in coarse graphs (WS and PubMed) generated by variation neighborhood algorithm. Reduction ratio is 0.7 and 0.9 respectively.} 
\end{figure}

We visualize the subgraphs corresponding to randomly sampled edges of coarse graphs. For example, in WS graphs, some subgraphs have only a few nodes and edges, while other subgraphs have some common patterns such as the dumbbell shape graph.  For PubMed, most subgraphs have tree-like structures, possibly due to the edge sparsity in the citation network.

\begin{figure}[ht]
\label{wdiff-fig}
\centering
\includegraphics[ trim=2.5cm 2cm 2.5cm 2cm,clip,width=.45\textwidth]{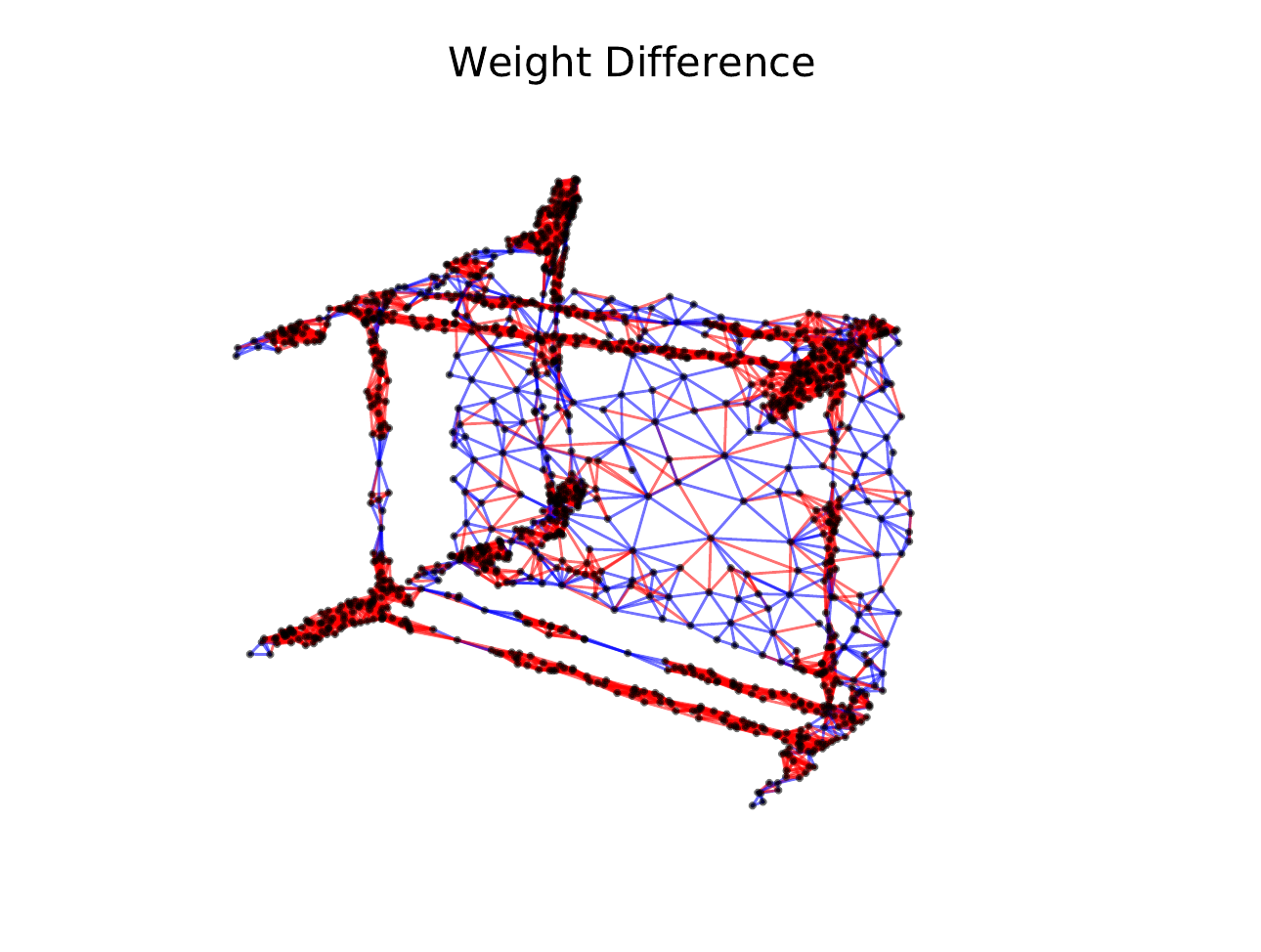}
\includegraphics[ trim=2.5cm 2cm 2.5cm 2cm,clip,width=.45\textwidth]{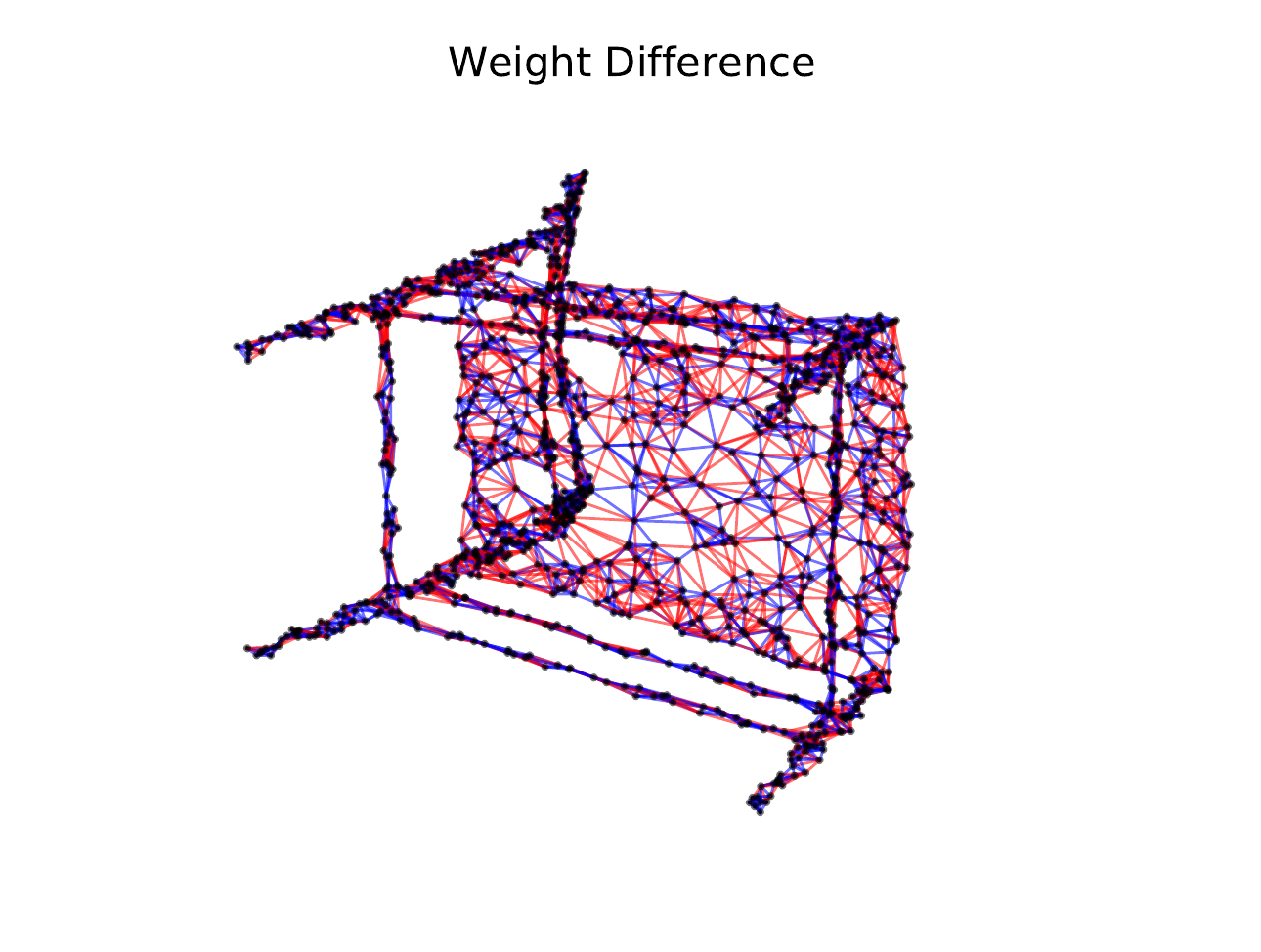}
\includegraphics[scale=.3]{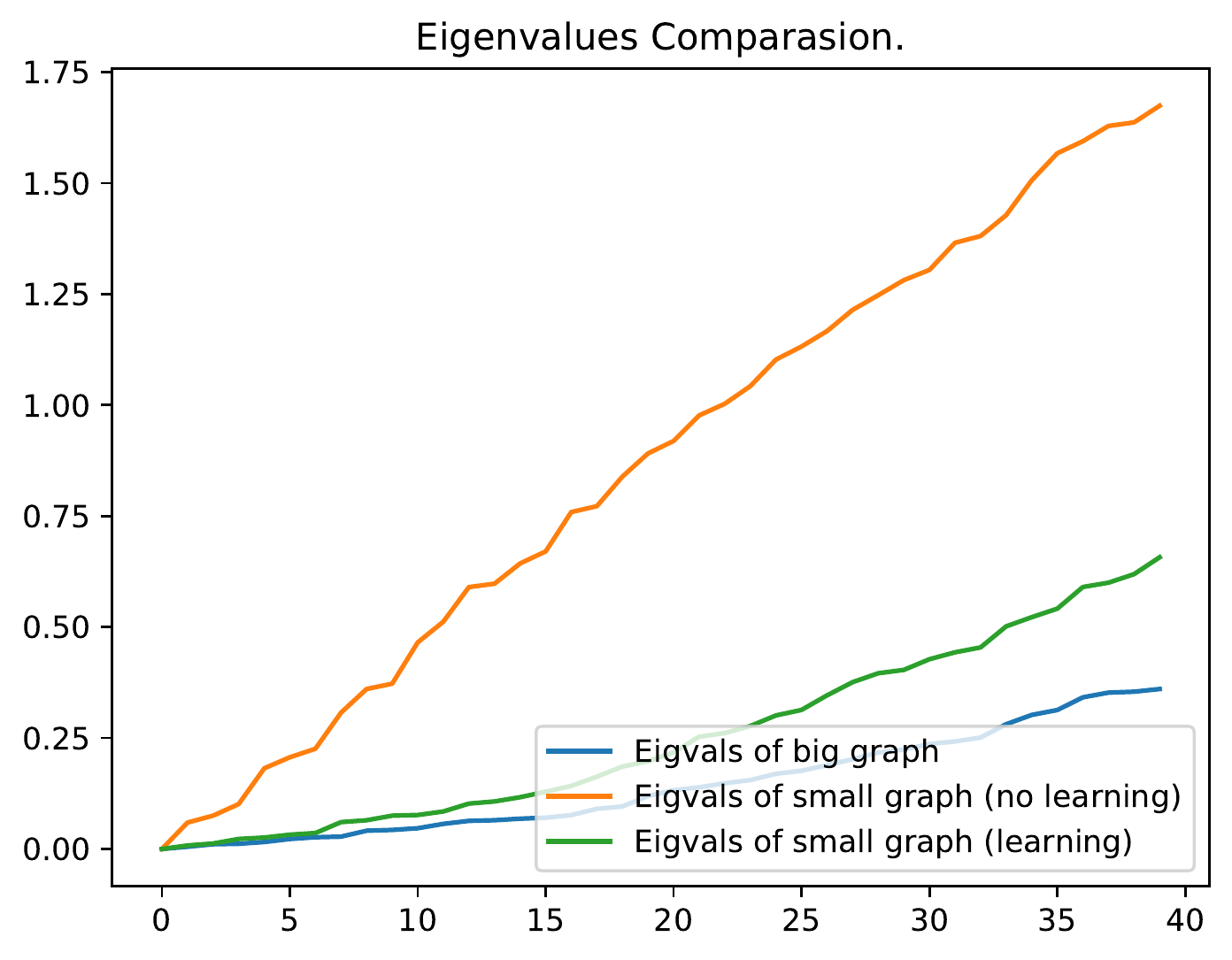}
\includegraphics[scale=.3]{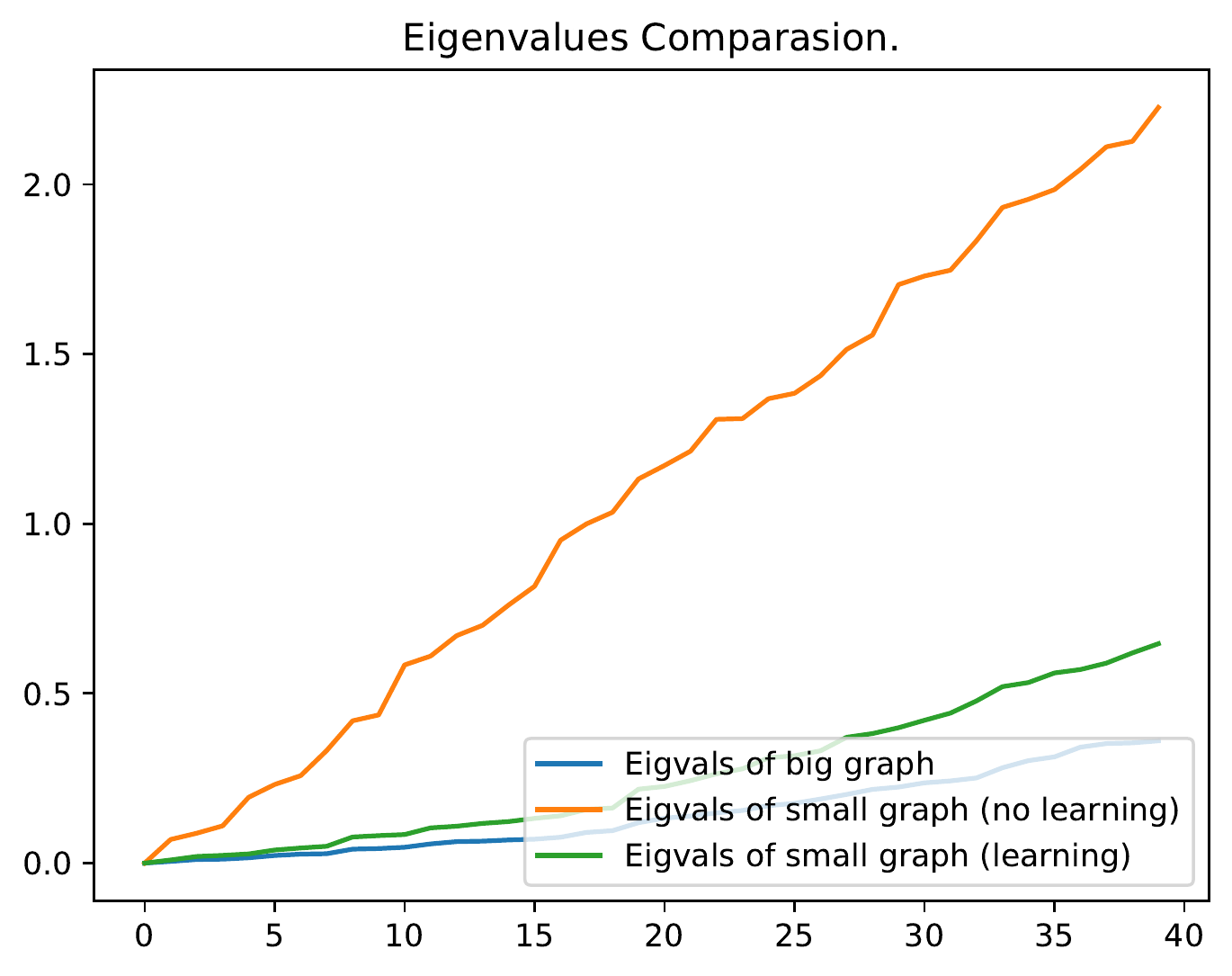}

\caption{The first row illustrates the weight difference for two coarsening methods. Blue (red) edges denote edges whose learned weights is smaller (larger) than the default ones. The second row shows the spectrum of the original graph Laplacian, coarse graph w/o learning, and coarse graph w/ learning.} 
\end{figure}

In \Cref{wdiff-fig}, we visualize the weight difference between coarsening algorithms with and without learning. We also plot the eigenvalues of coarse graphs, where the first 40 eigenvalues of the original graph are smaller than the coarse ones. After optimizing edge weights via \nntight, we see both methods produce graphs with eigenvalues closer to the eigenvalues of the original graphs. 

\section{Related Work}
\textbf{Graph sparsification}. Graph sparsification is firstly proposed to solve linear systems involving combinatorial graph Laplacian efficiently. \cite{spielman2011spectral, spielman2011graph} showed that for any undirected graph $G$ of $N$ vertices, a spectral sparsifier of $G$ with only $O(Nlog^cN/\epsilon^2)$ edges can be constructed in nearly-linear time. \footnote{The algorithm runs in $O(M.\text{polylog}N)$ time, where $M$ and $N$ are the numbers of edges and vertices.} Later on, the time complexity and the dependency on the number of the edges are reduced by various researchers \cite{batson2012twice, allen2015spectral, lee2018constructing, lee2017sdp}. 

\textbf{Graph coarsening}. Previous work on graph coarsening focuses on preserving different properties, usually related to the spectrum of the original graph and coarse graph. \cite{loukas2018spectrally, loukas2019graph} focus on the restricted spectral approximation, a modification of the spectral similarity measure used for graph sparsification. \cite{hermsdorff2019unifying} develop a probabilistic framework to preserve inverse Laplacian. 

\textbf{Deep learning on graphs}. As an effort of generalizing convolution neural network to the graphs and manifolds, graph neural networks is proposed to analyze graph-structured data. They have achieved state-of-the-art performance in node classification \cite{kipf2016semi}, knowledge graph completion \cite{schlichtkrull2018modeling}, link prediction \cite{dettmers2018convolutional, gurukar2019network}, combinatorial optimization \cite{li2018combinatorial, khalil2017learning}, property prediction \cite{duvenaud2015convolutional, xie2018crystal} and physics simulation \cite{sanchez2020learning}. 

\textbf{Deep generative model for graphs}.
To generative realistic graphs such as molecules and parse trees, various approaches have been taken to model complex distributions over structures and attributes, such as variational autoencoder \cite{simonovsky2018graphvae, ma2018constrained}, generative adversarial networks (GAN) \cite{de2018molgan, zhou2019misc}, deep autoregressive model \cite{liao2019efficient, you2018graphrnn, li2018learning}, and reinforcement learning type approach \cite{you2018graph}.  \cite{zhou2019misc} proposes a GAN-based framework to preserve the hierarchical community structure via algebraic multigrid method during the generation process. However, different from our approach, the coarse graphs in \cite{zhou2019misc} are not learned.

\section{Concluding Remarks}
We present a framework to compare original graph and the coarse one via the properly chosen Laplace operators and projection/lift map. Observing the benefits of optimizing over edge weights, we propose a GNN-based framework to learn the edge weights of coarse graph to further improve the existing coarsening algorithms. Through extensive experiments, we demonstrate that our method \nn significantly improves common graph coarsening methods under different metrics, reduction ratios, graph sizes, and graph types. 

\section{Missing Proofs}
\subsection{Choice of Laplace Operator}
\label{laplace-appendix}

\subsubsection{Laplace Operator on Weighted Simplicial Complex}
\label{simplicial-laplace-appendix}
Its most general form in the discrete case, presented as the operators on weighted simplicial complexes, is:
\begin{gather*}
\nL_{i}^{u p}=W_{i}^{-1} B_{i}^{T} W_{i+1} B_{i} \quad
\nL_{i}^{d o w n}=B_{i-1} W_{i-1}^{-1} B_{i-1}^{T} W_{i}
\end{gather*}
where $B_i$ is the matrix corresponding to the coboundary operator $\delta_i$, and $W_i$ is the diagonal matrix representing the weights of $i$-th dimensional simplices. See \cite{horak2013spectra} for details. When restricted to the graph (1 simplicial complex), we recover the most common graph Laplacians as special case of $\nL_{0}^{u p}$. Note that although the $\nL_{i}^{u p}$ and $\nL_{i}^{d o w n}$ is not symmetric, we can always symmetrize them by multiple a properly chosen diagonal matrix and its inverse from left and right without altering the spectrum.

\subsubsection{Missing Proofs}
\label{appendix-missing-proof}

\begin{wrapfigure}{r}{0.28\textwidth}
\vspace{-20pt}
  \begin{center}
  \label{appendix-fig}
    \includegraphics[width=0.26\textwidth]{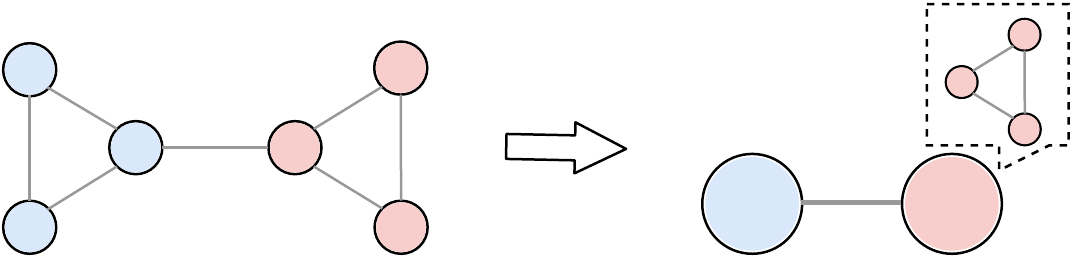}
  \end{center}
  \caption{A toy example.}
\end{wrapfigure}
We provide the missing proofs regarding the properties of the projection/lift map and the resulting operators on the coarse graph. 

Recall as an toy example, a coarsening algorithm will take graph on the left in \Cref{appendix-fig} and generate a coarse graph on the right, with coarsening matrix 
$P=\left[\begin{array}{cccccc}1 / 3 & 1 / 3 & 1 / 3 & 0 & 0 & 0 \\  0 & 0 & 0 & 1 / 3 & 1 / 3 & 1 / 3 \end{array}\right]$, 
$P^{+}=\left[\begin{array}{ll}
1 & 0  \\
1 & 0  \\
1 & 0  \\
0 & 1  \\
0 & 1  \\ 
0 & 1  
\end{array}\right]$, 
$\Gamma=\left[\begin{array}{cccc}3 &  0  \\ 0 & 3   \end{array}\right]$, 
$\IN = \left[\begin{array}{cccccc}
1 / 3 & 1 / 3 & 1 / 3 & 0 & 0 & 0\\
1 / 3 & 1 / 3 & 1 / 3 & 0 & 0 & 0\\
1 / 3 & 1 / 3 & 1 / 3 & 0 & 0 & 0\\
0 & 0 & 0 & 1 / 3 & 1 / 3 & 1 / 3 \\
0 & 0 & 0 & 1 / 3 & 1 / 3 & 1 / 3 \\
0 & 0 & 0 & 1 / 3 & 1 / 3 & 1 / 3
\end{array}\right]$. $\IN$ in general is a $\N \times \N$ block matrix of rank $\n$. All entries in each block $\IN_j$ is equal to $\frac{1}{\gamma_j}$ where $\gamma_j = \left|\vmap^{-1}(\vhat_j)\right|$.

\begin{table}[htp!]
\centering
\caption{Depending on the choice of $\mathcal{F}$ (quantity that we want to preserve) and $\genL$, we have different projection/lift operators and resulting $\sgenL$ on the coarse graph.}
\resizebox{\textwidth}{!}{
\begin{tabular}{@{}llllll@{}}
\toprule
Quantity $\mathcal{F}$ of interest & $\genL$  & Projection $\proj$ & Lift $\lift$ & $\sgenL$   & Invariant under $\lift$ \\ \midrule
Quadratic form $\Qform$ & $L$ &  $P$   & ${P^{+}}$  & Combinatorial Laplace $\widehat{L}$ & $\Qform_L(\lift \sx) = \Qform_{\widehat{L}}(\sx)$\\
 Rayleigh quotient $\RQ$ & $L$ & $\Gamma^{-1/2}{{(P^{+})}^T}$ & ${P^{+}}\Gamma^{-1/2}$ & Doubly-weighted Laplace $\sdwL$ & $\RQ_L(\lift \sx) = \RQ_{\sdwL} (\sx)$\\
 Quadratic form $\Qform$ & $\mathcal{L}$ & $ \widehat{D}^{1/2}PD^{-1/2} $  & $ D^{1/2}{(P^{+})} \widehat{D}^{-1/2}$ & Normalized Laplace $\widehat{\mathcal{L}}$  & $\Qform_\mathcal{L}(\lift \sx) = \Qform_{\widehat{\mathcal{L}}}(\sx)$\\ \bottomrule
\end{tabular}
}
\end{table}

We first make an observation about projection and lift operator, $\proj$ and $\lift$.

\begin{lemma}
$\proj \circ \lift = I.$ $ \lift \circ \proj = \IN. $
\end{lemma}
\begin{proof}
For the first case, it's easy to see $\proj \circ \lift =PP^+ = I$ and $\lift \circ \proj =P^+P = \IN$. \\

For the second case, 
$\proj \circ \lift = \Gamma^{-1/2}\ppt P^+ \Gamma^{-1/2} = \Gamma^{-1/2} \IN \Gamma^{-1/2} = I.$ $\lift \circ \proj =P^+ \Gamma^{-1} \ppt =  I$. \\

For the third case, 
\begin{align*}
\proj \circ \lift & = \Dhat^{1/2}PD^{-1/2} D^{1/2}{(P^{+})} \Dhat^{-1/2} \\
& = \Dhat^{1/2}P{(P^{+})} \Dhat^{-1/2} \\
& = \Dhat^{1/2}I \Dhat^{-1/2} = I. 
\end{align*}
\begin{align*}
\lift \circ \proj & = D^{1/2}{(P^{+})} \Dhat^{-1/2} \Dhat^{1/2}PD^{-1/2} \\
& = D^{1/2}{(P^{+})} PD^{-1/2} \\
&  = D^{1/2}\IN D^{-1/2} = \IN.
\end{align*}
\end{proof}
Now we prove the three lemmas in the this chapter. 
\begin{proposition}
For any vector $\sx \in \reals^\n$, we have that $\Qform_{\LtildeVhat}(\sx) = \Qform_{\combL }(P^+ \sx)$. 
In other words, set $\bx:= P^+ \sx$ as the lift of $\sx$ in $\reals^\N$, then $\sx^T \LtildeVhat \sx = \bx^T \combL  \bx$. 
\end{proposition}

\begin{proof}
$\Qform_L(\lift \sx) = (\lift \sx)^T L \lift \sx = \sx \ppt L P^{+} \sx^T = \sx^T \Lhat \sx = \Qform_{\Lhat}(\sx)$
\end{proof}

\begin{proposition}
For any vector $x \in \reals^\n$, we have that $\RQ_{\sdwL}(\sx) = \RQ_{\combL }(P^+ \Gamma^{-1/2} \sx)$. That is, set the lift of $\sx$ in $\reals^\N$ to be $\bx = P^+ \Gamma^{-1/2} \sx$, then we have that 
$\frac{\sx^T \sdwL \sx}{\sx^T \sx} = \frac{\bx^T \combL  \bx}{\bx^T \bx}$. 
\end{proposition} 

\begin{proof}
By definition $R_L(\lift \sx) = \frac{\Qform_L(\lift \sx)}{|| \lift \sx ||_2^2}$, $R_{\dwL}(x) = \frac{\Qform_{\Lhat}( x)}{||  x ||_2^2}.$ We will prove the lemma by showing
$\Qform_L(\lift \sx) = \Qform_{\dwL}( x)$ and $|| \lift \sx ||_2^2 = ||  x ||_2^2$.
\begin{align*}
\Qform_L(\lift \sx) & = (\lift \sx)^T L \lift \sx \\
& =  \sx^T \Gamma^{-1/2} \ppt L P^{+}\Gamma^{-1/2} \sx   \\
& = \sx^T \Gamma^{-1/2} \Lhat \Gamma^{-1/2} \sx \\
& = \sx^T\sdwL \sx\\
& =\Qform_{\sdwL}( \sx)
\end{align*}
$|| \lift \sx ||_2^2 = \sx^T\Gamma^{-1/2} \ppt P^{+}\Gamma^{-1/2} \sx = \sx^T\sx = ||  \sx ||_2^2$.
Since both numerator and denominator stay the same under the action of $\lift$, we conclude $R_L(\lift \sx) = R_{\sdwL}(\sx).$ 
\end{proof}

\begin{proposition}
For any vector $x \in \reals^\n$, we have that $\Qform_{\snL}(x) = \Qform_{\nL}(D^{1/2}P^{+}\Dhat^{1/2} x)$. That is, set the lift of $\sx$ in $\reals^\N$ to be $x: = D^{1/2}P^{+}\Dhat^{1/2} x$, then we have that 
$\sx^T \snL \sx = \bx^T \nL \bx$. 
\end{proposition} 

\begin{proof}
\begin{align*}
\Qform_\nL(\lift \sx) & = (\lift \sx)^T \nL \lift \sx \\
& = \sx \Dhat^{-1/2}\ppt D^{1/2} \nL D^{1/2}{(P^{+})} \Dhat^{-1/2} \sx \\
& = \sx \Dhat^{-1/2}\ppt L{(P^{+})} \Dhat^{-1/2} \sx \\
& = \sx \Dhat^{-1/2} \Lhat \Dhat^{-1/2} \sx  \\
& = \sx \snL \sx = \Qform_{\snL}(\sx) 
\end{align*}
\end{proof}

\subsection{Iterative Algorithm for Spectrum Alignment}
\label{iterative-algo-appendix}

\subsubsection{Problem Statement}
Given a graph $G$ and its coarse graph $\Ghat$ output by existing algorithm $\mathcal{A}$, ideally we would like to set edge weight of $\Ghat$ so that spectrum of $\widehat{L}$ (denoted as $\mathfrak{L} \w$ below) has prespecified eigenvalues $\lbd$, i.e, 
\begin{equation}
\label{original-opt-problem}
\begin{array}{ll}
& \mathfrak{L} \w = U \operatorname{Diag}(\lbd) U^{T} \\
 & \text { subject to }  \w \geq 0, U^{T} U=I
\end{array}
\end{equation}
We would like to make an important note that in general, given a sequence of non decreasing numbers $0=\lambda_1 \leq \lambda_2, ..., \lambda_n$ and a coarse graph $\Ghat$, it is not always possible to set the edge weights (always positive) so that the resulting eigenvalues of graph Laplacian of $\Ghat$ is $\{0=\lambda_1, \lambda_2, ..., \lambda_n\}$. We introduce some notations before we present the theorem. The theorem is developed in the context of \textit{inverse eigenvalue problem} for graphs \cite{barioli2004two, hogben2005spectral, fallat2020inverse}, which aims to characterize the all possible sets of eigenvalues that can be realized by symmetric matrices whose sparsity pattern is related to the topology of a given graph.

For a symmetric real $n\times n$ matrix $M$, the graph of $M$ is the graph with vertices $\{1, ..., n \}$ and edges $\left\{\{i, j\} \mid b_{i j} \neq 0 \text { and } i \neq j\right\}$. Note that the diagonal of $M$ is ignored in determining $\mathcal{G}(M)$. Let $S_n$ be the set of real symmetric $n \times n$ matrices. For a graph $\Ghat$ with $n$ nodes, define $\mathcal{S}(\Ghat)=\left\{M \in S_{n} \mid \mathcal{G}(M)=\Ghat\right\}$.

\begin{theorem}
\label{impossible-thm}\cite{barioli2004two, hogben2005spectral} 
If $T$ is a tree, for any $M\in \mathcal{S}(T)$, the diameter of $T$ is less than the number of distinct eigenvalues of $M$.
\end{theorem}

For any graph $\Ghat$, its Laplacian (both combinatorial and normalized Laplacian) belongs to $\mathcal{S}(\Ghat)$, the above theorem therefore applies. In other words, given a tree $T$ and given a sequence of non-decreasing numbers $0=\lambda_1 \leq \lambda_2, ... \lambda_n$, as long as the number of distinct values in sequences is less than the diameter of $T$, then this sequence can not be realized as the eigenvalues of graph Laplacian of $T$, no matter how we set the edge weights. 

Therefore Instead of looking for the a graph with exact spectral alignment with original graph, which is impossible for some nondecreasing sequences as illustrated by the theorem \ref{impossible-thm}, we relax the equality in equation \ref{original-opt-problem} by instead minimizing the $|| \mathfrak{L} \w-U \operatorname{Diag}(\lbd) U^{T} ||_F^2$. %
We first present an algorithm for the complete graph $\widehat{G}$ of size $n$. This algorithm is essentially the special case of \cite{kumar2019structured}. We then show relaxing $\Ghat$ from the complete graph to the arbitrary graph will not change the convergence result. Before that, we introduce some notations.

\subsubsection{Notation}
\begin{definition}
The linear operator $\mathfrak{L}: \w \in \mathbb{R}_{+}^{\frac{n(n-1)}{2}} \rightarrow \mathfrak{L} \w \in \mathbb{R}^{n \times n}$ is defined as
\begin{equation*}
[\mathfrak{L} \w]_{i j}=\left\{\begin{array}{ll}-w_{i+d_{j}} & i>j \\ {[\mathfrak{L} \w]_{j i}} & i > j \\ \sum_{i \neq j}[\mathfrak{L} \w]_{i j} & i=j\end{array}\right.
\end{equation*}
where $d_{j}=-j+\frac{j-1}{2}(2 n-j)$
\end{definition}

A toy example is given to illustrate the operators, Consider a
weight vector $\w=\left[w_{1}, w_{2}, w_{3}, w_{4}, w_{5}, w_{6}\right]^{T}$, The Laplacian operator $\mathfrak{L}$ on $\w$ gives
\begin{equation*}
\mathfrak{L} \w=\left[\begin{array}{cccc}\sum_{i=1,2,3} w_{i} & -w_{1} & -w_{2} & -w_{3} \\ -w_{1} & \sum_{i=1,4,5} w_{i} & -w_{4} & -w_{5} \\ -w_{2} & -w_{4} & \sum_{i=2,4,6} w_{i} & -w_{6} \\ -w_{3} & -w_{5} & -w_{6} & \sum_{i=3,5,6} w_{i}\end{array}\right]
\end{equation*}

Adjoint operator $\lapstar{}$ is defined to satisfy $\langle\lap{\w}, Y\rangle = \langle\w, \lapstar{Y}\rangle.$

\subsubsection{Complete Graph Case}
Recall our goal is to 

\begin{equation}
\label{opt-problem-equ}
\begin{array}{ll}
\underset{\w, U}{\mini} & \left\|\mathfrak{L} \w-U \operatorname{Diag}(\lbd) U^{T}\right\|_{F}^{2} \\ 
\text { subject to } & \w \geq 0, U^{T} U=I
\end{array}
\end{equation}

\begin{algorithm}[H]
\label{opt-algo-appendix}
\DontPrintSemicolon
  \KwInput{coarse graph $\widehat{G}$, error tolerance $\epsilon$, iteration limit $T$ }
  \KwOutput{coarse graph with new edge weights}
  Initialize $U$ as random element in orthogonal group $O(n, \mathbb{R})$ and $t=0$. \\
   \While{$\epsilon$ is smaller than the threshold or $t>T$}
   {
   		Update $\w^{t+1}, U^{t+1}$ according to \ref{update-w-appendix} and Lemma \ref{update-u-appendix}\;
		Compute Error $\epsilon$	\;
		 $t=t+1$
   }
   From $w^t$, output coarse graph with new edge weights.
\caption{Iterative algorithm for edge weight optimization}
\end{algorithm}

where $\lbd$ is the desired eigenvalues of the smaller graph. One choice of $\lbd$ can be the first $n$ eigenvalues of the original graph of size $N$. $\w$ and $U$ are variables of size $n(n-1)/2$ and $n \times n$. %

The algorithm proceeds by iteratively updating $U$ and $\w$ while fixing the other one.

\textbf{Update for $\w$:} It can be seen when $U$ is fixed, minimizing $\w$ is equivalent to a non-negative quadratic problem

\begin{equation}
\label{equ:updatew}
\underset{\w \geq 0 }\mini \quad f(\w)=\frac{1}{2}\|\mathfrak{L} \w\|_{F}^{2}-\mathbf{c}^{T} \w
\end{equation}
which is strictly convex where $\mathbf{c}=\mathfrak{L}^{*}(U \operatorname{Diag}(\boldsymbol{\lambda}) U^{T})$.
It is easy to see that the problem is strictly convex. However, due the the non-negativity constraint for $\w$, there is no closed form solution. Thus we derive a majorization function via the following lemma. 
\begin{lemma}
 The function $f(w)$ is majorized at $w_t$ by the function
 \begin{equation}
 \label{equ:major}
g(\w | \w^{t})=f(\w^{t})+(\w-\w^{t})^{T} \nabla f(\w^{t})+\frac{L_{1}}{2}\left\|\w-\w^{t}\right\|^{2}
\end{equation}

where $\w^{t}$ is the update from previous iteration an $L_{1}=\|\mathfrak{L}\|_{2}^{2}=2 n$.
\end{lemma}
After ignoring the constant terms in \ref{equ:major}, the majorized problem of \ref{equ:updatew} at $\w^{t}$ is given
\begin{equation}
\label{update-w-appendix}
\underset{\w \geq 0}{\mini} \quad g(\w | \w^{t})=\frac{1}{2} \w^{T} \w-\boldsymbol{a}^{T} \w,
\end{equation}
where $a=\w^{t}-\frac{1}{L_{1}} \nabla f(\w^{t})$ and $\nabla f(\w^{t})=\mathfrak{L}^{*}(\mathfrak{L} \w^{t})-\mathbf{c}$

\begin{lemma}
From the KKT optimality conditions we can easily obtain the optimal solution to \ref{equ:major} as
\begin{equation*}
\w^{t+1}=(\w^{t}-\frac{1}{L_{1}} \nabla f(\w^{t}))^{+}
\end{equation*}
where $(x)^{+}:=\max (x, 0)$.
\end{lemma}

\textbf{Update for $U$:} When $\w$ is fixed, the problem of optimizing $U$ is equivalent to

\begin{equation}
\label{kkt-equalation}
\begin{array}{ll}
\underset{U}{\mini} & \text{tr}(U^T\lap{\w}U Diag(\boldsymbol{\lambda}))
\\ \text { subject to } & U^TU = I
\end{array}
\end{equation}

It can be shown that the optimal $U$ at iteration $t$ is achieved by $U^{t+1} = \text{eigenvectors}(L_w)$.
\begin{lemma}\label{update-u-appendix}
From KKT optimality condition, the solution to \ref{kkt-equalation} is given by \\ $U^{t+1} = eigenvectors(\lap{\w}).$
\end{lemma}
The following theorem is proved at \cite{kumar2019structured}.
\begin{theorem}
The sequence $(\w^{t}, U^{t})$ generated by Algorithm \ref{opt-algo-appendix} converges to the set of $K K T$ points of \ref{opt-problem-equ}.
\end{theorem}

\subsubsection{Non-complete Graph Case}
\label{convergence-appendix}
The only complication in the case of the non-complete graph is that $\w$ has only $|E|$ number of free variables instead of $\frac{n(n-1)}{2}$ variables as the case of the complete graph. We will argue that $w$ will stay at the subspace of dimension $|E|$ during the iteration.

For simplicity, given a non-compete graph $\Ghat = (\Vhat, \Ehat)$, let us denote $\vhat = [n] = \{1, 2, ..., n\}$ and each edge will be represented as $(i, j)$ where $i > j$, and $i, j\in [n]$. It is easy to see that we can map each edge $(i, j)$ ($i>j$) to $k$-th coordinate of $\w$ via $k = \Phi(i, j) = i-j + \frac{(j-1)(2p-j)}{2}$.

 Let us denote $\mask{\w}$ (to emphasize its dependence on $\Ghat$, it is also denoted as $\wg$ later.) to be the same as $\w$ on coordinates that corresponds to edges in $G$ and 0 for the rest entries. In other words, %

\begin{equation*}
  \mask{\w}[k] =
  \begin{cases}
   \w[k] & \text{if } \Phi^{-1}(k)\in E \\
   0    & \text {o.w.}
  \end{cases}
\end{equation*}

Similarly, for any symmetric matrix $A$ of size $n \times n$
\begin{equation*}
  \mask{A}[i, j] =
  \begin{cases}
   A[i, j] & \text{if} (i, j) \in E \text{ or } (j, i) \in E \\
   0    & \text {o.w.}
  \end{cases}
\end{equation*}
Let us also define a $\Ghat$-subspace of $\w$ (denoted as $\Ghat$-subspace when there is no ambiguity) as $\{\mask{\w} | \w \in \mathbb{R}_{+}^{n(n-1) / 2} \}$.  
What we need to prove is that if we initialize the algorithm with $\wg$ instead of $\w$, $\wgt$ will remain in the $\Ghat$-subspace of $\w$ for any $t \in \mathbb{Z}_{+}$. 

First, we have the following lemma.
\begin{lemma}
We have 
\begin{enumerate}
\item $\lap{\mask{w}} = \mask{\lap{w}}$. 
\item $\langle\mask{\w_1}, \w_2\rangle = \langle\w_1, \mask{\w_2}\rangle = \langle\mask{\w_1}, \mask{\w_2}\rangle.$
\item $\mathfrak{L}^{*}\mask{Y} = \mask{\mathfrak{L}^{*}Y}$
\end{enumerate}
\end{lemma}

\begin{figure}[t]
\label{approx}
\begin{center}
\includegraphics[scale=.4]{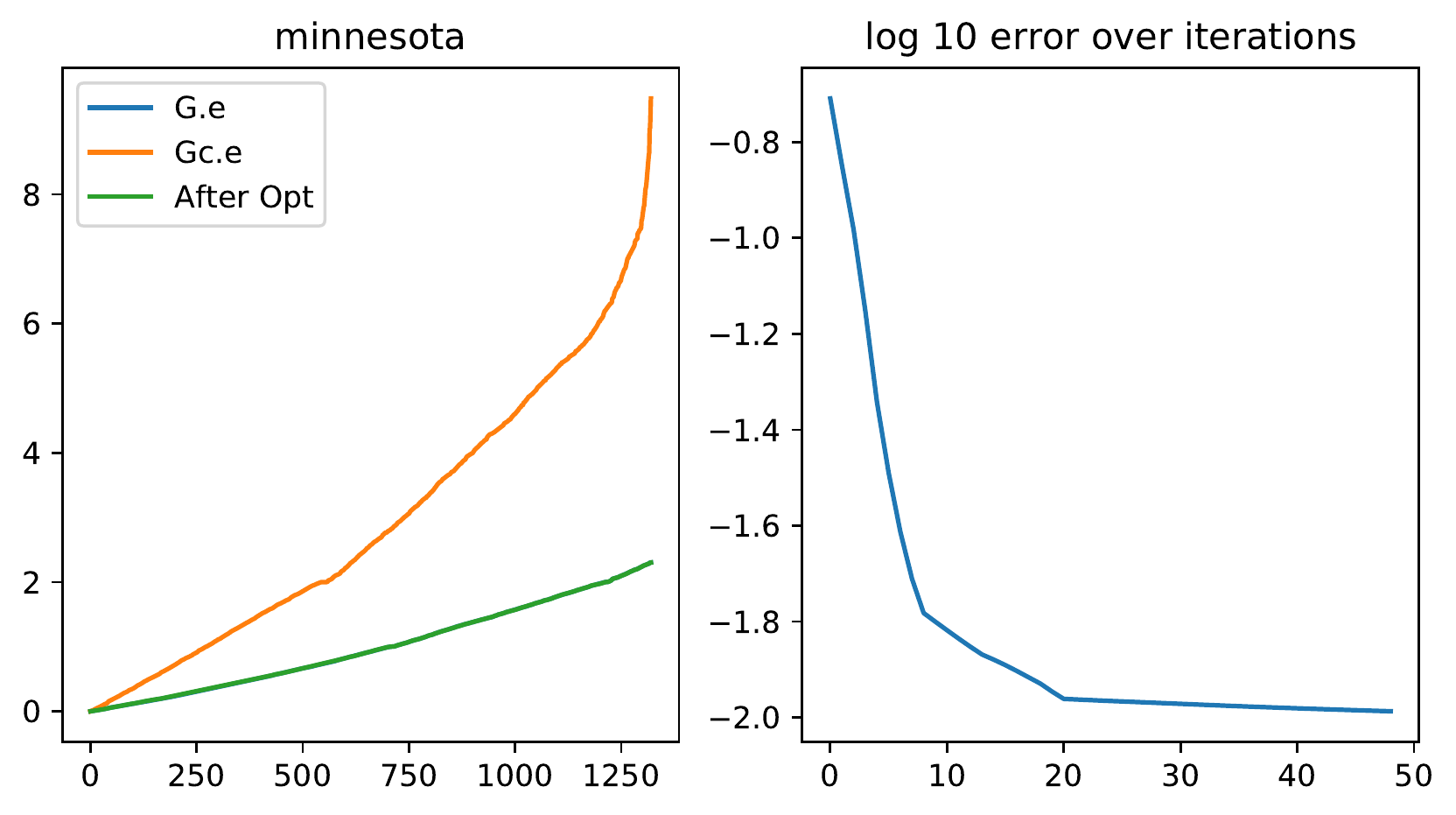}
\includegraphics[scale=.4]{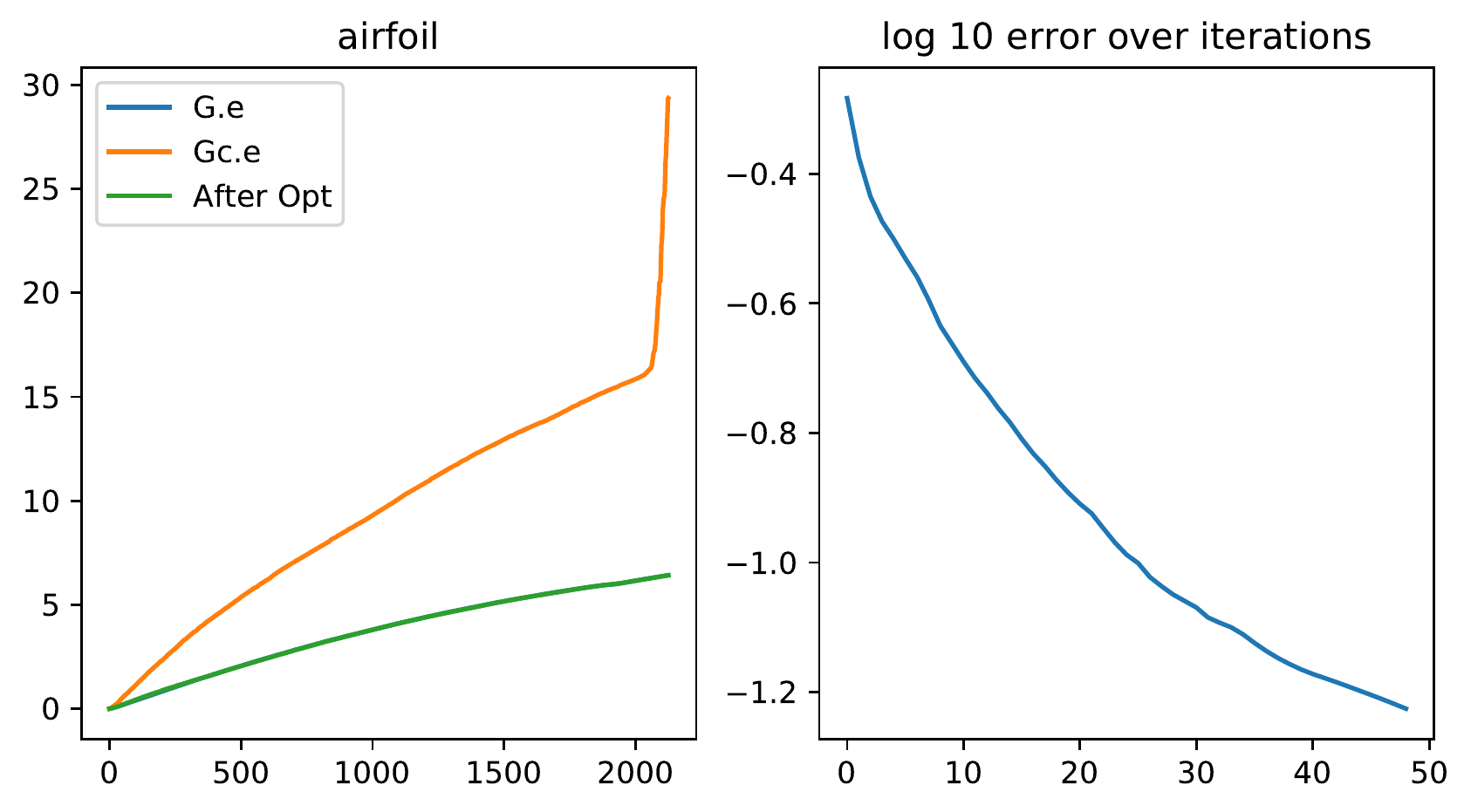}
\caption{After optimizing edge weights, we can construct a smaller graph with eigenvalues much closer to eigenvalues of original graph. $G.e$ and $Gc.e$ stand for the eigenvalues of original graph and coarse graph. After-Opt stands for the eigenvalues of graphs where weights are optimized.} %
\label{algo-appendix}
\end{center}
\end{figure}

\begin{proof}
Lemma 1 and 2 can be proved by definition. Now we prove the last lemma.
For any $\w$ $\in \mathbb{R}_{+}^{\frac{n(n-1)}{2}}$ and $Y\in \mathbb{R}^{n \times n}$
\begin{equation*}
\langle\w, \lapstar{\mask{Y}} \rangle = \langle\lap{\w}, \mask{Y}\rangle = \langle\mask{\lap{\w}}, Y\rangle
 = \langle\lap{\mask{\w}}, Y\rangle = \langle\mask{\w}, \lapstar{Y}\rangle = \langle\w, \mask{\lapstar{Y}}\rangle
\end{equation*}
where the fourth equation follows from the definition of $\lapstar{}$ and the other equations directly follows from the previous two lemmas.
Therefore $\lapstar{\mask{Y}} = \mask{\lapstar{Y}}$.
\end{proof}

Recall that we minimize the following objective when updating $\wg$

\begin{equation*}
\underset{\wg \geq 0}{\operatorname{minimize}} \quad \left \| \mathfrak{L} \wg-U \operatorname{Diag}(\boldsymbol{\lambda}) U^{T}\right\|_{F}^{2}
\end{equation*}

which is equivalent to be 
\begin{equation}
\label{mask-equ}
\underset{\wg \geq 0}{\operatorname{minimize}} \quad \left \| \mathfrak{L} \wg- \mask{\A} \right\|_{F}^{2}
\end{equation}

Now following the same process for the case of complete graph. Equation \ref{mask-equ} is equivalent to 
\begin{equation*}
\underset{\wg \geq 0}{\operatorname{minimize}} \quad f(\wg)=\frac{1}{2}\|\mathfrak{L} \wg\|_{F}^{2}-\mathbf{c}^{T} \wg
\end{equation*}
where $\mathbf{c}=\mathfrak{L}^{*}(\mask{\A} )$.

Use the same majorization function as the case of complete graph, we can get the following update rule
\begin{lemma}
From the KKT optimality conditions we can easily obtain the optimal solution to as
\begin{equation*}
\wg^{t+1}=(\wg^{t}-\frac{1}{L_{1}} \nabla f(\wg^{t}))^{+}
\end{equation*}
where $(x)^{+}:=\max (x, 0)$ and $\nabla f(\wg^{t}) = \lapstar{(\lap{\wgt} - \mask{\A})}$.
\end{lemma}
Since $\nabla f(\wgt) = \lapstar{(\lap{\mask{\wt}})-\mask{A})} =
 \lapstar{({\mask{\lap{\wt}} - \mask{A}})} 
 = \mask{\lapstar{(\lap{\wt} -A )}}$ 
 where \\ $A=\A$, therefore $\wg^{t+1}$ will remain in the $\Ghat$-subspace if $\wg^t$ is in the $\Ghat$-subspace. Since $\wg^{0}$ is initialized inside $\Ghat$-subspace, by induction $\wg^{t}$ stays in the $\Ghat$-subspace for any $t\in \mathbb{Z}^+$. Therefore, we conclude
\begin{theorem}
In the case of non-complete graph, the sequence $(\w^{t}, U^{t})$ generated by Algorithm \ref{opt-algo-appendix} converges to the set of $K K T$ points of \ref{opt-problem-equ}.
\end{theorem}

\textbf{Remark}: since for each iteration a full eigendecomposition is conducted, the computational complexity is $O(n^3)$ for each iteration, which is certainly prohibitive for large scale application. Another drawback is that the algorithm is not adaptive to the data so we have to run the same algorithm for graphs from the same generative distribution. The main takeaway of this algorithm is that it is possible to improve the spectral alignment of the original graph and coarse graph by optimizing over edge weights, as shown in Figure \ref{algo-appendix}.

\chapter{Discussion and Future Directions}\label{chapter:discussion}
In this thesis, we discussed three works that offer a local-to-global perspective on graph neural networks.
The first part of the thesis in \Cref{chapter:ign} introduces a class of global GNN, the Invariant Graph Network (IGN),
and provides a systematic study of its convergence property. 
The second part of the thesis in \Cref{chapter:mpnn_gt} studies the connection between local MPNNs and global Graph Transformers. It connects the local approach (MPNN)
and global approach (Graph Transformer), with DeepSets and Invariant Graph Network (IGN) serving as the conceptual bridge.
In the last part of the thesis at \Cref{chapter:cg}, we study the creative use of local MPNN to perform graph coarsening, a common subroutine used in modeling long-range interaction for large graphs. 
In the future, it would be interesting to explore the following directions. 

\textbf{Fine-grained understanding of the gap between MPNN and Graph Transformer}. The approximation of the self-attention layer
by MPNN establishes a link between MPNN + VN with Graph Transformer. However, it does not imply that MPNN + VN will achieve
equivalent empirical performance as Graph Transformer. In fact, we still observe a gap between the two. 
We think such limitation shares a similarity with research in universal permutational invariant functions. 
Both \DS{} \cite{zaheer2017deep} and Relational Network \cite{santoro2017simple} are universal permutational invariant architecture 
but there is still a representation gap between the two \cite{zweig2022exponential}. 
Under the restriction to analytic activation functions, one can construct a symmetric function acting on sets of size $n$ with elements in dimension $d$, which can be efficiently
approximated by the Relational Network, but provably requires width exponential in $n$ and $d$ for the \DS{}. 
We believe a similar representation gap also exists between GT and MPNN + VN and leave the characterization of functions lying in the gap as future work.

\textbf{Memory efficient global GNN}. Both IGN and GT require large memory consumption and high computational complexity, which 
limits their application to cases such as fluid dynamics, sea temperature forecasting, and congestion prediction in chip design, 
where modeling long-range interaction is crucial. The current heuristic resorts to techniques in 
 the efficient transformer literature that is agnostic to the graph structure. We believe understanding the continuous dynamics of the underlying phenomenon and leveraging such inductive bias in global GNN design will be a promising direction.

\textbf{Topological Methods for Graph Learning}.
Topology is a useful tool to understand the global property of the data and naturally fits into the approach to building GNN.
Topological methods \cite{dey2022computational,chen2019topological} such as persistence diagrams \cite{hofer2017deep,cai2021sanity,cai2020understanding}, multiparameter persistence modules \cite{loiseaux2022efficient,cai2021elder}, simplicial complex networks \cite{ebli2020simplicial} are promising tools to be integrated into future graph learning.

\textbf{Graph Coarsening}.
For graph coarsening work in \Cref{chapter:cg}, there are three directions to pursue. first, the topology of the coarse graph is currently determined by the coarsening
algorithm. It would be desirable to parametrize existing methods by neural networks so that the entire process can be trained end-to-end. Second, extending to other losses (maybe non-differentiable)
such as \cite{maretic2019got} which involves inverse Laplacian remains interesting and challenging.
Third, as there is no consensus on what specific properties should be preserved, understanding the
relationship between different metrics and the downstream task is important.

\bibliographystyle{plain} %
\bibliography{references}

\begin{thebibliography}{100}

\bibitem{albooyeh2019incidence}
Marjan Albooyeh, Daniele Bertolini, and Siamak Ravanbakhsh.
\newblock Incidence networks for geometric deep learning.
\newblock {\em arXiv preprint arXiv:1905.11460}, 2019.

\bibitem{allen2015spectral}
Zeyuan Allen-Zhu, Zhenyu Liao, and Lorenzo Orecchia.
\newblock Spectral sparsification and regret minimization beyond matrix
  multiplicative updates.
\newblock In {\em Proceedings of the forty-seventh annual ACM symposium on
  Theory of computing}, pages 237--245, 2015.

\bibitem{alon2020bottleneck}
Uri Alon and Eran Yahav.
\newblock On the bottleneck of graph neural networks and its practical
  implications.
\newblock {\em arXiv preprint arXiv:2006.05205}, 2020.

\bibitem{arora2019fine}
Sanjeev Arora, Simon Du, Wei Hu, Zhiyuan Li, and Ruosong Wang.
\newblock Fine-grained analysis of optimization and generalization for
  overparameterized two-layer neural networks.
\newblock In {\em International Conference on Machine Learning}, pages
  322--332. PMLR, 2019.

\bibitem{azizian2020expressive}
Wa{\"\i}ss Azizian and Marc Lelarge.
\newblock Expressive power of invariant and equivariant graph neural networks.
\newblock {\em arXiv preprint arXiv:2006.15646}, 2020.

\bibitem{barioli2004two}
Francesco Barioli and Shaun Fallat.
\newblock On two conjectures regarding an inverse eigenvalue problem for
  acyclic symmetric matrices.
\newblock {\em The Electronic Journal of Linear Algebra}, 11, 2004.

\bibitem{batson2012twice}
Joshua Batson, Daniel~A Spielman, and Nikhil Srivastava.
\newblock Twice-ramanujan sparsifiers.
\newblock {\em SIAM Journal on Computing}, 41(6):1704--1721, 2012.

\bibitem{battaglia2018relational}
Peter~W Battaglia, Jessica~B Hamrick, Victor Bapst, Alvaro Sanchez-Gonzalez,
  Vinicius Zambaldi, Mateusz Malinowski, Andrea Tacchetti, David Raposo, Adam
  Santoro, and Ryan Faulkner.
\newblock Relational inductive biases, deep learning, and graph networks.
\newblock {\em arXiv preprint arXiv:1806.01261}, 2018.

\bibitem{belkin2008discrete}
Mikhail Belkin, Jian Sun, and Yusu Wang.
\newblock Discrete laplace operator on meshed surfaces.
\newblock In {\em Proceedings of the twenty-fourth annual symposium on
  Computational geometry}, pages 278--287, 2008.

\bibitem{belkin2009constructing}
Mikhail Belkin, Jian Sun, and Yusu Wang.
\newblock Constructing laplace operator from point clouds in $r^d$.
\newblock In {\em Proceedings of the twentieth annual ACM-SIAM symposium on
  Discrete algorithms}, pages 1031--1040. SIAM, 2009.

\bibitem{bevilacqua2021equivariant}
Beatrice Bevilacqua, Fabrizio Frasca, Derek Lim, Balasubramaniam Srinivasan,
  Chen Cai, Gopinath Balamurugan, Michael~M Bronstein, and Haggai Maron.
\newblock Equivariant subgraph aggregation networks.
\newblock {\em arXiv preprint arXiv:2110.02910}, 2021.

\bibitem{brody2021attentive}
Shaked Brody, Uri Alon, and Eran Yahav.
\newblock How attentive are graph attention networks?
\newblock {\em arXiv preprint arXiv:2105.14491}, 2021.

\bibitem{bruna2013spectral}
Joan Bruna, Wojciech Zaremba, Arthur Szlam, and Yann LeCun.
\newblock Spectral networks and locally connected networks on graphs.
\newblock {\em arXiv preprint arXiv:1312.6203}, 2013.

\bibitem{cai2021sanity}
Chen Cai.
\newblock Sanity check for persistence diagrams.
\newblock In {\em ICLR 2021 Workshop on Geometrical and Topological
  Representation Learning}.

\bibitem{cai2019group}
Chen Cai, Yunfeng Cai, Mingming Sun, and Zhiqiang Xu.
\newblock Group representation theory for knowledge graph embedding.
\newblock {\em arXiv preprint arXiv:1909.05100}, 2019.

\bibitem{cai2023connection}
Chen Cai, Truong~Son Hy, Rose Yu, and Yusu Wang.
\newblock On the connection between mpnn and graph transformer.
\newblock {\em arXiv preprint arXiv:2301.11956}, 2023.

\bibitem{cai2021elder}
Chen Cai, Woojin Kim, Facundo M{\'e}moli, and Yusu Wang.
\newblock Elder-rule-staircodes for augmented metric spaces.
\newblock {\em SIAM Journal on Applied Algebra and Geometry}, 5(3):417--454,
  2021.

\bibitem{cai2023equivariant}
Chen Cai, Nikolaos Vlassis, Lucas Magee, Ran Ma, Zeyu Xiong, Bahador Bahmani,
  Teng-Fong Wong, Yusu Wang, and WaiChing Sun.
\newblock Equivariant geometric learning for digital rock physics: estimating
  formation factor and effective permeability tensors from morse graph.
\newblock {\em International Journal for Multiscale Computational Engineering},
  21(5), 2023.

\bibitem{cai2021graph}
Chen Cai, Dingkang Wang, and Yusu Wang.
\newblock Graph coarsening with neural networks.
\newblock {\em arXiv preprint arXiv:2102.01350}, 2021.

\bibitem{cai2018simple}
Chen Cai and Yusu Wang.
\newblock A simple yet effective baseline for non-attributed graph
  classification.
\newblock {\em arXiv preprint arXiv:1811.03508}, 2018.

\bibitem{cai2020note}
Chen Cai and Yusu Wang.
\newblock A note on over-smoothing for graph neural networks.
\newblock {\em arXiv preprint arXiv:2006.13318}, 2020.

\bibitem{cai2020understanding}
Chen Cai and Yusu Wang.
\newblock Understanding the power of persistence pairing via permutation test.
\newblock {\em arXiv preprint arXiv:2001.06058}, 2020.

\bibitem{cai2022convergence}
Chen Cai and Yusu Wang.
\newblock Convergence of invariant graph networks.
\newblock In {\em International Conference on Machine Learning}, pages
  2457--2484. PMLR, 2022.

\bibitem{chen2019topological}
Chao Chen, Xiuyan Ni, Qinxun Bai, and Yusu Wang.
\newblock A topological regularizer for classifiers via persistent homology.
\newblock In {\em The 22nd International Conference on Artificial Intelligence
  and Statistics}, pages 2573--2582. PMLR, 2019.

\bibitem{chen2022structure}
Dexiong Chen, Leslie O’Bray, and Karsten Borgwardt.
\newblock Structure-aware transformer for graph representation learning.
\newblock In {\em International Conference on Machine Learning}, pages
  3469--3489. PMLR, 2022.

\bibitem{chen2011algebraic}
Jie Chen and Ilya Safro.
\newblock Algebraic distance on graphs.
\newblock {\em SIAM Journal on Scientific Computing}, 33(6):3468--3490, 2011.

\bibitem{chen2018neural}
Ricky~TQ Chen, Yulia Rubanova, Jesse Bettencourt, and David Duvenaud.
\newblock Neural ordinary differential equations.
\newblock {\em arXiv preprint arXiv:1806.07366}, 2018.

\bibitem{child2019generating}
Rewon Child, Scott Gray, Alec Radford, and Ilya Sutskever.
\newblock Generating long sequences with sparse transformers.
\newblock {\em arXiv preprint arXiv:1904.10509}, 2019.

\bibitem{choromanski2020rethinking}
Krzysztof Choromanski, Valerii Likhosherstov, David Dohan, Xingyou Song,
  Andreea Gane, Tamas Sarlos, Peter Hawkins, Jared Davis, Afroz Mohiuddin, and
  Lukasz Kaiser.
\newblock Rethinking attention with performers.
\newblock {\em arXiv preprint arXiv:2009.14794}, 2020.

\bibitem{chung1996combinatorial}
Fan~RK Chung and Robert~P Langlands.
\newblock A combinatorial laplacian with vertex weights.
\newblock {\em journal of combinatorial theory, Series A}, 75(2):316--327,
  1996.

\bibitem{cybenko1989approximation}
George Cybenko.
\newblock Approximation by superpositions of a sigmoidal function.
\newblock {\em Mathematics of control, signals and systems}, 2(4):303--314,
  1989.

\bibitem{de2018molgan}
Nicola De~Cao and Thomas Kipf.
\newblock Molgan: An implicit generative model for small molecular graphs.
\newblock {\em arXiv preprint arXiv:1805.11973}, 2018.

\bibitem{defferrard2016convolutional}
Micha{\"e}l Defferrard, Xavier Bresson, and Pierre Vandergheynst.
\newblock Convolutional neural networks on graphs with fast localized spectral
  filtering.
\newblock {\em Advances in neural information processing systems},
  29:3844--3852, 2016.

\bibitem{dettmers2018convolutional}
Tim Dettmers, Pasquale Minervini, Pontus Stenetorp, and Sebastian Riedel.
\newblock Convolutional 2d knowledge graph embeddings.
\newblock In {\em Thirty-Second AAAI Conference on Artificial Intelligence},
  2018.

\bibitem{dey2010convergence}
Tamal~K Dey, Pawas Ranjan, and Yusu Wang.
\newblock Convergence, stability, and discrete approximation of laplace
  spectra.
\newblock In {\em Proceedings of the Twenty-First Annual ACM-SIAM Symposium on
  Discrete Algorithms}, pages 650--663. SIAM, 2010.

\bibitem{dey2013graph}
Tamal~Krishna Dey, Fengtao Fan, and Yusu Wang.
\newblock Graph induced complex on point data.
\newblock In {\em Proceedings of the twenty-ninth annual symposium on
  Computational geometry}, pages 107--116, 2013.

\bibitem{dey2022computational}
Tamal~Krishna Dey and Yusu Wang.
\newblock {\em Computational topology for data analysis}.
\newblock Cambridge University Press, 2022.

\bibitem{dhillon2007weighted}
Inderjit~S Dhillon, Yuqiang Guan, and Brian Kulis.
\newblock Weighted graph cuts without eigenvectors a multilevel approach.
\newblock {\em IEEE transactions on pattern analysis and machine intelligence},
  29(11):1944--1957, 2007.

\bibitem{dosovitskiy2020image}
Alexey Dosovitskiy, Lucas Beyer, Alexander Kolesnikov, Dirk Weissenborn,
  Xiaohua Zhai, Thomas Unterthiner, Mostafa Dehghani, Matthias Minderer, Georg
  Heigold, and Sylvain Gelly.
\newblock An image is worth 16x16 words: Transformers for image recognition at
  scale.
\newblock {\em arXiv preprint arXiv:2010.11929}, 2020.

\bibitem{du2019gradient}
Simon Du, Jason Lee, Haochuan Li, Liwei Wang, and Xiyu Zhai.
\newblock Gradient descent finds global minima of deep neural networks.
\newblock In {\em International Conference on Machine Learning}, pages
  1675--1685. PMLR, 2019.

\bibitem{du2018gradient}
Simon~S Du, Xiyu Zhai, Barnabas Poczos, and Aarti Singh.
\newblock Gradient descent provably optimizes over-parameterized neural
  networks.
\newblock {\em arXiv preprint arXiv:1810.02054}, 2018.

\bibitem{duvenaud2015convolutional}
David~K Duvenaud, Dougal Maclaurin, Jorge Iparraguirre, Rafael Bombarell,
  Timothy Hirzel, Al{\'a}n Aspuru-Guzik, and Ryan~P Adams.
\newblock Convolutional networks on graphs for learning molecular fingerprints.
\newblock In {\em Advances in neural information processing systems}, pages
  2224--2232, 2015.

\bibitem{dwivedi2020generalization}
Vijay~Prakash Dwivedi and Xavier Bresson.
\newblock A generalization of transformer networks to graphs.
\newblock {\em arXiv preprint arXiv:2012.09699}, 2020.

\bibitem{dwivedi2022long}
Vijay~Prakash Dwivedi, Ladislav Ramp{\'a}{\v{s}}ek, Mikhail Galkin, Ali Parviz,
  Guy Wolf, Anh~Tuan Luu, and Dominique Beaini.
\newblock Long range graph benchmark.
\newblock {\em arXiv preprint arXiv:2206.08164}, 2022.

\bibitem{d2021convit}
St{\'e}phane d’Ascoli, Hugo Touvron, Matthew~L Leavitt, Ari~S Morcos, Giulio
  Biroli, and Levent Sagun.
\newblock Convit: Improving vision transformers with soft convolutional
  inductive biases.
\newblock In {\em International Conference on Machine Learning}, pages
  2286--2296. PMLR, 2021.

\bibitem{ebli2020simplicial}
Stefania Ebli, Micha{\"e}l Defferrard, and Gard Spreemann.
\newblock Simplicial neural networks.
\newblock {\em arXiv preprint arXiv:2010.03633}, 2020.

\bibitem{eldridge2016graphons}
Justin Eldridge, Mikhail Belkin, and Yusu Wang.
\newblock Graphons, mergeons, and so on!
\newblock In {\em Advances in Neural Information Processing Systems}, pages
  2307--2315, 2016.

\bibitem{fallat2020inverse}
Shaun~M Fallat, Leslie Hogben, Jephian C-H Lin, and Bryan~L Shader.
\newblock The inverse eigenvalue problem of a graph, zero forcing, and related
  parameters.
\newblock {\em Notices of the American Mathematical Society}, 67(2), 2020.

\bibitem{fey2019fast}
Matthias Fey and Jan~Eric Lenssen.
\newblock Fast graph representation learning with pytorch geometric.
\newblock {\em arXiv preprint arXiv:1903.02428}, 2019.

\bibitem{gama2020stability}
Fernando Gama, Joan Bruna, and Alejandro Ribeiro.
\newblock Stability properties of graph neural networks.
\newblock {\em IEEE Transactions on Signal Processing}, 68:5680--5695, 2020.

\bibitem{garg2020generalization}
Vikas Garg, Stefanie Jegelka, and Tommi Jaakkola.
\newblock Generalization and representational limits of graph neural networks.
\newblock In {\em International Conference on Machine Learning}, pages
  3419--3430. PMLR, 2020.

\bibitem{gavish2010multiscale}
Matan Gavish, Boaz Nadler, and Ronald~R Coifman.
\newblock Multiscale wavelets on trees, graphs and high dimensional data:
  Theory and applications to semi supervised learning.
\newblock In {\em ICML}, pages 367--374, 2010.

\bibitem{geerts2020expressive}
Floris Geerts.
\newblock The expressive power of kth-order invariant graph networks.
\newblock {\em arXiv preprint arXiv:2007.12035}, 2020.

\bibitem{gilmer2017neural}
Justin Gilmer, Samuel~S Schoenholz, Patrick~F Riley, Oriol Vinyals, and
  George~E Dahl.
\newblock Neural message passing for quantum chemistry.
\newblock In {\em International conference on machine learning}, pages
  1263--1272. PMLR, 2017.

\bibitem{gleich2008matlabbgl}
David Gleich.
\newblock Matlabbgl. a matlab graph library.
\newblock {\em Institute for Computational and Mathematical Engineering,
  Stanford University}, 2008.

\bibitem{gurukarbenchmarking}
Saket Gurukar, Priyesh Vijayan, Balaraman Ravindran, Aakash Srinivasan,
  Goonmeet Bajaj, Chen Cai, Moniba Keymanesh, Saravana Kumar, Pranav Maneriker,
  and Anasua Mitra.
\newblock Benchmarking and analyzing unsupervised network representation
  learning and the illusion of progress.

\bibitem{gurukar2019network}
Saket Gurukar, Priyesh Vijayan, Aakash Srinivasan, Goonmeet Bajaj, Chen Cai,
  Moniba Keymanesh, Saravana Kumar, Pranav Maneriker, Anasua Mitra, and Vedang
  Patel.
\newblock Network representation learning: Consolidation and renewed bearing.
\newblock {\em arXiv preprint arXiv:1905.00987}, 2019.

\bibitem{hamilton2017inductive}
Will Hamilton, Zhitao Ying, and Jure Leskovec.
\newblock Inductive representation learning on large graphs.
\newblock {\em Advances in neural information processing systems}, 30, 2017.

\bibitem{han2022survey}
Kai Han, Yunhe Wang, Hanting Chen, Xinghao Chen, Jianyuan Guo, Zhenhua Liu,
  Yehui Tang, An~Xiao, Chunjing Xu, and Yixing Xu.
\newblock A survey on vision transformer.
\newblock {\em IEEE transactions on pattern analysis and machine intelligence},
  2022.

\bibitem{harel2000fast}
David Harel and Yehuda Koren.
\newblock A fast multi-scale method for drawing large graphs.
\newblock In {\em International symposium on graph drawing}, pages 183--196.
  Springer, 2000.

\bibitem{hendrickson1995multi}
Bruce Hendrickson and Robert~W Leland.
\newblock A multi-level algorithm for partitioning graphs.
\newblock {\em SC}, 95(28):1--14, 1995.

\bibitem{hermsdorff2019unifying}
Gecia~Bravo Hermsdorff and Lee Gunderson.
\newblock A unifying framework for spectrum-preserving graph sparsification and
  coarsening.
\newblock In {\em Advances in Neural Information Processing Systems}, pages
  7734--7745, 2019.

\bibitem{hofer2017deep}
Christoph Hofer, Roland Kwitt, Marc Niethammer, and Andreas Uhl.
\newblock Deep learning with topological signatures.
\newblock {\em Advances in neural information processing systems}, 30, 2017.

\bibitem{hogben2005spectral}
Leslie Hogben.
\newblock Spectral graph theory and the inverse eigenvalue problem of a graph.
\newblock {\em The Electronic Journal of Linear Algebra}, 14, 2005.

\bibitem{holst1980lengths}
Lars Holst.
\newblock On the lengths of the pieces of a stick broken at random.
\newblock {\em Journal of Applied Probability}, 17(3):623--634, 1980.

\bibitem{horak2013spectra}
Danijela Horak and J{\"u}rgen Jost.
\newblock Spectra of combinatorial laplace operators on simplicial complexes.
\newblock {\em Advances in Mathematics}, 244:303--336, 2013.

\bibitem{hornik1989multilayer}
Kurt Hornik, Maxwell Stinchcombe, and Halbert White.
\newblock Multilayer feedforward networks are universal approximators.
\newblock {\em Neural networks}, 2(5):359--366, 1989.

\bibitem{hu2021ogb}
Weihua Hu, Matthias Fey, Hongyu Ren, Maho Nakata, Yuxiao Dong, and Jure
  Leskovec.
\newblock Ogb-lsc: A large-scale challenge for machine learning on graphs.
\newblock {\em arXiv preprint arXiv:2103.09430}, 2021.

\bibitem{hu2020open}
Weihua Hu, Matthias Fey, Marinka Zitnik, Yuxiao Dong, Hongyu Ren, Bowen Liu,
  Michele Catasta, and Jure Leskovec.
\newblock Open graph benchmark: Datasets for machine learning on graphs.
\newblock {\em Advances in neural information processing systems},
  33:22118--22133, 2020.

\bibitem{hu2005efficient}
Yifan Hu.
\newblock Efficient, high-quality force-directed graph drawing.
\newblock {\em Mathematica Journal}, 10(1):37--71, 2005.

\bibitem{hussain2022global}
Md~Shamim Hussain, Mohammed~J Zaki, and Dharmashankar Subramanian.
\newblock Global self-attention as a replacement for graph convolution.
\newblock In {\em Proceedings of the 28th ACM SIGKDD Conference on Knowledge
  Discovery and Data Mining}, pages 655--665, 2022.

\bibitem{hwanganalysis}
EunJeong Hwang, Veronika Thost, Shib~Sankar Dasgupta, and Tengfei Ma.
\newblock An analysis of virtual nodes in graph neural networks for link
  prediction.
\newblock In {\em Learning on Graphs Conference}, 2022.

\bibitem{jacot2018neural}
Arthur Jacot, Franck Gabriel, and Cl{\'e}ment Hongler.
\newblock Neural tangent kernel: Convergence and generalization in neural
  networks.
\newblock {\em arXiv preprint arXiv:1806.07572}, 2018.

\bibitem{jeong2001lethality}
Hawoong Jeong, Sean~P Mason, A-L Barab{\'a}si, and Zoltan~N Oltvai.
\newblock Lethality and centrality in protein networks.
\newblock {\em Nature}, 411(6833):41--42, 2001.

\bibitem{kalyan2021ammus}
Katikapalli~Subramanyam Kalyan, Ajit Rajasekharan, and Sivanesan Sangeetha.
\newblock Ammus: A survey of transformer-based pretrained models in natural
  language processing.
\newblock {\em arXiv preprint arXiv:2108.05542}, 2021.

\bibitem{karypis1998fast}
George Karypis and Vipin Kumar.
\newblock A fast and high quality multilevel scheme for partitioning irregular
  graphs.
\newblock {\em SIAM Journal on scientific Computing}, 20(1):359--392, 1998.

\bibitem{katharopoulos2020transformers}
Angelos Katharopoulos, Apoorv Vyas, Nikolaos Pappas, and Fran{\c{c}}ois
  Fleuret.
\newblock Transformers are rnns: Fast autoregressive transformers with linear
  attention.
\newblock In {\em International Conference on Machine Learning}, pages
  5156--5165. PMLR, 2020.

\bibitem{keriven2020convergence}
Nicolas Keriven, Alberto Bietti, and Samuel Vaiter.
\newblock Convergence and stability of graph convolutional networks on large
  random graphs.
\newblock {\em arXiv preprint arXiv:2006.01868}, 2020.

\bibitem{keriven2021universality}
Nicolas Keriven, Alberto Bietti, and Samuel Vaiter.
\newblock On the universality of graph neural networks on large random graphs.
\newblock {\em arXiv preprint arXiv:2105.13099}, 2021.

\bibitem{keriven2019universal}
Nicolas Keriven and Gabriel Peyr{\'e}.
\newblock Universal invariant and equivariant graph neural networks.
\newblock {\em Advances in Neural Information Processing Systems},
  32:7092--7101, 2019.

\bibitem{khalil2017learning}
Elias Khalil, Hanjun Dai, Yuyu Zhang, Bistra Dilkina, and Le~Song.
\newblock Learning combinatorial optimization algorithms over graphs.
\newblock In {\em Advances in Neural Information Processing Systems}, pages
  6348--6358, 2017.

\bibitem{kim2022pure}
Jinwoo Kim, Tien~Dat Nguyen, Seonwoo Min, Sungjun Cho, Moontae Lee, Honglak
  Lee, and Seunghoon Hong.
\newblock Pure transformers are powerful graph learners.
\newblock {\em arXiv preprint arXiv:2207.02505}, 2022.

\bibitem{kingma2014adam}
Diederik~P Kingma and Jimmy Ba.
\newblock Adam: A method for stochastic optimization.
\newblock {\em arXiv preprint arXiv:1412.6980}, 2014.

\bibitem{kipf2016semi}
Thomas~N Kipf and Max Welling.
\newblock Semi-supervised classification with graph convolutional networks.
\newblock {\em arXiv preprint arXiv:1609.02907}, 2016.

\bibitem{kostrikov2018surface}
Ilya Kostrikov, Zhongshi Jiang, Daniele Panozzo, Denis Zorin, and Joan Bruna.
\newblock Surface networks.
\newblock In {\em Proceedings of the IEEE Conference on Computer Vision and
  Pattern Recognition}, pages 2540--2548, 2018.

\bibitem{kreuzer2021rethinking}
Devin Kreuzer, Dominique Beaini, Will Hamilton, Vincent L{\'e}tourneau, and
  Prudencio Tossou.
\newblock Rethinking graph transformers with spectral attention.
\newblock {\em Advances in Neural Information Processing Systems},
  34:21618--21629, 2021.

\bibitem{kriege2020survey}
Nils~M Kriege, Fredrik~D Johansson, and Christopher Morris.
\newblock A survey on graph kernels.
\newblock {\em Applied Network Science}, 5(1):1--42, 2020.

\bibitem{kumar2019structured}
Sandeep Kumar, Jiaxi Ying, Jose~Vinicius de~Miranda~Cardoso, and Daniel
  Palomar.
\newblock Structured graph learning via laplacian spectral constraints.
\newblock In {\em Advances in Neural Information Processing Systems}, pages
  11651--11663, 2019.

\bibitem{kushnir2006fast}
Dan Kushnir, Meirav Galun, and Achi Brandt.
\newblock Fast multiscale clustering and manifold identification.
\newblock {\em Pattern Recognition}, 39(10):1876--1891, 2006.

\bibitem{lafon2006diffusion}
Stephane Lafon and Ann~B Lee.
\newblock Diffusion maps and coarse-graining: A unified framework for
  dimensionality reduction, graph partitioning, and data set parameterization.
\newblock {\em IEEE transactions on pattern analysis and machine intelligence},
  28(9):1393--1403, 2006.

\bibitem{lee2019wide}
Jaehoon Lee, Lechao Xiao, Samuel Schoenholz, Yasaman Bahri, Roman Novak, Jascha
  Sohl-Dickstein, and Jeffrey Pennington.
\newblock Wide neural networks of any depth evolve as linear models under
  gradient descent.
\newblock {\em Advances in neural information processing systems},
  32:8572--8583, 2019.

\bibitem{lee2017sdp}
Yin~Tat Lee and He~Sun.
\newblock An sdp-based algorithm for linear-sized spectral sparsification.
\newblock In {\em Proceedings of the 49th annual acm sigact symposium on theory
  of computing}, pages 678--687, 2017.

\bibitem{lee2018constructing}
Yin~Tat Lee and He~Sun.
\newblock Constructing linear-sized spectral sparsification in almost-linear
  time.
\newblock {\em SIAM Journal on Computing}, 47(6):2315--2336, 2018.

\bibitem{lehoucq1998arpack}
Richard~B Lehoucq, Danny~C Sorensen, and Chao Yang.
\newblock {\em ARPACK users' guide: solution of large-scale eigenvalue problems
  with implicitly restarted Arnoldi methods}, volume~6.
\newblock Siam, 1998.

\bibitem{levie2021transferability}
Ron Levie, Wei Huang, Lorenzo Bucci, Michael Bronstein, and Gitta Kutyniok.
\newblock Transferability of spectral graph convolutional neural networks.
\newblock {\em Journal of Machine Learning Research}, 22(272):1--59, 2021.

\bibitem{li2018deeper}
Qimai Li, Zhichao Han, and Xiao-Ming Wu.
\newblock Deeper insights into graph convolutional networks for semi-supervised
  learning.
\newblock In {\em Proceedings of the AAAI conference on artificial
  intelligence}, volume~32, 2018.

\bibitem{li2018learning}
Yujia Li, Oriol Vinyals, Chris Dyer, Razvan Pascanu, and Peter Battaglia.
\newblock Learning deep generative models of graphs.
\newblock {\em arXiv preprint arXiv:1803.03324}, 2018.

\bibitem{li2018combinatorial}
Zhuwen Li, Qifeng Chen, and Vladlen Koltun.
\newblock Combinatorial optimization with graph convolutional networks and
  guided tree search.
\newblock In {\em Advances in Neural Information Processing Systems}, pages
  539--548, 2018.

\bibitem{liao2019efficient}
Renjie Liao, Yujia Li, Yang Song, Shenlong Wang, Will Hamilton, David~K
  Duvenaud, Raquel Urtasun, and Richard Zemel.
\newblock Efficient graph generation with graph recurrent attention networks.
\newblock In {\em Advances in Neural Information Processing Systems}, pages
  4255--4265, 2019.

\bibitem{lim2022sign}
Derek Lim, Joshua Robinson, Lingxiao Zhao, Tess Smidt, Suvrit Sra, Haggai
  Maron, and Stefanie Jegelka.
\newblock Sign and basis invariant networks for spectral graph representation
  learning.
\newblock {\em arXiv preprint arXiv:2202.13013}, 2022.

\bibitem{liu2021swin}
Ze~Liu, Yutong Lin, Yue Cao, Han Hu, Yixuan Wei, Zheng Zhang, Stephen Lin, and
  Baining Guo.
\newblock Swin transformer: Hierarchical vision transformer using shifted
  windows.
\newblock In {\em Proceedings of the IEEE/CVF International Conference on
  Computer Vision}, pages 10012--10022, 2021.

\bibitem{livne2012lean}
Oren~E Livne and Achi Brandt.
\newblock Lean algebraic multigrid (lamg): Fast graph laplacian linear solver.
\newblock {\em SIAM Journal on Scientific Computing}, 34(4):B499--B522, 2012.

\bibitem{loiseaux2022efficient}
David Loiseaux, Mathieu Carriere, and Andrew~J Blumberg.
\newblock Efficient approximation of multiparameter persistence modules.
\newblock {\em arXiv preprint arXiv:2206.02026}, 2022.

\bibitem{loukas2019graph}
Andreas Loukas.
\newblock Graph reduction with spectral and cut guarantees.
\newblock {\em Journal of Machine Learning Research}, 20(116):1--42, 2019.

\bibitem{loukas2018spectrally}
Andreas Loukas and Pierre Vandergheynst.
\newblock Spectrally approximating large graphs with smaller graphs.
\newblock {\em arXiv preprint arXiv:1802.07510}, 2018.

\bibitem{lu2018beyond}
Yiping Lu, Aoxiao Zhong, Quanzheng Li, and Bin Dong.
\newblock Beyond finite layer neural networks: Bridging deep architectures and
  numerical differential equations.
\newblock In {\em International Conference on Machine Learning}, pages
  3276--3285. PMLR, 2018.

\bibitem{ma2018constrained}
Tengfei Ma, Jie Chen, and Cao Xiao.
\newblock Constrained generation of semantically valid graphs via regularizing
  variational autoencoders.
\newblock In {\em Advances in Neural Information Processing Systems}, pages
  7113--7124, 2018.

\bibitem{maretic2019got}
Hermina~Petric Maretic, Mireille El~Gheche, Giovanni Chierchia, and Pascal
  Frossard.
\newblock Got: An optimal transport framework for graph comparison.
\newblock In {\em Advances in Neural Information Processing Systems}, pages
  13876--13887, 2019.

\bibitem{maron2019provably}
Haggai Maron, Heli Ben-Hamu, Hadar Serviansky, and Yaron Lipman.
\newblock Provably powerful graph networks.
\newblock {\em arXiv preprint arXiv:1905.11136}, 2019.

\bibitem{maron2018invariant}
Haggai Maron, Heli Ben-Hamu, Nadav Shamir, and Yaron Lipman.
\newblock Invariant and equivariant graph networks.
\newblock {\em arXiv preprint arXiv:1812.09902}, 2018.

\bibitem{maron2019universality}
Haggai Maron, Ethan Fetaya, Nimrod Segol, and Yaron Lipman.
\newblock On the universality of invariant networks.
\newblock In {\em International conference on machine learning}, pages
  4363--4371. PMLR, 2019.

\bibitem{mialon2021graphit}
Gr{\'e}goire Mialon, Dexiong Chen, Margot Selosse, and Julien Mairal.
\newblock Graphit: Encoding graph structure in transformers.
\newblock {\em arXiv preprint arXiv:2106.05667}, 2021.

\bibitem{murphy2018janossy}
Ryan~L Murphy, Balasubramaniam Srinivasan, Vinayak Rao, and Bruno Ribeiro.
\newblock Janossy pooling: Learning deep permutation-invariant functions for
  variable-size inputs.
\newblock {\em arXiv preprint arXiv:1811.01900}, 2018.

\bibitem{oono2019graph}
Kenta Oono and Taiji Suzuki.
\newblock Graph neural networks exponentially lose expressive power for node
  classification.
\newblock {\em arXiv preprint arXiv:1905.10947}, 2019.

\bibitem{park2022grpe}
Wonpyo Park, Woong-Gi Chang, Donggeon Lee, and Juntae Kim.
\newblock Grpe: Relative positional encoding for graph transformer.
\newblock In {\em ICLR2022 Machine Learning for Drug Discovery}, 2022.

\bibitem{paszke2017automatic}
Adam Paszke, Sam Gross, Soumith Chintala, Gregory Chanan, Edward Yang, Zachary
  DeVito, Zeming Lin, Alban Desmaison, Luca Antiga, and Adam Lerer.
\newblock Automatic differentiation in pytorch.
\newblock 2017.

\bibitem{preis1997party}
Robert Preis and Ralf Diekmann.
\newblock Party-a software library for graph partitioning.
\newblock {\em Advances in Computational Mechanics with Parallel and
  Distributed Processing}, pages 63--71, 1997.

\bibitem{pyke1965spacings}
Ronald Pyke.
\newblock Spacings.
\newblock {\em Journal of the Royal Statistical Society: Series B
  (Methodological)}, 27(3):395--436, 1965.

\bibitem{qi2017pointnet}
Charles~R Qi, Hao Su, Kaichun Mo, and Leonidas~J Guibas.
\newblock Pointnet: Deep learning on point sets for 3d classification and
  segmentation.
\newblock In {\em Proceedings of the IEEE conference on computer vision and
  pattern recognition}, pages 652--660, 2017.

\bibitem{rampavsek2022recipe}
Ladislav Ramp{\'a}{\v{s}}ek, Mikhail Galkin, Vijay~Prakash Dwivedi, Anh~Tuan
  Luu, Guy Wolf, and Dominique Beaini.
\newblock Recipe for a general, powerful, scalable graph transformer.
\newblock {\em arXiv preprint arXiv:2205.12454}, 2022.

\bibitem{renyi1953theory}
Alfr{\'e}d R{\'e}nyi.
\newblock On the theory of order statistics.
\newblock {\em Acta Mathematica Academiae Scientiarum Hungarica},
  4(3-4):191--231, 1953.

\bibitem{ron2011relaxation}
Dorit Ron, Ilya Safro, and Achi Brandt.
\newblock Relaxation-based coarsening and multiscale graph organization.
\newblock {\em Multiscale Modeling \& Simulation}, 9(1):407--423, 2011.

\bibitem{ruge1987algebraic}
John~W Ruge and Klaus St{\"u}ben.
\newblock Algebraic multigrid.
\newblock In {\em Multigrid methods}, pages 73--130. SIAM, 1987.

\bibitem{ruiz2020graphon}
Luana Ruiz, Luiz Chamon, and Alejandro Ribeiro.
\newblock Graphon neural networks and the transferability of graph neural
  networks.
\newblock {\em Advances in Neural Information Processing Systems}, 33, 2020.

\bibitem{ruiz2021graph}
Luana Ruiz, Fernando Gama, and Alejandro Ribeiro.
\newblock Graph neural networks: Architectures, stability, and transferability.
\newblock {\em Proceedings of the IEEE}, 109(5):660--682, 2021.

\bibitem{ruthotto2020deep}
Lars Ruthotto and Eldad Haber.
\newblock Deep neural networks motivated by partial differential equations.
\newblock {\em Journal of Mathematical Imaging and Vision}, 62(3):352--364,
  2020.

\bibitem{sanchez2020learning}
Alvaro Sanchez-Gonzalez, Jonathan Godwin, Tobias Pfaff, Rex Ying, Jure
  Leskovec, and Peter~W Battaglia.
\newblock Learning to simulate complex physics with graph networks.
\newblock {\em arXiv preprint arXiv:2002.09405}, 2020.

\bibitem{santoro2017simple}
Adam Santoro, David Raposo, David~G Barrett, Mateusz Malinowski, Razvan
  Pascanu, Peter Battaglia, and Timothy Lillicrap.
\newblock A simple neural network module for relational reasoning.
\newblock {\em Advances in neural information processing systems}, 30, 2017.

\bibitem{sato2020survey}
Ryoma Sato.
\newblock A survey on the expressive power of graph neural networks.
\newblock {\em arXiv preprint arXiv:2003.04078}, 2020.

\bibitem{schlichtkrull2018modeling}
Michael Schlichtkrull, Thomas~N Kipf, Peter Bloem, Rianne Van Den~Berg, Ivan
  Titov, and Max Welling.
\newblock Modeling relational data with graph convolutional networks.
\newblock In {\em European Semantic Web Conference}, pages 593--607. Springer,
  2018.

\bibitem{segol2019universal}
Nimrod Segol and Yaron Lipman.
\newblock On universal equivariant set networks.
\newblock {\em arXiv preprint arXiv:1910.02421}, 2019.

\bibitem{sen2008collective}
Prithviraj Sen, Galileo Namata, Mustafa Bilgic, Lise Getoor, Brian Galligher,
  and Tina Eliassi-Rad.
\newblock Collective classification in network data.
\newblock {\em AI magazine}, 29(3):93--93, 2008.

\bibitem{shi2022benchmarking}
Yu~Shi, Shuxin Zheng, Guolin Ke, Yifei Shen, Jiacheng You, Jiyan He, Shengjie
  Luo, Chang Liu, Di~He, and Tie-Yan Liu.
\newblock Benchmarking graphormer on large-scale molecular modeling datasets.
\newblock {\em arXiv preprint arXiv:2203.04810}, 2022.

\bibitem{shuman2015multiscale}
David~I Shuman, Mohammad~Javad Faraji, and Pierre Vandergheynst.
\newblock A multiscale pyramid transform for graph signals.
\newblock {\em IEEE Transactions on Signal Processing}, 64(8):2119--2134, 2015.

\bibitem{simonovsky2018graphvae}
Martin Simonovsky and Nikos Komodakis.
\newblock Graphvae: Towards generation of small graphs using variational
  autoencoders.
\newblock In {\em International Conference on Artificial Neural Networks},
  pages 412--422. Springer, 2018.

\bibitem{spielman2011graph}
Daniel~A Spielman and Nikhil Srivastava.
\newblock Graph sparsification by effective resistances.
\newblock {\em SIAM Journal on Computing}, 40(6):1913--1926, 2011.

\bibitem{spielman2004nearly}
Daniel~A Spielman and Shang-Hua Teng.
\newblock Nearly-linear time algorithms for graph partitioning, graph
  sparsification, and solving linear systems.
\newblock In {\em Proceedings of the thirty-sixth annual ACM symposium on
  Theory of computing}, pages 81--90, 2004.

\bibitem{spielman2011spectral}
Daniel~A Spielman and Shang-Hua Teng.
\newblock Spectral sparsification of graphs.
\newblock {\em SIAM Journal on Computing}, 40(4):981--1025, 2011.

\bibitem{tay2020efficient}
Yi~Tay, Mostafa Dehghani, Dara Bahri, and Donald Metzler.
\newblock Efficient transformers: A survey.
\newblock {\em ACM Computing Surveys (CSUR)}, 2020.

\bibitem{topping2021understanding}
Jake Topping, Francesco Di~Giovanni, Benjamin~Paul Chamberlain, Xiaowen Dong,
  and Michael~M Bronstein.
\newblock Understanding over-squashing and bottlenecks on graphs via curvature.
\newblock {\em arXiv preprint arXiv:2111.14522}, 2021.

\bibitem{turk1994zippered}
Greg Turk and Marc Levoy.
\newblock Zippered polygon meshes from range images.
\newblock In {\em Proceedings of the 21st annual conference on Computer
  graphics and interactive techniques}, pages 311--318, 1994.

\bibitem{vaswani2017attention}
Ashish Vaswani, Noam Shazeer, Niki Parmar, Jakob Uszkoreit, Llion Jones,
  Aidan~N Gomez, {\L}ukasz Kaiser, and Illia Polosukhin.
\newblock Attention is all you need.
\newblock {\em Advances in neural information processing systems}, 30, 2017.

\bibitem{velivckovic2017graph}
Petar Veli{\v{c}}kovi{\'c}, Guillem Cucurull, Arantxa Casanova, Adriana Romero,
  Pietro Lio, and Yoshua Bengio.
\newblock Graph attention networks.
\newblock {\em arXiv preprint arXiv:1710.10903}, 2017.

\bibitem{vishwanathan2010graph}
S~Vichy~N Vishwanathan, Nicol~N Schraudolph, Risi Kondor, and Karsten~M
  Borgwardt.
\newblock Graph kernels.
\newblock {\em The Journal of Machine Learning Research}, 11:1201--1242, 2010.

\bibitem{wagstaff2022universal}
Edward Wagstaff, Fabian~B Fuchs, Martin Engelcke, Michael~A Osborne, and Ingmar
  Posner.
\newblock Universal approximation of functions on sets.
\newblock {\em Journal of Machine Learning Research}, 23(151):1--56, 2022.

\bibitem{walshaw2000multilevel}
Chris Walshaw.
\newblock A multilevel algorithm for force-directed graph drawing.
\newblock In {\em International Symposium on Graph Drawing}, pages 171--182.
  Springer, 2000.

\bibitem{wang2020linformer}
Sinong Wang, Belinda~Z Li, Madian Khabsa, Han Fang, and Hao Ma.
\newblock Linformer: Self-attention with linear complexity.
\newblock {\em arXiv preprint arXiv:2006.04768}, 2020.

\bibitem{wang2022generative}
Wujie Wang, Minkai Xu, Chen Cai, Benjamin~Kurt Miller, Tess Smidt, Yusu Wang,
  Jian Tang, and Rafael G{\'o}mez-Bombarelli.
\newblock Generative coarse-graining of molecular conformations.
\newblock {\em arXiv preprint arXiv:2201.12176}, 2022.

\bibitem{wardetzky2008convergence}
Max Wardetzky.
\newblock Convergence of the cotangent formula: An overview.
\newblock {\em Discrete differential geometry}, pages 275--286, 2008.

\bibitem{weinan2017proposal}
E~Weinan.
\newblock A proposal on machine learning via dynamical systems.
\newblock {\em Communications in Mathematics and Statistics}, 5(1):1--11, 2017.

\bibitem{wolf2020transformers}
Thomas Wolf, Lysandre Debut, Victor Sanh, Julien Chaumond, Clement Delangue,
  Anthony Moi, Pierric Cistac, Tim Rault, R{\'e}mi Louf, and Morgan Funtowicz.
\newblock Transformers: State-of-the-art natural language processing.
\newblock In {\em Proceedings of the 2020 conference on empirical methods in
  natural language processing: system demonstrations}, pages 38--45, 2020.

\bibitem{wunodeformer}
Qitian Wu, Wentao Zhao, Zenan Li, David Wipf, and Junchi Yan.
\newblock Nodeformer: A scalable graph structure learning transformer for node
  classification.
\newblock In {\em Advances in Neural Information Processing Systems}, 2022.

\bibitem{wu2021representing}
Zhanghao Wu, Paras Jain, Matthew Wright, Azalia Mirhoseini, Joseph~E Gonzalez,
  and Ion Stoica.
\newblock Representing long-range context for graph neural networks with global
  attention.
\newblock {\em Advances in Neural Information Processing Systems},
  34:13266--13279, 2021.

\bibitem{xie2018crystal}
Tian Xie and Jeffrey~C Grossman.
\newblock Crystal graph convolutional neural networks for an accurate and
  interpretable prediction of material properties.
\newblock {\em Physical review letters}, 120(14):145301, 2018.

\bibitem{xu2004discrete}
Guoliang Xu.
\newblock Discrete laplace--beltrami operators and their convergence.
\newblock {\em Computer aided geometric design}, 21(8):767--784, 2004.

\bibitem{xu2018powerful}
Keyulu Xu, Weihua Hu, Jure Leskovec, and Stefanie Jegelka.
\newblock How powerful are graph neural networks?
\newblock {\em arXiv preprint arXiv:1810.00826}, 2018.

\bibitem{xu2019weighted}
Shijie Xu, Jiayan Fang, and Xiang-Yang Li.
\newblock Weighted laplacian and its theoretical applications.
\newblock {\em arXiv preprint arXiv:1911.10311}, 2019.

\bibitem{yang2020revisiting}
Chaoqi Yang, Ruijie Wang, Shuochao Yao, Shengzhong Liu, and Tarek Abdelzaher.
\newblock Revisiting over-smoothing in deep gcns.
\newblock {\em arXiv preprint arXiv:2003.13663}, 2020.

\bibitem{ying2021transformers}
Chengxuan Ying, Tianle Cai, Shengjie Luo, Shuxin Zheng, Guolin Ke, Di~He,
  Yanming Shen, and Tie-Yan Liu.
\newblock Do transformers really perform badly for graph representation?
\newblock {\em Advances in Neural Information Processing Systems},
  34:28877--28888, 2021.

\bibitem{you2018graph}
Jiaxuan You, Bowen Liu, Zhitao Ying, Vijay Pande, and Jure Leskovec.
\newblock Graph convolutional policy network for goal-directed molecular graph
  generation.
\newblock In {\em Advances in neural information processing systems}, pages
  6410--6421, 2018.

\bibitem{you2018graphrnn}
Jiaxuan You, Rex Ying, Xiang Ren, William~L Hamilton, and Jure Leskovec.
\newblock Graphrnn: Generating realistic graphs with deep auto-regressive
  models.
\newblock {\em arXiv preprint arXiv:1802.08773}, 2018.

\bibitem{zaheer2017deep}
Manzil Zaheer, Satwik Kottur, Siamak Ravanbakhsh, Barnabas Poczos, Russ~R
  Salakhutdinov, and Alexander~J Smola.
\newblock Deep sets.
\newblock {\em Advances in neural information processing systems}, 30, 2017.

\bibitem{zeng2019graphsaint}
Hanqing Zeng, Hongkuan Zhou, Ajitesh Srivastava, Rajgopal Kannan, and Viktor
  Prasanna.
\newblock Graphsaint: Graph sampling based inductive learning method.
\newblock {\em arXiv preprint arXiv:1907.04931}, 2019.

\bibitem{zhang2022composition}
Jie Zhang, Chen Cai, George Kim, Yusu Wang, and Wei Chen.
\newblock Composition design of high-entropy alloys with deep sets learning.
\newblock {\em npj Computational Materials}, 8(1):89, 2022.

\bibitem{zhang2015estimating}
Yuan Zhang, Elizaveta Levina, and Ji~Zhu.
\newblock Estimating network edge probabilities by neighborhood smoothing.
\newblock {\em arXiv preprint arXiv:1509.08588}, 2015.

\bibitem{zhao2019pairnorm}
Lingxiao Zhao and Leman Akoglu.
\newblock Pairnorm: Tackling oversmoothing in gnns.
\newblock {\em arXiv preprint arXiv:1909.12223}, 2019.

\bibitem{zhou2019misc}
Dawei Zhou, Lecheng Zheng, Jiejun Xu, and Jingrui He.
\newblock Misc-gan: A multi-scale generative model for graphs.
\newblock {\em Frontiers in Big Data}, 2:3, 2019.

\bibitem{zhou2021dirichlet}
Kaixiong Zhou, Xiao Huang, Daochen Zha, Rui Chen, Li~Li, Soo-Hyun Choi, and Xia
  Hu.
\newblock Dirichlet energy constrained learning for deep graph neural networks.
\newblock In {\em Thirty-Fifth Conference on Neural Information Processing
  Systems}, 2021.

\bibitem{zweig2022exponential}
Aaron Zweig and Joan Bruna.
\newblock Exponential separations in symmetric neural networks.
\newblock {\em arXiv preprint arXiv:2206.01266}, 2022.

\end{thebibliography}


\begin{thebibliography}{49}
\providecommand{\natexlab}[1]{#1}
\providecommand{\url}[1]{\texttt{#1}}
\expandafter\ifx\csname urlstyle\endcsname\relax
  \providecommand{\doi}[1]{doi: #1}\else
  \providecommand{\doi}{doi: \begingroup \urlstyle{rm}\Url}\fi

\bibitem[Albooyeh et~al.(2019)Albooyeh, Bertolini, and
  Ravanbakhsh]{albooyeh2019incidence}
Albooyeh, M., Bertolini, D., and Ravanbakhsh, S.
\newblock Incidence networks for geometric deep learning.
\newblock \emph{arXiv preprint arXiv:1905.11460}, 2019.

\bibitem[Arora et~al.(2019)Arora, Du, Hu, Li, and Wang]{arora2019fine}
Arora, S., Du, S., Hu, W., Li, Z., and Wang, R.
\newblock Fine-grained analysis of optimization and generalization for
  overparameterized two-layer neural networks.
\newblock In \emph{International Conference on Machine Learning}, pp.\
  322--332. PMLR, 2019.

\bibitem[Azizian \& Lelarge(2020)Azizian and Lelarge]{azizian2020expressive}
Azizian, W. and Lelarge, M.
\newblock Expressive power of invariant and equivariant graph neural networks.
\newblock \emph{arXiv preprint arXiv:2006.15646}, 2020.

\bibitem[Belkin et~al.(2008)Belkin, Sun, and Wang]{belkin2008discrete}
Belkin, M., Sun, J., and Wang, Y.
\newblock Discrete laplace operator on meshed surfaces.
\newblock In \emph{Proceedings of the twenty-fourth annual symposium on
  Computational geometry}, pp.\  278--287, 2008.

\bibitem[Belkin et~al.(2009)Belkin, Sun, and Wang]{belkin2009constructing}
Belkin, M., Sun, J., and Wang, Y.
\newblock Constructing laplace operator from point clouds in $r^d$.
\newblock In \emph{Proceedings of the twentieth annual ACM-SIAM symposium on
  Discrete algorithms}, pp.\  1031--1040. SIAM, 2009.

\bibitem[Bevilacqua et~al.(2021)Bevilacqua, Frasca, Lim, Srinivasan, Cai,
  Balamurugan, Bronstein, and Maron]{bevilacqua2021equivariant}
Bevilacqua, B., Frasca, F., Lim, D., Srinivasan, B., Cai, C., Balamurugan, G.,
  Bronstein, M.~M., and Maron, H.
\newblock Equivariant subgraph aggregation networks.
\newblock \emph{arXiv preprint arXiv:2110.02910}, 2021.

\bibitem[Bruna et~al.(2013)Bruna, Zaremba, Szlam, and LeCun]{bruna2013spectral}
Bruna, J., Zaremba, W., Szlam, A., and LeCun, Y.
\newblock Spectral networks and locally connected networks on graphs.
\newblock \emph{arXiv preprint arXiv:1312.6203}, 2013.

\bibitem[Cai \& Wang(2020)Cai and Wang]{cai2020note}
Cai, C. and Wang, Y.
\newblock A note on over-smoothing for graph neural networks.
\newblock \emph{arXiv preprint arXiv:2006.13318}, 2020.

\bibitem[Chen et~al.(2018)Chen, Rubanova, Bettencourt, and
  Duvenaud]{chen2018neural}
Chen, R.~T., Rubanova, Y., Bettencourt, J., and Duvenaud, D.
\newblock Neural ordinary differential equations.
\newblock \emph{arXiv preprint arXiv:1806.07366}, 2018.

\bibitem[Cybenko(1989)]{cybenko1989approximation}
Cybenko, G.
\newblock Approximation by superpositions of a sigmoidal function.
\newblock \emph{Mathematics of control, signals and systems}, 2\penalty0
  (4):\penalty0 303--314, 1989.

\bibitem[Defferrard et~al.(2016)Defferrard, Bresson, and
  Vandergheynst]{defferrard2016convolutional}
Defferrard, M., Bresson, X., and Vandergheynst, P.
\newblock Convolutional neural networks on graphs with fast localized spectral
  filtering.
\newblock \emph{Advances in neural information processing systems},
  29:\penalty0 3844--3852, 2016.

\bibitem[Dey et~al.(2010)Dey, Ranjan, and Wang]{dey2010convergence}
Dey, T.~K., Ranjan, P., and Wang, Y.
\newblock Convergence, stability, and discrete approximation of laplace
  spectra.
\newblock In \emph{Proceedings of the Twenty-First Annual ACM-SIAM Symposium on
  Discrete Algorithms}, pp.\  650--663. SIAM, 2010.

\bibitem[Du et~al.(2019)Du, Lee, Li, Wang, and Zhai]{du2019gradient}
Du, S., Lee, J., Li, H., Wang, L., and Zhai, X.
\newblock Gradient descent finds global minima of deep neural networks.
\newblock In \emph{International Conference on Machine Learning}, pp.\
  1675--1685. PMLR, 2019.

\bibitem[Du et~al.(2018)Du, Zhai, Poczos, and Singh]{du2018gradient}
Du, S.~S., Zhai, X., Poczos, B., and Singh, A.
\newblock Gradient descent provably optimizes over-parameterized neural
  networks.
\newblock \emph{arXiv preprint arXiv:1810.02054}, 2018.

\bibitem[Eldridge et~al.(2016)Eldridge, Belkin, and Wang]{eldridge2016graphons}
Eldridge, J., Belkin, M., and Wang, Y.
\newblock Graphons, mergeons, and so on!
\newblock In \emph{Advances in Neural Information Processing Systems}, pp.\
  2307--2315, 2016.

\bibitem[Finzi et~al.(2021)Finzi, Welling, and Wilson]{finzi2021practical}
Finzi, M., Welling, M., and Wilson, A.~G.
\newblock A practical method for constructing equivariant multilayer
  perceptrons for arbitrary matrix groups.
\newblock \emph{arXiv preprint arXiv:2104.09459}, 2021.

\bibitem[Gama et~al.(2020)Gama, Bruna, and Ribeiro]{gama2020stability}
Gama, F., Bruna, J., and Ribeiro, A.
\newblock Stability properties of graph neural networks.
\newblock \emph{IEEE Transactions on Signal Processing}, 68:\penalty0
  5680--5695, 2020.

\bibitem[Garg et~al.(2020)Garg, Jegelka, and Jaakkola]{garg2020generalization}
Garg, V., Jegelka, S., and Jaakkola, T.
\newblock Generalization and representational limits of graph neural networks.
\newblock In \emph{International Conference on Machine Learning}, pp.\
  3419--3430. PMLR, 2020.

\bibitem[Geerts(2020)]{geerts2020expressive}
Geerts, F.
\newblock The expressive power of kth-order invariant graph networks.
\newblock \emph{arXiv preprint arXiv:2007.12035}, 2020.

\bibitem[Gilmer et~al.(2017)Gilmer, Schoenholz, Riley, Vinyals, and
  Dahl]{gilmer2017neural}
Gilmer, J., Schoenholz, S.~S., Riley, P.~F., Vinyals, O., and Dahl, G.~E.
\newblock Neural message passing for quantum chemistry.
\newblock In \emph{International conference on machine learning}, pp.\
  1263--1272. PMLR, 2017.

\bibitem[Holst(1980)]{holst1980lengths}
Holst, L.
\newblock On the lengths of the pieces of a stick broken at random.
\newblock \emph{Journal of Applied Probability}, 17\penalty0 (3):\penalty0
  623--634, 1980.

\bibitem[Hornik et~al.(1989)Hornik, Stinchcombe, and
  White]{hornik1989multilayer}
Hornik, K., Stinchcombe, M., and White, H.
\newblock Multilayer feedforward networks are universal approximators.
\newblock \emph{Neural networks}, 2\penalty0 (5):\penalty0 359--366, 1989.

\bibitem[Jacot et~al.(2018)Jacot, Gabriel, and Hongler]{jacot2018neural}
Jacot, A., Gabriel, F., and Hongler, C.
\newblock Neural tangent kernel: Convergence and generalization in neural
  networks.
\newblock \emph{arXiv preprint arXiv:1806.07572}, 2018.

\bibitem[Keriven \& Peyr{\'e}(2019)Keriven and Peyr{\'e}]{keriven2019universal}
Keriven, N. and Peyr{\'e}, G.
\newblock Universal invariant and equivariant graph neural networks.
\newblock \emph{Advances in Neural Information Processing Systems},
  32:\penalty0 7092--7101, 2019.

\bibitem[Keriven et~al.(2020)Keriven, Bietti, and
  Vaiter]{keriven2020convergence}
Keriven, N., Bietti, A., and Vaiter, S.
\newblock Convergence and stability of graph convolutional networks on large
  random graphs.
\newblock \emph{arXiv preprint arXiv:2006.01868}, 2020.

\bibitem[Keriven et~al.(2021)Keriven, Bietti, and
  Vaiter]{keriven2021universality}
Keriven, N., Bietti, A., and Vaiter, S.
\newblock On the universality of graph neural networks on large random graphs.
\newblock \emph{arXiv preprint arXiv:2105.13099}, 2021.

\bibitem[Kipf \& Welling(2016)Kipf and Welling]{kipf2016semi}
Kipf, T.~N. and Welling, M.
\newblock Semi-supervised classification with graph convolutional networks.
\newblock \emph{arXiv preprint arXiv:1609.02907}, 2016.

\bibitem[Kostrikov et~al.(2018)Kostrikov, Jiang, Panozzo, Zorin, and
  Bruna]{kostrikov2018surface}
Kostrikov, I., Jiang, Z., Panozzo, D., Zorin, D., and Bruna, J.
\newblock Surface networks.
\newblock In \emph{Proceedings of the IEEE Conference on Computer Vision and
  Pattern Recognition}, pp.\  2540--2548, 2018.

\bibitem[Lee et~al.(2019)Lee, Xiao, Schoenholz, Bahri, Novak, Sohl-Dickstein,
  and Pennington]{lee2019wide}
Lee, J., Xiao, L., Schoenholz, S., Bahri, Y., Novak, R., Sohl-Dickstein, J.,
  and Pennington, J.
\newblock Wide neural networks of any depth evolve as linear models under
  gradient descent.
\newblock \emph{Advances in neural information processing systems},
  32:\penalty0 8572--8583, 2019.

\bibitem[Levie et~al.(2021)Levie, Huang, Bucci, Bronstein, and
  Kutyniok]{levie2021transferability}
Levie, R., Huang, W., Bucci, L., Bronstein, M., and Kutyniok, G.
\newblock Transferability of spectral graph convolutional neural networks.
\newblock \emph{Journal of Machine Learning Research}, 22\penalty0
  (272):\penalty0 1--59, 2021.

\bibitem[Li et~al.(2018)Li, Han, and Wu]{li2018deeper}
Li, Q., Han, Z., and Wu, X.-M.
\newblock Deeper insights into graph convolutional networks for semi-supervised
  learning.
\newblock In \emph{Thirty-Second AAAI conference on artificial intelligence},
  2018.

\bibitem[Lu et~al.(2018)Lu, Zhong, Li, and Dong]{lu2018beyond}
Lu, Y., Zhong, A., Li, Q., and Dong, B.
\newblock Beyond finite layer neural networks: Bridging deep architectures and
  numerical differential equations.
\newblock In \emph{International Conference on Machine Learning}, pp.\
  3276--3285. PMLR, 2018.

\bibitem[Maron et~al.(2018)Maron, Ben-Hamu, Shamir, and
  Lipman]{maron2018invariant}
Maron, H., Ben-Hamu, H., Shamir, N., and Lipman, Y.
\newblock Invariant and equivariant graph networks.
\newblock \emph{arXiv preprint arXiv:1812.09902}, 2018.

\bibitem[Maron et~al.(2019{\natexlab{a}})Maron, Ben-Hamu, Serviansky, and
  Lipman]{maron2019provably}
Maron, H., Ben-Hamu, H., Serviansky, H., and Lipman, Y.
\newblock Provably powerful graph networks.
\newblock \emph{arXiv preprint arXiv:1905.11136}, 2019{\natexlab{a}}.

\bibitem[Maron et~al.(2019{\natexlab{b}})Maron, Fetaya, Segol, and
  Lipman]{maron2019universality}
Maron, H., Fetaya, E., Segol, N., and Lipman, Y.
\newblock On the universality of invariant networks.
\newblock In \emph{International conference on machine learning}, pp.\
  4363--4371. PMLR, 2019{\natexlab{b}}.

\bibitem[Oono \& Suzuki(2019)Oono and Suzuki]{oono2019graph}
Oono, K. and Suzuki, T.
\newblock Graph neural networks exponentially lose expressive power for node
  classification.
\newblock \emph{arXiv preprint arXiv:1905.10947}, 2019.

\bibitem[Pyke(1965)]{pyke1965spacings}
Pyke, R.
\newblock Spacings.
\newblock \emph{Journal of the Royal Statistical Society: Series B
  (Methodological)}, 27\penalty0 (3):\penalty0 395--436, 1965.

\bibitem[R{\'e}nyi(1953)]{renyi1953theory}
R{\'e}nyi, A.
\newblock On the theory of order statistics.
\newblock \emph{Acta Mathematica Academiae Scientiarum Hungarica}, 4\penalty0
  (3-4):\penalty0 191--231, 1953.

\bibitem[Ruiz et~al.(2020)Ruiz, Chamon, and Ribeiro]{ruiz2020graphon}
Ruiz, L., Chamon, L., and Ribeiro, A.
\newblock Graphon neural networks and the transferability of graph neural
  networks.
\newblock \emph{Advances in Neural Information Processing Systems}, 33, 2020.

\bibitem[Ruiz et~al.(2021)Ruiz, Gama, and Ribeiro]{ruiz2021graph}
Ruiz, L., Gama, F., and Ribeiro, A.
\newblock Graph neural networks: Architectures, stability, and transferability.
\newblock \emph{Proceedings of the IEEE}, 109\penalty0 (5):\penalty0 660--682,
  2021.

\bibitem[Ruthotto \& Haber(2020)Ruthotto and Haber]{ruthotto2020deep}
Ruthotto, L. and Haber, E.
\newblock Deep neural networks motivated by partial differential equations.
\newblock \emph{Journal of Mathematical Imaging and Vision}, 62\penalty0
  (3):\penalty0 352--364, 2020.

\bibitem[Wardetzky(2008)]{wardetzky2008convergence}
Wardetzky, M.
\newblock Convergence of the cotangent formula: An overview.
\newblock \emph{Discrete differential geometry}, pp.\  275--286, 2008.

\bibitem[Weinan(2017)]{weinan2017proposal}
Weinan, E.
\newblock A proposal on machine learning via dynamical systems.
\newblock \emph{Communications in Mathematics and Statistics}, 5\penalty0
  (1):\penalty0 1--11, 2017.

\bibitem[Wu et~al.(2020)Wu, Pan, Chen, Long, Zhang, and
  Philip]{wu2020comprehensive}
Wu, Z., Pan, S., Chen, F., Long, G., Zhang, C., and Philip, S.~Y.
\newblock A comprehensive survey on graph neural networks.
\newblock \emph{IEEE transactions on neural networks and learning systems},
  32\penalty0 (1):\penalty0 4--24, 2020.

\bibitem[Xu(2004)]{xu2004discrete}
Xu, G.
\newblock Discrete laplace--beltrami operators and their convergence.
\newblock \emph{Computer aided geometric design}, 21\penalty0 (8):\penalty0
  767--784, 2004.

\bibitem[Xu et~al.(2018)Xu, Hu, Leskovec, and Jegelka]{xu2018powerful}
Xu, K., Hu, W., Leskovec, J., and Jegelka, S.
\newblock How powerful are graph neural networks?
\newblock \emph{arXiv preprint arXiv:1810.00826}, 2018.

\bibitem[Zhang et~al.(2015)Zhang, Levina, and Zhu]{zhang2015estimating}
Zhang, Y., Levina, E., and Zhu, J.
\newblock Estimating network edge probabilities by neighborhood smoothing.
\newblock \emph{arXiv preprint arXiv:1509.08588}, 2015.

\bibitem[Zhou et~al.(2020)Zhou, Cui, Hu, Zhang, Yang, Liu, Wang, Li, and
  Sun]{zhou2020graph}
Zhou, J., Cui, G., Hu, S., Zhang, Z., Yang, C., Liu, Z., Wang, L., Li, C., and
  Sun, M.
\newblock Graph neural networks: A review of methods and applications.
\newblock \emph{AI Open}, 1:\penalty0 57--81, 2020.

\bibitem[Zhou et~al.(2021)Zhou, Huang, Zha, Chen, Li, Choi, and
  Hu]{zhou2021dirichlet}
Zhou, K., Huang, X., Zha, D., Chen, R., Li, L., Choi, S.-H., and Hu, X.
\newblock Dirichlet energy constrained learning for deep graph neural networks.
\newblock In \emph{Thirty-Fifth Conference on Neural Information Processing
  Systems}, 2021.

\end{thebibliography}

\end{document}